\documentclass[submission, Phys,amsmath]{SciPost}
\usepackage{mathrsfs} \usepackage{amssymb}\usepackage{amsfonts} \usepackage{graphicx,color} \usepackage{bbm} \usepackage{multirow}\usepackage{ulem}
\usepackage{tabularx}
\usepackage{amsthm} 
\usepackage{mystyle1}
\usepackage{physics}
\usepackage{bbm} \usepackage{mathrsfs} \usepackage{commath}\usepackage{tikz}
\usetikzlibrary{shapes,arrows,positioning,calc}\usepackage{stmaryrd}\usepackage{bm}
\usepackage{mhchem,chemformula}
\usepackage{url}
\usepackage{grffile}
\usepackage{algpseudocode}
\usepackage{algorithm}

\hypersetup{
  colorlinks=true,
  linkcolor=blue,
  filecolor=magenta,
  citecolor=blue,
  urlcolor=blue,
}

\newcommand{\eq}[1]{Eq.~(\ref{#1})}
\newcommand{\nreq}[1]{(\ref{#1})}
\newcommand{\beq}{\begin{equation}} \newcommand{\eeq}{\end{equation}}
\newcommand{\beqn}{\begin{eqnarray}} \newcommand{\eeqn}{\end{eqnarray}}
\newcommand{\bmat}{\begin{mathdisplay}} \newcommand{\emat}{\end{mathdisplay}}

\newcommand{\black}[1]{\textcolor{black}{#1}}
\newcommand{\blue}[1]{\textcolor{black}{#1}}

\newcommand{\bs}{\blacksquare}

\newcommand{\bc}{\begin{center}}
\newcommand{\ec}{\end{center}}

\def\to{\rightarrow}  

\newcommand{\gdd}[1]{\mathcal{D}#1}
\newcommand{\lhs}{\text{(l.h.s.)}}

\newcommand{\V}[1]{\bm{#1} } %vector command
\newcommand{\Ave}[1]{\left\langle {#1} \right\rangle} %thermal average 
\newcommand{\mR}{\mathbb{R}}

\newcommand{\lb}{\left(}
\newcommand{\rb}{\right)}
\newcommand{\lbb}{\left\{}
\newcommand{\rbb}{\right\}}
\newcommand{\lsb}{ \left[ }
\newcommand{\rsb}{ \right] }

\newcommand{\Req}[1]{Eq.\ (\ref{eq:#1})}

\newcommand{\NReq}[1]{(\ref{eq:#1})}

\newcommand{\Leq}[1]{\label{eq:#1}}

\newcommand{\be}{\begin{eqnarray}}
\newcommand{\ee}{\end{eqnarray}}
\newcommand{\ba}{\begin{array}}
\newcommand{\ea}{\end{array}}
\newcommand{\no}{\nonumber}

\newcommand{\subbe}{\begin{subequations}}
\newcommand{\subee}{\end{subequations}}

\newcommand{\D}[1]{\mathcal{D}{#1}\;}

\newcommand{\Vset}{\mathtt{V}}
\newcommand{\Eset}{\mathtt{E}}
\newtheorem{lemma}{Lemma}
\newtheorem{corollary}{Corollary}

\begin{document}

\begin{center}{\Large \textbf{
      %relatively low-rank
%Tensor factorizaStatistical inference of an assembly of vectors with large-$M$ components through their $p$-body products      
      %      Graphical model for tensor factorization  by sparse sampling
      Graphical model for factorization  \blue{and completion of relatively high rank tensors} by sparse sampling
%Statistical inference of an assembly of vectors with large-$M$ components
%  through thier $p$-body products      
%Article Title, as descriptive as possible, ideally fitting in two lines (approximately 150 characters) or less
}}\end{center}

% TODO: write the author list here. Use initials + surname format.
% Separate subsequent authors by a comma, omit comma at the end of the list.
% Mark the corresponding author with a superscript *.
\begin{center}
  Angelo Giorgio Cavaliere  \textsuperscript{1}
  Riki Nagasawa  \textsuperscript{2}
  Shuta Yokoi  \textsuperscript{2}\\
  Tomoyuki Obuchi  \textsuperscript{3,4}
  Hajime Yoshino  \textsuperscript{1,2,4*}
\end{center}

\begin{center}
  {\bf 1} D3 Center, Osaka University
  %, Toyonaka, Osaka 560-0043, Japan
\\
  {\bf 2} Department of Physics, Graduate School of Science, Osaka University
  %, Toyonaka, Osaka 560-0043, Japan
\\
{\bf 3} Department of Systems Science, Graduate School of Informatics,
Kyoto University
%, Yoshida Hon-machi, Sakyo-ku, Kyoto-shi, Kyoto 606-8501, Japan
\\
  {\bf 4} RIKEN center for AIP
  %,Tokyo, Japan
\\
  * yoshino.hajime.cmc@osaka-u.ac.jp
\end{center}

\begin{center}
\today
\end{center}

%%%%%%%%%%%%%%%%%%%%%%%%%%%%%%%%%%%%%%%%%%%%%%%%%%%%%
%%%%%%%%%%%%%%%%%%%%%%%%%%%%%%%%%%%%%%%%%%%%%%%%%%%%%
%%%%%%%%%%%%%%%%%%%%%%%%%%%%%%%%%%%%%%%%%%%%%%%%%%%%%
\section*{Abstract}
         {\bf
           We consider tensor factorizations based on sparse measurements \blue{of the components of relatively high rank
             tensors}. The measurements are designed in a way that the underlying graph of interactions is a  random graph.
           The setup will be useful in cases where a substantial amount of data is missing, as in
           \blue{completion of relatively high rank matrices for} recommendation systems heavily used in social network services. In order to obtain theoretical insights on the setup, we consider statistical inference of the tensor factorization in a high dimensional limit, which we call as dense limit, where the graphs are large and dense but not fully connected.
We build message-passing algorithms and test them in a Bayes optimal teacher-student setting in some specific cases.           
We also develop a replica theory to examine the performance of statistical inference in the dense limit based on a cumulant expansion.
\blue{ The latter approach allows one to avoid
  blind usage of Gaussian ansatz which fails in some fully connected systems.}}
         
\vspace{10pt}
\noindent\rule{\textwidth}{1pt}
\tableofcontents\thispagestyle{fancy}
\noindent\rule{\textwidth}{1pt}
\vspace{10pt}

%%%%%%%%%%%%%%%%%%%%%%%%%%%%%%%%%%%%%%%%%%%%%%%%%%%%%
%%%%%%%%%%%%%%%%%%%%%%%%%%%%%%%%%%%%%%%%%%%%%%%%%%%%%
%%%%%%%%%%%%%%%%%%%%%%%%%%%%%%%%%%%%%%%%%%%%%%%%%%%%%
\section{Introduction}
%\input{introduction.v3}
%%%%%%%%%%%%%%%%%%%%%%%%%%%%%%%%%%%%%%%%%%%%%%%%%%%%%
%% introduction.tex
%%%%%%%%%%%%%%%%%%%%%%%%%%%%%%%%%%%%%%%%%%%%%%%%%%%%%

%%%%
We consider a reconstruction problem for $N$ $M$-dimensional vectors $\bm{x}_{i}\in \mR^{M}~(i=1,2,\ldots,N)$ from observations of $p$-plets
$\pi_{i_1,i_2,\ldots,i_p}$ of $\{\bm{x}_{i}\}$, which are defined as
\be
&&
\pi_{i_1,i_2,\ldots,i_p}
=\frac{\lambda }{\sqrt{M}}\sum_{\mu=1}^{M}F_{i_1,i_2,\ldots,i_p,\mu}x_{i_1 \mu}x_{i_2 \mu}\cdots x_{i_p \mu},
\Leq{x to pi}
\label{eq-def-pi}
\ee
where $\lambda$ controls the signal strength and $F_{i_1,i_2,\ldots,i_p,\mu}$ is the linear coefficient of the observation, which is later assumed to be unity or an independently and identically distributed (i.i.d.) random variable from certain distributions, depending on situations.
We will consider the cases of $p=2,3,\ldots$ in the present paper, while the special case $p=1$ with $N=1$ corresponds to linear estimation problems.
%In the latter cases the factor $F_{i}$ is often assumed to be random and called as random spreading factor.
To consider general observation processes, the actual observation $y_{i_1,i_2,\ldots,i_p}$ of $\pi_{i_1,i_2,\ldots,i_p}$ is assumed to be generated from the following output distribution:
\be
P_{\rm out}(y_{i_1,i_2,\ldots,i_p} \mid \pi_{i_1,i_2,\ldots,i_p}).
\ee
This defines the likelihood function for $\V{x}$ through the relation \NReq{x to pi}.

Similar problem settings have been investigated in the context of tensor/matrix factorization so far~\cite{rangan2012iterative,matsushita2013low,richard2014statistical,lesieur2017statistical,sakata2013statistical,sakata2013sample,krzakala2013phase,kabashima2016phase,donoho2009message,pourkamali2024matrix}, and the difference of our work from them lies in assuming the limit $N,M\to \infty$ while keeping the relation $N \gg M \gg 1$, and in assuming the observations are made uniformly randomly over all $p$-plets under the constraint that each vector $\bm{x}_{i}$ is observed $c = \alpha M$ times with $\alpha=O(1)$.
\black{Thus the inference is actually based on very sparse measurements of the tensor:
just  $O(NM)$ out of  $N^{p}$ elements of the tensor are measured.}
We call these assumptions {\it dense limit} \cite{yoshino2018,yoshino2023spatially}
hereafter. 
This is an interesting limit in the context of tensor completion:
%Here it is very interesting to view our problem as a tensor completion task.
%We are considering inference of $N$ vectors $\bm{x}_{i}$
%based on the sparce measurements of tensor products $\pi_{\bs}$ defined in \eq{eq-pi-bs}.
once we obtain an estimate of the vectors $\bm{x}_{i}$ $i=1,2,\ldots,N$, we can
make an inference of the whole tensor \eq{eq-def-pi} with $O(N^{p})$ elements,
just based on a vanishing fraction of it.
\blue{As we explain in the next section, the measurement is defined on random graphs
with connectivity $c(=\alpha M,~\alpha=O(1))$. Usually sparse graphs mean $c=O(1)$ so that our graphs are
not sparse in the usual sense. Nonetheless
they are much sparser compared with those with global couplings $c=O(N^{p})$
since we consider the dense limit $N \gg M \gg 1$  \cite{yoshino2018,yoshino2023spatially}.
\blue{More precisely $N$ must grow faster than any power of $c$.}
We call such an intermediately sparse graph as {\it dense graph} in the present paper.
}

The aim of this paper is to clarify the performance of the theoretically optimal estimation and also to provide an algorithm achieving optimality under the dense limit.
To this end, we work in the framework of the so-called Bayes optimal inference, assuming that the likelihood $P_{\rm out}$, the linear coefficient $F$, and the prior distribution $P_{\rm pri.}$ of $\bm{x}$ are known. The posterior distribution in the Bayes optimal setting is known to give the minimum mean square error (MMSE), and thus serves as a good benchmark for other frameworks and algorithms.

To analyze the Bayes optimal performance, we use the replica method from statistical mechanics. The result reveals how the MMSE depends on the parameters, as well as the existence of phase transitions connected to the computational difficulty of the posterior or the fundamental estimation limit. In addition, we invent an efficient algorithm computing the posterior average based on the framework of generalized approximate message passing (G-AMP)~\cite{rangan2011generalized}; the associated state evolution (SE) equations are derived and checked to be consistent with the replica result. Numerical experiments on finite-size systems based on G-AMP are also conducted, showing the consistency with the SE/replica predictions.

The consistency of the numerical experiments with the theoretical predictions suggest that our findings are exact in the limit we are currently considering, and, to the best of our knowledge, this is the first result of such precise asymptotics for tensor/matrix factorization and completion where the rank is not $O(1)$ and based on significantly sparse measurements of the tensor. This becomes possible by assuming the dense limit,  which allows intricate correlations between variables to be ignored. \blue{Actually, the replica theory we develop employs a cumulant expansion to study the effects of interaction between microscopic variables. The systematic treatment of the interactions allows us to avoid blind usage of the Gaussian ansatz, which we find to fail in some fully connected systems. The latter includes the matrix factorization problem ($p=2$) of the full-rank $M=N$ case.} This is an important technical consequence of this paper. 

The remainder of this paper is organized as follows. In the next section, we provide a detailed formulation of our problem setting.  Sec. \ref{sec-replica} explains the analysis using the replica method: we here elaborate how higher-order correlations among variables can be neglected under the dense limit.
In sec.~\ref{sec-mp}, the algorithm based on G-AMP and the associated SE equations are derived. In sec.~\ref{sec-result}, we examine some specific parameter settings and observe the behavior of the MSE, phase transitions, and the performance of the G-AMP algorithm. Sec.~\ref{sec-conclusion} concludes the paper with some remarks.

%In the appendix we explain the details of the message passing algorithms and of the replica theory.

%layered geometry with dense couplings which are restricted to those between neighboring layers.

%{\bf Perspective: M-p SG, space dependence like DNN (wetting transition)}

%The reason why this model?: Relation to a series of recent analyses of deep neural networks and disorder-free spin glass~\cite{10.21468/SciPostPhysCore.2.2.005,10.21468/SciPostPhys.4.6.040}; a generalization of LDPC; some motivating examples (large number of topics in topic model? large number of attributes in collaborative filtering? such examples seem to be hard to find...(sparsity ansatz does not hold))?

\label{introduction}
%%%%%%%%%%%%%%%%%%%%%%%%%%%%%%%%%%%%%%%%%%%%%%%%%%%%%
%%%%%%%%%%%%%%%%%%%%%%%%%%%%%%%%%%%%%%%%%%%%%%%%%%%%%
%%%%%%%%%%%%%%%%%%%%%%%%%%%%%%%%%%%%%%%%%%%%%%%%%%%%%
\section{Formulation and Related Work}
\label{sec-model}

%\input{model.v3}
%%%%%%%%%%%%%%%%%%%%%%%%%%%%%%%%%%%%%%%%%%%%%%%%%%%%%
%%%%%%%%%%%%%%%%%%%%%%%%%%%%%%%%%%%%%%%%%%%%%%%%%%%%%
%\subsection{Bayesian inference in the teacher-student setting}
\subsection{Notations, Graphical Models, and Quantities of Interest}
\label{sec-model-teacher-student}
Let us consider $N$ $M$-dimensional vectors $\bm{x}_{i}=(x_{i1},x_{i2},\ldots,x_{iM})^{\top} \in \mR^{M}~(i=1,2,\ldots,N)$ and denote the variable index set (identified with the variable set itself) by $\Vset=\{1,2,\ldots,N\}$. As mentioned in sec. \ref{introduction}, we make observations over a set of $p$-plets which is uniformly and randomly sampled over all $p$-plets without replacement under the constraint that each vector $\bm{x}_{i}$ is observed $c=\alpha M~(\alpha=O(1))$ times. The set of actually selected  $p$-plets is hereafter denoted as $\Eset$ and called edge set, with the intention of later associating this with a graphical model. For convenience, we introduce the symbol $\bs$ to index the selected $p$-plets. Among the $p$-plets, the subset that includes site $i (\in \Vset)$ is denoted by $\del i=\{\bs \in \Eset\mid i \in \bs \}$: the equality $|\del i|=c$ holds by the above constraint. In parallel, we express the variable set belonging to the $p$-plet $\bs$ as $\del \bs=\{i\in \bs \}$: the equality $|\del \bs|=p$ holds by definition. The noiseless observation associated to a $p$-plet $\bs$, $\pi_{\bs}$, is thus rewritten from \Req{x to pi} to
\be
\pi_{\bs}
=\frac{\lambda }{\sqrt{M}}\sum_{\mu=1}^{M}F_{\bs,\mu}\prod_{i\in \partial \bs}x_{i \mu}.
\label{eq-pi-bs}
\ee
Its observation $y_{\bs}$ is generated from the output distribution $P_{\rm out}(y_{\bs}\mid \pi_{\bs})$. Under these notations, the posterior distribution given all the observations $\{y_\bs\}_{\bs \in \Eset}$ is expressed as
\beqn
P_\text{pos.} \qty(\{\bm{x}_i\}_{i \in \Vset}\mid \{y_\bs\}_{\bs \in \Eset}) &
= 
\frac{1}{Z \qty(\{y_\bs\}_{\bs \in \Eset})}
\qty( \prod_{\bs \in \Eset}P_{\rm out}(y_{\bs}\mid \pi_{\bs}) )
\qty(\prod_{i\in \Vset}\prod_{\mu=1}^{M} P_\text{pri.}(x_{i\mu})),
\label{eq-formula-for-students}
\eeqn
where the prior distribution $P_\text{pri.}$ is assumed to be separable over each component and $Z \qty(\{y_\bs\}_{\bs \in \Eset})$ is a normalization constant expressed as
\beqn
Z \qty(\{y_\bs\}_{\bs \in \Eset})
=
\int \qty( \prod_{i\in \Vset}\prod_{\mu=1}^{M} dx_{i \mu} )
\qty( \prod_{\bs \in \Eset}P_{\rm out}(y_{\bs}\mid \pi_{\bs}) )
\qty(\prod_{i\in \Vset}\prod_{\mu=1}^{M} P_\text{pri.}(x_{i\mu})).
\label{eq-partition-function}
\eeqn
In addition, the edge set size is denoted as $N_{\bs}=|\Eset|$, and it is related to the other parameters as
\beq
N_{\bs}=Nc/p=(\alpha/p)NM
\label{eq-Nbs}
\eeq
Note that this edge set size is much smaller than the number of all $p$-plets $\binom{N}{p}=O(N^{p})$. The assumption of this sparse observation with the limit $M,N\to \infty$ keeping $N \gg M$ is the key to making our analysis exact, which has been thus termed {\it dense limit} as explained in sec. \ref{introduction}. The quantity
\beq
\gamma=\frac{\alpha}{p}
\label{eq-gamma}
\eeq
represents the ratio of the number of observations to the number of variables to be estimated, making it an important parameter that characterizes the estimation accuracy.

The above problem setup can be illustrated by a factor graph as shown in Fig.~\ref{fig:fig_graph}: the variable $\V{x}_i$ corresponds to a variable node represented by a gray circle, and the likelihood $P_{\rm out}(y_{\bs}\mid \pi_{\bs})$ is associated to a function node expressed by a black square.
\begin{figure}
  \begin{center}
    \includegraphics[width=0.95\textwidth]{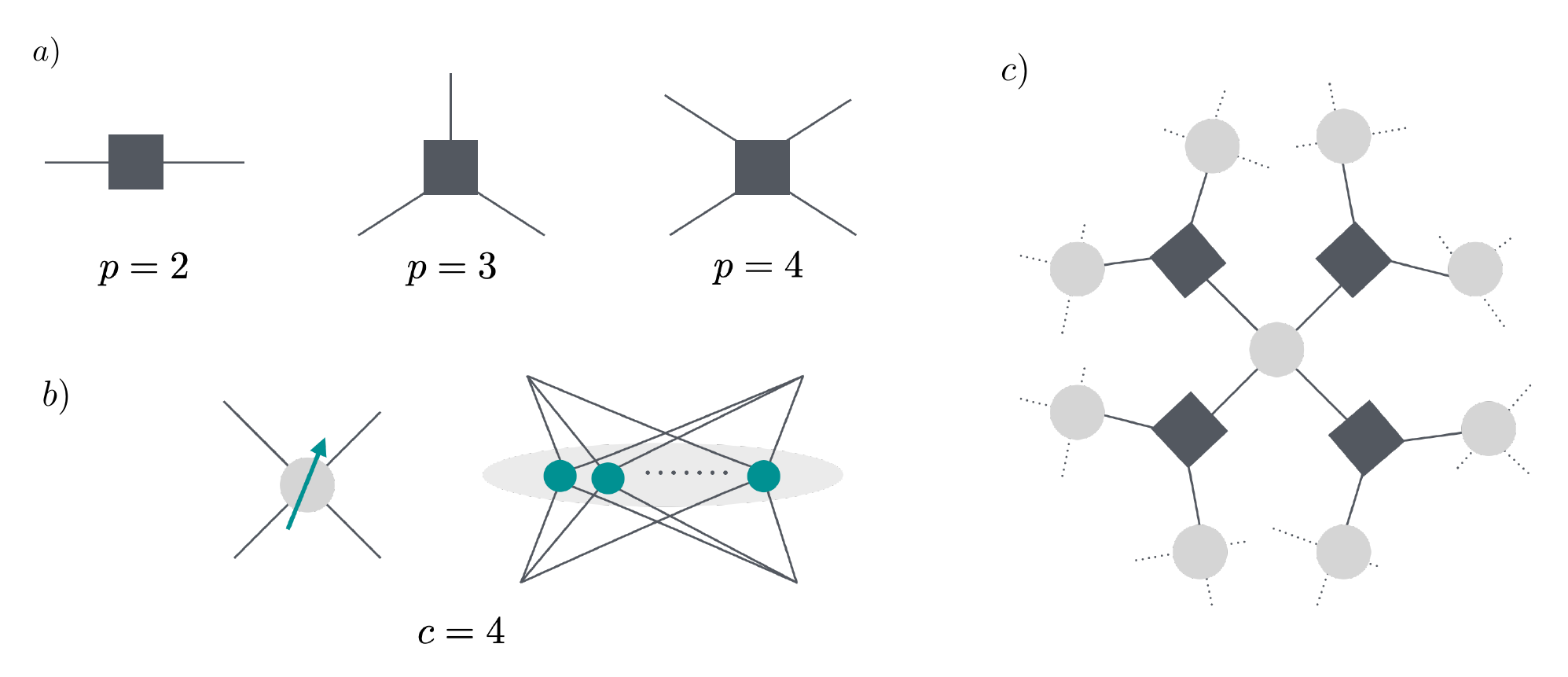}
    \end{center}
 	\caption{
          Graphical representations of our model.
          a) The black-squares $\bs$ represent the function nodes to each of which a $p$-plet $(i_{1},i_{2},\ldots,i_{p})$
          is assigned. Each function node has $p$ arms. 
          b) The gray circles represent the variable nodes $(i=1,2,\ldots,N)$
          to each of which a $M$-component vectorial variable ${\bm x}_{i}$ (arrow) is assigned.
          Each of the variable nodes has $c=\alpha M$ arms.
          Equivalently, we may also represent a variable node as a collection of sub-variable nodes (green circles)
          to each of which a component of the vector $x_{i\mu}$ is assigned.
          c) a factor graph is created by joining the variable and function nodes.}
	  \label{fig:fig_graph}
\end{figure}
The variable node can be decomposed into a collection of sub-variable nodes corresponding to the variable components $x_{i\mu}$, which are represented by green circles. From this graphical representation, the algorithm known as belief propagation (BP) is naturally derived, which will be explained in sec. \ref{sec-mp}.

%%%%%%%%%%%%%%%%%%%%%%%%%%%%%%%%%%%%%%%%%%%%%%%%%%%%%
\subsubsection{Technical Assumptions}
\label{sec-techinical-assumptions}

For each tensor component $\pi_{\bs}$ to be $O(1)$, the following two properties are hereafter assumed.
\begin{enumerate}
\item{ The prior distribution is zero-mean:
  \be
    \int \dd x P_\text{pri.}(x)x=0.
    \Leq{prior_constraint}
  \ee}
\item{The linear coefficient $F_{\bs,\mu}$ satisfy one of the following two conditions.
\begin{enumerate}
\item{(deterministic model)}
  \beq F_{\bs,\mu}=1.
  \label{eq-F-uniform}
  \eeq
\item{(random model) or (random spreading factor model)}

  $F_{\bs,\mu}$  is an i.i.d. random variable  with zero mean and unit variance.
  Writing the distribution function of $F$ as $P(F)$,
  \beq
  \int dF_{\bs,\mu} P(F_{\bs,\mu}) F_{\bs,\mu}=0, \qquad  \int dF_{\bs,\mu} P(F_{\bs,\mu}) F_{\bs,\mu}^{2}=1.
    \label{eq-F-random}
  \eeq
%where $P_{F}(\cdot)$ is the distribution function of $F$
%	The corresponding distribution of $F$ is denoted as .}
\end{enumerate}
For the deterministic model we can write $P(F_{\bs,\mu})=\delta (F_{\bs,\mu}-1)$.

}
\end{enumerate}

These two properties are crucial for making our analysis exact, which will become clear particularly in sec. \ref{sec-replica} for the replica analysis. Furthermore, we remark that the two conditions for the linear coefficient, deterministic and random, yield the identical macroscopic result in the dense limit. From the viewpoint of the tensor/matrix decomposition, the deterministic case $F=1$ corresponds to the so-called CP (Canonical Polyadic) decomposition and is more natural. These facts might seem to suggest that there is no reason to analyze the random model. We, however, argue that this is not the case. First, the fact that the deterministic and random cases yield the identical result in the limit is itself a nontrivial finding that has been revealed through our analysis. Second, in the message passing algorithms we develop, the random coefficient significantly improves convergence, particularly in the case of  $p=2$;
%which is presumably due to that the randomness breaks the rotational symmetry present in the $p=2$ case;
this as a result enhances the consistency between the microscopic behavior of the algorithm and the macroscopic behavior predicted by SE or the replica method, even for relatively small values of  $N$ and $M$. This means that the random case also has its own significance. 

Although any prior distributions satisfying \Req{prior_constraint} can be treated in the presented formulation, we mainly consider the two cases: Ising and Gaussian distributions. Their explicit formulae are
\beq
P_{\rm pri.}(x)=\frac{1}{2} \left[\delta(x-1)+\delta(x+1) \right]
\label{eq-prior-ising}
\eeq
for the Ising case and 
\beq
P_{\rm pri.}(x)=\frac{e^{-\frac{1}{2}x^{2}}}{\sqrt{2\pi}}
\label{eq-prior-gaussian}
\eeq
for the Gaussian case.

Similarly, although our theoretical treatment again allows us to treat generic output distributions as understood from the derivations shown in secs. \ref{sec-replica},\ref{sec-mp}, we consider two types of specific distributions for quantitative analysis. The first type treats an additive noise $w$. Namely
\beq
y = \pi + w.
\label{eq-additive-noise-output}
\eeq
This means the likelihood function can be written as
  \beqn
  P_\text{out}\qty(y|\pi) &= \int dw W(w)\delta(y-(\pi+w ))
=W(y-\pi),
%\frac{1}{\sqrt{2\pi\Delta_*^2}} e^{-\frac{1}{2\Delta_*^2} \qty(y-\pi)^2}
\label{eq-p-put-additive-noise}
  \eeqn
with $W(w)$ being the noise distribution. A particularly interesting case is the Gaussian noise:
\beq
%W(w)=\mathcal{N}(w\mid 0,\Delta^2).
W(w)=\mathcal{N}(w\mid 0,1).
\eeq
\blue{Here we chose to work with the case of zero mean and unit variance.}
The other type is the sign output function: 
\beq
  y = {\rm sgn}\qty(\pi).
 \label{eq-sign-output}
 \eeq
This yields
  %\siki{
  \beq
  P_\text{out}(y|\pi) = %\Theta_\text{H}(y\pi)
  %  \theta(y\pi)
  \delta(y-1)\theta(\pi)+  \delta(y+1)\theta(-\pi),
  \eeq
  where $\theta(x)$ is the Heaviside step function.

%%%%%%%%%%%%%%%%%%%%%%%%%%%%%%%%%%%%%%%%%%%%%%%%%%%%%
\subsubsection{Averages and Order Parameters}
\label{sec-order-parameters}
In the present problem, there are several different types of random variables. To discriminate the average over these things, we introduce two different average notations.

The first average is that of $x_{i \mu}$s over the posterior distribution \nreq{eq-formula-for-students} and is denoted as
\beq
\langle (\cdots) \rangle  =
\int \qty( \prod_{i\in \Vset}\prod_{\mu=1}^{M} dx_{i\mu} )
 P_{\text{pos.} }  
 \qty(  \{\bm{x}_i \}_{i \in \Vset} \mid \{y_{\bs} \}_{\bs \in \Eset}) (\cdots).
 \eeq
 This corresponds to the thermal average in statistical mechanics.

% For technical reasons, we also consider the posterior distribution without a likelihood factor specified by $\bs'$, namely we have
%\be
% P_{\text{pos.} }
% \qty(  \{\bm{x}_i \}_{i \in \Vset} \mid \{y_{\bs} \}_{\bs \in \Eset \backslash \bs'}) 
%\propto
%\qty( \prod_{\bs \in \Eset\backslash \bs'}P_{\rm out}(y_{\bs}\mid \pi_{\bs}) )
%\qty(\prod_{i\in \Vset}\prod_{\mu=1}^{M} P_\text{pri.}(x_{i\mu})),
% \ee
%and the corresponding average is denoted by
%\be
%\langle (\cdots) \rangle^{\backslash \bs'}
%=
%\int \qty( \prod_{i\in \Vset}\prod_{\mu=1}^{M} dx_{i\mu} )
% P_{\text{pos.} }
% \qty(  \{\bm{x}_i \}_{i \in \Vset} \mid \{y_{\bs} \}_{\bs \in \Eset \backslash \bs'}) (\cdots).
%\ee
%This average is mainly used in sec. \ref{sec-mp}.
%
 
The posterior average depends on the observations $\{y_\bs\}_{\bs \in \Eset}$ which are random variables themselves, and we denote the average over the $y_{\bs}$s by
\beqn
&&
    \mathbb{E}_{y}\lb \cdots \rb
    =
    \int \lb \prod_{\bs\in \Eset}dy_{\bs} \rb P(\{y_{\bs} \}_{\bs \in \Eset}) (\cdots),
    \label{eq-posteriror-average}
\eeqn
where $P(\{y_{\bs} \}_{\bs \in \Eset})$ is the distribution of $\{y_\bs\}_{\bs \in \Eset}$, which is computed from the likelihood and the prior distribution as follows:
\be
P(\{y_{\bs} \}_{\bs \in \Eset})
=
\int 
	\left(\prod_{\bs \in \Eset} 
	\prod_{\mu=1}^M dF_{\bs,\mu} P_{F}(F_{\bs,\mu})
	\right)
	\left(\prod_{i\in \Vset} 
	\prod_{\mu=1}^M dx_{i\mu} P_{\rm pri.}(x_{i\mu})
	\right)
\left( \prod_{\bs\in \Eset} P_{\rm out}(y_\bs \mid \pi_\bs) \right). \qquad
\ee
This average corresponds to the quenched average in statistical mechanics of disordered systems. Although the graph structure associated to $\Eset$ is also a random variable in the current setup, \blue{as we will see both in the replica approach
 (see sec~\ref{sec-interaction-part-of-free-energy}) and message passing approach (see sec~\ref{sec-SE-equations})} this randomness has no direct consequence on the analysis
of macroscopic observable in the limit $N, M \to \infty$, and hence we refrain from introducing a symbol to represent it.
\blue{(In numerical analysis of message passing algorithms on finite sized systems we perform averages over the realizations of $\Eset$.)}

Our estimation target is $\{ \V{x}_i \}_{i \in \Vset}$. To express this explicitly, we denote the true vectors  generated from the prior distribution $P_{\rm pri.}$ as  $\{ \V{x}_{*,i} \}_{i \in \Vset}$. Here, the symbol $*$ is interpreted to represent parameters to be estimated. The $p$-plets $\{ \pi_{\bs} \}$ are also considered to be the estimation targets and the same subscript rule is applied. Namely, we express the true value of the  $\pi_{\bs}$ as
\be
\pi_{*,\bs}=\frac{\lambda}{\sqrt{M}}\sum_{\mu=1}^{M}F_{\bs,\mu}\prod_{i\in\del \bs}x_{*,i\mu}.
\ee
The linear coefficient $F_{\bs,\mu}$ is assumed to be known and thus has no $*$-subscript. Hereafter, we call the system with these true values {\it teacher}, while the inference model is termed {\it student}. This terminology follows the so-called teacher-student setup in generic inference problems~\cite{zdeborova2016statistical}, which originally comes from neural networks.

As it will be evident in secs. \ref{sec-replica} and  \ref{sec-mp}, the following order parameters play the key roles in characterizing the macroscopic behavior:
\beqn
m &=
%\mathbb{E}\qty[x_{*,i\mu} m_{i\mu}^t]=
\mathbb{E}_{y}\left[ \frac{1}{NM} \sum_{i\in \Vset} \sum_{\mu=1}^M x_{*,i\mu} \langle x_{i\mu} \rangle \right], \label{eq-def-m}
\\
q &=
%\mathbb{E}\qty[(m_{i\mu}^t)^2]=
\mathbb{E}_{y} \left [ \frac{1}{NM} \sum_{i\in \Vset} \sum_{\mu=1}^M \langle x_{i\mu} \rangle^2 \right].
\label{eq-def-q}
\eeqn
The parameter $m$ is the overlap which measures the similarity between teacher and student.
We may call it as 'magnetization' due to the analogy to spin systems in physics.
When the meaningful inference is possible, the inequality $m > 0$ holds. Using the physics terminology, we
may call the phase with $m=0$ as paramagnetic phase (phase where inference is impossible)
and $m>0$ as ferromagnetic or magnetized phase.
On the other hand, the parameter $q$ is analogous to the Edwards-Anderson
order parameter used in spin-glass physics\cite{edwards1975theory}.

%In the present context it can be interpreted as the overlap between two students, say A and B, doing the same inference task under the same data $y_{\bs}$ provided by the teacher. This can be understood easily because we can write,   \beq \langle x_{i\mu} \rangle^2= \underbrace{\langle x_{i\mu} \rangle}_{\rm A}\underbrace{\langle x_{i\mu} \rangle}_{\rm B} \eeq

These order parameters are directly connected to two standard quantitative indicators of inference accuracy. One is the MSE concerning to the input $\V{x}_i$:
\be
\frac{1}{NM} \sum_{i\in \Vset} \sum_{\mu=1}^M \mathbb{E}_{y} \left[ \lb \left\langle x_{i\mu}\right\rangle -x_{*,i\mu} \rb^2\right] 
=1-2m+q,
\label{eq-def-MSE1}
\ee
where we assumed $\mathbb{E}_{y}x^2_{*,i\mu}=1$, which is the case according to our priors (see  \eq{eq-prior-ising} and \eq{eq-prior-gaussian}). The other is the MSE about the output $\pi_{i_1,i_2,\ldots, i_p}$:
\beqn
\left(
\begin{array}{c}
  N      \\
    p   
\end{array}
\right)^{-1}
\sum_{1 \leq i_{1} <  i_{2} < \cdots i_{p} \leq N}
    \mathbb{E}_{y}\lsb 
    \lb  \left\langle \pi_{i_1,i_2,\ldots,i_p}\right\rangle -\pi_{*,i_1,i_2,\ldots,i_p} \rb^{2}
 \rsb
 =\lambda^2 (1-2m^{p}+q^p).
 \label{eq-def-MSE2}
\eeqn
Here we assumed that the posterior distribution is separable and $\frac{1}{N}\sum_{i\in \Vset}\Ave{x_{i\mu}}=0$, $\frac{1}{N}\sum_{i\in \Vset}\Ave{x_{i\mu}}\Ave{x_{i\nu}}=0~(\mu \neq \nu)$, both of which asymptotically hold in the current setup, and $\mathbb{E}_{y}\lb\pi^{2}_{*,i_1,i_2,\ldots,i_p}\rb=\lambda^2$ which is the case in our priors.
The latter MSE can also be regarded as a 'generalization error' in the context of tensor completion
\blue{in the following sense. Given an inference of the vectors $\bm{x}_{i}$ $(i=1,2,\ldots,N)$ we can
  construct estimates of any components of the tensor $\pi_{i_{1},i_{2},\ldots,i_{p}}$
  through \eq{eq-def-pi} including those that have not been observed.
  The error of the inference of such tensor components not used in inference (training) may be
  regarded as 'generalization error'.
  Since the number of components observed $N_{\bs}=Nc=\alpha NM$ is a vanishing fraction of the
  total number pf the components $N^{p}$ in the dense limit $N \gg M \gg 1$, the MSE \eq{eq-def-MSE2}
  is dominated by non-observed ones.
  Thus, it can be regarded as a generalization error for tensor completion.
}

In the Bayes-optimal case, one can prove easily that  $m=q$ holds \cite{zdeborova2016statistical}, which is analogous to what happens along the Nishimori line in spin-glass systems \cite{nishimori2001statistical,iba1999nishimori}.
%In this situation, replica symmetry breaking (RSB) does not take place.
This idealized situation is the problem setup of this study.

%\subsection{Generalization Error for Tensor Completion}
%\label{sec-generalization-error-for-tensor-completion}
%
%As stated in the introduction it is interesting to view our statistical inference problem as a tensor completion. We are considering inference of $N$ $M$-dimensional
%vectors $\bm{x}_{i}=(x_{i1},x_{i2},\ldots,x_{iM})^{\top} \in \mR^{M}~(i=1,2,\ldots,N)$
%based on some sparce measurements of tensor products $\pi_{\bs}$ defined in \eq{eq-pi-bs}.
%Once we obtain an estimate of the vectors $\bm{x}_{i}$ $(i=1,2,\ldots,N)$, we can
%make an inference of the whole tensor
%$\pi_{i_1,i_2,\ldots,i_p}$  defined in \eq{eq-def-pi} with $O(N^{p-1})$ elements.
%Notice that $N_{\bs}/N^{p-1} \to 0$ in the dense limit $N \gg M \gg 1$ for $p \geq 2$
%which means that we just used a vanishing fraction of the tensor elements to do the inference.
%
%Thus it is reasonable to quantify the generalization error for the tensor completion task
%by a mean-square error,
%\beqn
%E&=&\frac{p!}{N(N-1)\cdots (N-p+1)}
%\sum_{1 \leq i_{1} <  i_{2} < \cdots i_{p} \leq N}
%    \mathbb{E}_{y}\lb 
%    ( (\pi_{*})_{i_1,i_2,\ldots,i_p} - \pi_{i_1,i_2,\ldots,i_p})^{2}
% \rb
%    \nonumber \\
%&=&2(1-m^{p})
%\eeqn
%Here we assumed that $\mathbb{E}_{y}\lb\pi^{2}_{*}\rb=\mathbb{E}_{y}\lb\pi^{2}\rb=1$
%which holds in the casese that we consider.

%%%%%%%%%%%%%%%%%%%%%%%%%%%%%%%%%%%%%%%%%%%%%%%%%%%%%
%%%%%%%%%%%%%%%%%%%%%%%%%%%%%%%%%%%%%%%%%%%%%%%%%%%%%
\subsection{Symmetries}
\label{sec-symmetries}

There are several global symmetries in our inference problems:
the statistical weight of a microscopic state $\{x_{i \mu}\}$ is invariant
  under the following operations depending on situations.
\begin{itemize}
\item {\bf reflection}: 
  In the cases of even $p$, the system is invariant 
  under $x_{i \mu} \to -x_{i \mu}$ for all $i \in V$.
  This is because we are considering prior distributions
  $P_{\rm pri.}(x)$ which are even functions of $x$
  (see  \eq{eq-prior-ising} and \eq{eq-prior-gaussian}) 
  .
\item {\bf permutation}: In the case $F_{\bs,\mu}=1$ (deterministic model),  the system is invariant under permutations of the index $\mu$ for the components.   In the random model (\eq{eq-F-random}), this symmetry is suppressed.
\item {\bf rotation}: In the special case of $p=2$ with the Gaussian prior
  (\eq{eq-prior-gaussian}), the system with 
$F_{\bs,\mu}=1$ (deterministic model, \eq{eq-F-uniform}) 
  is invariant under global rotations
  $x_{i \mu}=\sum_{\nu=1}^{M}R_{\mu \nu}x_{i \nu}$ with any rotation matrix
  $R_{\mu \nu}$.
  In the random model (\eq{eq-F-random}), this symmetry is suppressed.
\end{itemize}
Note that even in the random model (\eq{eq-F-random}) the reflection symmetry still remains for even $p$.

The order parameters defined in \eq{eq-def-m} and \eq{eq-def-q} become $0$ under the above symmetry transformations.
By the replica theory developed in sec.~\ref{sec-replica} and applied to specific models in sec~\ref{sec-result}, we investigate 
possibilities of spontaneous breakings of the symmetries in the dense limit $\lim_{c \to \infty}\lim_{N \to \infty}$.
By the message passing algorithm developed in sec.~\ref{sec-mp} and applied to specific models in sec~\ref{sec-result}, 
the remaining symmetry may be broken dynamically by the algorithm even in systems with finite $N$,$c$ and $M$.

%%%%%%%%%%%%%%%%%%%%%%%%%%%%%%%%%%%%%%%%%%%%%%%%%%%%%
%%%%%%%%%%%%%%%%%%%%%%%%%%%%%%%%%%%%%%%%%%%%%%%%%%%%%
\subsection{Related Work}
Let us sketch below how our problem can be compared with other problems.

%%%%%%%%%%%%%%%%%%%%%%%%%%%%%%%%%%%%%%%%%%%%%%%%%%%%%
\subsubsection{Matrix/Tensor Factorization}
Matrix factorization ($p=2$) has been extensively studied across various disciplines due to their broad applicability. For instance, dictionary learning seeks to represent signals or data samples as sparse linear combinations of learned basis elements, and it has proven effective in image denoising and compression tasks \cite{10.1145/1553374.1553463,4011956}. Another prominent line of research focuses on low-rank matrix completion, where the goal is to recover a low-rank structure from incomplete observations; a landmark result in this area showed that under certain conditions exact recovery can be achieved through convex optimization \cite{candes2012exact}. In applications such as array signal processing and compressed sensing, blind calibration has emerged to jointly estimate unknown sensor (or system) parameters and the underlying low-rank or sparse representation of data \cite{7919265}. Finally, robust PCA addresses the decomposition of data matrices into a low-rank component plus sparse outliers, ensuring stable principal component recovery even in the presence of gross corruptions \cite{10.1145/1970392.1970395}. The situation is similar for tensor factorization: many applications are found in fields such as signal processing and machine learning \cite{anandkumar2014tensor,cichocki2015tensor,sidiropoulos2017tensor}. These approaches collectively highlight the power of exploiting low-dimensional structures, whether in matrices or higher-order tensors, to solve a range of real-world problems.

It is important to note that the matrix/tensor factorization is performed often in the presence of missing entries \cite{acar2011scalable} which can be substantial in practice. This aligns with our current problem setup with sparse observation. When solving such problems, the estimated matrix/tensor's rank should be effectively low enough compared to its apparent dimensionality. \blue{Indeed matrix completion based on sparse measurements have been considered
also by statistical mechanis approaches
for low-rank cases \cite{keshavan2012efficient,gamarnik2016note,okajima2021matrix,stephan2024non}.}
However the magnitude of such effective rank in real-world data varies between different cases, and it is not always very low. For example, estimated tensor-rank can be very large (order $10^2$) for facial images \cite{zhao2015bayesian} and recommendation systems ($50$--$1000$) \cite{bell2007bellkor,koren2009matrix}, while $N$ is even larger by orders of magnitude. The assumption  $N \gg M$  in the dense limit that we consider in this paper is well aligned with this situation, suggesting that the analysis in this paper can provide meaningful insights in such a setting.

%%%%%%%%%%%%%%%%%%%%%%%%%%%%%%%%%%%%%%%%%%%%%%%%%%%%%
\subsubsection{Statistical Mechanical Approach to Tensor/Matrix Factorization}
Low-rank ($M=O(1)$) tensor/matrix factorization have been studied in a number of earlier works from statistical mechanical viewpoints. For example, \cite{rangan2012iterative} proposed an iterative estimation algorithm for rank-one matrices under various constraints (e.g., sparsity or non-negativity) and analyzed its performance using SE. \cite{matsushita2013low} combined low-rank matrix estimation with clustering and invented an AMP algorithm, showing that the AMP algorithm can outperform classical algorithms such as Lloyd’s K-means in certain regimes. \cite{richard2014statistical} focused on tensor PCA in the rank-one spiked tensor model and conducted a thorough analysis of both statistical and computational thresholds. In particular, this work derived an information-theoretic threshold (below which no estimator can reliably recover the spike) and compared it to the thresholds required by polynomial-time algorithms proposed also in the paper (naive power iteration, unfolding, and an AMP algorithm). The paper also discussed how the maximum likelihood estimation achieves the statistically optimal threshold, whereas practical, polynomial-time methods require a larger signal-to-noise ratio to succeed. This highlights the presence of easy and hard regions for spike recovery. \cite{lesieur2017statistical} also studied the spiked tensor model with rank $M=O(1)$ in a Bayes-optimal framework, and computed the mutual information, MMSE, and other statistical measures via the replica method. A key finding is that there is an information-theoretic threshold on the signal-to-noise ratio below which accurate recovery is impossible, regardless of computational resources. The paper also proposes an AMP algorithm and examines its performance, showing that while AMP can achieve the MMSE in certain parameter regimes, there remains a hard region where AMP (and presumably any polynomial-time algorithm
\cite{celentano2020estimation}
) cannot match the Bayes-optimal performance. 

On the other hand, the high-rank (extensive-rank) case $(M=O(N))$ has been analyzed in some papers but only in the context of matrix factorization or dictionary learning. \cite{sakata2013statistical,sakata2013sample} analyzed the reconstruction performance of the dictionaries (matrices) based on the frameworks of the reconstruction error minimization with the $\ell_0$ regularization and the Bayes-optimal one, respectively, using the replica method. The results show that the reconstruction is possible with $O(N)$ training samples both for the error minimization and the Bayes-optimal frameworks, underscoring that feasible sample complexities scale linearly with the dictionary dimension. In the Bayesian framework, both \cite{krzakala2013phase,kabashima2016phase} used the replica method to provide asymptotically exact descriptions of the MMSE and free-energy, and invented AMP algorithms. \cite{krzakala2013phase} mainly treated the problems of dictionary learning and blind calibration while \cite{kabashima2016phase} generalized a range of matrix factorization problems (dictionary learning, low-rank matrix completion, blind calibration, robust PCA) into a single Bayesian inference framework. Both established phase diagrams characterizing whether the matrix reconstruction is tractable or not, and clarified how they depend on the problem setting details.  

Actually, the high-rank case poses significant challenges. It has become recognized that the above analyses are not, in fact, precise after a while. To overcome the situation, \cite{maillard2022perturbative} proposed a systematic high-temperature or Plefka--Georges--Yedidia (PGY) expansion to derive the improved (but still approximate) mean-field description, showing that third-order corrections in the expansion lead to better predictions of the reconstruction error, highlighting the importance of including normally neglected terms. 
Progress has been made in sub-linearly high-rank cases $M=O(N^{a})$  with $0 < a < 1$
with full measurements $O(N^{2})$ \cite{reeves2020information,pourkamali2024matrix,barbier2024multiscale}.
Moreover in \cite{barbier2022statistical}, Barbier and Macris studied dictionary learning by combining random matrix theory with the replica method, introducing a new ansatz called spectral replica method. They derived exact or near-exact characterizations of the mutual information and mean-square errors for the matrix factorization task. In \cite{barbier2024phase},  Barbier et al. considered symmetric matrix denoising of a product-form matrix $XX^\top$ under additive Wigner noise, but with a non-rotationally-invariant prior on $X$. Their extensive numerical simulations show a distinct first-order denoising-factorization transition. Below this critical noise level, the rotationally invariant estimator (RIE) remains nearly optimal regardless of the prior. Above the threshold, however, the discrete and factorized structure of the prior strictly outperforms the RIE, though the authors observe that achieving this factorization-based improvement is often algorithmically difficult (due to a glassy posterior landscape).

As shown above, accurately analyzing the high-rank situation is challenging.
In the present paper we show it becomes relatively tractable under the dense limit ($N \gg M \gg 1$ with $O(NM)$ measurements) assumed in this paper for $p \geq 2$.
By leveraging the well established replica method and AMP within the dense limit, precise asymptotics are obtained. Interestingly, in contrast to \cite{pourkamali2024matrix} where the sub-linearly high-rank case $M=O(N^{\alpha})$ with $0 < \alpha < 1$ is treated under full measurements, our problems do not become equivalent to (a bunch of) the rank-one $M=1$ cases.

%%%%%%%%%%%%%%%%%%%%%%%%%%%%%%%%%%%%%%%%%%%%%%%%%%%%%
\subsubsection{Vectorial Constraint Satisfaction Problems}
\label{sec-vecotorialCSP}
Vectorial constraint satisfaction problems (CSP) was studied in \cite{yoshino2018}, and the problem studied in this paper is exactly an inverse problem or {\it planted version} of it. 

Let us briefly summarize the vectorial CSP problem. One considers the $M( \gg 1)$-dimensional vectorial 'spin' variables $\bm{x}_{i}$ ($i=1,2,\ldots,N$) interacting
with each other through $p$-body interactions specified by the following Hamiltonian:
\beq
H=\sum_{\bs=1}^{N_{\bs}} V(\delta-\pi_{\bs}),
\qquad
\pi_{\bs} =\frac{1}{\sqrt{M}} \sum_{\mu=1}^M F_{\bs,\mu}\prod_{j\in\del\bs} x_{j\mu},
%\label{eq-pi-bs}
\eeq
with a generic potential $V(x)$. As the 'spin' variables, Ising and spherical spins were considered.
For the factor $F_{\bs,\mu}$ (denoted as $X_{\bs \mu}$ in \cite{yoshino2018}), the deterministic, fully disordered, and intermediately disordered cases were considered. Glass transitions (clustering transitions) and jamming (SAT/UNSTAT transition) were studied in detail.
Here the number of components $M$ is assumed to be large $M \gg 1$.
This approach is complementary to the CSPs of
scalar continuous variables \cite{cavaliere2021optimization,cavaliere2025biased}.

In the case of linear potential
\beq
V(x)=Jx,
\label{eq-linear-potential}
\eeq
the problem was found to become essentially the same as the conventional (scalar $M=1$) $p$-spin (Ising/spherical) spin glass models with global $c \propto N^{p-1}$ coupling. In the case of $p=2$, the disordered
$F_{\bs,\mu}$ was needed to avoid crystallization. Otherwise, disorder-free glass transitions were found
within super-cooled paramagnetic phases.

More interesting cases are of non-linear potentials. In the quadratic potential
\beq
V(x)=\frac{\epsilon}{2}x^{2},
\label{eq-quadratic-potential}
\eeq
continuous replica symmetry breaking (RSB) was found even for the $p=2$ spherical model. 
In the case of the hard-core potential
\beq
e^{-\beta V(x)}=\theta(x),
\label{eq-hardcore-potential}
\eeq
various types of RSB and  
jamming transitions were found with spherical spins. The universality class of the jamming
turned out to be the same as that of hard-spheres \cite{charbonneau2014exact} for any $p \geq 1$,
including $p=1$ which corresponds to simple perceptrons \cite{franz2016simplest}.

%%%%%%%%%%%%%%%%%%%%%%%%%%%%%%%%%%%%%%%%%%%%%%%%%%%%%
\subsubsection{Error-correcting codes}
\label{sec-error-correcting-codes}

Our problem is also related to error correcting codes, which is a traditional topic in statistical mechanical informatics. In the Sourlas code \cite{sourlas1989spin} (which assumes global coupling $c \propto N^{p-1}$) and in the Kabashima-Saad code \cite{kabashima1998belief} (sparse coupling $c=O(1)$), $p$-body products of scalar ($M=1$) binary (Ising) variables are used to build the codes $\pi_{i_{1},i_{2},\ldots,i_{p}}$. Both codes attains the Shannon's information theoretical bound~\cite{shannon1948mathematical} in the limit $p \to \infty$.
An important difference is that the transmission rate $R=NM/N_{\bs}=M p/c$ in the thermodynamic limit $N \to \infty$ vanishes
in the Sourlas code ($M=1$) because of the global coupling $c \propto N^{p-1}$, while it can remain finite in the Kabashima-Saad code ($M=1$) by taking the limit
$c\to \infty$, after $N \to \infty$ limit, with the ratio $p/c$ fixed.
On the other hand, we consider vectorial variables with large number of components $M=c/\alpha \gg 1$ and 
intermediately dense coupling  $N \gg M \gg 1$. Our model also attains the Shannon's bound
with {\it finite} transmission rate $R=p/\alpha$ by taking the limit $p \to \infty$ with the ratio $p/\alpha$ fixed.

%%%%%%%%%%%%%%%%%%%%%%%%%%%%%%%%%%%%%%%%%%%%%%%%%%%%%
\subsubsection{Other Related Issues}
Linear observation ($p=1$) problems were widely studied in a number of contexts such as code division multiple access (CDMA)~\cite{tanaka2002statistical,kabashima2003cdma}, perceptron~\cite{gardner1989three,gyorgyi1990inference}, and compressed sensing~\cite{donoho2006compressed,kabashima2009typical,donoho2009message}. Our model can be regarded as a non-linear ($p> 1$) extension of these problems for a large number of vectors ($N \gg 1$).
In the linear estimations ($p=1$), a 'random spreading code' $F$ corresponding to \eq{eq-F-random} is indispensable.

The dense limit is used in this paper to make it possible to neglect higher-order loop effects. The essentially same technique was recently employed in the analysis of deep neural networks (DNNs)~\cite{yoshino2023spatially}. The dense limit serves as a useful asymptotic regime that enables analytical calculations. Beyond DNNs and tensor factorization, there may be other problems where this approach proves beneficial. 
%of width $N$ and depth $L$
%in a teacher-student analyzed in \cite{yoshino2020complex} by the replica method involves inference of
%the firing patterns $S_{i}^{\mu}=\pm 1$ of neurons $i=1,2,\ldots, NL$ in hidden layers
%for a large number of training data $\mu=1,2,\ldots,M$
%as well as synaptic couplings $J_{ij}$ through which the output of $j$-th perceptron
%is injected to $i$-th perceptron.

%%%%%%%%%%%%%%%%%%%%%%%%%%%%%%%%%%%%%%%%%%%%%%%%%%%%%
%%%%%%%%%%%%%%%%%%%%%%%%%%%%%%%%%%%%%%%%%%%%%%%%%%%%%
%%%%%%%%%%%%%%%%%%%%%%%%%%%%%%%%%%%%%%%%%%%%%%%%%%%%%
\section{Replica Theory}
\label{sec-replica}

\blue{In this section we develop a replica theory which we use in sec.~\ref{sec-result}
to analyze specific models. 
In sec~\ref{sec-replica-basic-formalism} we explain the basic strategy.
The crucial ingredicent in our scheme, absent in standard literatures, is
a cumulant exapnsion to study the effects of interaction between microscopic variables
discussed in sec.~\ref{sec-interaction-part-of-free-energy} and appendix. This allows one to avoid blind usage of Gaussian ansatz which fails in some fully connected systems. Readers may skip secs.~\ref{sec-replica-entropic-part-of-free-energy}-\ref{sec-replica-total-free-energy} if they are not interested with technical but standard details. In secs.~\ref{sec-error-correcting-code} and \ref{sec-remarks-on-non-linearity}, we discuss the significance of our model from the viewpoint of error correcting code and of matrix factorization ($p=2$) in the presence of additive Gaussian noise, respectively. In the latter, we explain why our problem does not become equivalent to a low-rank problem in contrast to sub-linearly rank cases $M=O(N^{a})$ with $0 < a < 1$ based on full measurements $O(N^{2})$. In sec.~\ref{sec-independence-on-choices-of-F} we point out independence of macroscopic quantities on choices of the (spreading) factor $F_{\bs,\mu}$.
}

%%%%%%%%%%%%%%%%%%%%%%%%%%%%%%%%%%%%%%%%%%%%%%%%%%%%%
%%%%%%%%%%%%%%%%%%%%%%%%%%%%%%%%%%%%%%%%%%%%%%%%%%%%%
%%%%%%%%%%%%%%%%%%%%%%%%%%%%%%%%%%%%%%%%%%%%%%%%%%%%%

\subsection{Basic formalism}
\label{sec-replica-basic-formalism}

\subsubsection{Introduction of replicated system}

The logarithm of the normalization constant  $Z$ (see \eq{eq-partition-function}), or the free energy in statistical mechanics, contains all the information about the system and serves as the central quantity in the discussion here. This quantity is expected to show the self-averaging property and thus converges to its expectation value in our dense limit. To analytically deal with this expectation, the replica method utilizes the following identity:
\be
\mathbb{E}_{y} \ln Z=\lim_{n\to 0} \frac{\partial}{\partial n}
\ln
\mathbb{E}_{y} Z^n .
\ee
The free energy $f$ is thus defined and computed as
\beq
 - f = \lim_{M,N\to\infty}\frac{1}{MN}\mathbb{E}_{y} \ln Z \lb \{y_\bs\} \rb 
    = \lim_{M,N\to\infty}\lim_{n\to 0} \frac{1}{MN}\frac{\partial}{\partial n}
    \ln 
    \mathbb{E}_{y} Z^n
    \lb \{y_\bs\} \rb.
 \label{eq-def-f}
 \eeq
 with the partition function $Z$ defined in \eq{eq-partition-function}.
In practice, we first set $n$ as a positive integer in order to perform calculations, obtain an analytically continuable expression under suitable ansatz, and then take the limit  $n \to 0$ using this expression.	 For any positive integer $n$, we may write the replicated partition function $Z^n$ as follows: 
\beq
Z^{n}(\{y_\bs\})= \int 
\prod_{a=1}^n
\lbb 
\left(\prod_{i\in \Vset}\prod_{\mu=1}^M dx_{i\mu}^a\right)
\left(\prod_{i\in \Vset} \prod_{\mu=1}^{M}P^{_{a}}_{\rm pri.}
(x_{i\mu}^a)\right)\left( \prod_{\bs\in \Eset} 
P_{\rm out} \lb y_\bs\lvert\pi^a_\bs \rb \right)
\rbb
,
\label{def:Zn-replica}
\eeq
where $\pi_{\bs}^{a}$ is defined as (see \eq{eq-pi-bs})
\beq
\pi^{a}_{\bs} =\frac{1}{\sqrt{M}} \sum_{\mu=1}^M F_{\bs,\mu}\prod_{j\in\del\bs} x^{a}_{j\mu} .
\label{eq-pi-bs-replica}    
\eeq
Recalling how the teacher generates the measurement $y_{\bs}$
(see \eq{eq-posteriror-average}), we find
\be
&&
    \mathbb{E}_{y}\lb \cdots  \rb
    =
    \int \lb \prod_{\bs\in \Eset} dy_{\bs}
    \prod_{\mu=1}^M dF_{\bs,\mu} P_{F}(F_{\bs,\mu})
     \rb
\no  \\  &&
\times    \int \left(\prod_{i\in \Vset}\prod_{\mu=1}^M dx_{*,i\mu}\right)
\left(\prod_{i\in \Vset}\prod_{\mu=1}^{M}P_{\rm pri.}(x_{*,i\mu})\right)
\left( \prod_{\bs\in \Eset} P_{\rm out}(y_\bs\lvert\pi_{*,\bs})\right) 
\lb \cdots \rb.
\ee
Assigning the replica index $0$ to the teacher or the ground truth
as $x_{*,i\mu}=x_{i\mu}^0,~\pi_{*}=\pi^0$, we have
\beqn
&&\mathbb{E}_{y} Z^n\lb \{y_\bs\}  \rb
=\int  \lb \prod_{\bs\in \Eset} dy_{\bs} 
\prod_{\mu=1}^M dF_{\bs,\mu} P_{F}(F_{\bs,\mu})
\rb
\no \\ && 
\times 
\int
\prod_{a=0}^n
\lbb
\left(\prod_{i\in \Vset}\prod_{\mu=1}^M dx_{i\mu}^a\right)
\left(\prod_{i\in \Vset}\prod_{\mu=1}^{M}P^{_{a}}_{\rm pri.}(x_{i\mu}^a)\right)
\left( \prod_{\bs\in \Eset} P_{\rm out}(y_\bs\lvert\pi^a_\bs)\right) 
\rbb
\nonumber \\
&&
=
\left.
Z_{1+n}[\partial/\partial h]
%\underbrace{\int\prod_{a=0}^n\left(\prod_{i\in \Vset}\prod_{\mu=1}^M dx_{i\mu}^a\right)
%\left(\prod_{i\in \Vset}\prod_{\mu=1}^{M}P^{_{a}}_{\rm pri.}(x_{i\mu}^a)\right)
%e^{\sum_{\bs\in \Eset} \sum_{a=0}^{n}\pi_{\bs}^{a} \frac{\partial}{\partial h_{\bs}^{a}}}}_{Z_{1+n}[\partial/\partial h_{\bs}^{a}]}
\prod_{\bs\in \Eset}
\lb 
\int  dy_{\bs}
\prod_{a=0}^nP_{\rm out}(y_\bs\lvert h^a_\bs)
\rb
\right|_{h=0},
\label{eq-1+n-system}
\eeqn
where we introduced
\be
&&
Z_{1+n}[\partial/\partial h]=
    \int \lb \prod_{\bs\in \Eset} \prod_{\mu=1}^M dF_{\bs,\mu} P_{F}(F_{\bs,\mu})
     \rb
     \no \\ &&
\times 
\int\prod_{a=0}^n
\lbb
\left(\prod_{i\in \Vset}\prod_{\mu=1}^M dx_{i\mu}^a\right)
\left(\prod_{i\in \Vset}\prod_{\mu=1}^{M}P^{_{a}}_{\rm pri.}(x_{i\mu}^a)\right)
\rbb
e^{\sum_{\bs\in \Eset} \sum_{a=0}^{n}\pi_{\bs}^{a} \frac{\partial}{\partial h_{\bs}^{a}}}.
\label{eq-z-1+n}
\ee
The last equation of \eq{eq-1+n-system} can be understood from the following relation:
\beq
P_{\rm out}^{}(y\lvert\pi)=
e^{\pi\frac{\partial}{\partial h}}
\left.
%W^{}(y+h)
P_{\rm out}^{}(y\lvert h)
\right|_{h=0}
\eeq
\blue{which can be proven formally doing power series expansion of the exponential.}

\subsubsection{Strategy to extract free-energy expressions}

In the limit $M \to \infty$ we wish to analyze the order parameters (\eq{eq-def-m} and \eq{eq-def-q}) which can detect the spontaneous breaking of the global symmetries
discussed in sec.~\ref{sec-symmetries}.
To this end we introduce overlaps between
replicas including the teacher $a=0$ and students $a=1,2,\ldots,n$:
\beqn
%m^a_i=\frac{1}{M}\sum_{\mu=1}^Mx^*_{i\mu}x^a_{i\mu} \\
Q_i^{ab}=\frac{1}{M}\sum_{\mu=1}^Mx^a_{i\mu}x^b_{i\mu}.
\label{eq-Qab}
\eeqn
%{As we noted before  $\pi_{\bs}$ defined in \eq{eq-pi-bs} (so that $\pi_{\bs}^{a}$ defined in \eq{eq-pi-bs-replica} ) is invariant under permutations of the index $\mu$ for the components in the case $F_{\bs,\mu}=1$ (deterministic case). Thus the replicated partition function \eq{eq-1+n-system} becomes invariant under permutations of the components done independently on different replicas $a=0,1,2,..,n$. With the order parameter \eq{eq-Qab} we are aniticipating thermodynamic phases where this symmetry is spontaneously broken. In the presence of random $F_{\bs,\mu}$ this symmetry is removed from the outset.}
In order to introduce the order parameters
%\eq{eq-Qab}
%and to enforce the normalization condition \eq{eq-Qaa}
in our analysis we use an identity
\begin{eqnarray}
  1 &=& \int\prod_{i\in \Vset}
  %\left(\prod_{a=1}^nM{\rm d}m^a_i\delta\Bigl(Mm^a_i-\sum_{\mu=1}^Mx^*_{i\mu}x^a_{i\mu}\Bigr)\right)
  \prod_{a < b}\left(M{\rm d}Q_i^{ab}\delta\Bigl(MQ_i^{ab}-\sum_{\mu=1}^Mx^a_{i\mu}x^b_{i\mu}\Bigr)\right) \nonumber \\
  &=& \int\prod_{i\in \Vset}
  %\left(\prod_{a=1}^nM{\rm d}m^a_i\delta\Bigl(Mm^a_i-\sum_{\mu=1}^Mx^*_{i\mu}x^a_{i\mu}\Bigr)\right)
\prod_{a < b}  \left(M{\rm d}Q_i^{ab}
\frac{d \epsilon^{ab}_{i}}{2\pi i}
\exp\left[\hat{\epsilon}_{i}\left(
\sum_{\mu=1}^Mx^a_{i\mu}x^b_{i\mu}-MQ_i^{ab}\right)
\right]\right)
    \label{eq:inserting_1}
\end{eqnarray}
in the following computation.
{
It is also convenient to introduce the diagonal elements of the overlap matrix $Q_{i}^{aa}$ ($a=0,1,2,\ldots,n$, $i=1,2,\ldots,N$) which reflect normalization of the variables $x^{a}_{i \mu}$.
In the case of the Ising prior (\eq{eq-prior-ising}), $Q_{i}^{aa}=1$ holds 
for any $M$. However, in the case of the Gaussian prior (\eq{eq-prior-gaussian}), $Q_{i}^{aa}=1$ becomes satisfied strictly only in the $M \to \infty$ limit due to the law of large numbers.
Thus we introduce $Q_{i}^{aa}$ as additional order parameter and consider, in addition to  \eq{eq:inserting_1},
\beqn
1=
\int \frac{d \epsilon^{aa}_{i}}{4\pi i} \exp\left[\frac{\epsilon^{aa}_{i}}{2}\left( \sum_{\mu=1}^Mx^a_{i\mu}x^a_{i\mu}-MQ_i^{aa}\right) \right].
    \label{eq:inserting_2}
\eeqn
The integral will be evaluated by the saddle point method in the $M \to \infty$ limit.
We will find that $\epsilon^{aa}_{i}$ remains arbitrary in the case of the Ising prior.
}

Using the identities  \eq{eq:inserting_1} and \eq{eq:inserting_2}
%{and enforcing of the
%normalization condition as \eq{eq-modify-prior}}
we can rewrite \eq{eq-z-1+n} as,
\beq
Z_{1+n}[ \partial/\partial h ]=\int \prod_{i\in \Vset} \prod_{a <  b}(M{\rm d}Q_i^{ab})
(M \prod_{a}dQ^{aa}_{i})
e^{- \hat{F}[\hat{Q}][ \partial/\partial h ]},
\label{eq-z-1+n-and-F}
\eeq
where
\beq
e^{- \hat{F}[\hat{Q}][ \partial/\partial h ]}=\int_{-i \infty}^{i \infty}\prod_{i\in \Vset} \prod_{a <  b} \frac{d \epsilon^{ab}_{i}}{2\pi i}
\prod_{a} \frac{d \epsilon^{aa}_{i}}{4\pi i}
e^{- \frac{M}{2} \sum_{i\in \Vset}\sum_{a,b}\epsilon^{ab}_{i}Q_{i}^{ab}}
e^{- \hat{G}[\hat{\epsilon}][ \partial/\partial h ]}.
\label{eq-hat-F}
\eeq
Here we introduced
\beq
- \hat{G}[\hat{\epsilon}][ \partial/\partial h ]=- G_{0}[\hat{\epsilon}]+
%\ln
%  \prod_{\bs\in \Eset}
%   \int dy_{\bs}
\ln  \left \langle
%  \prod_{\bs\in \Eset}
%\prod_{a=0}^{n}P_{\rm out}(y_\bs\lvert\pi^a_\bs)
e^{\sum_{\bs\in \Eset} \sum_{a=0}^{n}\pi_{\bs}^{a} \frac{\partial}{\partial h_{\bs}^{a}}}
\right\rangle_{\epsilon},
\label{eq-def-G}
\eeq
with
\beq
- G_{0}[\hat{\epsilon}]= - \sum_{i\in \Vset}\sum_{\mu =1}^{M} 
g_{0}[\epsilon_{i}]
,
\quad
 - g_{0}[\epsilon_i]=
\ln
\int \prod_{a=0}^{n} dx^{a} P^{_{a}}_{\rm pri.}(x^a) e^{\frac{1}{2}\sum_{a,b}\epsilon^{ab}_i x^{a}x^{b}},
\label{eq-def-G0}
\eeq
and
\be
&&
\hspace{-1cm}
\langle (\cdots) \rangle_{\epsilon}
=
Z^{-1}_{\epsilon}
 \int \lb \prod_{\bs\in \Eset} \prod_{\mu=1}^M dF_{\bs,\mu} P_{F}(F_{\bs,\mu}) \rb
\no \\ && 
\hspace{-1cm}
\times
\lb
\prod_{i\in \Vset}\prod_{\mu=1}^{M}
\lb 
\prod_{a=0}^{n} dx^{a}_{i\mu}\rb P_{\rm pri.}(x_{i\mu}^a) 
e^{\frac{1}{2}\sum_{a,b}\epsilon_{i}^{ab} x_{i\mu}^{a}x_{i\mu}^{b}} 
\rb
(\cdots),
\label{eq-averaging-with-epsilon}
\\ &&
\hspace{-1cm}
Z_{\epsilon}=
 \int \lb \prod_{\bs\in \Eset} \prod_{\mu=1}^M dF_{\bs,\mu} P_{F}(F_{\bs,\mu}) \rb
\lb
\prod_{i\in \Vset}\prod_{\mu=1}^{M}
\lb 
\prod_{a=0}^{n} dx^{a}_{i\mu}\rb P_{\rm pri.}(x_{i\mu}^a) 
e^{\frac{1}{2}\sum_{a,b}\epsilon_{i}^{ab} x_{i\mu}^{a}x_{i\mu}^{b}} 
\rb.
\ee
Here the subscript $0$ is meant to emphasize that the average is took with a 'non-interacting' system.
\blue{Note that $\pi_{\bs}$ which appears in the 2nd term of \eq{eq-def-G}
depends on $F_{\bs,\mu}$ and $x_{i\mu}^{a}$ (see \eq{eq-pi-bs-replica}) 
so that the average is not trivial.}

In the above expressions, $\hat{F}[\hat{Q}]$ and $\hat{G}[\hat{\epsilon}]$ are functionals of
the order parameters $Q_{i}^{ab}$ and the conjugated fields $\epsilon_{i}^{ab}(=\epsilon_{i}^{ba})$ for $i=1,2,\ldots,N$.
\blue{We will perform the integrals over the extensive number of variables by a saddle point method
exploiting the fact that $M \gg 1$. Such somewhat unusual saddle point integrals have been considered in
\cite{yoshino2018,yoshino2020complex,yoshino2023spatially}.
Although we do not have a fully rigorous justification of the method we believe the results we obtain is valid.
%Indeed we will find that the equation of states obtained by the present approach fully agree with the state evolution equations obtained
%by the message passing approach (see sec.~\ref{sec-comparison-BP-replica}).
For clarity we briefly discuss the method in appendix \ref{appendix-sp-indegration}. 
}

Then we find,
\beq
- \hat{F}[\hat{Q}]= - \sum_{i\in \Vset} \frac{M}{2}\sum_{a,b}(\epsilon_{i}^{ab})^{*}[\hat{Q}] \;Q_{i}^{ab}- \hat{G}[\hat{\epsilon}], \qquad
Q_{i}^{ab}=\left. \frac{\partial (- \hat{G}[\hat{\epsilon}])}{\partial \epsilon_{i}^{ab}} \right |_{\epsilon=\epsilon^{*}[\hat{Q}]}
\label{eq-F-G},
\eeq
which can be regarded as a Legendre transformation.

We will evaluate the 2nd term of \eq{eq-def-G} perturbatively in a way analogous to the Plefka expansion \cite{plefka1982convergence,GY91,georges1990low}.
Introducing a parameter $A$ as a book keeping device to organize an expansion, we will obtain 
\beq
\hat{G}=G_{0}+A \hat{G}_{1}+\frac{A^{2}}{2} \hat{G}_{2}+\cdots .
\label{eq-G-plefka-expansion}
\eeq
Later we will find that $A=\lambda^{2}$ is the proper choice as terms proportional to
odd powers of $\lambda$ vanish for a symmetry reason.
Then the functional $F[\hat{Q}]$ and the corresponding $\epsilon^{*}$ can be obtained through \eq{eq-F-G} as
\beq
\hat{F}=F_{0}+A \hat{F}_{1}+\frac{A^{2}}{2} \hat{F}_{2}+\cdots, \qquad
\epsilon=\epsilon_{0}+A \epsilon_{1}+\frac{A^{2}}{2} \epsilon_{2}+\cdots,
\label{eq-F-and-epsilon-plefka-expansion}
\eeq
which yields
\beqn
- F_{0}[\hat{Q}]&=&- G_{0}[\epsilon_{0}]-\frac{M}{2}\sum_{i}\sum_{a,b}(\epsilon_{0})_{i}^{ab}Q_{i}^{ab}, \nonumber \\
- \hat{F}_{1}[\hat{Q}]&=&- \hat{G}_{1}[\epsilon_{0}], \nonumber \\
- \hat{F}_{2}[\hat{Q}]&=& - \hat{G}_{2}[\epsilon_{0}]+\frac{(G_{0}^{'}[\epsilon_{0}])^{2}}{G_{0}^{''}[\epsilon_{0}]}  \qquad \cdots,
\label{eq-F-G-relations}
\eeqn
where $\epsilon_{0}$ is determined by
\beq
Q^{ab}_i=\left. \frac{\partial (- g_{0}[\epsilon_i])}{\partial \epsilon^{ab}_i} \right |_{\epsilon=\epsilon_{0}}.
\label{eq-epsilon0}
\eeq
Fortunately we will find that higher order terms $O(A^{2})$ vanish in the dense limit
$N \gg c \gg 1$ in our models.

Wrapping up the above formal results we can write using
\eq{eq-1+n-system}, \eq{eq-z-1+n}, \eq{eq-z-1+n-and-F}
\beqn
&&\mathbb{E}_{y}Z^n\lb \{y_\bs\} \rb
=\int \prod_{i\in \Vset} \prod_{a\leq b}M{\rm d}Q_i^{ab}
e^{- F[\hat{Q}]},
\eeqn
with $- F[\hat{Q}]$ defined as
\beq
- F[\hat{Q}]=- F_{0}[\hat{Q}]- F_{{\rm ex}}[\hat{Q}],
\label{eq-F-total-summary}
\eeq
which may be regarded as replicated free-energy functional. 
Here $- F_{0}[\hat{Q}]$ can be regarded as entropic part of the free-energy
while $- F_{{\rm ex}}[\hat{Q}]$
is the interaction part of the free-energy defined as
\beq
- F_{{\rm ex}}[\hat{Q}]
= \ln \left \{ \left. 
e^{- \hat{F}_{{\rm ex}}[\hat{Q}][\partial/\partial h]}
\prod_{\bs\in \Eset}\int  dy_{\bs}\prod_{a=0}^nP_{\rm out}(y_\bs\lvert h^a_\bs)
\right|_{\{h_{a}=0\}}
\right \},
\label{eq-F-ex}
\eeq
with
\beq
- \hat{F}_{{\rm ex}}[\hat{Q}][\partial/\partial h]=
- F_{1}[\hat{Q}][\partial/\partial h]-(A^{2}/2) \hat{F}_{2}[\hat{Q}][\partial/\partial h]+\cdots.
\eeq

Correspondingly, for the matrix $\epsilon^{ab}$ we assume the similar form as \eq{eq-simple-ansatz}. We will denote the field conjugated to $m$ as $\psi$ in the Ising prior case:
\beqn
\epsilon^{0a} &=& \epsilon^{a0} =\psi,  \qquad a \in \{ 1,2,\ldots,n \}.
\label{eq-simple-ansatz-phi}
\eeqn

\subsubsection{Cumulant expansion for the interaction part of the free-energy}
\label{sec-interaction-part-of-free-energy}

Now we consider more closely the interaction part of the free-energy $- F_{\rm ex}[\hat{Q}]$.
To this end we start from the cumulant expansion of the 2nd term of \eq{eq-def-G}, which
allows us to obtain the Plefka expansion of the functional $\hat{G}[\epsilon_{0}]$
(see \eq{eq-G-plefka-expansion}),
%$\ln
%  \left \langle
%e^{\sum_{\bs\in \Eset} \sum_{a=0}^{n}\pi_{\bs}^{a} \frac{\partial}{\partial h_{\bs}^{a}}}
%\right\rangle_{\epsilon}$.
%For instance, in the case of the additive Gaussian noise we have
%\beq
%\left. \frac{e^{-\frac{(y+h)^{2}}{2}}}{\sqrt{2\pi}} \right|_{h=0}
%\eeq
%\beq
%=\frac{e^{-\frac{(y-\pi)^{2}}{2}}}{\sqrt{2\pi}}
%\eeq
\beqn
A(-G_{1})[\epsilon_{0}]+\frac{A^{2}}{2}(-G_{2}) [\epsilon_{0}] \ldots &=&
\ln \langle  e^{-\lambda \sum_{\bs}\sum_{a}\pi^{a}_{\bs}\frac{\partial}{\partial h_{\bs}^{a}}}\rangle_{\epsilon} .
%\ln
%  \left \langle
%\exp\left[
%\prod_{\bs\in \Eset}
%\prod_{a=0}^{n}P_{\rm out}(y_{\bs}\lvert\pi_{\bs}^a)
%\right]
%\right\rangle_{\epsilon} \nonumber \\
%&=&
%\prod_{\bs}
%\prod_{a=0}^{n}
%\left.
%\frac{e^{-\frac{(y_{\bs}+h_{\bs}^{a})^{2}}{2}}}{\sqrt{2\pi}}
%W^{}(y_{\bs}+h_{\bs}^{a})
%P_{\rm out}^{}(y_{\bs}\lvert h_{\bs}^{a})
%\right|_{\{h_{\bs}^{a}=0\}} \qquad
\label{eq-cumulant-expansion-interaction-0}
\eeqn
Now let us evaluate the last expression
by a cumulant expansion,
%$\langle  e^{-\sum_{\bs}\sum_{a}\pi^{a}_{\bs}\frac{\partial}{\partial h_{\bs}^{a}}}\rangle_{\epsilon}$ 
\beqn
&& \ln \langle  e^{-\lambda\sum_{\bs}\sum_{a}\pi^{a}_{\bs}\frac{\partial}{\partial h_{\bs}^{a}}}\rangle_{\epsilon}
 = -\lambda\sum_{\bs}\sum_{a}
\langle
\pi^{a}_{\bs}
\rangle_{\epsilon}
\frac{\partial}{\partial h_{\bs}^{a}} \nonumber \\
&+&\frac{\lambda^{2}}{2}
\sum_{\bs_{1},\bs_{2}}\sum_{a,b}
\left[
\langle
\pi^{a}_{\bs_{1}}
\pi^{b}_{\bs_{2}}
\rangle_{\epsilon}
-
\langle
\pi^{a}_{\bs_{1}}
\rangle_{\epsilon}
\langle
\pi^{b}_{\bs_{2}}
\rangle_{\epsilon}
\right]
\frac{\partial^{2}}{\partial h_{\bs_{1}}^{a}\partial h_{\bs_{2}}^{b}} 
+  O(\lambda^{3}) \ldots
\label{eq-cumulant-expansion-interaction}
\eeqn
where we introduced a parameter $\lambda$ to organize a perturbation series.
We find below that $A=\lambda^{2}$ is the proper choice.
%As we see below we will find that $O(\lambda)$ term is zero for a symmetry reason (In the average $\langle \ldots \rangle_{\epsilon}$
%defined in \eq{eq-averaging-with-epsilon} the weight is symmetric with respect to the change of sign of $x^{a}_{i\mu}$ of all replicas $a=0,1,2,.,n$)
%while $O(\lambda^{2})$ term is not zero. 
%so that it is natural to choose $A=\lambda^{2}$ in the Plefka expansion 
%\eq{eq-G-plefka-expansion} and \eq{eq-F-and-epsilon-plefka-expansion}.
After extracting $G_{1},G_{2},\ldots$ through \eq{eq-cumulant-expansion-interaction-0}
we will consider the Legendre transformation \eq{eq-F-G} order by order in the Plefka expansion.

It is convenient to represent various terms generated by the cumulant expansion \eq{eq-cumulant-expansion-interaction} by diagrams as those depicted
in Fig.~\ref{fig:diagrams}. The details of the analysis is presented
in appendix~\ref{appendix-cumulant-expansion}. 
The key observations in the analysis is the following.
\begin{itemize}
\item Disconnected diagrams do not appear in the cumulant expansion meaning that correlation functions
which appear in $\hat{G}$ are {\it connected correlation functions}.
On each variable node we should put a connected correlation function like
$\langle x^{a}_{i\mu}x^{b}_{i\mu}\rangle^{c}_{\epsilon}=\langle x^{a}_{i\mu}x^{b}_{i\mu}\rangle_{\epsilon}-\langle x^{a}_{i\mu}\rangle_{\epsilon}\langle x^{b}_{i\mu}\rangle_{\epsilon}=Q_{i}^{ab}$,
$\langle x^{a}_{i\mu}x^{b}_{i\mu}x^{c}_{i\mu}x^{d}_{i\mu}\rangle^{c}_{\epsilon}$, ... Note also that odd cumulants
like $\langle x^{a}_{i\mu}\rangle^{c}_{\epsilon}$,
$\langle x^{a}_{i\mu}x^{b}_{i\mu}x^{c}_{i\mu} \rangle^{c}_{\epsilon}$,...
are zero due to the reflection symmetry mentioned in sec.~\ref{sec-symmetries}:
$P^{_{a}}_{\rm pri.}(x)=P^{_{a}}_{\rm pri.}(-x)$ which holds both in the Ising and Gaussian priors.
This is the reason why $A=\lambda^{2}$ becomes the proper choice.

\item One particle reducible diagrams, namely
diagrams which become disconnected into two parts by cutting
connections between two adjacent factor nodes
(as those shown in Fig.~\ref{fig:diagrams} b))
disappear by the Legendre transform $\hat{G} \to \hat{F}$ \eq{eq-F-G}
so that only {\it one particle irreducible (1PI) diagrams} contribute to $\hat{F}$ 
\cite{hansen1990theory,zinn2021quantum}.
They can be classified into two types.
\begin{itemize}
\item[(A)] Diagrams which consist of one factor node
as represented by those shown in Fig.~\ref{fig:diagrams} a)
replicated into $2,4,6,8...$ replicas. Among these, we will find that only $2$ replica case is relevant and all others vanish in the dense limit $N \gg c \gg 1$.

\item[(B)]  Diagrams which consist of closed loops
  such as those depicted in Fig.~\ref{fig:diagrams} c).
  We will find that they vanish in the dense limit $N \gg c \gg 1$.
However it is not necessarily the case  in the case of global coupling $c \propto N^{p-1}$.
\end{itemize}
\end{itemize}

Remarkably, as we noted above, we find that the contributions
from higher order cumulants $O(A^{2})$,$O(A^{4})$... vanish
in the dense limit $N \gg c \gg 1$.
\blue{More precisely, as we explain the details in the appendix sec \ref{appendix-cumulant-expansion},
$N$ must grow faster than any polynomial of $c$.}
As the result we find (see \eq{eq-F1-appendix}),
\beq
 \hat{F}_{\rm ex}[\hat{Q}][\partial/\partial h_{\bs}^{a}]=- F_{1}[\hat{Q}][\partial/\partial h_{\bs}^{a}]
=
\frac{1}{2} \sum_{\bs} \sum_{a,b=0}^{n}
\prod_{j \in \partial\bs}Q_{j}^{ab} \frac{\partial^{2}}{\partial h_{\bs}^{a}\partial h_{\bs}^{b}}.
\label{eq-Fex}
\eeq
Using this in \eq{eq-F-ex} we obtain
\beq
- F_{{\rm ex}}[\hat{Q}]
= \ln \left \{
e^{\frac{1}{2} \sum_{\bs} \sum_{a,b=0}^{n}
\prod_{j \in \partial\bs}Q_{j}^{ab} \frac{\partial^{2}}{\partial h_{\bs}^{a}\partial h_{\bs}^{b}}}
\prod_{\bs\in \Eset}\int  dy_{\bs}\prod_{a=0}^nP_{\rm out}(y_\bs\lvert h^a_\bs)
\right \}.
\label{eq-f-ex}
\eeq

{The vanishment of the higher order contributions is reminiscent of
  what one encounters in
  the standard procedure to prove the central limit theorem.
Indeed, the above result can be reproduced assuming a Gaussian ansatz
as commonly employed in the literature \cite{sakata2013statistical,kabashima2016phase},
in which one explicitly
assumes that fluctuation of $\pi$ within the canonical ensemble
follows a Gaussian distribution with zero mean and the variance parametrized by $Q_{i}^{ab}$.
Similar situation happens in the structural glass transition of
supercooled liquids in the large dimensional limit \cite{kurchan2012exact},
as well as in the inverse problem of the current problem \cite{yoshino2018}.
However, let us emphasize that the vanishment of
  the higher order contributions does {\it not} hold generally,
  including in the case $c \propto N^{p-1}$. For instance,
  the matrix factorization problem ($p=2$) suffers this issue
  when one considers full rank case $M=O(N)$.
  This point has been noticed
  also in the analysis of the TAP equation for the matrix factorization problem
  \cite{maillard2022perturbative}
  and the problem of  machine learning by multi-layer perceptrons\cite{yoshino2023spatially}.}

\blue{In the present paper we are considering regular random graphs in which the connectivity $c$ is
fixed. Then the term contributing the interaction part of the free-energy functional $F_{\rm ex}[\hat{Q}]$
\eq{eq-f-ex} is completely regular across the system. Thus it does not depend on specific realizations
of the random graph. We expect introduction of some  mild fluctuations of the connectivity would not change
the free-energy density $F_{\rm ex}[\hat{Q}]/(NM)$ in the dense limit $N \gg M \gg 1$.
}

\begin{figure}[h]
\centering
\includegraphics[width=0.9\textwidth]{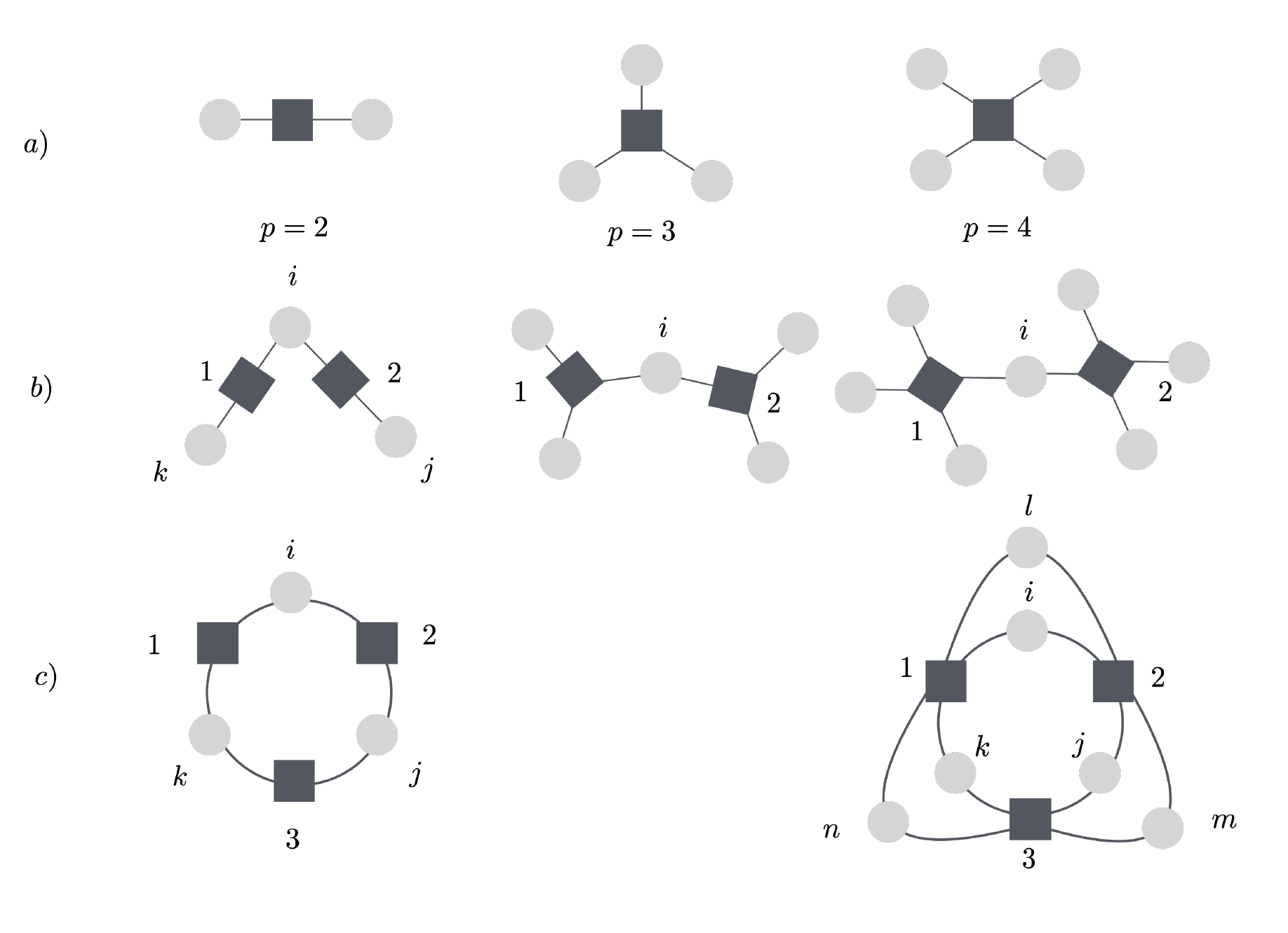}
\caption{Some representative diagrams which appear in the cumulant expansion.}
\label{fig:diagrams}
\end{figure}

%%%%%%%%%%%%%%%%%%%%%%%%%%%%%%%%%%%%%%%%%%%%%%%%%%%%%

%%%%%%%%%%%%%%%%%%%%%%%%%%%%%%%%%%%%%%%%%%%%%%%%%%%%%
%%%%%%%%%%%%%%%%%%%%%%%%%%%%%%%%%%%%%%%%%%%%%%%%%%%%%

\subsubsection{Ansats for the replica overlaps}

In the present paper we will consider spatially homogeneous solution:
\beqn
Q^{ab}_{i}=Q^{ab},\qquad \epsilon_i^{ab}=\epsilon^{ab}.
\label{eq-homogeneous-ansatz}
\eeqn
This is justified since there are no inhomogeneities in the system (e.g. boundary), and thus the dependence on the variable index $i$ of these quantities is hereafter removed. The simplest ansatz for the matrix $Q$ is the replica symmetric one:
\beqn
Q^{ab}_{i}=(Q-q)\delta_{ab}+q,  \qquad a,b \in \{0,1,2,\ldots,n\},
\label{eq-RS-ansatz}
\eeqn
which is justified since we will focus on Bayes optimal settings in the present paper.
{Here $Q$ represents the normalization of the variables $x_{i \mu}^{a}$ for all replicas $a=0,1,2\ldots,n$.}
\blue{However, to facilitate consistency checks and 
also to help future works to explore cases beyond the idealized Bayes optimal case, }
we will consider a slightly more general ansatz:
\beqn
Q^{ab} &=& (Q-q)\delta_{ab}+q,  \qquad a,b \in \{1,2,\ldots,n\}, \nonumber \\
Q^{0a} &=& Q^{a0} =m,  \qquad a \in \{1,2,\ldots,n \}, \nonumber \\
Q^{00} &=& Q_{0}.
\label{eq-simple-ansatz}
\eeqn
{Here $Q$ and $Q_{0}$  parameterize normalization of the
variables $x_{i \mu}^{a}$ for the students $a=1,2\ldots,n$ and the teacher $a=0$ respectively.}

The parameters $m$ and $q$ are precisely the order parameters
introduced in \eq{eq-def-m} and \eq{eq-def-q}.
Note that  \eq{eq-RS-ansatz} is recovered assuming $Q=Q_{0}$ and $m=q$ with
$m$ being the overlaps between the teacher ($0$-th replica) and students and $q$ being the overlaps between the students. In Bayes optimal inferences $m=q$  generally holds much as on the Nishimori line in spinglass models \cite{iba1999nishimori,zdeborova2016statistical,nishimori2001statistical}.

\subsection{Entropic part of the free-energy}
\label{sec-replica-entropic-part-of-free-energy}

Now let us evaluate explicitly the entropic part of the free-energy $F_{0}[\hat{Q}]$ for the two specific priors, Ising and Gaussian prior.
To this end we have to evaluate $g_{0}[\hat{\epsilon}]$ (see \eq{eq-def-G0}).

\subsubsection{Ising prior}
%\textcolor{red}
For the Ising model we have the prior distribution \eq{eq-prior-ising} which yields
\beqn
- g_{0}\lsb \hat{\epsilon} \rsb
&=& \ln \prod_{a=0}^{n} \sum_{x^{a}=\pm 1} e^{\frac{1}{2}\sum_{a, b=0}^{n} \epsilon^{ab}x^{a}x^{b}}  \nonumber \\
&=& \ln \prod_{a=0}^{n} \sum_{x^{a}=\pm 1}
\left.
e^{\frac{1}{2}\sum_{a,b=1}^{n} \epsilon^{ab}\frac{\partial^{2}}{\partial h^{a}\partial h^{b}}
+\frac{1}{2}\epsilon^{00}+x^{0}\sum_{a=1}^{n} \epsilon^{0a}\frac{\partial}{\partial h^{a}}}
e^{\sum_{a=1}^{n}h^{a}x^{a}} \right |_{h=0} \nonumber \\
%&=& \ln \prod_{a=0}^{n} \sum_{x^{a}=\pm 1}
%\left.
%e^{\frac{1}{2}\sum_{a,b=1}^{n} \epsilon^{ab}\frac{\partial^{2}}{\partial h^{a}\partial h^{b}}}
%e^{\epsilon^{00}+\sum_{a=1}^{n} \epsilon^{a0}\frac{\partial}{\partial h^{a}}}
%e^{\sum_{a=1}^{n}h^{a}x^{a}} \right |_{h=0} \nonumber \\
&=& \frac{\epsilon^{00}}{2}+\ln 
\left. e^{\frac{1}{2}\sum_{a,b=1}^{n} \epsilon^{ab}\frac{\partial^{2}}{\partial h^{a}\partial h^{b}}} \prod_{a=1}^{n}
(2\cosh(h^{a}+\epsilon^{0a})) \right |_{h=0}.
\eeqn
To derive the last equation we used 
lemma~\ref{lemma:uno} and $\cosh(h)=\cosh(-h)$.

Then from \eq{eq-epsilon0} we find
\beqn
Q^{ab}&=&\frac{
  \left. e^{\frac{1}{2}\sum_{a,b=1}^{n} (\epsilon^{*}_{0})^{ab}\frac{\partial^{2}}{\partial h^{a}\partial h^{b}}} \prod_{c=1}^{n}
(2\cosh(h^{c}+(\epsilon^{*}_{0})^{0c})) \tanh(h^{a}+(\epsilon^{*}_{0})^{0a})\tanh(h^{b}+(\epsilon^{*}_{0})^{0b})\right |_{h=0}
}{
  \left. e^{\frac{1}{2}\sum_{a,b=1}^{n} (\epsilon^{*}_{0})^{ab}\frac{\partial^{2}}{\partial h^{a}\partial h^{b}}} \prod_{c=1}^{n}
(2\cosh(h^{c}+(\epsilon_{0}^{*})^{0c})) \right |_{h=0}
} \nonumber \\
&& \hfill (a \neq b, a \geq 1, b \geq 1), \nonumber \\
Q^{0a}&=&Q^{a0}=\frac{
  \left. e^{\frac{1}{2}\sum_{a,b=1}^{n} (\epsilon^{*}_{0})^{ab}\frac{\partial^{2}}{\partial h^{a}\partial h^{b}}} \prod_{c=1}^{n}
(2\cosh(h^{c}+(\epsilon^{*}_{0})^{0c})) \tanh(h^{a}+(\epsilon^{*}_{0})^{0a})\right |_{h=0}
}{
  \left. e^{\frac{1}{2}\sum_{a,b=1}^{n} (\epsilon^{*}_{0})^{ab}\frac{\partial^{2}}{\partial h^{a}\partial h^{b}}} \prod_{c=1}^{n}
(2\cosh(h^{c}+(\epsilon^{*}_{0})^{0c})) \right |_{h=0} 
}  \qquad (a \geq 1), \nonumber \\
Q^{aa}&=&1 \qquad (a=0,1,2,\ldots,n),
\label{eq-Qab-ising-2}
\eeqn
which determines $(\epsilon^{*}_{0})^{ab}$
(see \eq{eq-epsilon0}).
To derive the last equation we used $e^{\frac{d^{2}}{dh^{2}}}\cosh(h)=\cosh(h)$.

Note that $Q^{aa}=1$ is satisfied with $\epsilon^{aa}$ remaining arbitrary.
Then the 1st equation of \eq{eq-F-G-relations} yields
\beq
\frac{- F_{0}[\hat{Q}]}{M}=
-\frac{1}{2}\sum_{i}\sum_{a,b=0}^{n}(\epsilon^{*}_{0})_{i}^{ab}Q_{i}^{ab}+\sum_{i}
  \ln e^{\frac{1}{2}\sum_{a,b=1}^{n}(\epsilon^{*}_{0})_{i}^{ab}\frac{\partial^{2}}{\partial h_{a}\partial h_{b}}}
  \prod_{c=1}^{n}  \left.  2\cosh(h_{c}+(\epsilon^{*}_{0})^{0a}) \right|_{h=0}.
  \eeq

  We now assume the replica symmetric ansatz \eq{eq-simple-ansatz}. For the matrix $\epsilon^{ab}$ we assume the similar form
  with \eq{eq-simple-ansatz-phi},
  \beqn
\frac{- F_{0}(q,m)}{NM}&=&
-\frac{1}{2}\left[  (1+n)\epsilon^{*}_{0} + (n^{2}-n)\epsilon_{0}^{*} q +2n \psi^{*}_{0} m \right]
+\frac{\epsilon^{*}_{0}}{2}+
  \ln e^{\frac{\epsilon^{*}_{0}}{2}\frac{\partial^{2}}{\partial h^{2}}}
  \left.  (2\cosh(h+\psi_{0}^{*})^{n} \right|_{h=0} \nonumber \\
  &=&
  -\frac{1}{2}\left[  n\epsilon^{*}_{0} + (n^{2}-n)\epsilon_{0}^{*} q +2n \psi^{*}_{0} m \right]
+
\ln \int \mathcal{D}z
(2\cosh(\psi_{0}^{*}+\sqrt{\epsilon^{*}_{0}}z)^{n}.
  \label{eq-f0-ising-RS}
    \eeqn
    Note that we assumed the replica symmetric form $(\epsilon^*_{0})^{ab}=\epsilon^*_0~(a,b>1)$ and for convenience the diagonal components were also set to be equal to the non-diagonal ones $((\epsilon^*_{0})^{aa}=\epsilon^*_{0})$ using the arbitrariness.
    $\epsilon_{0}^{*}$ and $\psi_{0}^{*}$ are to be determined
    by \eq{eq-epsilon0}.
    Here we
    introduced a short-hand notation,
    \beq
    \int\mathcal{D}z\ldots=\int_{-\infty}^{\infty}\frac{dz}{\sqrt{2\pi}}
    e^{-\frac{z^{2}}{2}}\ldots,
    \label{eq-def-Dz}
    \eeq
    and used lemma~\ref{lemma:due}.
    We find in $n \to 0$ limit, within the ansatz \eq{eq-simple-ansatz-phi},
    \beqn
    q=\int \mathcal{D}z  \tanh^{2}(\sqrt{\epsilon_{0}^{*}}z+\psi^{*}_{0}),
%    \frac{e^{\frac{\epsilon^{*}_{0}}{2}\frac{\partial^{2}}{\partial h^{2}} }
%    \left.  (2\cosh(h+\psi^{*}_{0}))^{n} \tanh^{2}(h+\psi^{*}_{0})\right|_{h=0}}{e^{\frac{\epsilon^{*}_{0}}{2}\frac{\partial^{2}}{\partial h^{2}}}
    %      \left.  (2\cosh(h+\psi^{*}_{0}))^{n} \right|_{h=0}}
    \nonumber \\
    m= \int \mathcal{D}z  \tanh(\sqrt{\epsilon_{0}^{*}}z+\psi^{*}_{0}).
%    \frac{e^{\frac{\epsilon^{*}_{0}}{2}\frac{\partial^{2}}{\partial h^{2}} }
%    \left.  (2\cosh(h+\psi_{0}^{*}))^{n} \tanh(h+\psi^{*}_{0})\right|_{h=0}}{e^{\frac{\epsilon^{*}_{0}}{2}\frac{\partial^{2}}{\partial h^{2}}}
    %      \left.  (2\cosh(h+\psi^{*}_{0}))^{n} \right|_{h=0}}
  \label{eq-f0-ising-RS-derivatives}    
    \eeqn

%%%%%%%%%%%%%%%%%%%%%%%%%%%%%%%%%%%%%%%%%%%%%%%%%%%%%
\subsubsection{Gaussian prior}

For the Gaussian model we have
the prior distribution \eq{eq-prior-gaussian}
which yields
\beqn
- g_{0}[\hat{\epsilon}]&=&\ln \int \prod_{a=0}^{n} dx^{a} \frac{e^{-\frac{1}{2}x_{a}^{2}}}{\sqrt{2\pi}}
e^{\frac{1}{2}\sum_{a, b=0}^{n} \epsilon^{ab}x^{a}x^{b}} \nonumber \\
&=&\ln \int \prod_{a=0}^{n} \frac{dx^{a}}{\sqrt{2\pi}} 
e^{ \frac{1}{2}\sum_{a,b=0}^{n}(\epsilon^{ab}-\delta_{ab})x^{a}x^{b}}=-\frac{1}{2}\ln {\rm det} (-(\hat{\epsilon}-\hat{I})).
\eeqn
Then from \eq{eq-epsilon0} we find
\beq
\hat{Q}=(-(\hat{\epsilon}^{*}_{0}-\hat{I}))^{-1},
\eeq
where $\hat{I}$ is the identity matrix.
Finally the  1st equation of \eq{eq-F-G-relations} yields,
\beq
\frac{- F_{0}[\hat{Q}]}{M}=
-\frac{1}{2}\sum_{i}\sum_{a=0}^{n}Q_{i}^{aa}
+\frac{n N}{2}\ln (2\pi)
+\frac{1}{2}\sum_{i} \ln {\rm det} \hat{Q}_{i}.
\eeq

In the homogeneous and replica symmetric ansatz \eq{eq-RS-ansatz}, we find
  using lemma~\ref{lemma:determinant-RS},
  \beq
%  \frac{  - F_{0}(q)}{NM}
  \frac{  - F_{0}(Q,q)}{NM}  
  =\frac{1}{2}\left[
%  \ln (1+nq)+ n \ln (1-q)
    \ln (Q+nq)+ n \ln (Q-q)  
    \right]-\frac{1+n}{2}Q.
  \eeq
  In the slightly generalized ansatz \eq{eq-simple-ansatz} we find
  using lemma~\ref{lemma:determinant-RS}
and  lemma~\ref{lemma:determinant},
  \beqn
%  \frac{  - F_{0}(q,m)}{NM}
  \frac{  - F_{0}(Q_{0},Q,q,m)}{NM}  
  &=&\frac{1}{2}\ln Q_{0}+
  \frac{1}{2}\ln {\rm det}((q-m^{2}/Q_{0})+(Q-q)\delta_{ab}) \nonumber \\
 & =&\frac{1}{2}\ln Q_{0}+\frac{1}{2}\left[
  %    \ln (1-q+n(q-m^{2}))+ (n-1) \ln (1-q)
    \ln (Q-q+n(q-m^{2}/Q_{0}))+ (n-1) \ln (Q-q)  
    \right]\nonumber \\
&&  -\frac{1}{2}Q_{0}-\frac{n}{2}Q.
  \label{eq-f0-gaussian-RS}
  \eeqn

      We will need derivatives of $\partial_{n} F_{0}(q)|_{n=0}$ with respect to $Q$,$q$ and $m$
      which we find as,
      \beqn
      \partial_{n} \frac{\partial }{\partial Q}  \frac{ - F_{0}(Q,q,m)}{NM} |_{n=0} &=&
      \frac{1}{2}\frac{Q-2q+m^{2}/Q_{0}}{(Q-q)^{2}}-\frac{1}{2},\nonumber \\      
      \partial_{n} \frac{\partial }{\partial q}  \frac{ - F_{0}(Q,q,m)}{NM} |_{n=0} &=&
      %      \frac{1}{2}\frac{q-m^{2}}{(1-q)^{2}}\nonumber \\
      \frac{1}{2}\frac{q-m^{2}/Q_{0}}{(Q-q)^{2}},\nonumber \\      
      \partial_{n} \frac{\partial }{\partial m}  \frac{ - F_{0}(Q,q,m)}{NM} |_{n=0} &=&
      %      -\frac{m}{1-q}
            -\frac{m/Q_{0}}{Q-q}.
  \label{eq-f0-gaussian-RS-derivatives}      
  \eeqn
It can be noticed that in the Bayes optimal case for which $q=m$ holds, 
  $\partial_{n} \frac{\partial }{\partial Q}  \frac{ - F_{0}(Q,q,m)}{NM} |_{n=0}=0$ provided also that $Q=Q_{0}=1$ holds.
  
%%%%%%%%%%%%%%%%%%%%%%%%%%%%%%%%%%%%%%%%%%%%%%%%%%%%%
%%%%%%%%%%%%%%%%%%%%%%%%%%%%%%%%%%%%%%%%%%%%%%%%%%%%%

\subsection{Interaction part of the free-energy}
\label{sec-replica-interactionc-part-of-free-energy}

Let now turn to explicitly evaluate the interaction part of the free-energy $- F_{\rm ex}[\hat{Q}]$
for the output functions of the two specific cases, additive noise and sign output.

\subsubsection{Additive noise}
\label{sec-additive-noise-replica}

In the case of the additive noise \eq{eq-p-put-additive-noise} we have
\beq
P_{\rm out}^{}(y\lvert\pi)=W^{}(y-\pi).
\eeq
Using this in \eq{eq-F-ex} we find
\beqn
&&- F_{\rm ex}[\hat{Q}]=\sum_{\bs} \ln \int  dy_{\bs}
\exp \left[\frac{\lambda^{2}}{2} \sum_{a,b=0}^{n}
\prod_{j \in \partial\bs}Q_{j}^{ab} \frac{\partial^{2}}{\partial h_{\bs}^{a}\partial h_{\bs}^{b}}
  \right]
\prod_{a=0}^{n}
\left.
%\frac{e^{-\frac{(y_{\bs}+h_{\bs}^{a})^{2}}{2}}}{\sqrt{2\pi}}
W^{}(y_{\bs}+h_{\bs}^{a})
\right|_{\{h_{\bs}^{a}=0\}} \label{eq-F1-additive} \\
&&=\sum_{\bs} \ln \int  dy_{\bs}
% \frac{e^{-\frac{y_{\bs}^{2}}{2}}}{\sqrt{2\pi}}
    W(y_{\bs})
\exp \left[\frac{\lambda^{2}}{2} \sum_{a,b=1}^{n} \left(
  \prod_{j \in \partial\bs}Q_{j}^{00}
  -\prod_{j \in \partial\bs}Q_{j}^{0b}
  -\prod_{j \in \partial\bs}Q_{j}^{a0}
+ \prod_{j \in \partial\bs}Q_{j}^{ab}
  \right)
  \frac{\partial^{2}}{\partial h_{\bs}^{a}\partial h_{\bs}^{b}}   \right]  \nonumber \\
&& \hspace*{2cm} \prod_{a=1}^{n}
\left.
%\frac{e^{-\frac{(y_{\bs}+h_{\bs}^{a})^{2}}{2}}}{\sqrt{2\pi}}
W^{}(y_{\bs}+h_{\bs}^{a})
\right|_{\{h_{\bs}^{a}=0\}}.
\label{eq-F1-additive-noise}
\eeqn
To derive the 2nd equation we performed integrations by parts.

In the homogeneous and replica symmetric ansatz \eq{eq-RS-ansatz} we find from the 1st equation
of \eq{eq-F1-additive-noise},
  \beq
\frac{  - F_{\rm ex}(Q,q)}{NM}= \frac{\alpha}{p}
  \ln \int dy
 %   W(y)  
  e^{\frac{\lambda^{2}}{2}q^{p}\frac{\partial^{2}}{\partial h^{2}}}
  \left.  
\left(
  e^{\frac{\lambda^{2}}{2}(Q^{p}-q^{p})\frac{\partial^{2}}{\partial h^{2}}}
  %  \frac{e^{-\frac{(y_{\bs}+h)^{2}}{2}}}{\sqrt{2\pi}}
  W^{}(y+h)
  \right)^{1+n} \right|_{h=0},
  \eeq
  where we used $N_{\bs}=NM(\alpha/p)$ (see \eq{eq-Nbs} and \eq{eq-gamma}).
  In the slightly generalized ansatz \eq{eq-simple-ansatz} we find from the 2nd
  equation of \eq{eq-F1-additive-noise},
  \beqn
  &&
\frac{  - F_{\rm ex}(Q_{0},Q,q,m)}{NM} = \frac{\alpha}{p}
  \ln \int dy
    W(y)  
    e^{\frac{\lambda^{2}}{2}(Q^{p}_{0}-2m^{p}+q^{p})\frac{\partial^{2}}{\partial h^{2}}}
  \left.  
\left(
  e^{\frac{\lambda^{2}}{2}(Q^{q}-q^{p})\frac{\partial^{2}}{\partial h^{2}}}
    W^{}(y+h)
    \right)^{n} \right|_{h=0} \nonumber \\
    &=& \frac{\alpha}{p}
  \ln \int dy
  W(y)
 \int \mathcal{D}z_{0} 
%    e^{\frac{\lambda^{2}}{2}(1-2m^{p}+q^{p})\frac{\partial^{2}}{\partial h^{2}}}
  \left.  
  \left(
 \int \mathcal{D}z_{1}   
%  e^{\frac{\lambda^{2}}{2}(1-q^{p})\frac{\partial^{2}}{\partial h^{2}}}
    W^{}(y+\sqrt{Q_{0}^{p}-2m^{p}+q^{p}}z_{0}+\sqrt{Q^{p}-q^{p}}z_{1})
    \right)^{n} \right|_{h=0},
    \label{eq-f1-additive-noise-RS}
    \eeqn
    where we used the short-hand notation \eq{eq-def-Dz} and lemma~\ref{lemma:due}.

  We will need derivatives of $\partial_{n} F_{\rm ex}(Q_{0},Q,q,m)|_{n=0}$ with respect to $Q$,$q$ and $m$
  which we find as
  \beqn
    \partial_{n} \frac{\partial }{\partial Q}  \frac{ - F_{\rm ex}(Q_{0},Q,q,m)}{NM} |_{n=0}
&  =&\frac{\lambda^{2}}{2}\alpha Q^{p-1}\theta_{1}(\lambda,Q_{0},Q,m,q), \nonumber \\
  \partial_{n} \frac{\partial }{\partial q}  \frac{ - F_{\rm ex}(Q_{0},Q,q,m)}{NM} |_{n=0}
&  =&-\frac{\lambda^{2}}{2}\alpha q^{p-1}\theta_{0}(\lambda,Q_{0},Q,m,q), \nonumber \\
  \partial_{n}   \frac{\partial }{\partial m}  \frac{ - F_{\rm ex}(Q_{0},Q,q,m)}{NM} |_{n=0}&=&
  \lambda^{2}\alpha m^{p-1}
  \left(
  \theta_{0}(\lambda,m,q)-\theta_{1}(\lambda,Q_{0},Q,m,q)
  \right),
      \label{eq-f1-additive-noise-RS-derivatives}
  \eeqn
  with
  \beqn
  \theta_{0}(\lambda,Q_{0},Q,m,q)&=&\int dy
    W(y)  
    e^{\frac{\lambda^{2}}{2}(Q_{0}^{p}-2m^{p}+q^{p})\frac{\partial^{2}}{\partial h^{2}}}
    \left.
  \left(    
  \frac{    
  e^{\frac{\lambda^{2}}{2}(Q^{p}-q^{p})\frac{\partial^{2}}{\partial h^{2}}}
    \frac{d}{dh}W^{}(y+h)
    }{
%\left(
  e^{\frac{\lambda^{2}}{2}(Q^{p}-q^{p})\frac{\partial^{2}}{\partial h^{2}}}
    W^{}(y+h)
%    \right)    
  }
  \right)^{2}
  \right|_{h=0}
  \nonumber \\
  &=& \int dy W(y) \int \mathcal{D}z_{0}
  \left(
  \frac{ \int \mathcal{D}z_{1} (d/d\Xi)W^{}(\Xi)}{\int \mathcal{D}z_{1}
    W^{}(\Xi)}
  \right)^{2}, \nonumber \\
%  \eeqn
%  and similarly,
%    \beqn
  \theta_{1}(\lambda,Q_{0},Q,m,q)&=&\int dy
    W(y)  
    e^{\frac{\lambda^{2}}{2}(Q_{0}^{p}-2m^{p}+q^{p})\frac{\partial^{2}}{\partial h^{2}}}
    \left.
%  \left(    
  \frac{    
  e^{\frac{\lambda^{2}}{2}(Q^{p}-q^{p})\frac{\partial^{2}}{\partial h^{2}}}
    \frac{d^{2}}{dh^{2}}W^{}(y+h)
    }{
%\left(
  e^{\frac{\lambda^{2}}{2}(Q^{p}-q^{p})\frac{\partial^{2}}{\partial h^{2}}}
    W^{}(y+h)
%    \right)    
  }
%  \right)
  \right|_{h=0}
  \nonumber \\
  &=& \int dy W(y) \int \mathcal{D}z_{0}
 % \left(
  \frac{ \int \mathcal{D}z_{1} (d^{2}/d\Xi^{2})W^{}(\Xi)}{\int \mathcal{D}z_{1}
    W^{}(\Xi)},
  %  \right)
  \label{eq-theta0-theta1}
  \eeqn
and
  \beq
  \Xi=y-\lambda\sqrt{Q_{0}^{p}-2m^{p}+q^{p}}z_{0}-\lambda\sqrt{Q^{p}-q^{p}}z_{1}.
  \label{eq-Xi}
  \eeq
In the Bayes optimal case, for which $m=q$ holds,
one finds $\theta_1=0$ as shown in Lemma \ref{lemma:quattro}
provided also that $Q_{0}=Q$. The latter certainly holds for the Ising prior.

More specifically in the case of Gaussian noise
  \beq
  W(w)=\frac{e^{-\frac{w^{2}}{2}}}{\sqrt{2\pi}},
  \label{eq-gaussian-noise-replica}
  \eeq
  we find
  \beqn
  \theta_{0}(\lambda,Q_{0},Q,m,q)&=&
  \frac{\lambda^{2}(Q_{0}^{p}-2m^{p}+q^{p})+1}{
    (\lambda^{2}(Q^{p}-q^{p})+1)^{2}}, \nonumber \\
    \theta_{1}(\lambda,Q_{0},Q,m,q) & = & \theta_{0}(\lambda,Q_{0},Q,m,q)
    -\frac{1}{\lambda^{2}(Q^{p}-q^{p})+1}.
    \label{eq-theta0-theta1-gaussian-noise-replica}
    \eeqn
In the Bayes-optimal case, for which $m=q$ must hold, we find again that $\theta_{1}=0$ holds
provided $Q_{0}=Q$ also holds. 

%%%%%%%%%%%%%%%%%%%%%%%%%%%%%%%%%%%%%%%%%%%%%%%%%%%%%
\subsubsection{Sign output}

In the case of the sign output we have
\beq
P_{\rm out}^{}(y\lvert\pi)=\theta(y\pi)
%=\theta(y)\theta(\pi)+\theta(-y)\theta(-\pi)
=\delta(y-1)\theta(\pi)+\delta(y+1)\theta(-\pi).
\eeq
Thus
\beqn
\prod_{a}P_{\rm out}^{}(y\lvert\pi^{a})=\prod_{a}\theta(y\pi^{a})
%=\theta(y)\prod_{a}\theta(\pi^{a})+\theta(-y)\prod_{a}\theta(-\pi^{a}) \nonumber \\
=\delta(y-1)\prod_{a}\theta(\pi^{a})+\delta(y+1)\prod_{a}\theta(-\pi^{a}),
\eeqn
since $\theta(x)\theta(-x)=0$.

From the above result we find
\beqn
- F_{\rm ex}[\hat{Q}]
&=&\sum_{\bs} \ln \int  dy_{\bs}
\exp \left[\frac{\lambda^{2}}{2} \sum_{a,b=0}^{n}
\prod_{j \in \partial\bs}Q_{j}^{ab} \frac{\partial^{2}}{\partial h_{\bs}^{a}\partial h_{\bs}^{b}}
  \right] \nonumber \\
&& \hspace*{2cm} \prod_{a=0}^{n}
\left. 
%\frac{e^{-\frac{(y_{\bs}+h_{\bs}^{a})^{2}}{2}}}{\sqrt{2\pi}}
%W^{}(y_{\bs}+h_{\bs}^{a})
\left[\delta(y-1)\prod_{a}\theta(h^{a})+\delta(y+1)\prod_{a}\theta(-h^{a})\right]
\right|_{\{h_{\bs}^{a}=0\}}  \nonumber \\
&=&
\sum_{\bs} \ln
\left[
\exp \left[\frac{\lambda^{2}}{2}\sum_{a,b=0}^{n}
\prod_{j \in \partial\bs}Q_{j}^{ab} \frac{\partial^{2}}{\partial h_{\bs}^{a}\partial h_{\bs}^{b}}
  \right]
\prod_{a=0}^{n}
\left. \theta(h_{\bs}^{a}) \right|_{\{h_{\bs}^{a}=0\}}
\right] \nonumber \\
&& +N_{\bs} \ln 2.
\label{eq-F1-sign}
\eeqn
In the following we will omit the last constant term.

In the homogeneous and replica symmetric ansatz \eq{eq-RS-ansatz} we find,
  \beq
 \frac{ - F_{\rm ex}(Q,q)}{NM}= \frac{\alpha}{p}
  \ln
  e^{\frac{\lambda^{2}}{2}q^{p}\frac{\partial^{2}}{\partial h^{2}}}
  \left.  
\left(
  e^{\frac{\lambda^{2}}{2}(Q^{p}-q^{p})\frac{\partial^{2}}{\partial h^{2}}}
  \theta(h)
  %  \frac{e^{-\frac{(y_{\bs}+h)^{2}}{2}}}{\sqrt{2\pi}}
  \right)^{1+n}
  \right|_{h=0}.
  \eeq
  In the slightly generalized ansatz \eq{eq-simple-ansatz} we can write
  \beq
  Q_{ab}^{p}=m^{p}+(Q_{0}^{p}-m^{p})\delta_{a,0}\delta_{b,0}
  +I(a \geq 1)I(b \geq 1)[(q^{p}-m^{p})+(Q^{p}-q^{p})\delta_{a,b}],
  \eeq
  where $I(..)$ is the indicator function. Then we find,
  \beqn
  &&
  \hspace{-1cm}
  \frac{- F_{\rm ex}(Q_{0},Q,q,m)}{NM}= \frac{\alpha}{p}
  \ln
  e^{\frac{\lambda^{2}}{2}m^{p}\frac{\partial^{2}}{\partial h^{2}}}\left\{
\left( e^{\frac{\lambda^{2}}{2}(Q_{0}^{p}-m^{p})\frac{\partial^{2}}{\partial h^{2}}}  \theta(h)\right)
\left.
e^{\frac{\lambda^{2}}{2}(m^{p}-q^{p})\frac{\partial^{2}}{\partial h^{2}}}
\left(
  e^{\frac{\lambda^{2}}{2}(Q^{p}-q^{p})\frac{\partial^{2}}{\partial h^{2}}}
  \theta(h)
  %  \frac{e^{-\frac{(y_{\bs}+h)^{2}}{2}}}{\sqrt{2\pi}}
  \right)^{n}
  \right \}
  \right|_{h=0}\nonumber \\
&&
  \hspace{-1cm}
=   \frac{\alpha}{p}
  \ln
  \int \mathcal{D}z_{0}
  \left\{
%  \left(
  \int \mathcal{D}z_{3}  
  \theta(\lambda\sqrt{m^{p}}z_{0}+\lambda\sqrt{Q_{0}^{p}-m^{p}}z_{3})
  %\right)
  \right. \nonumber \\
&& \left.   
    \int \mathcal{D}z_{1}
\left(
    \int \mathcal{D}z_{2}
  \theta(\lambda\sqrt{m^{p}}z_{0}+\lambda\sqrt{q^{p}-m^{p}}z_{1}+\lambda\sqrt{Q^{p}-q^{p}}z_{2})
  \right)^{n}
  \right \} \nonumber \\
&&  \hspace{-1cm}
=  \frac{\alpha}{p} 
  \ln \int \mathcal{D}z_{0} H(X_{\rm teacher})  \int \mathcal{D}z_{1} H^{n}(X_{\rm student}),
  \label{eq-f1-sign-RS}
  \eeqn
where we used the short-hand notation \eq{eq-def-Dz} and lemma~\ref{lemma:due}. We also introduced a function $H(x)$,
  \beq
  H(x)=\frac{1}{2}{\rm erfc}\left(\frac{x}{\sqrt{2}}\right)=\int_{x}^{\infty}\frac{dt}{\sqrt{2\pi}}e^{-\frac{t^{2}}{2}},
  \label{eq-H}
  \eeq  
and also
  \beq
  X_{\rm teacher}=-\frac{\sqrt{m^{p}}}{\sqrt{Q_{0}^{p}-m^{p}}}z_{0}, \qquad
  X_{\rm student}=-\frac{\sqrt{m^{p}}z_{0}+\sqrt{q^{p}-m^{p}}z_{1}}{\sqrt{Q^{p}-q^{p}}}.
  \label{eq-X-teacher-student}
  \eeq

 We will need derivatives of $\partial_{n} F_{\rm ex}(Q_{0},Q,q,m)|_{n=0}$ with respect to  $Q$,$q$ and $m$  which we find as,  
  \beqn
    \partial_{n} \frac{\partial }{\partial Q}  \frac{ - F_{\rm ex}(Q_{0},Q,q,m)}{NM} |_{n=0}
  &  =& \lambda^{2} \alpha Q^{p-1}
  e^{\frac{\lambda^{2}}{2}m^{p}\frac{\partial^{2}}{\partial h^{2}}}\left\{
\left( e^{\frac{\lambda^{2}}{2}(Q_{0}^{p}-m^{p})\frac{\partial^{2}}{\partial h^{2}}}  \theta(h)\right) \right. \nonumber \\
&& \hspace{2cm} \left. \left.
e^{\frac{\lambda^{2}}{2}(m^{p}-q^{p})\frac{\partial^{2}}{\partial h^{2}}}
\left(
\frac{
 \frac{\partial^{2}}{\partial h^{2}}   e^{\frac{\lambda^{2}}{2}(Q^{p}-q^{p})\frac{\partial^{2}}{\partial h^{2}}}
  \theta(h)
}{
  e^{\frac{\lambda^{2}}{2}(Q^{p}-q^{p})\frac{\partial^{2}}{\partial h^{2}}}
  \theta(h)}
  %  \frac{e^{-\frac{(y_{\bs}+h)^{2}}{2}}}{\sqrt{2\pi}}
  \right)
  \right \}
  \right|_{h=0}
  \nonumber \\
  &=&
  \frac{\lambda^{2} \alpha Q^{p-1}}{Q^{p}-q^{p}}
  \int \mathcal{D}z_{\rm 0}H(X_{\rm teacher})
  \int \mathcal{D}z_{\rm 1}
\frac{H^{''}(X_{\rm student})}{H(X_{\rm student})},
    \label{eq-f1-sign-RS-derivative-wrt-Q}  
  \eeqn
%  which we find doing integrations by parts.
\beqn
  \nonumber \\
  \partial_{n} \frac{\partial }{\partial q}  \frac{ - F_{\rm ex}(Q_{0},Q,q,m)}{NM} |_{n=0}
  &  =&-\lambda^{2} \alpha q^{p-1}
  e^{\frac{\lambda^{2}}{2}m^{p}\frac{\partial^{2}}{\partial h^{2}}}\left\{
\left( e^{\frac{\lambda^{2}}{2}(Q_{0}^{p}-m^{p})\frac{\partial^{2}}{\partial h^{2}}}  \theta(h)\right) \right. \nonumber \\
&& \hspace{2cm} \left. \left.
e^{\frac{\lambda^{2}}{2}(m^{p}-q^{p})\frac{\partial^{2}}{\partial h^{2}}}
\left(
\frac{
 \frac{\partial}{\partial h}   e^{\frac{\lambda^{2}}{2}(Q^{p}-q^{p})\frac{\partial^{2}}{\partial h^{2}}}
  \theta(h)
}{
  e^{\frac{\lambda^{2}}{2}(Q^{p}-q^{p})\frac{\partial^{2}}{\partial h^{2}}}
  \theta(h)}
  %  \frac{e^{-\frac{(y_{\bs}+h)^{2}}{2}}}{\sqrt{2\pi}}
  \right)^{2}
  \right \}
  \right|_{h=0}   \nonumber \\
  &=& - \frac{\lambda^{2}\alpha q^{p-1}}{Q^{p}-q^{p}}\int \mathcal{D}z_{\rm 0} H(X_{\rm teacher})
  \int \mathcal{D}z_{1}
  \left(\frac{H'(X_{\rm student})}{H(X_{\rm student})} \right)^{2},
    \label{eq-f1-sign-RS-derivatives-1}  
\eeqn
\beqn
  \partial_{n}   \frac{\partial }{\partial m}  \frac{ - F_{\rm ex}(Q_{0},Q,q,m)}{NM} |_{n=0}&=&
  \lambda^{2}\alpha m^{p-1}
    e^{\frac{\lambda^{2}}{2}m^{p}\frac{\partial^{2}}{\partial h^{2}}}\left\{
    \left    (  \frac{\partial}{\partial h} e^{\frac{\lambda^{2}}{2}(Q_{0}^{p}-m^{p})\frac{\partial^{2}}{\partial h^{2}}}  \theta(h)
    \right    )
    \nonumber \right. \\
&& \hspace*{2cm} \left. \left.
e^{\frac{\lambda^{2}}{2}(m^{p}-q^{p})\frac{\partial^{2}}{\partial h^{2}}}
%\left(
\frac{
 \frac{\partial}{\partial h}   e^{\frac{\lambda^{2}}{2}(Q^{p}-q^{p})\frac{\partial^{2}}{\partial h^{2}}}
  \theta(h)
}{
  e^{\frac{\lambda^{2}}{2}(Q^{p}-q^{p})\frac{\partial^{2}}{\partial h^{2}}}
  \theta(h)}
  %  \frac{e^{-\frac{(y_{\bs}+h)^{2}}{2}}}{\sqrt{2\pi}}
%  \right)
  \right \}
  \right|_{h=0}   \nonumber \\  
  &=& \frac{\lambda^{2}\alpha m^{p-1}}{Q_{0}^{p}-m^{p}} \int \mathcal{D}z_{\rm 0} H^{'}(X_{\rm teacher})
  \int \mathcal{D}z_{1} \frac{H'(X_{\rm student})}{H(X_{\rm student})}.
  \label{eq-f1-sign-RS-derivatives-2}  
  \eeqn
Here we used
$e^{\frac{\lambda^{2}}{2}m^{p}\frac{\partial^{2}}{\partial h^{2}}} e^{\frac{\lambda^{2}}{2}(Q_{0}^{p}-m^{p})\frac{\partial^{2}}{\partial h^{2}}}\theta(h)=e^{\frac{\lambda^{2}}{2}\frac{\partial^{2}}{\partial h^{2}}}\theta(h)=\int_{0}^{\infty} \frac{dy}{\sqrt{2\pi \lambda^{2}}}e^{-\frac{y^{2}}{2\lambda^{2}}}=1/2$.

%%%%%%%%%%%%%%%%%%%%%%%%%%%%%%%%%%%%%%%%%%%%%%%%%%%%%
%%%%%%%%%%%%%%%%%%%%%%%%%%%%%%%%%%%%%%%%%%%%%%%%%%%%%
\subsection{Total free-energy and equation of states}
\label{sec-replica-total-free-energy}

%%%%%%%%%%%%%%%%%%%%%%%%%%%%%%%%%%%%%%%%%%%%%%%%%%%%%

\subsubsection{In the absence of students}
\label{sec-without-students}

In the absence of students $n \to 0$ we find
\eq{eq-f-ex} becomes
\beq
\lim_{n \to 0}- F_{{\rm ex}}[\hat{Q}]=0
\eeq
because of the normalization of the output function $\int dy P_{\rm out}(y|h)=1$.
Then in the case of Gaussian prior we find the total free-energy, using
\eq{eq-f0-gaussian-RS}, in the $n \to 0$ limit as
\beq
\lim_{n \to 0}
\frac{  - F(Q_{0},Q,q,m)}{NM} =\frac{1}{2}\ln Q_{0}-\frac{1}{2}Q_{0}.
  \eeq
Taking derivative with respect to $Q_{0}$ we find
$0=\lim_{n \to 0}\frac{\partial}{\partial Q_{0}}  \frac{ - F_{0}(Q_{0},Q,q,m)}{NM}$  
is satisfied with $Q_{0}=1$, as expected because of the law of large numbers.
Thus we assume $Q_{0}=1$ in the following.

\subsubsection{Ising prior and additive noise}
\label{sec-Ising-Gaussian-replica}

Using  \eq{eq-f0-ising-RS} and  \eq{eq-f1-additive-noise-RS}
together with $Q_{0}=Q=1$ which holds in the Ising case,
we obtain the total replicated free-energy as
\beqn
\frac{- F(q,m)}{NM} &= &
-\frac{1}{2}\left[  n\epsilon_{0}^{*} + (n^{2}-n)\epsilon_{0}^{*} q +2n \psi^{*}_0 m \right] 
+ \ln
\int \mathcal{D}z
 (2\cosh(\psi_{0}^{*}+\sqrt{\epsilon^{*}_{0}}z))^{n} \nonumber \\
&+&
\frac{\alpha}{p}
  \ln \int dy
  W(y)  \int \mathcal{D}z_{0} 
  \left(
 \int \mathcal{D}z_{1}   
    W^{}(y+\lambda\sqrt{1-2m^{p}+q^{p}}z_{0}+\lambda\sqrt{1-q^{p}}z_{1})
    \right)^{n} \qquad
    \eeqn
from which we find the free-energy \eq{eq-def-f} as
\beqn
- f &= & 
-\frac{1}{2}\left[ \epsilon_{0}^{*} - \epsilon_{0}^{*} q +2 \psi^{*}_{0} m \right] 
+  \int \mathcal{D}z \ln (2\cosh(\psi_{0}^{*}+\sqrt{\epsilon_{0}^{*}}z)) \nonumber \\
&+&
\frac{\alpha}{p}
%\ln
\int dy
  W(y)  \int \mathcal{D}z_{0} 
\ln  \left(
 \int \mathcal{D}z_{1}   
    W^{}(y+\lambda\sqrt{1-2m^{p}+q^{p}}z_{0}+\lambda\sqrt{1-q^{p}}z_{1})
    \right).\qquad
    \label{eq-f-ising-RS}
    \eeqn
    Extremizing the free-energy with respect to $q$ and $m$ we find, using \eq{eq-f1-additive-noise-RS-derivatives} 
    \beqn
    \epsilon_{0}^{*} &=& \lambda^{2}\alpha q^{p-1}\theta_{0}(\lambda,m,q)\nonumber \\
   \psi_{0}^{*} &=& \lambda^{2}\alpha m^{p-1}(\theta_{0}(\lambda,m,q)-\theta_{1}(\lambda,m,q))
   \label{eq-eps-phi-Ising-RS}
    \eeqn
    Then the equations of state (saddle point equations) for the order parameters $q$ and $m$ are found
    using \eq{eq-f0-ising-RS-derivatives} as,
    \beqn
    q&=& \int \mathcal{D}z  \tanh^{2}(\sqrt{\lambda^{2}\alpha q^{p-1}\theta_{0}(\lambda,m,q)}z+\lambda^{2}\alpha m^{p-1}(\theta_{0}(\lambda,m,q)-\theta_{1}(\lambda,m,q))),\nonumber \\
    m&=& \int \mathcal{D}z  \tanh(\sqrt{\lambda^{2}\alpha q^{p-1}\theta_{0}(\lambda,m,q)}z+\lambda^{2}\alpha m^{p-1}(\theta_{0}(\lambda,m,q)-\theta_{1}(\lambda,m,q))),\qquad
    \label{eq-SP-ising-additive-noise-RS}
    \eeqn
    where $\theta_{0}$ and $\theta_{1}$ are given by \eq{eq-theta0-theta1} with \eq{eq-Xi}.

    In the Bayes optimal case $m=q$, we have already shown using Lemma \ref{lemma:quattro} that $\theta_1=0$. Then with $A\equiv\alpha\lambda^2m^{p-1}\theta_0$ we find the two equations
    \eq{eq-SP-ising-additive-noise-RS} become the same,
\begin{align}
    m-q &= \int\mathcal{D}z\tanh(A+\sqrt{A}z)-\int\mathcal{D}z\tanh^2(A+\sqrt{A}z) = \nonumber \\
    &=\frac{e^{-\frac{A}{2}}}{\sqrt{2\pi A}}\int {\rm d}u\,e^{-\frac{u^2}{2A}}\tanh{u}\,\frac{1}{\cosh{u}}=0, \quad\quad \forall A
\end{align}
where we have used a change of variable $u=A+\sqrt{A}z$ to recognize in the integrand an odd function of $u$.

    Specifically we will analyze the Bayes-optimal case $m=q$ with Gaussian noise  \eq{eq-gaussian-noise-replica}. In this case one finds,
    using \eq{eq-theta0-theta1-gaussian-noise-replica},
    \begin{equation}
      m=\int\mathcal{D}z\tanh(A+\sqrt{A}z),
\qquad A=\frac{\alpha\lambda^2m^{p-1}}{1+\lambda^2(1-m^p)},
    \label{eq:m_Ising_Bayesopt}
    \end{equation}
 For this case, the free-energy \eq{eq-f-ising-RS}
-\eq{eq-SP-ising-additive-noise-RS} becomes, 
\begin{align}
     f(m,\alpha,\lambda) = \,&\frac{1}{2}A(m,\alpha,\lambda)(1+m) - \int\mathcal{D}z\ln\cosh\left(A(m,\alpha,\lambda)+z\sqrt{A(m,\alpha,\lambda)}\right) \nonumber \\
    &+\frac{\alpha}{2p}\ln\left(1+\lambda^2(1-m^p)\right).
    \label{eq:free_ene_Ising_Bayesopt}
\end{align}

\subsubsection{Gaussian prior and additive noise}
\label{sec-Gaussian-Gaussian-replica}

Using  \eq{eq-f0-gaussian-RS} and
\eq{eq-f1-additive-noise-RS} together with $Q_{0}=1$
(which we found in sec.~\ref{sec-without-students}),
we obtain the total replicated free-energy as
\beqn
\frac{- F(Q,q,m)}{NM} &= & \frac{1}{2}\left[
    \ln (Q-q+n(q-m^{2}))+ (n-1) \ln (Q-q)
    \right]
-\frac{1}{2}-\frac{n}{2}Q
\nonumber \\
&+&\frac{\alpha}{p}
%\ln
\int dy
  W(y)  \int \mathcal{D}z_{0} 
  \left(
 \int \mathcal{D}z_{1}   
    W^{}(y+\sqrt{1-2m^{p}+q^{p}}z_{0}+\sqrt{Q^{p}-q^{p}}z_{1})
    \right)^{n}, \qquad
    \eeqn
        from which we find the free-energy \eq{eq-def-f} as
\beqn
- f &= &
\frac{1}{2}\left[\frac{q-m^{2}}{Q-q}+\ln(Q-q)\right]-\frac{Q}{2}
%    \ln (1-q+n(q-m^{2}))+ (n-1) \ln (1-q)
%    \right]
\nonumber \\
&+&
\frac{\alpha}{p}
  \ln \int dy
  W(y)  \int \mathcal{D}z_{0} 
  \ln
  \left(
 \int \mathcal{D}z_{1}   
    W^{}(y+\sqrt{1-2m^{p}+q^{p}}z_{0}+\sqrt{Q^{p}-q^{p}}z_{1})
    \right). \qquad
    \label{eq-f-gaussian-additive-noise-RS-2}
    \eeqn
    Extremizing the free-energy with respect to $Q$, $q$ and $m$ we find, using \eq{eq-f0-gaussian-RS-derivatives} and \eq{eq-f1-additive-noise-RS-derivatives},
    \beqn
    0 &=& \frac{Q-2q-m^{2}}{(Q-q)^{2}}-1-\lambda^{2}\alpha Q^{p-1}\theta_{1}(\lambda,1,Q,m,q)),
    \nonumber \\
    0 &=& \frac{q-m^{2}}{(Q-q)^{2}}-\lambda^{2}\alpha q^{p-1}\theta_{0}(\lambda,1,Q,m,q),\nonumber \\
    0 &=& -\frac{m}{Q-q}+\lambda^{2}\alpha m^{p-1}(\theta_{0}(\lambda,m,q)-\theta_{1}(\lambda,1,Q,m,q)).
    \label{eq-SP-gaussian-additive-noise-RS}    
    \eeqn

In the Bayes-optimal case, $m=q$ must hold. 
As noted before in sec \ref{sec-additive-noise-replica}
we find $\theta_{1}=0$ if $m=q$ and $Q_{0}=Q$.
In this case the first equation of \eq{eq-SP-gaussian-additive-noise-RS}  become satisfied
and the last two equations of \eq{eq-SP-gaussian-additive-noise-RS} become the same.
Thus by solving the latter we find a self-consistent solution with $q=m$ and $Q=Q_{0}=1$.

We will analyze specifically the Bayes-optimal case with Gaussian noise. In this case one finds,
using \eq{eq-theta0-theta1-gaussian-noise-replica},
\begin{equation}
    m = \frac{\alpha\lambda^2m^{p-1}}{1+\lambda^2(1-m^p)+\alpha\lambda^2m^{p-1}}.
    \label{eq:m_Bayesopt_gauss}
\end{equation}
For this case, the free-energy 
\eq{eq-f-gaussian-additive-noise-RS-2}
%-\eq{eq-SP-ising-additive-noise-RS}
becomes, 
\begin{align}
  - f(m,\alpha,\lambda) &= 
\frac{1}{2}\left[m+\ln(1-m)\right]
%\nonumber \\
%&+
+\frac{\alpha}{p}
  \ln \int dy
  W(y)
%  \int \mathcal{D}z_{0} 
  \ln
  \left(
 \int \mathcal{D}z
    W (y+\sqrt{1-m^{p}}z)
    \right).
      \label{eq:free_ene_gaussian_Bayesopt}
\end{align}

%{\bf (Todo) Complete this.}
%The cases $p=2$ and $p=3$ are simple enough to solve eq.\eqref{eq:m_Bayesopt_gauss} analytically.

%%%%%%%%%%%%%%%%%%%%%%%%%%%%%%%%%%%%%%%%%%%%%%%%%%%%%
\subsubsection{Gaussian prior and sign output}
\label{sec-Gaussian-sign-replica}

Using  \eq{eq-f0-gaussian-RS} and  \eq{eq-f1-sign-RS}
together with $Q_{0}=1$
(which we found in sec.~\ref{sec-without-students})
we obtain the total replicated free-energy as
\beqn
\frac{- F(Q,q,m)}{NM} &= &
\frac{1}{2}\left[
    \ln (Q-q+n(q-m^{2}))+ (n-1) \ln (-q)
    \right]
 -\frac{1}{2}-\frac{n}{2}Q
\nonumber \\
&& +
\frac{\alpha}{p} 
  \ln \int \mathcal{D}z_{0} H(X_{\rm teacher})  \int \mathcal{D}z_{1} H^{n}(X_{\rm student}) ,
\nonumber
\eeqn
with $H(x)$ defined in \eq{eq-H} and $X_{\rm teacher}$ and $X_{\rm student}$ defined in  \eq{eq-X-teacher-student}.
Then we find the free-energy \eq{eq-def-f} as
\beqn
- f &= &
\frac{1}{2}\left[\frac{q-m^{2}}{Q-q}+\ln(Q-q)\right]-\frac{Q}{2}
+ \frac{\alpha}{p} 
\ln \int \mathcal{D}z_{0} H(X_{\rm teacher})  \int \mathcal{D}z_{1} \ln H(X_{\rm student}). \qquad
\label{eq-free-ene-signoutput}
  \eeqn
  Extremizing the free-energy with respect to $Q$, $q$ and $m$ we find, using \eq{eq-f0-gaussian-RS-derivatives}
  \eq{eq-f1-sign-RS-derivative-wrt-Q}, \eq{eq-f1-sign-RS-derivatives-1} and  \eq{eq-f1-sign-RS-derivatives-2},
  \beqn
  0 &=&
  \frac{Q-2q-m^{2}}{(Q-q)^{2}}-1
+   \frac{\lambda^{2} \alpha Q^{p-1}}{Q^{p}-q^{p}}
  \int \mathcal{D}z_{\rm 0}H(X_{\rm teacher})
  \int \mathcal{D}z_{\rm 1}
\frac{H^{''}(X_{\rm student})}{H(X_{\rm student})},
  \nonumber \\
    0 &=& \frac{q-m^{2}}{(Q-q)^{2}}-\frac{2\lambda^{2}\alpha q^{p-1}}{Q^{p}-q^{p}} \int \mathcal{D}z_{\rm 0} H(X_{\rm teacher})
  \int \mathcal{D}z_{1}
\left(\frac{H'(X_{\rm student})}{H(X_{\rm student})} \right)^{2}, \nonumber \\
    0 &=& -\frac{m}{1-q}+\frac{\lambda^{2}\alpha m^{p-1}}{1-m^{p}} \int \mathcal{D}z_{\rm 0} H'(X_{\rm teacher})
  \int \mathcal{D}z_{1} \frac{H'(X_{\rm student})}{H(X_{\rm student})}.
    \label{eq-SP-gaussian-sign-RS}    
    \eeqn
We find the 1st equation becomes satisfied in the Bayes optimal case $q=m$ with $Q_{0}=Q$.

In the Bayes optimal case $q=m$ must hold. 
Assuming also $Q=1$ we find the last two equations shown above become the same,
\beq
0 = -\frac{m}{1-m}+\frac{\lambda^{2}\alpha m^{p-1}}{1-m^{p}} \int \mathcal{D}z_{\rm 0}
%H'(X_{\rm teacher})
%    \int \mathcal{D}z_{1}
\frac{(H'(X))^{2}}{H(X)},
\qquad    X=-\frac{\sqrt{m^{p}}}{\sqrt{1-m^{p}}}z.
    \label{eq-SP-gaussian-sign-RS-BayesOptimal}    
    \eeq
By solving this equation we find a self-consistent solution with $q=m$ and $Q=Q_{0}=1$.

In the Bayes optimal case $q=m$ and assuming $Q=1$, the free-energy \eq{eq-free-ene-signoutput} becomes,
\beq
- f = 
\frac{1}{2}\left[\frac{q-m^{2}}{1-m}+\ln(1-m)\right]
+ \frac{\alpha}{p} 
\ln \int \mathcal{D}z H(X) \ln H(X), \qquad    X=-\frac{\sqrt{m^{p}}}{\sqrt{1-m^{p}}}z.
\label{eq-free-ene-signoutput-BayesOptimal}
\eeq

%%%%%%%%%%%%%%%%%%%%%%%%%%%%%%%%%%%%%%%%%%%%%%%%%%%%%
%%%%%%%%%%%%%%%%%%%%%%%%%%%%%%%%%%%%%%%%%%%%%%%%%%%%%
\subsection{Error correcting code}
\label{sec-error-correcting-code}

\blue{As mentioned in sec.~\ref{sec-error-correcting-codes}, our problem is
also related to error correcting codes}.
Similarly to the Sourlas code\cite{sourlas1989spin},
our inference problem with the Ising prior and additive Gaussian noise
can be viewed as an error correcting code.
The teacher (encoder) is trying to send 
$NM$ bits (ground truth) encoded into $N_{\bs}$ data and
student (decoder) is trying to infer the ground truth given the
data. The transmission rate is given by
\beq
R=\frac{NM}{N_{\bs}}=\frac{p}{\alpha}.
\eeq

Now let us consider to take $p \to \infty$, $\alpha \to \infty$
with the ratio $R$ fixed.
In this limit \eq{eq:m_Ising_Bayesopt} admits only two solutions
\beq
m=0,1.
\eeq
Thus in this limit the inference completely fails $m=0$ or succeed perfectly $m=1$.
From \eq{eq:free_ene_Ising_Bayesopt} we find the free-energy associated with
the two solutions as,
\beqn
- f(m=0) &=& \ln 2 -\frac{1}{2R}\left[ 1+\ln (2\pi)+\ln (1+\lambda^{2})\right], \\
- f(m=1) &=&  -\frac{1}{2R}\left[ 1+\ln (2\pi)\right] .
\eeqn
From this we find there is a 1st order transition
at $\lambda_{c}$ given by
\beq
R=\frac{1}{2} \ln_{2}(1+\lambda_{\rm c}^{2}),
\eeq
or
\beq
\lambda_{\rm c}=\sqrt{2R\log 2-1}.
\eeq
For $\lambda < \lambda_{\rm c}$ the paramagnetic solution $m=0$ is the thermodynamically relevant one which is replaced by the ferromagnetic solution $m=1$ for $\lambda > \lambda_{\rm c}$.

On the other hand the capacity of Gaussian channel is given by
\beq
C=\frac{1}{2}\ln_{2} (1+\lambda^{2}),
\eeq
thus we find 
\beq
C(\lambda_{\rm c})=R.
\eeq
This means that the Shannon's bound~\cite{shannon1948mathematical},
\beq
R \leq C(\lambda)
\eeq
is satisfied for $\lambda \geq \lambda_{\rm c}$ and the upper bound is
attained at the transition point  $\lambda \to \lambda_{\rm c}^{+}$.

\blue{ The Shannonn's bound says that the code must have sufficient redundancy such that the transmission rate $R$ is smaller than the channel capacity $C$ for fully successful error correction $m=1$. The above result means that our code can do fully successful inferences $m=1$ at the largest possible transmission rate allowed by the bound at the transition point $\lambda_{c}$ in the $p \to \infty$ limit. While this feature is the same as the Sourlas code \cite{sourlas1989spin}, the notable differences lies in the transmission rate $R$. The transmission rate $R$ (as well as the capacity $C$) of the Sourlas code 
vanishes in the thermodynamic limit $N \to \infty$ since $R \propto 1/N^{p-1}$ because of the global coupling. This is a negative aspect of the Sourlas code. However, thanks to the dense coupling which is spacer than the global coupling, the transmission rate $R$ of our model can remain finite
by considering $p,\alpha \to \infty$ with the ratio $R=p/\alpha$ being fixed. }
%In the appendix sec.~\ref{sec-comparison-to-sourlas-code} we discuss some additional
%differences between our model viewed as an error correcting code and the Sourlas code.

\subsection{{Some remarks on non-linearity of effective potentials}}
\label{sec-remarks-on-non-linearity}
%\input{error-correcting-code}
%\input{comparison-with-rank-one}

%Let us discuss comparisons of our problem in the dense limit $N \gg M \gg 1$ based on sparce measurement $N_{\bs}=O(NM)$
%with some rank-one $M=1$ problems.

Each of our inference problems can be viewed as a system
in which  $x_{i\mu}$s are interacting with each other 
as described by an effective Hamiltonian,
\beq
H=\sum_{\bs} V(\pi_{\bs}) \qquad
V(\pi_{\bs})=-\ln P_{\rm out}(y_{\bs}|\pi_{\bs})
\eeq
Here we omitted local potential terms due to $P_{\rm pri}(x_{i\mu})$.
In the cases that we considered in the present paper, the effective potentials $V(\pi_{\bs})$ are non-linear functions of $\pi_{\bs}$. Importantly our problems cannot be linealized and this point brings about non-trivial consequences as we discuss below.

In sec.~\ref{sec-vecotorialCSP} we mentioned that our present problem is exactly the inverse problem or
{\it planted version} of the vectorial constraint satisfaction problems (CSP)
\blue{on dense (but not fully connected) graphs} studied in \cite{yoshino2018}.
In the case of 'linear potential' $V(x)=Jx$\ (see sec. \eq{eq-linear-potential}),
it was found that the vectorial CSP problems in the dense limit become essentially equivalent
to \blue{(just a bunch of)} the $M=1$ problem
which are the usual mean-field $p$-spin spinglass models.
\blue{The situation is analogous to the sub-linearly high-rank models of matrix factorization (see below).}
However such correspondences disappear for non-linear potentials.
For instance, in the case of quadratic potential $V(x)=\frac{\epsilon}{2}x^{2}$ \eq{eq-quadratic-potential},
full RSB was found even for $p=2$ {\it spherical} model
(continuous $x_{j\mu}$ with the constraint $\sum_{\mu=1}^{M} x_{j\mu}^{2}=M$)
while such RSB is absent in the case of linear potential for the $p=2$ spherical model (spherical SK model \cite{kosterlitz1976spherical}). \blue{We sketch below why non-linear models do not reduce to
linear (and thus effectively $M=1$) models in the case $c=O(M)$.}

Similarly to the vectorial CSP problems, we believe that our vectorial inference problem do {\it not} reduce to rank-one $M=1$ problem in general except for the special case of linear potential $V(x)=Jx$,
i.e. $\ln P_{\rm out}(y_{\bs}|\pi_{\bs}) \propto -\pi_{\bs}$.
Indeed in the case of the Ising model with additive Gaussian noise
$\ln P_{\rm out}(y_{\bs}|\pi_{\bs}) \propto -(y_{\bs}-\pi_{\bs})^{2}$,
which can be viewed as a system with quadratic potential  $V(x)=\frac{\epsilon}{2}x^{2}$,
we obtained the equation of state \eq{eq:m_Ising_Bayesopt} which has a term $\lambda^2(1-m^p)$
in the denominator. Such a term is absent in the case of the Sourlas code which is a rank-one $M=1$ problem
with additive Gaussian noise. As the result some properties of our model are {\it qualitatively}
different from those of the Sourlas code.
Most importantly the phase behavior in the $p=2$ case: mixture of 2nd and 1st order transitions as discussed in 
sec.~\ref{sec:result-p=2-ising-gaussian-noise}. Such a feature is absent in the case of the Sourlas code.

Let us discuss more closely the case of the quadratic effective potential,
\beqn
H&=&-\frac{1}{2}\sum_{\bs} \left(
(\pi_{\bs}^{*}+w_{\bs})- \pi_{\bs}\right)^{2} \\
&=&\sum_{\bs} (\pi_{\bs}^{*}+w_{\bs})\pi_{\bs}
-\frac{1}{2}\sum_{\bs} (\pi_{\bs})^{2}
 +{\rm const},
\label{eq-quadractic-or-linear-potential}
\eeqn
where $w_{\bs}$ is a quenched Gaussian random number with zero mean and unit variance and 
\beqn
\pi_{\bs} = \frac{\Lambda}{\sqrt{c}} \sum_{\mu=1}^M
F_{\bs \mu}\prod_{j\in\del\bs} x_{j\mu}.
\eeqn
In the last equation of \eq{eq-quadractic-or-linear-potential}, "const" represents a irrelevant constant term
which does not involve $\pi_{\bs}$.
\begin{enumerate}
\item  Globally coupled systems $c \propto N^{p-1}$

  In this case $\pi_{\bs}$ scales as $\pi_{\bs}=O((M N^{1-p})^{1/2})$ with $p \geq 2$.
  Here we assumed that the sum $\sum_{\mu=1}^{M}\cdots$ scales as $O(M^{1/2})$. This is valid as long as correlation between different components
  $\mu$ can be neglected.
Thus for $M=N^{a}$ we find $\pi_{\bs} \to 0$ in $N \to \infty$ limit if the exponent
$a$ is small enough.
This applies to the case of the Sourlas code ($M=1$) and sub-linearly high-rank cases
if $0 < a <p-1$.
In such a situation the 2nd term $\pi_{\bs}^{2}$ in the last equation \eq{eq-quadractic-or-linear-potential}
can be safely dropped in $N \to \infty$ and the problem reduces to models linear in $\pi_{\bs}$.

Indeed the Sourlas code can be mapped exactly to the ferromagnetically biased  $p$-spin Ising spin-glass model
\cite{derrida1980random,gross1984simplest,nishimori1999statistical}.
More recent findings that sub-linearly high-rank models
with global coupling \cite{reeves2020information,pourkamali2024matrix,barbier2024multiscale}
become essentially equivalent to \blue{(just a bunch of)} rank-one $M=1$ cases is consistent with the above observation.

The situation is different if the exponent $a$ is not small enough, including the case of full rank $M=O(N)$ case ($a=1$) in
matrix factorization ($p=2$)
\cite{sakata2013statistical,sakata2013sample}.
In such cases the problems {\it cannot} be reduced to models linear in $\pi_{\bs}$.

\item  Our model $c \propto M$ ($\Lambda=\lambda/\sqrt{\alpha}$)

The situation is very different in this case:
$\pi=O(1)$ so that the  2nd term $\pi_{\bs}^{2}$ in the last equation \eq{eq-quadractic-or-linear-potential} cannot be dropped. Thus our problem cannot be reduced to models linear in $\pi_{\bs}$.

\end{enumerate}

%%%%%%%%%%%%%%%%%%%%%%%%%%%%%%%%%%%%%%%%%%%%%%%%%%%%%
%%%%%%%%%%%%%%%%%%%%%%%%%%%%%%%%%%%%%%%%%%%%%%%%%%%%%
\subsection{Independence on choices of factor $F$}
\label{sec-independence-on-choices-of-F}

In the results presented in this section one can notice
that expressions for the macroscopic quantities like the order-parameters, equation of states, free-energy... are the same
for the deterministic model with $F=1$ \eq{eq-F-uniform} and the disordered model with random $F$  \eq{eq-F-random}.
Note however that the remaining global symmetries (see sec~\ref{sec-symmetries}) are different between the two cases.

\clearpage

%%%%%%%%%%%%%%%%%%%%%%%%%%%%%%%%%%%%%%%%%%%%%%%%%%%%%
%%%%%%%%%%%%%%%%%%%%%%%%%%%%%%%%%%%%%%%%%%%%%%%%%%%%%
%%%%%%%%%%%%%%%%%%%%%%%%%%%%%%%%%%%%%%%%%%%%%%%%%%%%%
\section{Message Passing Algorithms}
\label{sec-mp}

\blue{In this section we develop a message passing scheme for the present problem following a standard procedure  \cite{rangan2010estimation,kabashima2016phase}, which will be used in sec.~\ref{sec-result} to analyze specific models.
If the readers are not interested in the technical but standard details they may skip
secs.~\ref{sec-BP}-\ref{sec-SE-equations}.
In sec.~\ref{sec-comparison-BP-replica} we find the results of SE (state evolution) equations agree with
the corresponding equations of states obtained by the replica approach presented in sec.~\ref{sec-replica-total-free-energy}.
}

%%%%%%%%%%%%%%%%%%%%%%%%%%%%%%%%%%%%%%%%%%%%%%%%
%%%%%%%%%%%%%%%%%%%%%%%%%%%%%%%%%%%%%%%%%%%%%%%%
\subsection{Belief Propagation (BP)}
\label{sec-BP}

Based on the graphical model in Fig. \ref{fig:fig_graph}, we here introduce the well-known belief propagation (BP) algorithm. The graph composed of gray circles and squares corresponds to the random and sparse observations, and it forms a sparse random graph. Therefore, the BP algorithm that uses the messages for $\bm{x}_i = (x_{i1}, x_{i2}, \ldots, x_{iM})^{\top}$ corresponding to the gray circles is expected to yield asymptotically exact results. Fortunately, in our dense limit, a further simplification is possible. Specifically, even if we consider messages and BP for each variable component $x_{i\mu}$ corresponding to the green circles, the results are still expected to be asymptotically exact. This is precisely the same situation in which BP becomes asymptotically exact in fully connected models. In summary, for two reasons--the randomness and sparsity of observations and the divergence of $M$--the use of BP for each $x_{i\mu}$ is justified. 

The BP equations for the messages of $x_{i\mu}$ can be written formally as
\beqn
\tilde{\phi}_{\bs\to i\mu}^t(x_{i\mu}) &= \frac{1}{\tilde{z}_{\bs\to i\mu}^t} \expval{P_{\rm out} \qty(y_\bs|\pi_\bs)}^{\setminus i\mu}_t,
  \label{eq:p-ary_problem_BP1}\\
\phi_{i\mu\to\bs}^{t+1}(x_{i\mu}) &= \frac{1}{z^{t+1}_{i\mu\to\bs}} P_\text{pri.}(x_{i\mu}) \prod_{\bs'\in\del i\setminus\bs} \tilde{\phi}^t_{\bs'\to i\mu}(x_{i\mu}),
\label{eq:p-ary_problem_BP2}
\eeqn
where we introduced an average defined as
\beqn
\expval{\cdots}^{\setminus i\mu}_t &= \int \qty( \prod_{j\in\del\bs\setminus i} \prod_{\nu=1}^M \dd{x_{j\nu}} \phi_{j\nu\to\bs}^t(x_{j\nu}) ) \qty( \prod_{\nu(\neq\mu)} \dd{x_{i\nu}} \phi_{i\nu\to\bs}^t(x_{i\nu}) )\; \qty(\cdots).
 \label{eq:HRTD-BP-expectation}
\eeqn
Note that the statistical weight used in the average is completely factorized
into the product of the messages $\phi_{j\nu\to\bs}^t(x_{j\nu})$ from different $(j,\nu)$s.
%$j \in \partial \bs$ and $\nuu=1,2,\ldots,M$.
Thus we have for any observables 
$\langle A(x_{j_{1}\nu_{1}})B(x_{j_{2}\nu_{2}})\rangle_{t}^{\setminus i\mu}=\langle A(x_{j_{1}\nu_{1}})\rangle^{\setminus i\mu}_{t}\langle B(x_{j_{2}\nu_{2}})\rangle^{\setminus i\mu}_{t}$ 
unless $(j_{1},\nu_{1})=(j_{2},\nu_{2})$.

%%%%%%%%%%%%%%%%%%%%%%%%%%%%%%%%%%%%%%%%%%%%%%%%
%%%%%%%%%%%%%%%%%%%%%%%%%%%%%%%%%%%%%%%%%%%%%%%%
\subsection{From BP to r-BP}
\label{sec-r-BP}
In the regime $M \gg 1$  which we are interested in, the BP equations become intractable. In the following we
derive r-BP (relaxed BP)\cite{rangan2010estimation,kabashima2016phase} which is suitable for $M \gg 1$.

%%%%%%%%%%%%%%%%%%%%%%%%%%%%%%%%%%%%%%%%%%%%%%%%
\subsubsection{The 1st BP equation}
\label{subsubsec-1stBP}

We first consider the 1st BP equation \eq{eq:p-ary_problem_BP1}.
Let us introduce a Fourier representation of $P_\text{out}(y_\bs|\pi_\bs)$
\beqn
P_{\rm out}(y_\bs|\pi_\bs) &= \int \frac{\dd{k}}{2\pi} \dd{z} e^{ik\qty(\pi_\bs-z)} P_{\rm out}(y_\bs|z).
\eeqn
Then we can rewrite \eq{eq:p-ary_problem_BP1} as,
\beqn
\tilde{\phi}_{\bs\to i\mu}^t (x_{i\mu}) &= \frac{1}{\tilde{z}_{\bs\to i\mu}^t} \int \frac{\dd{k}}{2\pi} \dd{z} \expval{e^{ik\pi_\bs}}^{\setminus i\mu}_t e^{-ikz} P_{\rm out}(y_\bs|z).
\eeqn
Now we evaluate $\expval{e^{ik\pi_\bs}}^{\setminus i\mu}_t$ by a cumulant expansion,
\beqn
\ln\expval{e^{ik\pi_\bs}}^{\setminus i\mu}_t &= ik O^{(1),t}_{\bs\to i\mu} - \frac12 k^2 O^{(2),t}_{\bs\to i\mu} - \frac16 i k^3 O^{(3),t}_{\bs\to i\mu} +\dots.
\eeqn

Using \eq{eq:HRTD-BP-expectation} and the comment mentioned there, we find the first few moments as,
\beqn
O^{(1),t}_{\bs\to i\mu} &=& \expval{\pi_\bs}^{\setminus i\mu}_t \notag\\
	&=& \frac{\lambda}{\sqrt{M}}F_{\bs,\mu} x_{i\mu} \prod_{j\in\del\bs\setminus i} m_{j\mu\to\bs}^t + \sum_{\nu(\neq\mu)} \frac{\lambda}{\sqrt{M}}F_{\bs\nu} \prod_{j\in\del\bs} m_{j\nu\to\bs}^t \notag\\
&=& \frac{\lambda}{\sqrt{M}} F_{\bs\mu} \rho^t_{\bs\setminus i,\mu} x_{i\mu} + \omega^t_{\bs\to\mu}, \\
O^{(2),t}_{\bs\to i\mu} &=& \expval{\qty(\pi_\bs)^2}^{\setminus i\mu}_t - \qty(\expval{\pi_\bs}^{\setminus i\mu}_t)^2 \notag\\
&=& \qty(\frac{\lambda}{\sqrt{M}}F_{\bs\mu})^{2} x_{i\mu}^2 \qty(\prod_{j\in\del\bs\setminus i} v_{j\mu\to\bs}^t - \prod_{j\in\del\bs\setminus i} \qty(m^t_{j\mu\to\bs})^2) \notag \nonumber \\
&& + \sum_{\nu(\neq\mu)} \qty(\frac{\lambda}{\sqrt{M}}F_{\bs\nu})^{2} \qty(\prod_{j\in\del\bs} v_{j\nu\to\bs}^t - \prod_{j\in\del\bs} \qty(m_{j\nu\to\bs}^t)^2)\\
&=& V_{\bs\to\mu}^t + \order{M^{-1}}, \\
O^{(3),t}_{\bs\to i\mu} &=& \expval{\qty(\pi_\bs)^3}^{\setminus i\mu}_t - 3 \expval{\pi_\bs}^{\setminus i\mu}_t \expval{\qty(\pi_\bs)^2}^{\setminus i\mu}_t + 2 \qty(\expval{\pi_\bs}^{\setminus i\mu}_t)^3 \notag\\
	&=& \sum_{\nu(\neq\mu)} \qty(\frac{\lambda}{\sqrt{M}}F_{\bs\nu})^{3} \qty( \prod_{j\in\del\bs} t_{j\nu\to\bs}^t - 3 \prod_{j\in\del\bs} m_{j\nu\to\bs}^t v_{j\nu\to\bs}^t + 2 \prod_{j\in\del\bs} \qty(m_{j\nu\to\bs}^t)^3 ) + \order{M^{-1}} \notag\\
	&&= \order{M^{-1/2}}.
\eeqn

Here we introduced moments of $x_{i\mu}$ associated with the distribution function $\phi_{i\mu\to\bs}^t(x_{i\mu})$ \eqref{eq:p-ary_problem_BP2},
\beqn
m_{i\mu\to\bs}^t &=& \int \dd{x_{i\mu}} \phi_{i\mu\to\bs}^t(x_{i\mu}) \; x_{i\mu} \label{eq:HRTD-rBP_def_m},\\
	v_{i\mu\to\bs}^t &=& \int \dd{x_{i\mu}} \phi_{i\mu\to\bs}^t(x_{i\mu}) \; x_{i\mu}^2, \label{eq:HRTD-rBP_def_v}\\
	t_{i\mu\to\bs}^t &=& \int \dd{x_{i\mu}} \phi_{i\mu\to\bs}^t(x_{i\mu}) \; x_{i\mu}^3.
        \eeqn
        We call these also as messages. Furthermore we also introduced
        \beqn
        \rho_{\bs\setminus i, \mu}^t &= \prod_{j\in\del\bs\setminus i} m_{j\mu\to\bs}^t, \label{eq:HRTD-rBP_def_rho}\\
	\omega_{\bs\to\mu}^t &= \sum_{\nu(\neq\mu)} \frac{\lambda}{\sqrt{M}}F_{\bs\nu} \prod_{j\in\del\bs} m_{j\nu\to\bs}^t,
	 \label{eq:HRTD-rBP_def_omega} \\
	V_{\bs\to\mu}^t &= \sum_{\nu(\neq\mu)} \qty(\frac{\lambda}{\sqrt{M}}F_{\bs\nu})^{2} \qty(\prod_{j\in\del\bs} v_{j\nu\to\bs}^t - \prod_{j\in\del\bs} \qty(m_{j\nu\to\bs}^t)^2). \label{eq:HRTD-rBP_def_V}
        \eeqn

        For $M\gg1$, we can drop terms of $\order{M^{-1/2}}$ so that the BP equation \eqref{eq:p-ary_problem_BP1} becomes,
        \beqn
&&  \tilde{\phi}_{\bs\to i\mu}^t(x_{i\mu}) \propto \int \dd{k} \dd{z} \exp\qty[ik\qty(\frac{\lambda}{\sqrt{M}}F_{\bs\mu} \rho_{\bs\setminus i,\mu}^t x_{i\mu} + \omega_{\bs\to\mu}^t -z)-\frac12 k^2 V_{\bs\to \mu}^{t}] \; P_\text{out}(y_\bs|z) \notag\\
	&\propto& \int\dd{z} \exp\qty[-\frac{1}{2V_{\bs\to\mu}^t} \qty(\frac{\lambda}{\sqrt{M}}F_{\bs\mu} \rho_{\bs\setminus i,\mu}^t x_{i\mu} + \omega_{\bs\to\mu}^t - z)^2]\; P_\text{out} \qty(y_\bs|z) \notag\\
 &\propto& \exp\qty[-\frac{1}{2V_{\bs\to\mu}^t} \qty(\frac{\lambda}{\sqrt{M}}F_{\bs\mu} \rho_{\bs\setminus i,\mu}^t)^2 x_{i\mu}^2 ]
 \expval{ e^{ \frac{1}{V_{\bs\to\mu}^t} \qty(z-\omega_{\bs\to\mu}^t) \frac{\lambda}{\sqrt{M}}F_{\bs\mu}
       \rho_{\bs\setminus i,\mu}^t x_{i\mu}}  ; \omega_{\bs\to\mu}^t ,y_\bs, V_{\bs\to\mu}^t}_\text{out}. \qquad \label{eq:p-ary_rBP_Culculate1}
 \eeqn
 Here we defined $\expval{\cdots;\omega,y,V}_\text{out}$,
 \beqn
 \expval{\cdots; \omega,y,V}_\text{out} &= \frac{ \int\dd{z} P_\text{out}\qty(y|z) \;\qty(\cdots)\; \exp(-\frac{1}{2V} \qty(z-\omega)^2 ) } { \int\dd{z} P_\text{out}\qty(y|z) \exp(-\frac{1}{2V} \qty(z-\omega)^2 )}.
 \eeqn
Again  we can evaluate $\expval{\cdots;\omega,y,V}_\text{out}$ in \eq{eq:p-ary_rBP_Culculate1} by a cumulant expansion,
\beqn
 &&\ln\expval{ e^{ \frac{1}{V} \qty(z-\omega) \frac{\lambda}{\sqrt{M}}X \rho x } ; \omega, y, V }_\text{out} \notag\\
&&\qquad = g_\text{out}\qty(\omega,y,V) \frac{\lambda}{\sqrt{M}}X  \rho x + \frac12 \qty( g_\text{out,II}(\omega,y,V) - g^2_\text{out}\qty(\omega,y,V) )
\qty(\frac{\lambda}{\sqrt{M}}X)^{2} \rho^2 x^2 \nonumber \\
&& + \order{M^{-3/2}}, \label{eq:p-ary_rBP_Culculate2}
\eeqn
where we defined  'output functions'
\beqn
g_\text{out}(\omega,y,V) &=& \frac{ \int\dd{z} P_\text{out}\qty(y|z) (z-\omega) e^{-(z-\omega)^2/2V} }{ V\int\dd{z} P_\text{out}\qty(y|z) e^{-(z-\omega)^2/2V} }, \label{eq:HRTD-rBP_def_outputFunction}\\
	 g_\text{out,II}\qty(\omega,y,V) &=& \frac{\int\dd{z} P_\text{out}\qty(y|z) \qty(z-\omega)^2 e^{-(z-\omega)^2/2V}}{V^2\int\dd{z} P_\text{out}(y|z) e^{-(z-\omega)^2/2V}},
\eeqn
which verify an identity
\beqn
g_\text{out,II}\qty(\omega,y,V) &= \frac1V + \pdv{\omega} g_\text{out}\qty(\omega,y,V) + g_\text{out}^2\qty(\omega,y,V). \label{eq:p-ary_rBP_Culculate3}
  \eeqn
  Using \eq{eq:p-ary_rBP_Culculate2} and \eq{eq:p-ary_rBP_Culculate3} we find
  \eq{eq:p-ary_rBP_Culculate1} becomes,
  \beqn
  \tilde{\phi}_{\bs\to i\mu}^t(x_{i\mu}) &\propto& \exp( -\frac12 A_{\bs\to i\mu}^t x_{i\mu}^2 + B^t_{\bs\to i\mu} x_{i\mu} ),
  \eeqn
  where we defined,
  \beqn
 A_{\bs\to i\mu}^t &=& - \qty(\frac{\lambda}{\sqrt{M}}F_{\bs\mu} \rho_{\bs\setminus i,\mu}^t)^2 \pdv{\omega} g_\text{out}\qty(\omega_{\bs\to\mu}^t,y_\bs,V_{\bs\to\mu}^t), \\
 B_{\bs\to i\mu}^t &=& \frac{\lambda}{\sqrt{M}}F_{\bs\mu} \rho_{\bs\setminus i,\mu}^t \, g_\text{out}\qty(\omega_{\bs\to\mu}^t, y_\bs, V_{\bs\to\mu}^t).
 \label{eq:HRTD-AB}
 \eeqn
 As shown in \eq{eq:HRTD-rBP_def_rho}, \eq{eq:HRTD-rBP_def_omega} and  \eq{eq:HRTD-rBP_def_V},
the parameters $\rho, \omega$ and $V$ 
depend on the messages $m_{i\mu\to\bs}^t$ and $v_{i\mu\to\bs}^t$ defined in \eq{eq:HRTD-rBP_def_m} and \eq{eq:HRTD-rBP_def_v}.

%%%%%%%%%%%%%%%%%%%%%%%%%%%%%%%%%%%%%%%%%%%%%%%%
 \subsubsection{The 2nd BP equation}
 \label{subsubsec-2ndBP}
 
Using the above result, we can write  the 2nd BP equation \eqref{eq:p-ary_problem_BP2} as,
 \beqn
\phi_{i\mu\to\bs}^{t+1}(x_{i\mu}) &\propto P_\text{pri.}(x_{i\mu}) \prod_{\bs'\in\del i\setminus\bs} \exp( -\frac12 A_{\bs'\to i\mu}^t x_{i\mu}^2 + B_{\bs'\to i\mu}^t x_{i\mu} ) \notag\\
&\propto P_\text{pri.}(x_{i\mu}) \exp( -\frac{1}{2\,\Sigma_{i\mu\to\bs}^{t+1}} \qty(x_{i\mu}-T_{i\mu\to\bs}^{t+1})^2 ).
\eeqn
Here we introduced
\beqn
\Sigma_{i\mu\to\bs}^{t+1} &= \frac1{\sum_{\bs'\in\del i\setminus\bs} A_{\bs'\to i\mu}^t}, \qquad T_{i\mu\to\bs}^{t+1} = \Sigma_{i\mu\to\bs}^{t+1} \times \sum_{\bs'\in\del i\setminus\bs} B_{\bs'\to i\mu}^t.  \label{eq:HRTD-del_Sigma_T}
\eeqn

%%%%%%%%%%%%%%%%%%%%%%%%%%%%%%%%%%%%%%%%%%%%%%%%
 \subsubsection{Back to the 1st BP equation}

 When we analyzed the 1st BP equation in sec. \ref{subsubsec-1stBP}, we found
 messages $m_{i\mu\to\bs}^t$ and $v_{i\mu\to\bs}^t$ defined in \eq{eq:HRTD-rBP_def_m} and \eq{eq:HRTD-rBP_def_v},
 are needed to be evaluated using $\phi_{i\mu\to\bs}^t$.
 Using the results in sec~\ref{subsubsec-2ndBP} we can write these using
$\Sigma_{i\mu\to\bs}^{t}$ and $T_{i\mu\to\bs}^{t}$ as,
\beqn
m_{i\mu\to\bs}^t &=& \int\dd{x_{i\mu}} \phi_{i\mu\to\bs}^t(x_{i\mu})\;x_{i\mu} = f_\text{input}\qty(\Sigma_{i\mu\to\bs}^{t},T_{i\mu\to\bs}^{t}) \label{eq:HRTD-rBP_input_m},
\\
       v_{i\mu\to\bs}^t &=& \int\dd{x_{i\mu}} \phi_{i\mu\to\bs}^t(x_{i\mu})\;x_{i\mu}^2 = f_\text{input,II}\qty(\Sigma_{i\mu\to\bs}^{t},T_{i\mu\to\bs}^{t}) \notag\\
       &=& \Sigma_{i\mu\to\bs}^{t} \times \pdv{T} f_\text{input}\qty(\Sigma_{i\mu\to\bs}^{t},T_{i\mu\to\bs}^{t}) + f_\text{input}^2\qty(\Sigma_{i\mu\to\bs}^{t},T_{i\mu\to\bs}^{t}), \label{eq:HRTD-rBP_input_v}
       \eeqn
       where we defined
       \beqn
 f_\text{input}\qty(\Sigma,T) &=& \frac{ \int\dd{x} P_\text{pri.}(x) \,x\, e^{-(x-T)^2/2\Sigma} }{ \int\dd{x} P_\text{pri.}(x) e^{-(x-T)^2/2\Sigma} }, \label{eq:HRTD-rBP_def_inputFunction}\\
 f_\text{input,II}\qty(\Sigma,T) &=& \frac{ \int\dd{x} P_\text{pri.}(x) \,x^2\, e^{-(x-T)^2/2\Sigma} }{ \int\dd{x} P_\text{pri.}(x) e^{-(x-T)^2/2\Sigma}}.
 \eeqn
 Note that an identity
 \beqn
 f_\text{input,II}\qty(\Sigma,T) &= \Sigma \times \pdv{T} f_\text{input}\qty(\Sigma,T) + f_\text{input}^2\qty(\Sigma,T)
 \eeqn
 holds.

%%%%%%%%%%%%%%%%%%%%%%%%%%%%%%%%%%%%%%%%%%%%%%%%
 \subsubsection{Marginal distribution and the moments}

 The marginal distribution is obtained as,
 \beqn
  \phi_{i\mu}^{t+1}(x_{i\mu}) &=& \frac{1}{z^{t+1}_{i\mu}} P_\text{pri.}(x_{i\mu}) \prod_{\bs'\in\del i} \tilde{\phi}^t_{\bs'\to i\mu}(x_{i\mu}) \nonumber \\
 && \propto P_\text{pri.}(x_{i\mu}) \exp( -\frac{1}{2\,\Sigma_{i\mu}^{t+1}} \qty(x_{i\mu}-T_{i\mu}^{t+1})^2 ), \label{eq:HRTD-rBP_marginals}
 \eeqn
with
% \beqn
% \Sigma_{i\mu}^{t+1} &= \frac1{\sum_{\bs\in\del i} A_{\bs\to i\mu}^t}, \qquad T_{i\mu}^{t+1} = \Sigma_{i\mu}^{t+1} \times \sum_{\bs\in\del i} B_{\bs\to i\mu}^t \label{eq:HRTD-def_Sigma_T_full}
% \eeqn
% with
%   \beqn
% A_{\bs\to i\mu}^t &=& - \qty(\frac{\lambda}{\sqrt{M}}F_{\bs\mu} \rho_{\bs\setminus i,\mu}^t)^2 \pdv{\omega} g_\text{out}\qty(\omega_{\bs\to\mu}^t,y_\bs,V_{\bs\to\mu}^t), \\
% B_{\bs\to i\mu}^t &=& \frac{\lambda}{\sqrt{M}}F_{\bs\mu} \rho_{\bs\setminus i,\mu}^t \, g_\text{out}\qty(\omega_{\bs\to\mu}^t, y_\bs, V_{\bs\to\mu}^t)
%% \label{eq:HRTD-AB}
% \eeqn
%
  \beqn
  (\Sigma_{i\mu}^{t+1})^{-1} &= - \sum_{\bs\in\del i}
  \qty(\frac{\lambda}{\sqrt{M}}F_{\bs\mu} \rho_{\bs\setminus i,\mu}^t)^2 \pdv{\omega} g_\text{out}\qty(\omega_{\bs\to\mu}^t,y_\bs,V_{\bs\to\mu}^t),
  \label{eq:HRTD-def_Sigma_full_summary}  
  \\
  \frac{T_{i\mu}^{t+1}}{\Sigma_{i\mu}^{t+1}} & =   \sum_{\bs\in\del i} 
\frac{\lambda}{\sqrt{M}}F_{\bs\mu} \rho_{\bs\setminus i,\mu}^t \, g_\text{out}\qty(\omega_{\bs\to\mu}^t, y_\bs, V_{\bs\to\mu}^t).
  \label{eq:HRTD-def_Sigma_T_full_summary}
 \eeqn

The 1st and 2nd moments of the 
marginal distribution function \eqref{eq:HRTD-rBP_marginals} are obtained as,
 \beqn
 m_{i\mu}^{t+1} &= \int \dd{x_{i\mu}} \phi_{i\mu}^{t+1}(x_{i\mu}) \; x_{i\mu}= f_\text{input}\qty(\Sigma_{i\mu}^{t+1},T_{i\mu}^{t+1}), \label{eq:HRTD-rBP_def_m_full}\\
	v_{i\mu}^{t+1} &= \int \dd{x_{i\mu}} \phi_{i\mu}^{t+1}(x_{i\mu}) \; x_{i\mu}^2= f_\text{input,II}\qty(\Sigma_{i\mu}^{t+1},T_{i\mu}^{t+1}). \label{eq:HRTD-rBP_def_v_full} 
\eeqn

Collecting the results we obtain the r-BP algorithm~\ref{alg:p-ary_r-BP}.
%%%%%%%%%%%%%%%%%%%%%
%%--Alg
%\begin{figure}[!t]
  \begin{algorithm}[H]
    \caption{r-BP algorithm}
    \label{alg:p-ary_r-BP}
    \begin{algorithmic}[1]
      \Require $\mathbb{G}=(\Vset,\Eset)$, matrix $F$, input function $f_\text{input}$, output function $g_\text{out}$
      \Initialize ~Initialize messages (see sec.~\ref{appendix-mp-detail})
      $m_{i\mu\to\bs}^{t}$ and $v_{i\mu\to\bs}^{t}$.
      \Repeat
      \State Update $\omega_{\bs\to\mu}^t$ and $V_{\bs\to\mu}^t$.
      \beq
      \begin{aligned}
      \omega_{\bs\to\mu}^t &\leftarrow \sum_{\nu(\neq\mu)} \frac{\lambda}{\sqrt{M}}F_{\bs\nu} \prod_{j\in\del\bs} m_{j\nu\to\bs}^t,  \\
      V_{\bs\to\mu}^t &\leftarrow \sum_{\nu(\neq\mu)} \qty(\frac{\lambda}{\sqrt{M}}F_{\bs\nu})^{2} \qty(\prod_{j\in\del\bs} v_{j\nu\to\bs}^t - \prod_{j\in\del\bs} \qty(m_{j\nu\to\bs}^t)^2).
      \end{aligned}      
      \tag{Alg1.1}
      \eeq
      \State Update the values of the output functions.
        \beq
        g_{\bs\to\mu}^t \leftarrow g_\text{out}\qty(\omega_{\bs\to\mu}^t,y_\bs,V_{\bs\to\mu}^t), \qquad \del_\omega g^t_{\bs\to\mu} \leftarrow \pdv{\omega} g_\text{out}\qty(\omega_{\bs\to\mu}^t,y_\bs,V_{\bs\to\mu}^t).
        \tag{Alg1.2}
        \eeq
        \State Update $\Sigma_{i\mu\to\bs}^{t+1}$ and $T_{i\mu\to\bs}^{t+1}$.
\beq
\begin{aligned} \frac{1}{\Sigma_{i\mu\to\bs}^{t+1}} &\leftarrow \sum_{\bs'\in\del i\setminus\bs} \qty(\frac{\lambda}{\sqrt{M}}F_{\bs'\mu})^{2} \qty( - \del_\omega g^t_{\bs'\to\mu} ) \prod_{j\in\del\bs'\setminus i} \qty(m_{j\mu\to\bs'}^t)^2 , \\
  \frac{T_{i\mu\to\bs}^{t+1} }{\Sigma_{i\mu\to\bs}^{t+1}} &\leftarrow \sum_{\bs'\in\del i\setminus\bs} \frac{\lambda}{\sqrt{M}}F_{\bs'\mu} g_{\bs'\to\mu}^t \prod_{j\in\del\bs'\setminus i} m_{j\mu\to\bs'}^t. \end{aligned}
\tag{Alg1.3}
\label{eq:Alg1_Sigma_T}
\eeq
			\State Update messages $m_{i\mu\to\bs}^{t+1}$ and $v_{i\mu\to\bs}^{t+1}$.
                        \beq
                        m_{i\mu\to\bs}^{t+1} \leftarrow f_\text{input}\qty(\Sigma_{i\mu\to\bs}^{t+1},T_{i\mu\to\bs}^{t+1}), \qquad v_{i\mu\to\bs}^{t+1} \leftarrow f_\text{input,II}\qty(\Sigma_{i\mu\to\bs}^{t+1},T_{i\mu\to\bs}^{t+1}).
                        \tag{Alg1.4}
                        \eeq
                        \Until $m_{i\mu\to\bs}^{t}$ and $v_{i\mu\to\bs}^{t}$ converge.
                        		\Ensure { Messages $m_{i\mu}$, $v_{i\mu}$ which are the average and variance of the variables.
                                         \beq
                                          \begin{aligned} \frac{1}{\Sigma_{i\mu}} & \leftarrow \sum_{\bs\in\del i} \qty(\frac{\lambda}{\sqrt{M}}F_{\bs\mu})^{2} \qty( - \del_\omega g_{\bs\to\mu} ) \prod_{j\in\del\bs\setminus i} \qty(m_{j\mu\to\bs})^2, \nonumber \\
                      \frac{T_{i\mu} }{\Sigma_{i\mu}} & \leftarrow \sum_{\bs\in\del i}
                                            \frac{\lambda}{\sqrt{M}}F_{\bs\mu} g_{\bs\to\mu} \prod_{j\in\del\bs\setminus i} m_{j\mu\to\bs}, \\
                &  m_{i\mu} \leftarrow f_\text{input}\qty(\Sigma_{i\mu},T_{i\mu}), \qquad v_{i\mu} \leftarrow f_\text{input,II}\qty(\Sigma_{i\mu},T_{i\mu}). \end{aligned}
                                          \tag{Alg1.5}
                                          \eeq
			}
    \end{algorithmic}
  \end{algorithm}
%%%%%%%%%%%%%%%%%%%%%

%%%%%%%%%%%%%%%%%%%%%%%%%%%%%%%%%%%%%%%%%%%%%%%%
%%%%%%%%%%%%%%%%%%%%%%%%%%%%%%%%%%%%%%%%%%%%%%%%
  \subsection{From r-BP to G-AMP}
\label{sec-G-AMP}
  
In the limit $M \gg 1$, we can further simplify the r-BP obtained in the previous section to derive AMP (approximate message passing) algorithm.
While BP defines recursion formulae for 'messages' of the form $A_{\bs \to i\mu}$, $B_{i\mu \to \bs}$ between the factor and variable nodes,
AMP defines a set of recursion formulae for completely local quantities like $A_{\bs}$, $B_{i\mu}$. The AMP equations are derived from the r-BP ones through a perturbative manner which is essentially the same technique to derive the so called TAP (Thouless-Anderson-Palmer) equations in mean-field spin-glass models with global couplings.

From \eq{eq:HRTD-rBP_def_omega} and \eqref{eq:HRTD-rBP_def_V} we find,
    \beqn
      \omega_{\bs\to\mu}^t &=& \omega_\bs^t - \frac{\lambda}{\sqrt{M}}F_{\bs\mu} \prod_{j\in\del\bs} m_{j\mu\to\bs}^t, \\
      V_{\bs\to\mu}^t &=& V_\bs^t + \order{M^{-1}},
      \label{eq:HRTD-AMP-expansion-omega-and-V}
      \eeqn
where we introduced,
  \beqn
  \omega_\bs^t = \sum_{\nu=1}^M \frac{\lambda}{\sqrt{M}}F_{\bs\nu} \prod_{j\in\del\bs} m_{j\nu\to\bs}^t, \qquad V_\bs^t = \sum_{\nu=1}^M
  \qty(\frac{\lambda}{\sqrt{M}}F_{\bs\nu})^2
  \qty(\prod_{j\in\del\bs} v_{j\nu\to\bs}^t - \prod_{j\in\del\bs} \qty(m_{j\nu\to\bs}^t)^2).
\qquad
  \label{eq:HRTD-def_omega_V_full}
    \eeqn

    From \eq{eq:HRTD-del_Sigma_T} and \eq{eq:HRTD-AB}, we find,
    \beqn 
    \Sigma_{i\mu\to\bs}^{t+1} &= \Sigma_{i\mu}^{t+1} + \order{M^{-1}}, \\
    T_{i\mu\to\bs}^{t+1} &= T_{i\mu}^{t+1} - \Sigma_{i\mu}^{t+1} \frac{\lambda}{\sqrt{M}}F_{\bs\mu} g_\text{out}\qty(\omega_{\bs}^t, y_\bs, V_{\bs}^t) \prod_{j\in\del\bs\setminus i}m_{j\mu\to\bs}^t + \order{M^{-1}},
    \label{eq:HRTD-AMP-expansion-Sigma-and-T}
    \eeqn
    where we used an expansion of the output function \eqref{eq:HRTD-rBP_def_outputFunction},
    \beqn
    g_\text{out}\qty(\omega_{\bs\to\mu}^t,y_\bs,V_{\bs\to\mu}^t) &= g_\text{out}\qty(\omega_\bs^t,y_\bs,V_\bs^t)
    - \frac{\lambda}{\sqrt{M}}F_{\bs\mu} \qty(\prod_{j\in\del\bs} m_{j\mu\to\bs}^t) \pdv{\omega}g_\text{out}\qty(\omega_\bs^t,y_\bs,V_\bs^t) \nonumber \\
    + \order{M^{-1}},
    \eeqn
    which follows from \eq{eq:HRTD-AMP-expansion-omega-and-V}.

    From \eq{eq:HRTD-rBP_input_m} and \eqref{eq:HRTD-rBP_input_v} we find,
    \beqn
    m_{i\mu\to\bs}^{t+1}
    &=& m_{i\mu}^{t+1} \nonumber \\
    && - \pdv{T} f_\text{input}\qty(\Sigma_{i\mu}^{t+1},T_{i\mu}^{t+1}) \; \Sigma_{i\mu}^{t+1}
    \frac{\lambda}{\sqrt{M}}F_{\bs\mu} g_\text{out}\qty(\omega_\bs^{t},y_\bs,V_\bs^{t}) \prod_{j\in\del\bs\setminus i}m_{j\mu\to\bs}^t + \order{M^{-1}} \notag\\
    &=& m_{i\mu}^{t+1} \nonumber \\
    &&- \qty( f_\text{input,II}\qty(\Sigma_{i\mu}^{t+1},T_{i\mu}^{t+1}) - f_\text{input}^2\qty(\Sigma_{i\mu}^{t+1},T_{i\mu}^{t+1}) ) \frac{\lambda}{\sqrt{M}}F_{\bs\mu} g_\text{out}\qty(\omega_\bs^{t},y_\bs,V_\bs^{t}) \prod_{j\in\del\bs\setminus i}m_{j\mu\to\bs}^t \nonumber \\
    &&+ \order{M^{-1}} \notag\\
		&=& m_{i\mu}^{t+1} - \qty( v_{i\mu}^{t+1} - \qty(m_{i\mu}^{t+1})^2 ) \frac{\lambda}{\sqrt{M}}F_{\bs\mu} g_\text{out}\qty(\omega_\bs^{t}, y_\bs, V_\bs^{t}) \prod_{j\in\del\bs\setminus i}m_{j\mu}^t + \order{M^{-1}}, \label{eq:HRTD-GAMP_m_approximation}\\
    v_{i\mu\to\bs}^{t+1} &=& v_{i\mu}^{t+1} - \pdv{T} f_\text{input,II}\qty(\Sigma_{i\mu}^{t+1},T_{i\mu}^{t+1}) \; \Sigma_{i\mu}^{t+1} \frac{\lambda}{\sqrt{M}}F_{\bs\mu} g_\text{out}\qty(\omega_\bs^{t},y_\bs,V_\bs^{t}) \prod_{j\in\del\bs\setminus i}m_{j\mu}^t \nonumber \\
    &&+ \order{M^{-1}},
    \eeqn
    where we used an expansion of the input function \eqref{eq:HRTD-rBP_def_inputFunction},
    \beqn
f_\text{input}\qty(\Sigma_{i\mu\to\bs}^t,T_{i\mu\to\bs}^t) &=& f_\text{input}\qty(\Sigma_{i\mu}^t,T_{i\mu}^t) \notag\\
&&- \pdv{T}f_\text{input}\qty(\Sigma_{i\mu}^t,T_{i\mu}^t)
\; \Sigma_{i\mu}^t \frac{\lambda}{\sqrt{M}}F_{\bs\mu}  g_\text{out}\qty(\omega_\bs^{t-1},y_\bs,V_\bs^{t-1}) \prod_{j\in\del\bs\setminus i}m_{j\mu\to\bs}^{t-1} \nonumber \\
&& + \order{M^{-1}},
\eeqn
which follows from  \eq{eq:HRTD-AMP-expansion-Sigma-and-T}.

Now collecting the above results we obtain the following, dropping $O(M^{-1})$ corrections,
\beqn
\omega_\bs^t &=& \sum_{\nu=1}^M \frac{\lambda}{\sqrt{M}}F_{\bs\nu} \left (
\prod_{j\in\del\bs} m_{j\nu}^t \right. \nonumber \\
&& \left. - \frac{\lambda}{\sqrt{M}}F_{\bs\nu} g_\text{out}\qty(\omega_\bs^{t-1},y_\bs,V_\bs^{t-1}) \sum_{j\in\del\bs} \qty(v_{j\nu}^t - \qty(m_{j\nu}^t)^2) \prod_{k\in\del\bs\setminus j} m_{k\nu}^t m_{k\nu}^{t-1}
\right), \\
	V_\bs^t &=& \sum_{\nu=1}^M \qty(\frac{\lambda}{\sqrt{M}}F_{\bs\nu})^2 \qty(\prod_{j\in\del\bs} v_{j\nu}^t - \prod_{j\in\del\bs} \qty(m_{j\nu}^t)^2), \\
\Sigma_{i\mu}^{t+1} &=& -\frac{1}{\sum_{\bs\in\del i} \qty(\frac{\lambda}{\sqrt{M}}F_{\bs\mu})^2 \pdv{\omega} g_\text{out}\qty(\omega_{\bs}^t,y_\bs,V_\bs^t) \prod_{j\in\del\bs\setminus i} \qty(m_{j\mu}^t)^2 }, \\
\frac{T_{i\mu}^{t+1}}{\Sigma_{i\mu}^{t+1}} &=& \frac{m_{i\mu}^t}{\Sigma_{i\mu}^{t+1}} + \sum_{\bs\del i} \frac{\lambda}{\sqrt{M}}F_{\bs\mu} g_\text{out}\qty(\omega_\bs^t, y_\bs, V_\bs^t) \left( \prod_{j\in\del\bs} m_{j\mu}^t \right. \notag\\
&&\hspace{2cm} \left. - \frac{\lambda}{\sqrt{M}}F_{\bs\mu} g_\text{out}\qty(\omega_\bs^{t-1},y_\bs,V_\bs^{t-1}) m_{i\mu}^{t-1} \sum_{j\in\del\bs} \qty(v_{j\mu}^t - \qty(m_{j\mu}^t)^2) \prod_{k\in\del\bs\setminus i,j} m_{k\mu}^t m_{k\mu}^{t-1} \right). \qquad
\eeqn
This set of equations gives the G-AMP algorithm \ref{alg:p-ary_G-AMP}.
%%%%%%%%%%%%%%%%%%%%%
%%--Alg
%\begin{figure}[!t]
  \begin{algorithm}[H]
    \caption{G-AMP algorithm}
    \label{alg:p-ary_G-AMP}
    \begin{algorithmic}[1]
      \Require graph $\mathbb{G}=(\Vset,\Eset)$, matrix $F$, input function $f_\text{input}$, output function $g_\text{out}$.
      \Initialize~Initialize messages
 (see sec.~\ref{appendix-mp-detail})
      $m_{i\mu}^t$, $m_{i\mu}^{t-1}$, $v_{i\mu}^t$ and $g_\bs^{t-1}$.
      \Repeat
      \State Update $\omega_{\bs\to\mu}^t$ and $V_{\bs\to\mu}^t$.
      \beq
      	\begin{aligned} \omega_\bs^t &\leftarrow \sum_{\nu=1}^M \frac{\lambda}{\sqrt{M}}F_{\bs\nu} \qty( \prod_{j\in\del\bs} m_{j\nu}^t - \frac{\lambda}{\sqrt{M}}F_{\bs\nu} g_\bs^{t-1} \sum_{j\in\del\bs} \qty(v_{j\nu}^t-\qty(m_{j\nu}^t)^2) \prod_{k\in\del\bs\setminus j} m_{k\nu}^t m_{k\nu}^{t-1} ), \\
	  V_\bs^t &\leftarrow \sum_{\nu=1}^M \qty(\frac{\lambda}{\sqrt{M}}F_{\bs\nu})^2 \qty(\prod_{j\in\del\bs} v_{j\nu}^t - \prod_{j\in\del\bs} \qty(m_{j\nu}^t)^2). \end{aligned} \tag{Alg2.1} \label{eq:Alg2.1}
        \eeq
      \State Update the values of the output function
      \beq
      g_{\bs}^t \leftarrow g_\text{out}\qty(\omega_{\bs}^t,y_\bs,V_{\bs}^t), \qquad \del_\omega g^t_{\bs} \leftarrow \pdv{\omega} g_\text{out}\qty(\omega_{\bs}^t,y_\bs,V_{\bs}^t). \tag{Alg2.2}
      \eeq
      \State Update $\Sigma_{i\mu\to\bs}^t$ and $T_{i\mu\to\bs}^t$.
      \beq
      \begin{aligned} \frac{1}{\Sigma_{i\mu}^{t+1}} &\leftarrow \sum_{\bs\in\del i}
        \qty(\frac{\lambda}{\sqrt{M}}F_{\bs\mu})^2 \qty(- \del_\omega g_\bs^t ) \prod_{j\in\del\bs\setminus i} \qty(m_{j\mu}^t)^2, \\
						\frac{T_{i\mu}^{t+1}}{\Sigma_{i\mu}^{t+1}} &\leftarrow \frac{m_{i\mu}^t}{\Sigma_{i\mu}^{t+1}} + \sum_{\bs\in\del i} \frac{\lambda}{\sqrt{M}}F_{\bs\mu} g_\bs^t \left( \prod_{k\in\del\bs\setminus i} m_{k\mu}^t \right. \notag\\
						&\hspace{3cm} \left. - \frac{\lambda}{\sqrt{M}}F_{\bs\mu} g_\bs^{t-1} m_{i\mu}^{t-1} \sum_{j\in\del\bs\setminus i} \qty(v_{j\mu}^t-\qty(m_{j\mu}^t)^2) \prod_{k\in\del\bs\setminus i,j} m_{k\mu}^t m_{k\mu}^{t-1} \right). \end{aligned}\tag{Alg2.3}
        \eeq
	\State Update messages $m_{i\mu}^t$ and  $v_{i\mu}^t$.
        \beq
	m_{i\mu}^{t+1} \leftarrow f_\text{input}\qty(\Sigma_{i\mu}^{t+1},T_{i\mu}^{t+1}), \qquad v_{i\mu}^{t+1} \leftarrow f_\text{input,II}\qty(\Sigma_{i\mu}^{t+1},T_{i\mu}^{t+1}). \tag{Alg2.4}
        \eeq
      \Until $m_{i\mu}^t$ and $v_{i\mu}^t$ converge.
      \Ensure { Messages $m_{i\mu}$, $v_{i\mu}$ which are the average and variance of the variables.}
    \end{algorithmic}
  \end{algorithm}
%%%%%%%%%%%%%%%%%%%%%

%%%%%%%%%%%%%%%%%%%%%%%%%%%%%%%%%%%%%%%%%%%%%%%%
%%%%%%%%%%%%%%%%%%%%%%%%%%%%%%%%%%%%%%%%%%%%%%%%
  \subsection{Computational costs}
  \label{sec-computational-cost-BP}
  
In r-BP we have $NM \times c=\alpha NM^{2}$ messages and in
\eq{eq:Alg1_Sigma_T} we do summation over $c=\alpha M$ terms so that
the computational cost of r-BP is order $O(NMc^{2})=O(NM^{3})$.
On the other hand, in  G-AMP, we have $NM$ messages
and the computational cost is reduced to $O(NMc)=O(NM^{2})$. In both cases, the computational cost does not depend on $p$, which is in contrast to the fully connected case.
%Note that in fully connected systems $c\propto N^{p-1}$ while $c=\alpha M$ with $\alpha=O(1)$ independently of $p$ in our case.

%%%%%%%%%%%%%%%%%%%%%%%%%%%%%%%%%%%%%%%%%%%%%%%%
%%%%%%%%%%%%%%%%%%%%%%%%%%%%%%%%%%%%%%%%%%%%%%%%
\subsection{Some prior distributions}
  \label{sec-prior-BP}

%%%%%%%%%%%%%%%%%%%%%%%%%%%%%%%%%%%%%%%%%%%%%%%%
 \subsubsection{Ising model}
    \label{subsubsec-ising-model}

    Let us consider the case that the ground truth data takes Ising values
\eq{eq-prior-ising}.
%    , i.~.e. $x_{i\mu}= \pm 1$
%  in an unbiased manner. We have,
%  \be
%P_\text{pri.} (x)  = \frac12 \delta\qty(x-1) + \frac12 \delta\qty(x+1) 
%\ee
In this case we find the input functions \eq{eq:HRTD-rBP_def_inputFunction} as,
\beq
f_\text{input}\qty(\Sigma,T) = \tanh(\frac{T}{\Sigma}), \qquad f_\text{input,II}\qty(\Sigma,T) = 1. \label{eq:HRTD-input_Ising} 
\eeq
  
%%%%%%%%%%%%%%%%%%%%%%%%%%%%%%%%%%%%%%%%%%%%%%%%
\subsubsection{Gaussian model}
  \label{subsubsec-gaussian-model}

  We also consider the case that data takes real values which obey Gaussian distribution
  \eq{eq-prior-gaussian}.
%\be
%P_\text{pri.}(x) &= \frac{1}{\sqrt{2\pi}} e^{-\frac12 x^2}
%\ee
In this case we find the input functions as,
\beqn
f_\text{input}\qty(\Sigma,T) &= \frac{T}{\Sigma+1}, \qquad f_\text{input,II}\qty(\Sigma,T) = \frac{\Sigma}{\Sigma+1} + \frac{T^2}{\qty(\Sigma+1)^2}. \label{eq:HRTD-input_Spherical}
\eeqn

%%%%%%%%%%%%%%%%%%%%%%%%%%%%%%%%%%%%%%%%%%%%%%%%
%%%%%%%%%%%%%%%%%%%%%%%%%%%%%%%%%%%%%%%%%%%%%%%%
\subsection{Some likelihood functions}
  \label{sec-likelihood-BP}
%%%%%%%%%%%%%%%%%%%%%%%%%%%%%%%%%%%%%%%%%%%%%%%%
 \subsubsection{Additive  Noise}
  \label{subsubsec-additive-noise}

  Let us consider the case that the observed value $y$ contains
  additive noise $w$ to $\pi_*$  (see \eq{eq-additive-noise-output}),
  \beq
  y = \pi_* + w.
  \eeq
  This  means the likelihood function is
  \beqn
  P_\text{out}\qty(y|\pi) &= \int dw W(w)\delta(y-(\pi+w ))
=W(y-\pi), %\frac{1}{\sqrt{2\pi\Delta_*^2}} e^{-\frac{1}{2\Delta_*^2} \qty(y-\pi)^2}
  \eeqn
  with $W(w)$ being the distribution function of the noise.
This yields the output functions \eq{eq:HRTD-rBP_def_outputFunction}
\beqn
g_\text{out}\qty(\omega,y,V) &= \frac{
\int \frac{\dd z}{\sqrt{2\pi}}e^{-z^{2}/2}
 (-\frac{\partial}{\partial y})W(y-\omega-\sqrt{V}z)}{\int \frac{\dd z}{\sqrt{2\pi}}e^{-z^{2}/2} W(\omega-y-\sqrt{V}z)}.
%\nonumber \\
%-\pdv{\omega} g_\text{out}\qty(\omega,y,V)  &= \frac{1}{V+\Delta_*^2/\beta}
\label{eq:HRTD-output_AWGN}
\eeqn

In particular, in the case of Gaussian noise
\beq
W(w)=\mathcal{N}(0,\Delta_*^2),
\eeq
we find
\beqn
g_\text{out}\qty(\omega,y,V) &= \frac{y-\omega}{V+\Delta_*^2}, \qquad -\pdv{\omega} g_\text{out}\qty(\omega,y,V) = \frac{1}{V+\Delta_*^2}. \label{eq:HRTD-output_AWGN_Gaussian}
\eeqn

%%%%%%%%%%%%%%%%%%%%%%%%%%%%%%%%%%%%%%%%%%%%%%%%
  \subsubsection{Sign output}

  We also consider that the observed value $y$ is just the sign of $\pi_*$,
   (see \eq{eq-sign-output})
  \beq
  y = {\rm sgn}\qty(\pi_*).
  \eeq
  This means
  %\siki{
  \beq
  P_\text{out}(y|\pi) = %\Theta_\text{H}(y\pi)
  %  \theta(y\pi)
  \delta(y-1)\theta(\pi)+  \delta(y+1)\theta(-\pi).
  \eeq
  This yields the output function \eq{eq:HRTD-rBP_def_outputFunction}
  \beq
  g_\text{out}\qty(\omega,y,V)=\sum_{\sigma=-1,1}\delta(y-\sigma)
\tilde{g}_\text{out}\qty(\omega,\sigma,V),
%  \frac{\sigma}{\sqrt{2\pi V} H\qty(-\sigma\frac{\omega}{\sqrt{V}})} \exp(-\frac1{2V} \omega^2)
\qquad \tilde{g}_\text{out}\qty(\omega,\sigma,V)= 
  \left. \frac{-\sigma}{\sqrt{V}}\frac{H'(x)}{H(x)}\right |_{x=-\sigma\omega/\sqrt{V}}.
  %,
%  \qquad -\pdv{\omega}g_\text{out}\qty(\omega,y,V) = g_\text{out}^2\qty(\omega,y,V) + \frac{\omega}{V} g_\text{out}\qty(\omega,y,V) 
\label{eq:HRTD-outputFunction_sign}
\eeq
Here we introduced
\beq
H(x) = \int_x^\infty \frac{\dd{t}}{\sqrt{2\pi}} \; e^{-\frac12 t^2} = \frac12 \text{erfc}\qty(\frac{x}{\sqrt{2}}),
\label{eq:def-H-function}
\eeq
with $\text{erfc}(x)=(2/\sqrt{\pi})\int_{x}^{\infty}dy e^{-y^{2}}$ being the complementary error function.
%Note that
%\beq
%g_\text{out}\qty(\omega,-y,V)=-g_\text{out}\qty(-\omega,y,V)
%\eeq
%holds.

%%%%%%%%%%%%%%%%%%%%%%%%%%%%%%%%%%%%%%%%%%%%%%%%
%%%%%%%%%%%%%%%%%%%%%%%%%%%%%%%%%%%%%%%%%%%%%%%%
\subsection{State Evolution}
\label{sec-SE-equations}
In order to analyze the performance of the algorithms explained above at macroscopic scales $N, M \to \infty$ we now turn to analyze the so called state evolution (SE). The equations governing SE will be compared with the equations of states obtained by the replica approach.

Let us introduce the order parameters which characterize the macroscopic behavior of the algorithms,
\beqn
m^t &=
%\mathbb{E}\qty[x_{*,i\mu} m_{i\mu}^t]=
\frac{1}{NM} \sum_{i=1}^N \sum_{\mu=1}^M x_{*,i\mu} m_{i\mu}^t, \label{eq:HRTD-SE_definition_m} \\
q^t &=
%\mathbb{E}\qty[(m_{i\mu}^t)^2]=
\frac{1}{NM} \sum_{i=1}^N \sum_{\mu=1}^M \qty(m_{i\mu}^t)^2,\label{eq:HRTD-SE_definition_q} \\
Q^t &=
%\mathbb{E}\qty[v_{i\mu}^t] =
\frac{1}{NM} \sum_{i=1}^N \sum_{\mu=1}^M v_{i\mu}^t, \label{eq:HRTD-SE_definition_Q}
\eeqn
where $m_{i\mu}^{t}$ and $v_{i\mu}^{t}$ are given by
(see \eq{eq:HRTD-rBP_def_v_full} ),
 \beqn
 m_{i\mu}^{t} &=  f_\text{input}\qty(\Sigma_{i\mu}^{t},T_{i\mu}^{t}),\\
	v_{i\mu}^{t} &=  f_\text{input,II}\qty(\Sigma_{i\mu}^{t},T_{i\mu}^{t}).
\eeqn

The individual microscopic variables $m_{i\mu}^t$ and $v_{i\mu}^t$
fluctuate depending on realizations of the quenched random variables
$F_{\bs\mu}$, $x_{*,i\mu}$ and $w_\bs$
and the realization of the graph $G$. However we expect the macroscopic
observables $m^t$, $q^{t}$ and $Q^{t}$ are self-averaging such that they become independent of the
realizations of the quenched random variables in the thermodynamic limit $N, M \to \infty$.
Then we can evaluate them as,
\beqn
m^{t}=\mathbb{E}_{y} [m_{i\mu}^{t}]=\mathbb{E}_{y} [ x_{*,i\mu} f_\text{input}\qty(\Sigma_{i\mu}^{t},T_{i\mu}^{t})],  \label{eq:HRTD-SE_eval_m} \\
q^{t}=\mathbb{E}_{y} [(m_{i\mu}^{t})^{2}]=\mathbb{E}_{y} [ f^{2}_\text{input}\qty(\Sigma_{i\mu}^{t},T_{i\mu}^{t})], \label{eq:HRTD-SE_eval_q} \\
Q^{t}=\mathbb{E}_{y} [v_{i\mu}^{t}]=\mathbb{E}_{y} [ f_\text{input,II}\qty(\Sigma_{i\mu}^{t},T_{i\mu}^{t})]. \label{eq:HRTD-SE_eval_Q}
\eeqn

To evaluate the order parameters we need $\Sigma_{i\mu}^{t}$
given by \eq{eq:HRTD-def_Sigma_full_summary}
and $T_{i\mu}^{t}$ given by  \eq{eq:HRTD-def_Sigma_T_full_summary}.
Both involve $g_\text{out} \qty(\omega_{\bs\to\mu}^t,y_{\bs},V_{\bs\to\mu}^t)$.
So let us first consider $V_{\bs\to\mu}^t$ defined in \eq{eq:HRTD-rBP_def_V}.
Since it is given by a summation over a large
number of terms in $M \to \infty$ limit we can evaluate it as,
\beqn
V_{\bs\to\mu}^t &= \sum_{\nu(\neq\mu)} 
\left(\frac{\lambda}{\sqrt{M}}F_{\bs\nu}\right)^2
 \qty( \prod_{j\in\del\bs} v_{j\nu\to\bs}^t - \prod_{j\in\del\bs} \qty(m_{j\nu\to\bs}^t)^2 ) \notag\\
	&= M \mathbb{E}_{F}\qty[\left(\frac{\lambda}{\sqrt{M}}F_{\bs\nu}\right)^2] \biggl( \Bigl(\mathbb{E}_{y}\qty[v_{i\nu}^t]\Bigr)^p - \Bigl(\mathbb{E}_{y}\qty[\qty(m_{i\nu}^t)^2]\Bigr)^p \biggr) + \order{M^{-1}} \notag\\
 &= \lambda^2 \qty( \qty(Q^t)^p - \qty(q^t)^p ) + \order{M^{-1}} \simeq V^t,
 \label{eq-eval-Vt}
 \eeqn
 with
\beq
V^t = \lambda^2 \qty( \qty(Q^t)^p - \qty(q^t)^p ).
\label{eq-def-Vt}
\eeq
Here we assumed that, in the BP algorithms, the messages $\qty{v_{j\nu \to \bs}^t}$
and  $\qty{m_{j\nu \to \bs}^t}$
are independent from each other and from $F$ and that messages are independent.
We also used $\mathbb{E}_{F}\qty[F_{\bs \nu}^{2}]=1$ (see \eq{eq-F-random}) where 
$\mathbb{E}_{F}$ denotes the average over the linear coefficients $F$.
In the following we write the observation that the student receives as
\beq
y_\bs = h_{\rm out} \qty(\pi_\bs,w_\bs), \label{eq:HRTD-definition_h_out}
\eeq
where
\beq
h_{\rm out} \qty(\pi,w)=\pi+w,
\label{eq-hout-additive-noise}
\eeq
in the case of additive noise (see \eq{eq-additive-noise-output})
and
\beq
h_{\rm out} \qty(\pi,w)={\rm sgn}(\pi),
\label{eq-hout-sign-output}
\eeq
in the case of sign output (see \eq{eq-sign-output}). Using $V^{t}$ defined in \eq{eq-def-Vt}
let us introduce some disorder-averaged quantities,
\beqn
\hat{\chi}^t &= \mathbb{E}_{y}\qty[-\pdv{\omega} g_\text{out} \qty(\omega_{\bs\to\mu}^t,h_\text{out}\qty(\pi_{*,\bs\to\mu},w_\bs),V^{t}
%  \lambda^2 \qty( \qty(Q^t)^p - \qty(q^t)^p)
  )
], \label{eq:HRTD-SE_definition_hat_chi} \\
\hat{m}^t &= \mathbb{E}_{y}\qty[\pdv{\pi_*} g_\text{out} \qty(\omega_{\bs\to\mu}^t,h_\text{out}\qty(\pi_{*,\bs\to\mu},w_\bs),
  V^{t}
  %\lambda^2 \qty( \qty(Q^t)^p - \qty(q^t)^p )
  ) ], \label{eq:HRTD-SE_definition_hat_m} \\
\hat{q}^t &= \mathbb{E}_{y}\qty[ g^2_\text{out} \qty(\omega_{\bs\to\mu}^t,h_\text{out}\qty(\pi_{*,\bs\to\mu},w_\bs),
V^{t}
%  \lambda^2 \qty( \qty(Q^t)^p - \qty(q^t)^p )
) ], \label{eq:HRTD-SE_definition_hat_q}
\eeqn
which become useful in the following.

Now let us turn to evaluate $\Sigma_{i\mu}^{t+1}$
given by \eq{eq:HRTD-def_Sigma_full_summary}.
Similarly to $V_{\bs\to\mu}^t$ discussed above, it is given by a summation over large number of variables
in $c(=\alpha M) \to \infty$ limit so that it can be evaluated as,
\begin{align}
\qty(\Sigma_{i\mu}^{t+1})^{-1} &= - \sum_{\bs\in\del i} \left(\frac{\lambda}{\sqrt{M}}F_{\bs\mu}\right)^2 \qty(\prod_{j\in\del\bs\setminus i} \qty(m_{j\mu\to\bs})^2) \pdv{\omega} g_\text{out}\qty(\omega_{\bs\to\mu}^t,h_\text{out}\qty(\pi_{*,\bs},w_\bs),V_{\bs\to\mu}^t) \notag\\
&= \alpha M \mathbb{E}_{F}\qty[\left(\frac{\lambda}{\sqrt{M}}F_{\bs\mu}\right)^2] \qty(\mathbb{E}_{y}\qty[\qty(m_{i\mu}^t)^2])^{p-1} \nonumber \\
& \hspace*{1cm} \mathbb{E}_{y} \qty[-\pdv{\omega} g_\text{out}\qty(\omega_{\bs\to\mu}^t,h_\text{out}\qty(\pi_{*,\bs\to\mu},w_\bs),V^t)] + \order{M^{-1}} \notag\\
&\simeq \alpha \lambda^2 \qty(q^t)^{p-1} \hat{\chi}^t.
\label{eq-eval-inv-Sigma}
\end{align}
We also used the fact that $F_{\bs\mu}$, which is not independent from $\pi_{*,\bs}$,  is independent from
\beq
\pi_{*,\bs\to\mu} = \pi_{*,\bs} - \frac{\lambda}{\sqrt{M}} F_{\bs\mu}\prod_{j\in\del\bs} x_{*,j\mu}.
\label{eq-pi-mu-pi}
\eeq

Next let us examine
%$F_{\square\mu}$ is not independent from $\pi_{*,\square}$ but it is independent from $\pi_{*,\square\to\mu}$,
$T_{i\mu}^{t+1}$ defined by \eq{eq:HRTD-def_Sigma_T_full_summary}. It  can be evaluated
in $M \to \infty$ limit as,
\begin{align}
 \frac{T_{i\mu}^{t+1}}{\Sigma_{i\mu}^{t+1}} &= \sum_{\bs\in\del i} \left(\frac{\lambda}{\sqrt{M}}F_{\bs\mu}\right) \qty(\prod_{j\in\del\bs\setminus i} m_{j\mu\to\bs}^t) g_\text{out}\qty(\omega_{\bs\to\mu}^t,h_\text{out}\qty(\pi_{*,\bs},w_\bs),V_{\bs\to\mu}^t) \notag\\
	&= \sum_{\bs\in\del i} \left(\frac{\lambda}{\sqrt{M}}F_{\bs\mu}\right) \qty(\prod_{j\in\del\bs\setminus i} m_{j\mu\to\bs}^t) \Biggl\{ g_\text{out}\qty(\omega_{\bs\to\mu}^t,h_\text{out}\qty(\pi_{*,\bs\to\mu},w_\bs),V^t) \notag\\
 &	\hspace{4cm} + \left(\frac{\lambda}{\sqrt{M}}F_{\bs\mu}\right) \qty(\prod_{j\in\del\bs} x_{*,j\mu}) \pdv{\pi_*} g_\text{out}\qty(\omega_{\bs\to\mu}^t,h_\text{out}\qty(\pi_{*,\bs\to\mu},w_\bs),V^t) \Biggr\}
 \nonumber \\
 & + \order{M^{-1}} \notag\\
 &= \sum_{\bs\in\del i} \left(\frac{\lambda}{\sqrt{M}}F_{\bs\mu}\right)
 \qty(\prod_{j\in\del\bs\setminus i} m_{j\mu\to\bs}^t)
 g_\text{out}\qty(\omega_{\bs\to\mu}^t,h_\text{out}\qty(\pi_{*,\bs\to\mu},w_\bs),V^t) \notag\\
 &
+\sum_{\bs\in\del i}
\left(\frac{\lambda}{\sqrt{M}}F_{\bs\mu}\right)^{2}
\qty(\prod_{j\in\del\bs\setminus i} m_{j\mu\to\bs}^t x_{*,j\mu})  x_{*,i\mu} \pdv{\pi_*} g_\text{out}\qty(\omega_{\bs\to\mu}^t,h_\text{out}\qty(\pi_{*,\bs\to\mu},w_\bs),V^t) + \order{M^{-1}} \notag\\
& \simeq \sqrt{\alpha\lambda^2 \qty(q^{t})^{p-1} \hat{q}^t} \mathcal{N}(0,1) + \alpha \lambda^2 \qty(m^t)^{p-1} \hat{m}^t x_{*,i\mu}.
\label{eq-eval-T-over-Sigma}
\end{align}
The 1st term in the last equation on righthand side (rhs) is due
to the 1st term in the 3rd equation on rhs, which 
is a summation over large number $c(=\alpha M) \gg 1$ of statistically independent variables.
Here we are using again the fact that $F_{\bs\mu}$ is independent from $\pi_{*,\bs\to\mu}$.
Then it can be considered as a Gaussian variable due to the central limit theorem. 
Its mean is evaluated as,
\beq
\mathbb{E}_{F}\qty[\sum_{\bs\in\del i} \left(\frac{\lambda}{\sqrt{M}}F_{\bs\mu}\right)]
\mathbb{E}_{y}\qty[ \qty(\prod_{j\in\del\bs\setminus i} m_{j\mu\to\bs}^t)  ]
\mathbb{E}_{y}\qty[ g_\text{out}\qty(\omega_{\bs\to\mu}^t,h_\text{out}\qty(\pi_{*,\bs\to\mu},w_\bs),V^t)  ] 
=0
\eeq
This holds because of the reflection symmetry
  of the prior distribution
  $P_\text{pri.}(x)=P_\text{pri.}(-x)$
  (see sec.~\ref{sec-symmetries}) which ensures
  $\mathbb{E}_{y}\qty[ \qty(\prod_{j\in\del\bs\setminus i} m_{j\mu\to\bs}^t)  ]=0$ for $p>1$. For the $p=1$ case (linear estimation),
  which we do not consider in the present
  paper,  the similar mean vanishment happens thanks to the random spreading code yielding $\mathbb{E}_{F} [F_{\bs, \mu}]=0$ in some earlier contexts such as CDMA and compressed sensing.
%\textcolor{blue}{Obuchi: I think the statement here is problematic since we also consider nonrandom $F$ case at $p=2$, where this term vanishment is not justified totally according to this statement. Is this true?  }
%If $\mathbb{E}_{F} [F_{\bs, \mu}] \neq 0$ the average still vanishes if
%\begin{enumerate}[label=(\alph*)]
%	\item $p>2$ : the average of $x$ over the prior distribution $P_\text{pri.}(x)$ is zero.
%        \item $p=1$ :  the average of $x$ over the prior distribution $P_\text{pri.}(x)$ is $0$
%          and the output function $g_\text{out}$ satisfies
%\beq
%%\mathbb{E}_{y}\qty[g_\text{out}\qty(\omega_{\bs\to\mu}^t,h_\text{out}\qty(\pi_{*,\bs\to\mu},w_\bs),V^t)] = 0 \label{eq:HRTD-SE-gout_condition3}
%\eeq
%\cite{caltagirone2014convergence}.
On the other hand, similarly to $V_{\bs\to\mu}^t$
(see \eq{eq-eval-Vt})
and $\Sigma_{i\mu}^{t+1}$ (see \eq{eq-eval-inv-Sigma}), 
the variance is evaluated as
\be
&&
\mathbb{E}_{y}\qty[ \qty(\sum_{\bs\in\del i} \left(\frac{\lambda}{\sqrt{M}}F_{\bs\mu}\right) \qty(\prod_{j\in\del\bs\setminus i} m_{j\mu\to\bs}^t) g_\text{out}\qty(\omega_{\bs\to\mu}^t,h_\text{out}\qty(\pi_{*,\bs\to\mu},w_\bs),V^t) )^2 ] 
\no \\ &&
\simeq \alpha\lambda^2 \qty(q^t)^{p-1} \hat{q}^t.
\ee
The 2nd term in the last equation on rhs of \eq{eq-eval-T-over-Sigma}
is derived from the 2nd term in the 3rd equation on rhs similarly.

%As we show below the order parameters \eqref{eq:HRTD-SE_definition_m}-\eqref{eq:HRTD-SE_definition_Q}
%can be expressed through these quantities.

%\begin{itemize}
%\item Additional noise
%  
%With the distribution function of the noise $W(w)$ we can write,
%\beq
%P_{\rm out}\qty(y|\pi) = \int \dd{w} W(w) \; \delta\qty( y - h_{\rm out}(\pi,w) )
%\eeq
%\item Sign output function
%\end{itemize}
%

%For the random matrix $X$ we assume 
%\beq
%\frac{\lambda^2}{M} = \mathbb{E}\qty[F_{\bs\mu}^2] = \frac{1}{LM} \sum_{\bs=1}^L \sum_{\mu=1}^M \left(\frac{\lambda}{\sqrt{M}}F_{\bs\mu}\right)^2 \label{eq:HRTD-SE-F^2}
%\eeq
%holds.

%For simplicity we assume the following holds (otherwise one needs to consider the effect of non-zero average

%Here $\mathbb{E}_{y}\qty[\cdots]$ represents the averages over realizations of the random matrix $F$,
%the ground truth signal $x_{*}$ and realizations of the noise $w$.
%\end{enumerate}

%The measurement$y_\bs$ is a function of the ground truth product $\pi_{*,\bs}$ and the noise $w_\bs$
%(see \eq{eq:HRTD-definition_h_out}).

To sum up, we now have expressions of the two quantities
$\qty(\Sigma_{i\mu}^{t})^{-1}$
\eq{eq-eval-inv-Sigma}
and $\frac{T_{i\mu}^{t}}{\Sigma_{i\mu}^{t}}$ \eq{eq-eval-T-over-Sigma}
needed to evaluate the order parameters $m^{t}$,  $q^{t}$,  $Q^{t}$
as given by \eq{eq:HRTD-SE_eval_m}-\eq{eq:HRTD-SE_eval_Q}
expressed in terms of $m^{t}$,  $q^{t}$,  $Q^{t}$
and  $\hat{\chi}^t$, $\hat{m}^t$ and $\hat{q}^{t}$ given by
\eq{eq:HRTD-SE_definition_hat_chi}-\eq{eq:HRTD-SE_definition_hat_q}.
Now we have to examine  the average over the quenched randomness
$\mathbb{E}_{y}$ which appear in these equations.
The average $\mathbb{E}_{y}$ over
$F_{\bs\mu}$, $x_{*,i\mu}$ and $w_\bs$ can be
regarded as the average over $\omega_{\bs\to\mu}^t$, $\pi_{*,\bs\to\mu}$ and  $w_{\bs}$.
They are obtained as
%Given (a)-(c), the averages and variances of  $\omega_{\bs\to\mu}^t$ and $\pi_{*,\bs\to\mu}$ are obtained as,
\beqn
 \mathbb{E}_{y}\qty[\omega_{\bs\to\mu}^t] &= \mathbb{E}_{x_*}\qty[\pi_{*,\bs}] = 0, \\
\mathbb{E}_{y} \bigl[ \qty(\omega_{\bs\to\mu}^t )^2 \bigr] &= \lambda^2 \qty(q^t)^p + \order{M^{-1}}, \\
\mathbb{E}_{y} \qty[ \omega_{\bs\to\mu}^t \pi_{*,\bs} ] &= \lambda^2 \qty(m^t)^p + \order{M^{-1}}, \\
\mathbb{E}_{y} \qty[ \pi_{*\bs}^2 ] &= \lambda^2 \qty(\mathbb{E}_{x_*}\qty[x_{*,i\mu}^2])^p=1.
\eeqn

Introducing the bivariate normal distribution function
$\mathcal{N}\qty[x_{1},x_{2};C]$
of variables $x_{1}$ and $x_{2}$ with $0$ mean and covariance matrix $C=(C^{t})$,
\beq
\mathcal{N}\qty[x_{1},x_{2};C]=
\frac{1}{2\pi\sqrt{{\rm det}C}}
\exp \left[-\frac{1}{2}((C^{-1})_{11}x_{1}^{2}+2(C^{-1})_{12}x_{1}x_{2}+(C^{-1})_{22}x_{2}^{2}) \right],
\label{eq-bivariate-normal-distribution}
\eeq
with
\beq
C=\left (\begin{array}{cc}
C_{11} & C_{12} \\
C_{21} & C_{22} 
\end{array}
\right)
=\lambda^2
\left (
\begin{array}{cc}
 \qty(q^t)^p & \qty(m^t)^p \\
 \qty(m^t)^p &  1
\end{array}
\right),
\qquad {\rm det C}=\lambda^{4} \left((q^{t})^{p}-(m^{t})^{2p}\right).
\label{eq-def-covariance-matrix}
\eeq
\eq{eq:HRTD-SE_definition_hat_chi}-\eqref{eq:HRTD-SE_definition_hat_q} can be expressed as
\beqn
\hat{\chi}^t &= \int\dd{w} W(w) \int \dd{\xi} \dd{z} \mathcal{N}\qty[ \xi , z;C] \; \qty( - \pdv{\xi} g_\text{out} \qty(\xi,h_{\rm out}(z,w),V^t ) ), \label{eq:HRTD-SE_hat_chi_integral} \\
\hat{m}^t &= \int\dd{w} W(w) \int \dd{\xi} \dd{z} \mathcal{N}\qty[ \xi , z; C] \;\; \pdv{z} g_\text{out} \qty(\xi,h_{\rm out}(z,w),V^t), \label{eq:HRTD-SE_hat_m_integral} \\
\hat{q}^t &= \int\dd{w} W(w) \int \dd{\xi} \dd{z} \mathcal{N}\qty[ \xi , z; C] \;\; g^2_\text{out} \qty(\xi,h_{\rm out}(z,w),V^t ). \label{eq:HRTD-SE_hat_q_integral}
\eeqn
%\beqn
%\hat{\chi}^t &= \int\dd{w} W(w) \int \dd{\xi} \dd{z} \mathcal{N}\qty[ \xi , z; \lambda^2 \qty(q^t)^p, \lambda^2 (\mathbb{E}_{x_*}\qty[x_{*,i\mu}^2])^p , \lambda^2 \qty(m^t)^p] \; \qty( - \pdv{\xi} g_\text{out} \qty(\xi,h_{\rm out}(z,w),V^t ) ) \label{eq:HRTD-SE_hat_chi_integral} \\
%\hat{m}^t &= \int\dd{w} W(w) \int \dd{\xi} \dd{z} \mathcal{N}\qty[ \xi , z; \lambda^2 \qty(q^t)^p, \lambda^2 (\mathbb{E}_{x_*}\qty[x_{*,i\mu}^2])^p , \lambda^2 \qty(m^t)^p] \;\; \pdv{z} g_\text{out} \qty(\xi,h_{\rm out}(z,w),V^t) \label{eq:HRTD-SE_hat_m_integral} \\
%\hat{q}^t &= \int\dd{w} W(w) \int \dd{\xi} \dd{z} \mathcal{N}\qty[ \xi , z; \lambda^2 \qty(q^t)^p, \lambda^2 (\mathbb{E}_{x_*}\qty[x_{*,i\mu}^2])^p , \lambda^2 \qty(m^t)^p] \;\; g^2_\text{out} \qty(\xi,h_{\rm out}(z,w),V^t ) \label{eq:HRTD-SE_hat_q_integral}
%\eeqn

Finally we find the order parameters \eqref{eq:HRTD-SE_definition_m}-\eqref{eq:HRTD-SE_definition_Q}
satisfy the following self-consistent equations,
%can be expressed using $\hat{\chi}^t$, $\hat{m}^t$, $\hat{q}^t$ as,
\beqn
m^{t+1} 
&= \int \dd{x_*} P_\text{pri.}\qty(x_*) \int\D{z} x_* f_\text{input} \qty( \frac{1}{\alpha\lambda^2 \qty(q^t)^{p-1} \hat{\chi}^t} , \frac{\alpha \lambda^2 \qty(m^t)^{p-1} \hat{m}^t x_* + z \sqrt{\alpha\lambda^2\qty(q^t)^{p-1} \hat{q}^t} }{\alpha\lambda^2 \qty(q^t)^{p-1} \hat{\chi}^t} ),\label{eq:HRTD-SE_EqState_m}
\\
q^{t+1} 
	&= \int \dd{x_*} P_\text{pri.}\qty(x_*) \int\D{z} f_\text{input}^2 \qty( \frac{1}{\alpha\lambda^2 \qty(q^t)^{p-1} \hat{\chi}^t} , \frac{\alpha \lambda^2 \qty(m^t)^{p-1} \hat{m}^t x_* + z \sqrt{\alpha\lambda^2\qty(q^t)^{p-1} \hat{q}^t} }{\alpha\lambda^2 \qty(q^t)^{p-1} \hat{\chi}^t} ), \label{eq:HRTD-SE_EqState_q}\\
Q^{t+1}  
&= \int \dd{x_*} P_\text{pri.}\qty(x_*) \int\D{z} f_\text{input,II} \qty( \frac{1}{\alpha\lambda^2 \qty(q^t)^{p-1} \hat{\chi}^t} , \frac{\alpha \lambda^2 \qty(m^t)^{p-1} \hat{m}^t x_* + z \sqrt{\alpha\lambda^2\qty(q^t)^{p-1} \hat{q}^t} }{\alpha\lambda^2 \qty(q^t)^{p-1} \hat{\chi}^t} ). \label{eq:HRTD-SE_EqState_Q}
\eeqn
These are the SE equations for the current system. The corresponding algorithm is given in algorithm \ref{alg:p-ary_SE}.

\blue{Note that the macroscopic observables such as $m^t$,$q^t$ etc. do not depend on specific realizations
of the random graph. This is consistent with the observation in the replica theory (see sec~\ref{sec-interaction-part-of-free-energy}).
}

%%%%%%%%%%%%%%%%%%%%%
%%--Alg
%\begin{figure}[!t]
  \begin{algorithm}[H]
    \caption{SE algorithm}
    \label{alg:p-ary_SE}
    \begin{algorithmic}[1]
      \Require  input function $f_\text{input}$, output function $g_\text{out}$, maximum time step $t_{\rm max}$.
      \Initialize~ Set initial values of $m^{t}$,$q^{t}$ and $Q^{t}$.
      \Repeat
      \State Update $V^{t}$ and the matrix $C$
\beq
V^t \leftarrow \lambda^2 \qty( \qty(Q^t)^p - \qty(q^t)^p )
%\eeq
%\beq
\qquad
C \leftarrow
%\left (\begin{array}{cc}
%C_{11} & C_{12} \\
%C_{21} & C_{22} 
%\end{array}
%\right)
%=
\lambda^2
\left (
\begin{array}{cc}
 \qty(q^t)^p & \qty(m^t)^p \\
 \qty(m^t)^p &  1
\end{array}
\right)
\tag{Alg3.1}
\label{eq:Alg3.1}
\eeq
\State Update $\hat{\chi}^t$,$\hat{m}^t$ and $\hat{q}^t$
\beq
\begin{aligned}
\hat{\chi}^t & \leftarrow \int\dd{w} W(w) \int \dd{\xi} \dd{z} \mathcal{N}\qty[ \xi , z;C] \; \qty( - \pdv{\xi} g_\text{out} \qty(\xi,h_{\rm out}(z,w),V^t ) ),
\notag\\
\hat{m}^t & \leftarrow \int\dd{w} W(w) \int \dd{\xi} \dd{z} \mathcal{N}\qty[ \xi , z; C] \;\; \pdv{z} g_\text{out} \qty(\xi,h_{\rm out}(z,w),V^t),
\notag\\
\hat{q}^t & \leftarrow \int\dd{w} W(w) \int \dd{\xi} \dd{z} \mathcal{N}\qty[ \xi , z; C] \;\; g^2_\text{out} \qty(\xi,h_{\rm out}(z,w),V^t )
\end{aligned}
\label{eq:Alg3.2}
\tag{Alg3.2}
\eeq
\State Update $m^t$,$q^t$ and $Q^t$
\beq
\begin{aligned}
  m^{t+1} 
  &\leftarrow \int \dd{x_*} P_\text{pri.}\qty(x_*) \int\D{z} x_* f_\text{input} \qty( \frac{1}{\alpha\lambda^2 \qty(q^t)^{p-1} \hat{\chi}^t} , \frac{\alpha \lambda^2 \qty(m^t)^{p-1} \hat{m}^t x_* + z \sqrt{\alpha\lambda^2\qty(q^t)^{p-1} \hat{q}^t} }{\alpha\lambda^2 \qty(q^t)^{p-1} \hat{\chi}^t} ),
\notag\\
q^{t+1} 
& \leftarrow \int \dd{x_*} P_\text{pri.}\qty(x_*) \int\D{z} f_\text{input}^2 \qty( \frac{1}{\alpha\lambda^2 \qty(q^t)^{p-1} \hat{\chi}^t} , \frac{\alpha \lambda^2 \qty(m^t)^{p-1} \hat{m}^t x_* + z \sqrt{\alpha\lambda^2\qty(q^t)^{p-1} \hat{q}^t} }{\alpha\lambda^2 \qty(q^t)^{p-1} \hat{\chi}^t} ),
\notag\\
Q^{t+1}  
& \leftarrow \int \dd{x_*} P_\text{pri.}\qty(x_*) \int\D{z} f_\text{input,II} \qty( \frac{1}{\alpha\lambda^2 \qty(q^t)^{p-1} \hat{\chi}^t} , \frac{\alpha \lambda^2 \qty(m^t)^{p-1} \hat{m}^t x_* + z \sqrt{\alpha\lambda^2\qty(q^t)^{p-1} \hat{q}^t} }{\alpha\lambda^2 \qty(q^t)^{p-1} \hat{\chi}^t} ). 
 \end{aligned}
\tag{Alg3.3}
\label{eq:Alg3.3}
\eeq
      \Until $m^{t}$,$q^{t}$ and $Q^{t}$ converge or $t$ reaches the maximum step $t_{\rm max}$.
      \Ensure {Time sequences of the order parameters $\{(m^{t},q^{t},Q^{t})\}_t$.}
    \end{algorithmic}
  \end{algorithm}
%%%%%%%%%%%%%%%%%%%%%

%%%%%%%%%%%%%%%%%%%%%%%%%%%%%%%%%%%%%%%%%%%%%%%%
\subsubsection{Ising model}

In the case of the Ising input  (see sec.~\ref{subsubsec-ising-model}),
we find using \eq{eq:HRTD-input_Ising}
in \eq{eq:HRTD-SE_EqState_m}-\eq{eq:HRTD-SE_EqState_Q},
\beqn
m^{t+1}
&=& \int\D{z} \tanh\left(
\alpha \lambda^2 \qty(m^t)^{p-1} \hat{m}^t + z \sqrt{\alpha\lambda^2\qty(q^t)^{p-1} \hat{q}^t}
\right), \label{eq:HRTD-SE_EqState_m_ising}\\
q^{t+1} &=& 
\int\D{z} \tanh^{2}\left(
\alpha \lambda^2 \qty(m^t)^{p-1} \hat{m}^t + z \sqrt{\alpha\lambda^2\qty(q^t)^{p-1} \hat{q}^t} \right),  \label{eq:HRTD-SE_EqState_q_ising} \\
Q^{t+1} &=& 1.
\eeqn

In Bayes optimal case we expect $m^{t}=q^{t}$ holds at the fixed point. 
Later we will find that $\hat{m}^{t}=\hat{q}^{t}$ holds if $m^{t}=q^{t}$
and $Q^{t}=1$. Then one can show that $m^{t+1}=q^{t+1}$ holds
using \eq{eq:HRTD-SE_EqState_m_ising}, \eq{eq:HRTD-SE_EqState_q_ising}
and Lemma A.

%%%%%%%%%%%%%%%%%%%%%%%%%%%%%%%%%%%%%%%%%%%%%%%%
\subsubsection{Gaussian model}

In the case of the Gaussian input  (see sec.~\ref{subsubsec-gaussian-model}),
we find using
\eq{eq:HRTD-input_Spherical}
in \eq{eq:HRTD-SE_EqState_m}-\eq{eq:HRTD-SE_EqState_Q},
\beqn
m^{t+1} &=& \int \dd{x_*} P_\text{pri.}\qty(x_*) \int\D{z} x_*
\frac{\alpha \lambda^2 \qty(m^t)^{p-1} \hat{m}^t x_* + z \sqrt{\alpha\lambda^2\qty(q^t)^{p-1} \hat{q}^t}}{
  1+\alpha\lambda^2 \qty(q^t)^{p-1} \hat{\chi}^t} 
=\frac{\alpha \lambda^2 \qty(m^t)^{p-1} \hat{m}^t}{1+\alpha\lambda^2 \qty(q^t)^{p-1} \hat{\chi}^t}, \label{eq:HRTD-SE_EqState_m_Gaussian} \\
q^{t+1}  &=& \int \dd{x_*} P_\text{pri.}\qty(x_*) \int\D{z}
\left[
  \left(\frac{\alpha \lambda^2 \qty(m^t)^{p-1} \hat{m}^t}{1+\alpha\lambda^2 \qty(q^t)^{p-1} \hat{\chi}^t}\right)^{2} (x_*)^{2}
  +z^{2} \frac{\alpha \lambda^2 \qty(q^t)^{p-1} \hat{q}^t}{(1+\alpha\lambda^2 \qty(q^t)^{p-1} \hat{\chi}^t)^{2}}
  +2zx_{*} \cdots
  \right]  \nonumber \\
&=& (m^{t+1})^{2}+\frac{
  \alpha \lambda^2 \qty(q^t)^{p-1} \hat{q}^t
}{
  (1+\alpha\lambda^2 \qty(q^t)^{p-1} \hat{\chi}^t)^{2}
}, \label{eq:HRTD-SE_EqState_q_Gaussian} \\
Q^{t+1}  &=& q^{t+1}+\frac{1}{1+\alpha \lambda^{2}(q^{t})^{p-1}\hat{\chi}^{t}}.
\label{eq:HRTD-SE_EqState_Q_Gaussian}
\eeqn
Assuming the spin normalization $Q^{t+1}=1$, the last equation implies,
\beq
\frac{1}{1-q^{t+1}}=1+\alpha \lambda^{2}(q^{t})^{p-1}\hat{\chi}^{t}.
\eeq
Using this we find
\beqn
\frac{m^{t+1}}{1-q^{t+1}} &=& \alpha \lambda^{2} (m^{t})^{p-1}\hat{m}^{t},
\label{eq:HRTD-SE_EqState_m_gaussian}
\\
\frac{q^{t+1}-(m^{t+1})^{2}}{(1-q^{t+1})^{2}}&=&\alpha \lambda^{2} (q^{t})^{p-1}\hat{q}^{t}.
\label{eq:HRTD-SE_EqState_q_gaussian}
\eeqn

In Bayes optimal case we expect $m^{t}=q^{t}$ holds at the fixed point.
Later we will find that $\hat{\chi}^{t}=\hat{m}^{t}=\hat{q}^{t}$ 
holds if $m^{t}=q^{t}$ and $Q^{t}=1$ hold.
Then given these we find that $q^{t+1}=m^{t+1}$ holds
using \eq{eq:HRTD-SE_EqState_m_gaussian}
and \eq{eq:HRTD-SE_EqState_q_gaussian}. We also find that $Q^{t+1}=1$
holds if $\hat{\chi}^{t}=\hat{m}^{t}=\hat{q}^{t}$ using \eq{eq:HRTD-SE_EqState_m_Gaussian}-\eq{eq:HRTD-SE_EqState_Q_Gaussian}.
%{\bf [Q] Beyond the Bayes optimal situation $Q^{t+1}=1$ dos not hold??}

%%%%%%%%%%%%%%%%%%%%%%%%%%%%%%%%%%%%%%%%%%%%%%%%
\subsubsection{Additive noise}

In the case of additive noise (see \eq{eq-additive-noise-output}) we have
$h_{\rm out} \qty(\pi,w)=\pi+w$ as given by \eq{eq-hout-additive-noise}.
In sec \ref{subsubsec-additive-noise} we obtained
\eq{eq:HRTD-output_AWGN}) which implies,
\beq
g_\text{out}\qty(\xi,h_{\rm out} \qty(z,w),V)=g_\text{out}\qty(w+(z-\xi),V).
\eeq
The fact that the output function depends on $z$ and $\xi$ just through
the difference $z-\xi$ leads to a simplification.
For a generic function $f(z)$,
we find the average weighted by the bivariate normal distribution
\eq{eq-bivariate-normal-distribution}
%of variables $\xi$ and $z$ with $0$ mean and
%covariance matrix $C$ as $\mathcal{N}\qty[x_{1},x_{2};C]$
becomes
\beq
\int \dd\xi \dd z \mathcal{N}\qty[\xi,z;C] f(z-\xi)
=\int {\cal D}z_{0} f(\sqrt{C_{11}+C_{22}-2C_{12}}z_{0}).
\eeq
Using this and the explicit form of the output function given by \eq{eq:HRTD-output_AWGN}) we find
\eq{eq:HRTD-SE_hat_chi_integral}-\eq{eq:HRTD-SE_hat_q_integral}
become
\beqn
\hat{\chi}^t &=&\hat{m}^t = \int\dd{w} W(w) \int {\cal D}z_{0} \frac{\partial}{\partial w} g_\text{out} \qty(w-\lambda\sqrt{\qty(q^t)^p+
  %(\mathbb{E}_{x_*}\qty[x_{*,i\mu}^2])^p
  1
  -2 \qty(m^t)^p]}z_{0},V^t) \nonumber \\
  & = &\theta_{0}-\theta_{1},
    \label{eq-SE-q-additive-noise}
\\
\hat{q}^t &=& \int\dd{w} W(w) \int {\cal D}z_{0}
g^{2}_\text{out} \qty(w-\lambda\sqrt{\qty(q^t)^p
  %  +(\mathbb{E}_{x_*}\qty[x_{*,i\mu}^2])^p
  +1
  -2 \qty(m^t)^p]}z_{0},V^t) \nonumber \\
  &=& \theta_{0},
  \label{eq-SE-chi-additive-noise}
  \eeqn
 where we introduced
  \beqn
  \theta_{0} &=
  \int\dd{w} W(w) \int {\cal D}z_{0} 
  \left(
  \frac{\int {\cal D}z_{1}\frac{\partial}{\partial \omega}W(\Xi)}{\int {\cal D}z_{1}W(\Xi)}
  \right)^{2},
  \label{eq-SE-theta0-additive-noise}  
  \\
  \theta_{1} & =
   \int\dd{w} W(w) \int {\cal D}z_{0} 
%  \left(
  \frac{\int {\cal D}z_{1}
    \frac{\partial^{2}}{\partial \omega^{2}}
    \frac{\partial}{\partial \omega}W(\Xi)}{\int {\cal D}z_{1}W(\Xi)},
  %  \right)^{2}
  \label{eq-SE-theta1-additive-noise}    
  \eeqn
  with
  \beq
  \Xi=w-
  \lambda  \sqrt{ \qty(q^t)^p+
    %    (\mathbb{E}_{x_*}\qty[x_{*,i\mu}^2])^p
    1
    -2 \qty(m^t)^p]}z_{0}
    -\sqrt{V^{t}}z_{1}.
  \label{eq-SE-Xi-additive-noise}        
   \eeq

In the Bayes optimal case we expect $m^{t}=q^{t}$ holds at the fixed point.
Let us also assume $Q^{t}=1$ holds. Then we find
   (see \eq{eq-def-Vt})
   $\Xi=w- \lambda  \sqrt{1- \qty(q^t)^p]}(z_{0}+z_{1})$.
   Then it is possible to show
   that $\theta_{1}=0$ (Lemma B).
   This in turn implies $\hat{\chi}^{t}=\hat{q}^{t}=\hat{m}^{t}$.

%%%%%%%%%%%%%%%%%%%%%%%%%%%%%%%%%%%%%%%%%%%%%%%%
\subsubsection{Sign output}
In the case of sign output (see \eq{eq-sign-output}) we have
$h_{\rm out} \qty(\pi,w)={\rm sgn}(\pi)$ as given by \eq{eq-hout-sign-output}.
%  \beq
%  g_\text{out}\qty(\omega,y,V)=\sum_{\sigma=-1,1}\delta(y-\sigma)
%  \frac{\sigma}{\sqrt{2\pi V} H\qty(-\sigma\frac{\omega}{\sqrt{V}})} \exp(-\frac1{2V} \omega^2)
  %,
%  \qquad -\pdv{\omega}g_\text{out}\qty(\omega,y,V) = g_\text{out}^2\qty(\omega,y,V) + \frac{\omega}{V} g_\text{out}\qty(\omega,y,V) 
%\eeq
The explicit form of the output function is given by \eq{eq:HRTD-outputFunction_sign} which reads,
\beq
g_\text{out}\qty(\omega,y,V)=\sum_{\sigma=-1,1}\delta(y-\sigma)
\tilde{g}_{\rm out}(\xi,\sigma,V),
  \qquad
  \tilde{g}_{\rm out}(\xi,\sigma,V)=
   \left. \frac{-\sigma}{\sqrt{V}}\frac{H'(x)}{H(x)}\right |_{x=-\sigma\xi/\sqrt{V}},
   \eeq
   where we find a useful relation $\tilde{g}_{\rm out}(\xi,-1,V)=-\tilde{g}_{\rm out}(-\xi,+1,V)$. Using these we find $\hat{\chi}^{t}$, $\hat{m}^{t}$ and $\hat{q}^{t}$ become the following.
\beqn
\hat{\chi}^{t} &=&2\int d\xi
\frac{e^{-\frac{\xi^{2}}{2C_{11}}}}{\sqrt{2\pi C_{11}}}
  H(-\gamma \xi)
  \left[
    \frac{\xi}{V}\tilde{g}_{\rm out}(\xi,+1,V^{t})+\tilde{g}^{2}_{\rm out}(\xi,+1,V^{t})
    \right], \\
  \hat{m}^{t} &=& 2\int d\xi
  \frac{e^{-\frac{(C^{-1})_{11}\xi^{2}}{2}}}{2\pi\sqrt{{\rm det} C}}
\tilde{g}_{\rm out}(\xi,+1,\sigma,V),   \\
%  \frac{e^{-\frac{\xi^{2}}{2 V_{t}}}}{\sqrt{2\pi V^{t}}}
%  \frac{1}{H\left(-\frac{\xi}{\sqrt{V^{t}}}\right)}
  \hat{q}^{t} &=& 2\int d\xi
\frac{e^{-\frac{\xi^{2}}{2C_{11}}}}{\sqrt{2\pi C_{11}}}
H(-\gamma \xi)
\tilde{g}^{2}_{\rm out}(\xi,+1,\sigma,V),
\eeqn
with
\beq
\gamma=\frac{C_{12}}{\sqrt{{\rm det} C}}\frac{1}{\sqrt{C_{11}}}=
\frac{1}{\left\{\lambda^{2}\left[\left(\frac{q^{t}}{m^{t}}\right)^{2p}-(q^{t})^{p}\right]
  \right \}^{1/2}
},
\eeq
for the covariance matrix $C$ (see \eq{eq-def-covariance-matrix}).
Since we assume the Bayes optimal setting we expect $q^{t}=m^{t}$ holds
at the fixed point.
Assuming also $Q^{t}=1$ which must hold due to the normalization of spins we find
(see \eq{eq-def-Vt}),
\beq
\gamma=\frac{1}{\sqrt{\lambda^{2}(1-q^{t})^{p}}}=\frac{1}{\sqrt{V^{t}}},
\eeq
which implies
\beq
\hat{\chi}^{t} =\hat{q}^{t}=\hat{m}^{t}
=2\int d\xi
\frac{e^{-\frac{\xi^{2}}{2\lambda^{2}(q^{t})^{p}}}}{\sqrt{2\pi \lambda^{2}(q^{t})^{p}}}
\frac{1}{2\pi V^{t}}\frac{e^{-\frac{\xi^{2}}{V^{t}}}}{H\left(-\frac{\xi}{\sqrt{V^{t}}}\right)}.
\label{eq-SE-chi-q-m-sign-output}
\eeq

%{\bf [Q] Comparing with the replica result \eq{eq-SP-gaussian-additive-noise-RS},
%  we find agreement in the Bayes optimal case $m^{t}=q^{t}$. Beyond
%  the Bayes optimal situation the replica result involves one more integral
%  $\int {\cal D}z_{1}$
%  which we do not find in the SE equation. Why?}

%where
%\beq
%H(x) = \int_x^\infty \frac{\dd{t}}{\sqrt{2\pi}} \; e^{-\frac12 t^2} = \frac12 \text{erfc}\qty(\frac{x}{\sqrt{2}})
%\eeq
%with $\text{erfc}(x)=(2/\sqrt{\pi})\int_{x}^{\infty}dy e^{-y^{2}}$ being the complementary error function.

%%%%%%%%%%%%%%%%%%%%%%%%%%%%%%%%%%%%%%%%%%%%%%%%
%%%%%%%%%%%%%%%%%%%%%%%%%%%%%%%%%%%%%%%%%%%%%%%%
\subsection{Comparison with the replica theory}
\label{sec-comparison-BP-replica}

Here let us compare the SE equations with the equation of states obtained by the replica approach. Schematically the SE equations
look like $A_{t+1}=F_{A}(A_{t},B_{t},...)$,
$B_{t+1}=F_{B}(A_{t},B_{t},...),...$
while the equations of states by the replica approach
look like $A=\tilde{F}_{A}(A,B,...),B=\tilde{F}_{B}(A,B,...),...$.
  The two approaches agree if
  $F_{A}=\tilde{F}_{A},F_{B}=\tilde{F}_{B},\ldots$.

For clarity let us focus on the explicit expressions for the three specific cases.
In the case of Ising prior and additive noise, we find that the SE equations
using \eq{eq:HRTD-SE_EqState_m_ising} and \eq{eq:HRTD-SE_EqState_q_ising}
combined with \eq{eq-SE-q-additive-noise}-\eq{eq-SE-Xi-additive-noise} agree with
the two equations of states obtained by the replica approach
given by \eq{eq-SP-ising-additive-noise-RS}.
Similarly in the case of Gaussian prior and additive noise, the SE equations (\eq{eq:HRTD-SE_EqState_m_gaussian} and \eq{eq:HRTD-SE_EqState_q_gaussian} combined with \eq{eq-SE-chi-additive-noise}-\eq{eq-SE-Xi-additive-noise}) again agree with the replica result given by 
\eq{eq-SP-gaussian-additive-noise-RS}. For the case of Gaussian prior and sign output in the Bayes optimal setting, the SE equation (\eq{eq:HRTD-SE_EqState_m_gaussian} with  \eq{eq-SE-chi-q-m-sign-output}) matches with the replica result \eq{eq-SP-gaussian-sign-RS-BayesOptimal}. In all the cases, the equivalence between the SE and replica results is confirmed.

%%%%%%%%%%%%%%%%%%%%%%%%%%%%%%%%%%%%%%%%%%%%%%%%
%%%%%%%%%%%%%%%%%%%%%%%%%%%%%%%%%%%%%%%%%%%%%%%%
\subsection{Discussions}
\label{sec-discussion-BP}

In the replica approach we have shown that loop corrections vanish in the dense limit $\lim_{c \to \infty}\lim_{N \to \infty}$.
Unfortunately such considerations are absent in the message passing approach discussed in this section.
We believe that the arguments used to derive the algorithms (r-BP and G-AMP) are valid
only in the dense limit $\lim_{c \to \infty}\lim_{N \to \infty}$.
Indeed the analysis of TAP equation through PGY expansion \cite{maillard2022perturbative}
on the $p=2$ system (matrix factorization) have shown that loop corrections are inevitable in the globally coupled system.

In the results presented above one can notice that expressions for the macroscopic quantities like the order-parameters, equation of states  are the same for the deterministic model with $F=1$ \eq{eq-F-uniform} and the disordered model with random $F$ \eq{eq-F-random}. Let us recall that we found the same in the replica approach discussed in sec~\ref{sec-replica}. However we found that convergence of the algorithms (r-BP and G-AMP) can be very different between the two models in some specific cases as we explain in sec~\ref{sec-result}.

\clearpage

%%%%%%%%%%%%%%%%%%%%%%%%%%%%%%%%%%%%%%%%%%%%%%%%%%%%%
%%%%%%%%%%%%%%%%%%%%%%%%%%%%%%%%%%%%%%%%%%%%%%%%%%%%%
%%%%%%%%%%%%%%%%%%%%%%%%%%%%%%%%%%%%%%%%%%%%%%%%%%%%%
\section{Analysis of specific models}
\label{sec-result}

Now let us examine some representative cases closely.
In the following we limit ourselves to the Bayes optimal setting for which $m=q$
and $Q_{0}=Q=1$ holds.
We discuss  the cases of Ising prior and additive Gaussian noise in sec.~\ref{sec-ising-gaussian-noise},
Gaussian prior and additive Gaussian noise in sec.~\ref{sec-gaussian-gaussian-noise},
Gaussian prior and sign output in sec.~\ref{sec-gaussian-sign-output}.
In each of the cases, we first examine the equation of states for the order parameter $m$
(obtained by the replica and message passing approaches which agree) and establish phase diagrams.
Then we examine the performance of message passing algorithms.
For the details on some numerical prescriptions for the message passing algorithms see appendix~\ref{appendix-mp-detail}.
Some detailed comparisons between different algorithms are presented in appendix~\ref{sec-comparisons-finite-size}.
\blue{Our codes for the numerical analysis can be found in the
GitHub repository~\cite{mymodel2026}}.

%%%%%%%%%%%%%%%%%%%%%%%%%%%%%%%%%%%%%%%%%%%%%%%%%%%%%
%%%%%%%%%%%%%%%%%%%%%%%%%%%%%%%%%%%%%%%%%%%%%%%%%%%%%
\subsection{Ising prior and additive Gaussian noise}
\label{sec-ising-gaussian-noise}
Here we examine the case of the Bayes optimal additive Gaussian noise. The corresponding equation of state and free energy in the replica approach are given by \eq{eq:m_Ising_Bayesopt} and  \eq{eq:free_ene_Ising_Bayesopt}. For the message passing approach, we will examine the r-BP algorithm (Algorithm~\ref{alg:p-ary_r-BP}) and G-AMP algorithm (Algorithm~\ref{alg:p-ary_G-AMP}) with the input function \eq{eq:HRTD-input_Ising}  and output function \eq{eq:HRTD-output_AWGN_Gaussian}.
The SE equation can be found as \eq{eq:HRTD-SE_EqState_m_ising} and \eq{eq:HRTD-SE_EqState_q_ising} combined with \eq{eq-SE-q-additive-noise}-\eq{eq-SE-Xi-additive-noise}, showing a consistency with the replica formula. 

%We will provide the numerical solution for the two cases $p=2$ and $p=3$ in figure \ref{fig:ising_gauss_p2} and~\ref{fig:ising_gauss_p3} respectively.

%%%%%%%%%%%%%%%%%%%%%%%%%%%%%%%%%%%%%%%%%%%%%%%%%%%%%
\subsubsection{$p=2$ case}
\label{sec:result-p=2-ising-gaussian-noise}
Let us first examine the case $p=2$. By solving the equation of state \eq{eq:m_Ising_Bayesopt}
numerically we obtained results as shown in Fig.~\ref{fig:ising_gauss_p2}.
%%%%%%%%%%%%%%%%%%%%%%
\begin{figure}[h]
\centering
\includegraphics[width=0.450\textwidth]{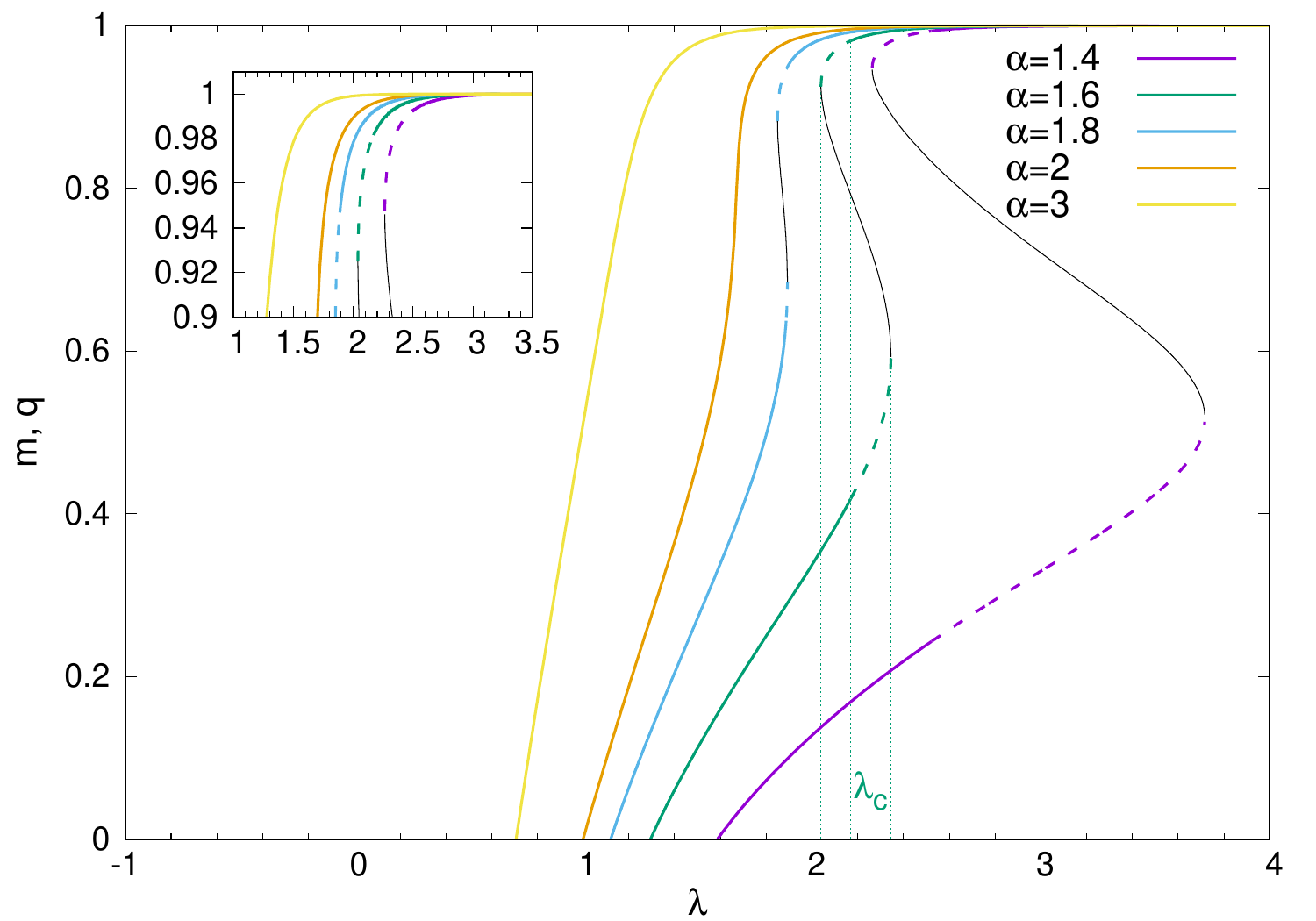}
\includegraphics[width=0.410\textwidth]{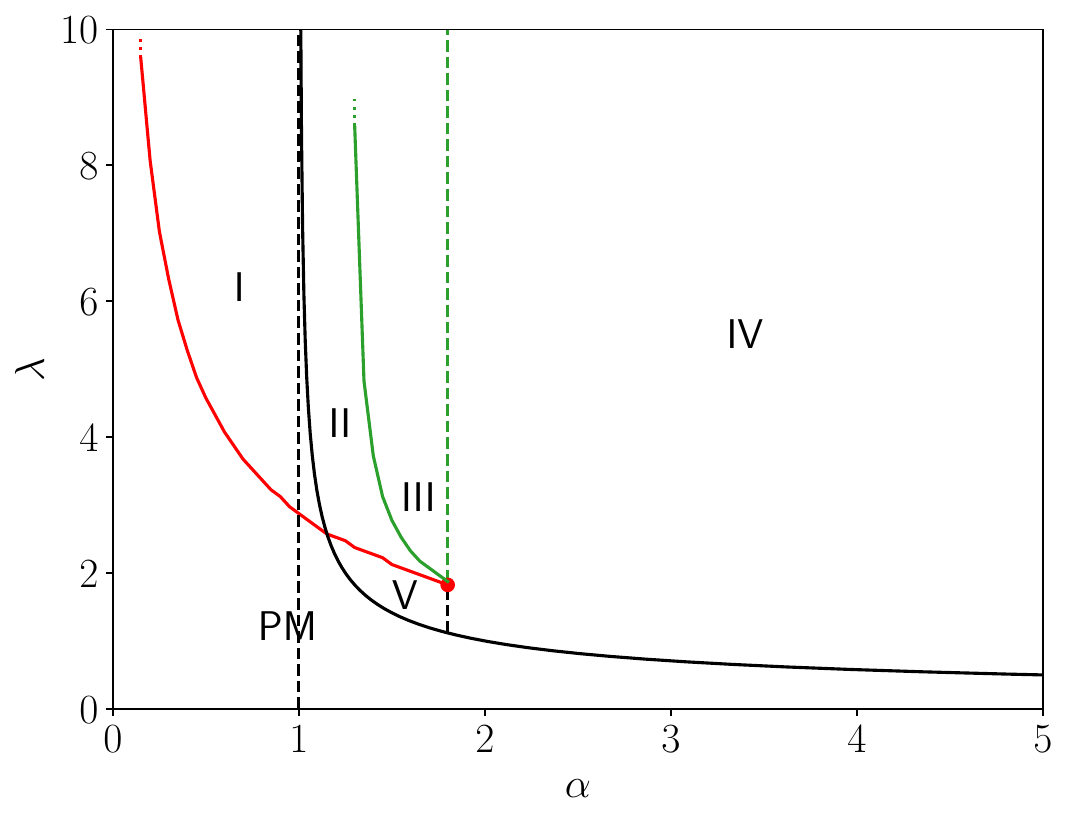}
\caption{Left: Order parameter $m=q$ in the Bayes optimal case for Ising prior, additive Gaussian noise and $p=2$. \blue{The dashed portion of the lines indicates a metastable magnetized state associated to a first-order transition. Black continuous lines indicate branches of unstable solutions.} For $\alpha=1.6$ three vertical lines are shown: the central one $\lambda_c$ is the thermodynamic first-order transition where the difference in free energy between the two branches of solutions changes sign; the leftmost one is the spinodal point $\lambda_{\rm d}(\alpha)$ for the high-magnetization solution (corresponding to the red line in the right panel); the rightmost one is the spinodal point for the low-magnetization solution (green line in the right panel).
 Right: phase diagram in the $\alpha-\lambda$ plane. The black solid line is the spinodal line (stability limit) of the paramagnetic (PM) solution $\lambda^{*}(\alpha)$; for $\alpha<1$ the PM solution is stable $\forall \lambda$. Region I: coexistence of PM and a high-magnetization solution, separated by a first-order transition (not shown for clarity). Region II: coexistence of low and high-magnetization phases, separated by a first-order transition (not shown). Region III: only the high-magnetization phase survives beyond the spinodal point for the low-magnetization solution. Regions IV and V: as in region III, there is a unique magnetized phase, but decreasing $\lambda$ it transitions continuously to the PM phase.
  }
\label{fig:ising_gauss_p2}
\end{figure}
%%%%%%%%%%%%%%%%%%%%%%
%\black{Todo: add also analysis of stability...just the logidutinal mode is sufficient.}
Note that $m=0$ is always a solution.  At large enough $\alpha$ we find another solution $m>0$ whose magnitude grows continuously increasing $\lambda$ passing a critical point $\lambda^*(\alpha)=\frac{1}{\sqrt{\alpha-1}}$ (black line in the right panel of Fig~\ref{fig:ising_gauss_p2}), \black{which coincides with the stability condition of the paramagnetic solution computed in appendix~\ref{appendix-stability-parmagnet-ising}.}
Here one observes a continuous transition by increasing $\lambda$ between a state with $m=0$ and a magnetized state with $m > 0$. At small enough $\alpha$, the transition in $\lambda$ becomes first-order (with the paramagnetic state remaining stable up to $\lambda=\infty$ for  $\alpha<1$). For intermediate values of $\alpha$, the phase behavior is more complex: the system exhibits first a continuous transition from paramagnet to a non-trivial low-overlap state, followed by a first-order transition to a higher-overlap state.

\begin{figure}[h]
	\centering
	\includegraphics[width=0.495\textwidth]{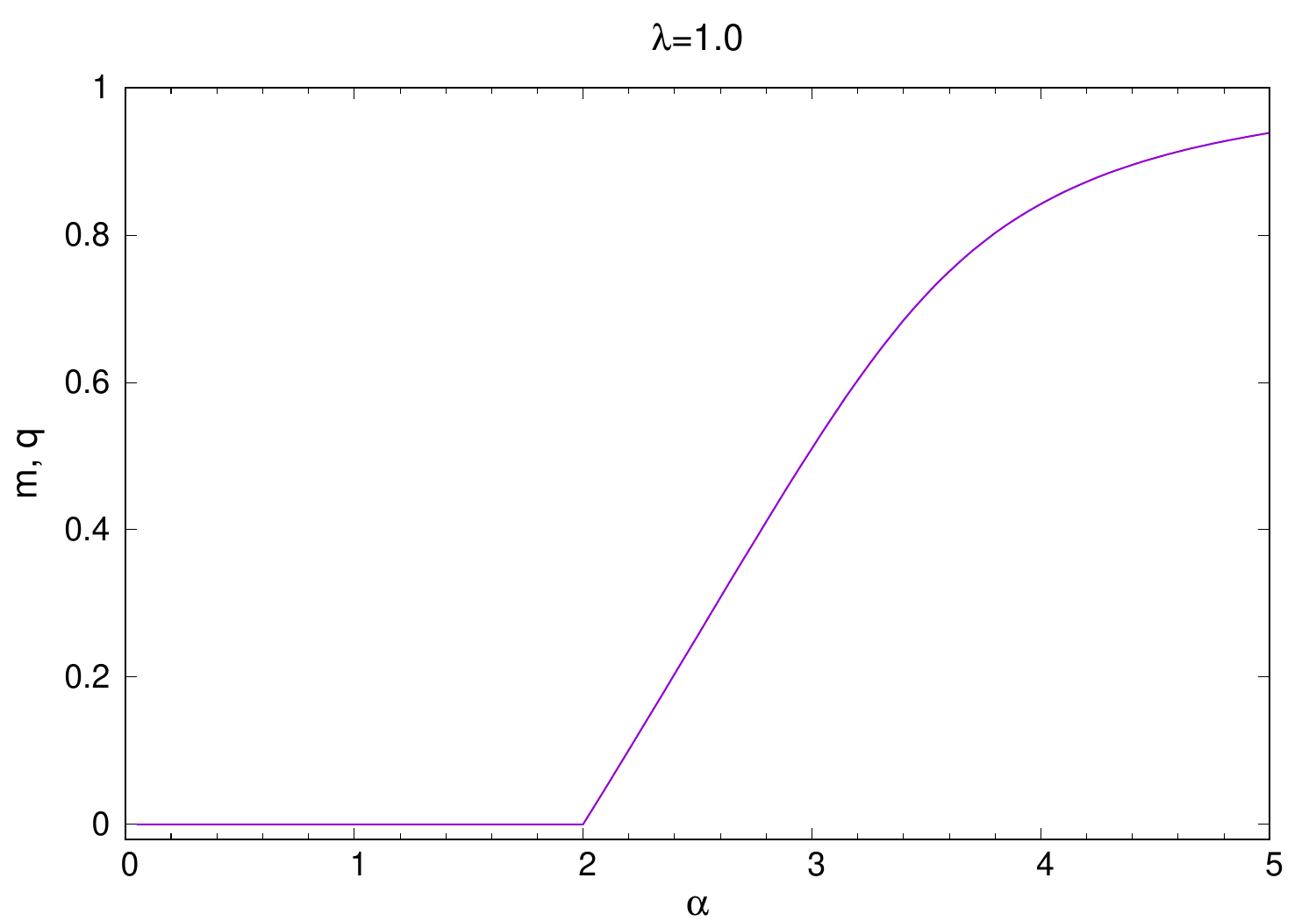}
	\includegraphics[width=0.495\textwidth]{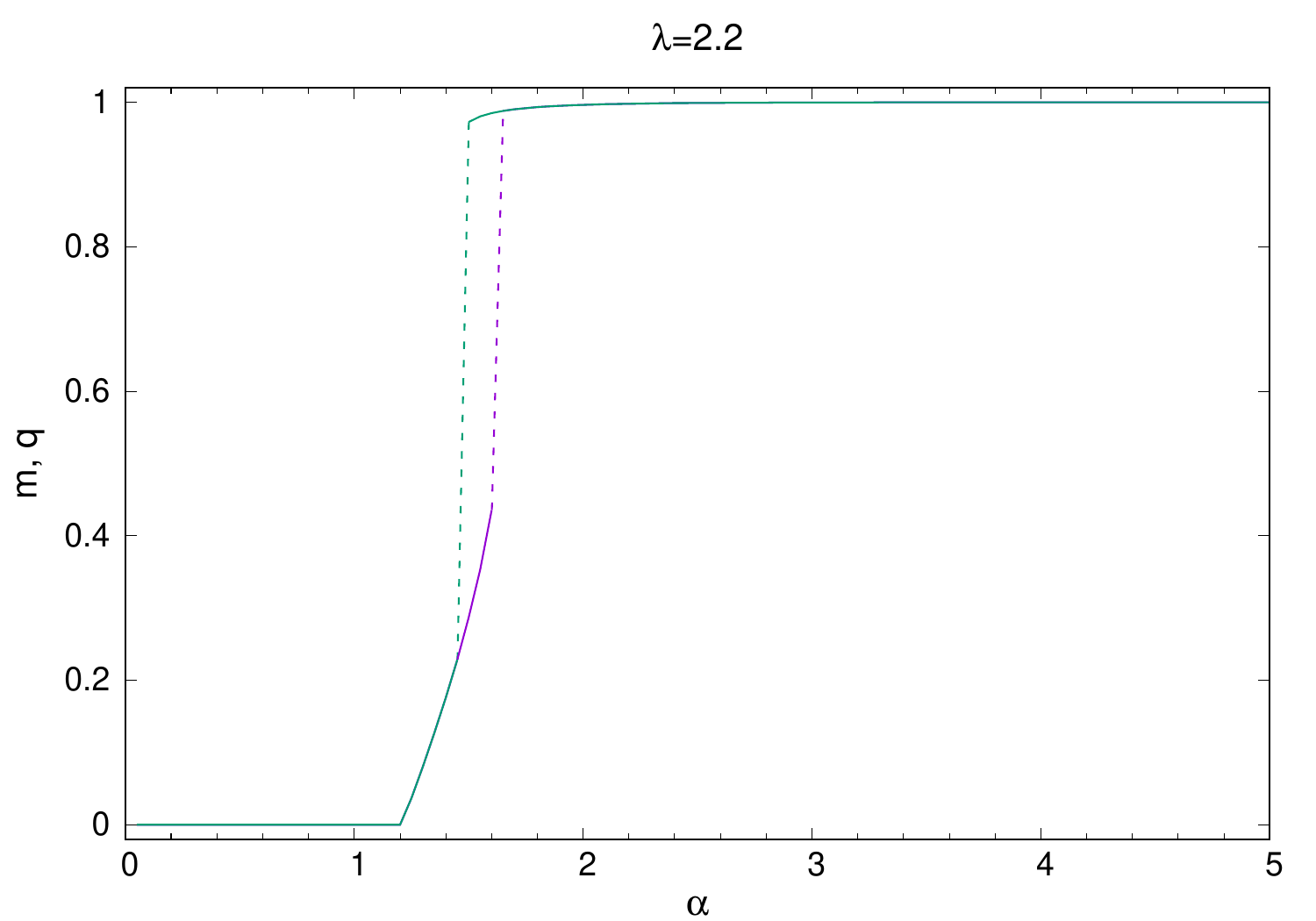}
	\includegraphics[width=0.495\textwidth]{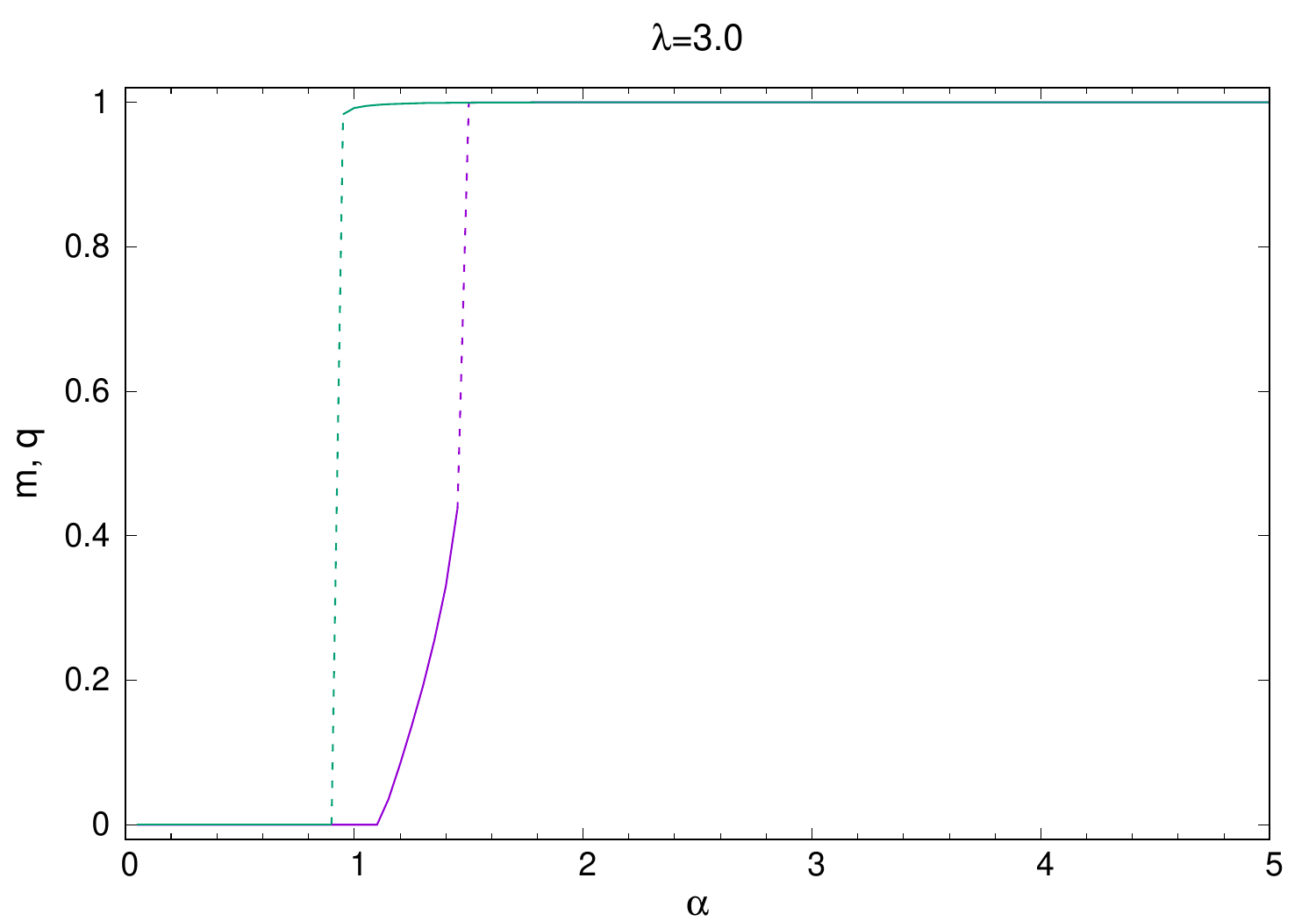}
	\includegraphics[width=0.495\textwidth]{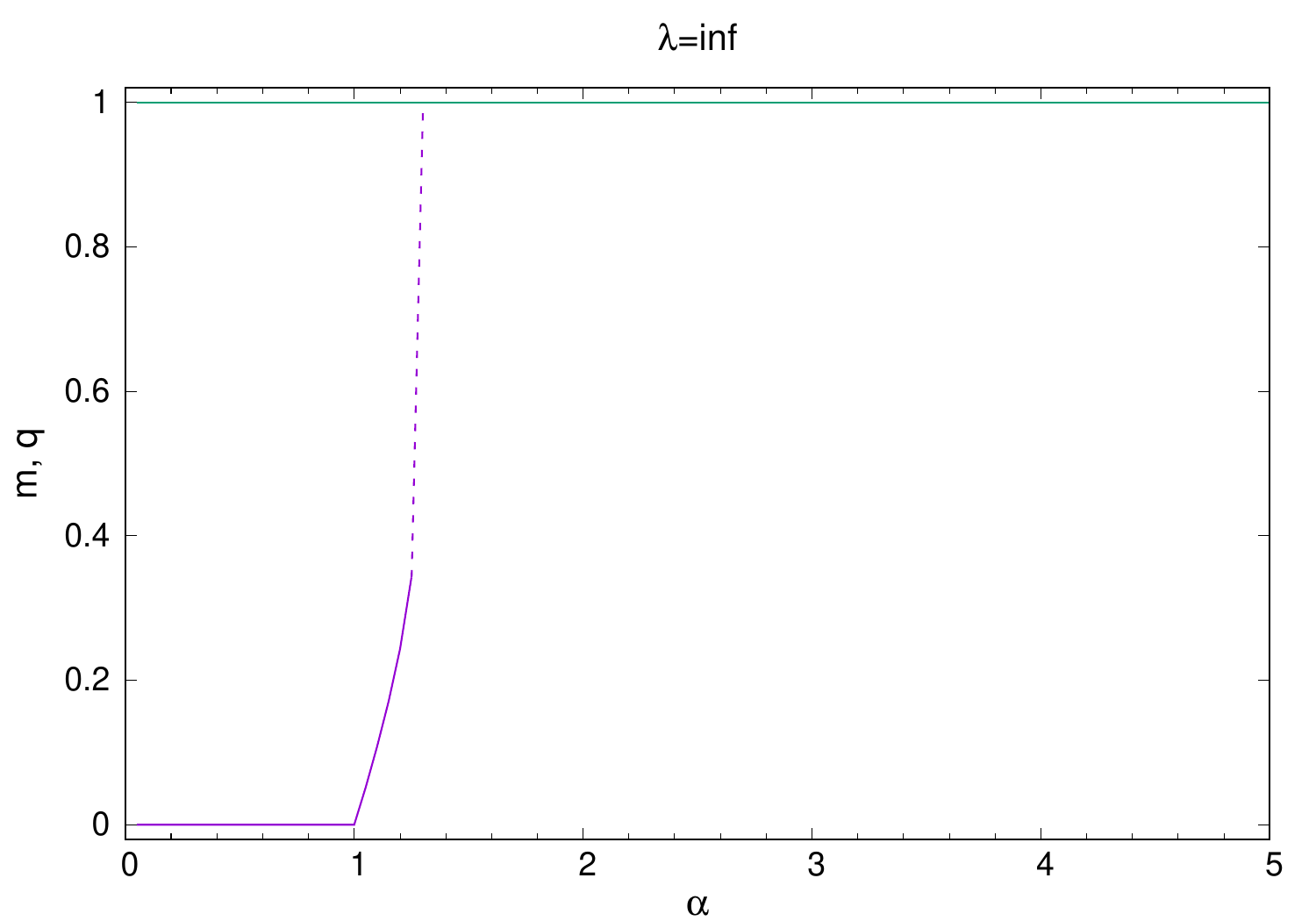}
	\caption{Order parameter $m=q$ in the Bayes optimal case for Ising prior, additive Gaussian noise and $p=2$. The order parameter is presented as a function of $\alpha$ for different values of $\lambda$. These figures present the same information as in Fig.~\ref{fig:ising_gauss_p2}, but shown in a different, complementary way. The first panel shows a unique branch of solutions, transitioning continuously from the paramagnetic one $m=0$. Increasing $\lambda$, in the second and third panels, also a high overlap branch (shown in green) appears discontinuously. Last panel shows the behaviour of the order parameter in the $\lambda\to\infty$ limit. The vertical dashed line coincides with the easy-to-hard threshold $\alpha_P$ for perfect reconstruction in the $\lambda\to\infty$ limit, while the boundary for the existence of the perfect recovery solution is given by $\alpha_s=0$.
	}
	\label{fig:ising_gauss_p2_mVSalpha}
\end{figure}
%%%%%%%%%%%%%%%%%%%%%%

The consequences of this phase diagram on the performance of inference are as follows (see the right panel Fig~\ref{fig:ising_gauss_p2}).  In Region I inference is hard due to the stability of the paramagnetic solution, in the sense that G-AMP, and presumably any polynomial time algorithms, will not be able to find a solution correlated with the signal when starting from uninformative (low overlap) initialization. Regions II and V can also be considered hard, but in a slightly more complicated way. In fact, it is easy to find a solution correlated with the signal more than a random guess, since the paramagnetic phase loses its stability. However, this solution is not optimal, since it corresponds to the low-overlap branch, while the high-overlap branch remains inaccessible in polynomial time. This means that the perfect reconstruction of the signal is impossible even in the $\lambda\to \infty$ limit when starting from the uninformative initialization. By solving the equation of state (\eq{eq:m_Ising_Bayesopt}) for $\lambda\to\infty$, we obtained that the green line separating Regions II and III asymptotically approaches $\alpha \approx 1.30$. This value defines the easy-to-hard threshold $\alpha_P$: in the whole region $\alpha<\alpha_P$ the perfect reconstruction of the signal is not possible even in the $\lambda\to\infty$ limit (noiseless limit), see also the last panel of Fig.~\ref{fig:ising_gauss_p2_mVSalpha}. Conversely, in regions III and IV inference becomes easy, as there is a unique solution which has a large overlap with the signal, which converges to the perfect reconstruction solution ($m=1$) in the limit $\lambda\to\infty$. 

\blue{Let us consider now a slightly different kind of question, that is whether perfect recovery of the planted solution is possible or not. More specifically, we are interested in understanding for which values of the parameters $\alpha$ and $\lambda$ the $m=1$ state exists and it is locally stable. There are two relevant limits one may think of: the $\alpha\to\infty$ limit (increasing number of observations) and the $\lambda\to\infty$ limit (corresponding effectively to the noiseless limit). Since the point of our setting is considering how sparse measurements can still allow meaningful inference, it is more natural and interesting for us to focus on the noiseless limit $\lambda\to\infty$. This suggests us to introduce a threshold $\alpha_s$ which represents the boundary of existence of the perfect recovery solution ($m=1$). In other words, it corresponds to the smallest value of $\alpha$ for which the $m=1$ solution exists in the limit $\lambda\to\infty$.  
}

\blue{Interestingly, as we can see from Fig.~\ref{fig:ising_gauss_p2_mVSalpha}, such a boundary is given in this case by $\alpha_s=0$, since the $m=1$ solution always exists in the $\lambda\to\infty$ limit. This is very in contrast to the $p=2$ Gaussian prior case as shown further below, where we will find $\alpha_s=1$ (meaning that the perfect recovery solution does not exist even in the noiseless limit if the fraction of observations is too small, due to the continuous nature of the variables).
}

\blue{We will refer in the following to $\alpha_s$ (where the subscript s stands for ``success'') also as the possible-to-impossible threshold. Some readers might find this terminology confusing, so let us explain it more carefully. In the context of statistical physics, a threshold called impossible-to-hard threshold is commonly defined with respect to the specific problem of accessing the planted state when starting from an uninformative initial condition~\cite{zdeborova2016statistical}. According to this setting, the simple existence and stability of the $m=1$ solution does not grant that inference will  be algorithmically feasible starting from an uninformative initialization. Here instead, the term ``possible'' does not imply algorithmic considerations, our goal being just to highlight the difference in the thermodynamic behaviour of the specific models under consideration. Thus we will simply call ``possible'' the region where the $m=1$ solution exists, and``impossible'' otherwise. In the end, $\alpha_s$ corresponds also to the location of the asymptote of the $\lambda_d(\alpha)$ curve (the spinodal point for the high-magnetization solution) of Fig.~\ref{fig:ising_gauss_p2} (right).
}

%Interestingly, another threshold $\alpha_s$ called possible-to-impossible threshold, which represents the boundary for the existence of the perfect recovery solution $(m=1)$, is here given by $\alpha_s=0$ since the $m=1$ solution always exists in the $\lambda\to \infty$ limit. This is very in contrast to the Gaussian prior case as shown below.

Next let us examine the performance of message passing algorithms. In Fig.~\ref{fig:G_AMP_Ising_Gaussian_p=2} we show time evolution of the overlaps $m^{t}$ and $q^{t}$ computed using G-AMP.
%%%%%%%%%%%%%%%%%%%%%%
\begin{figure}[h]
	\centering
	\includegraphics[width=0.495\textwidth]{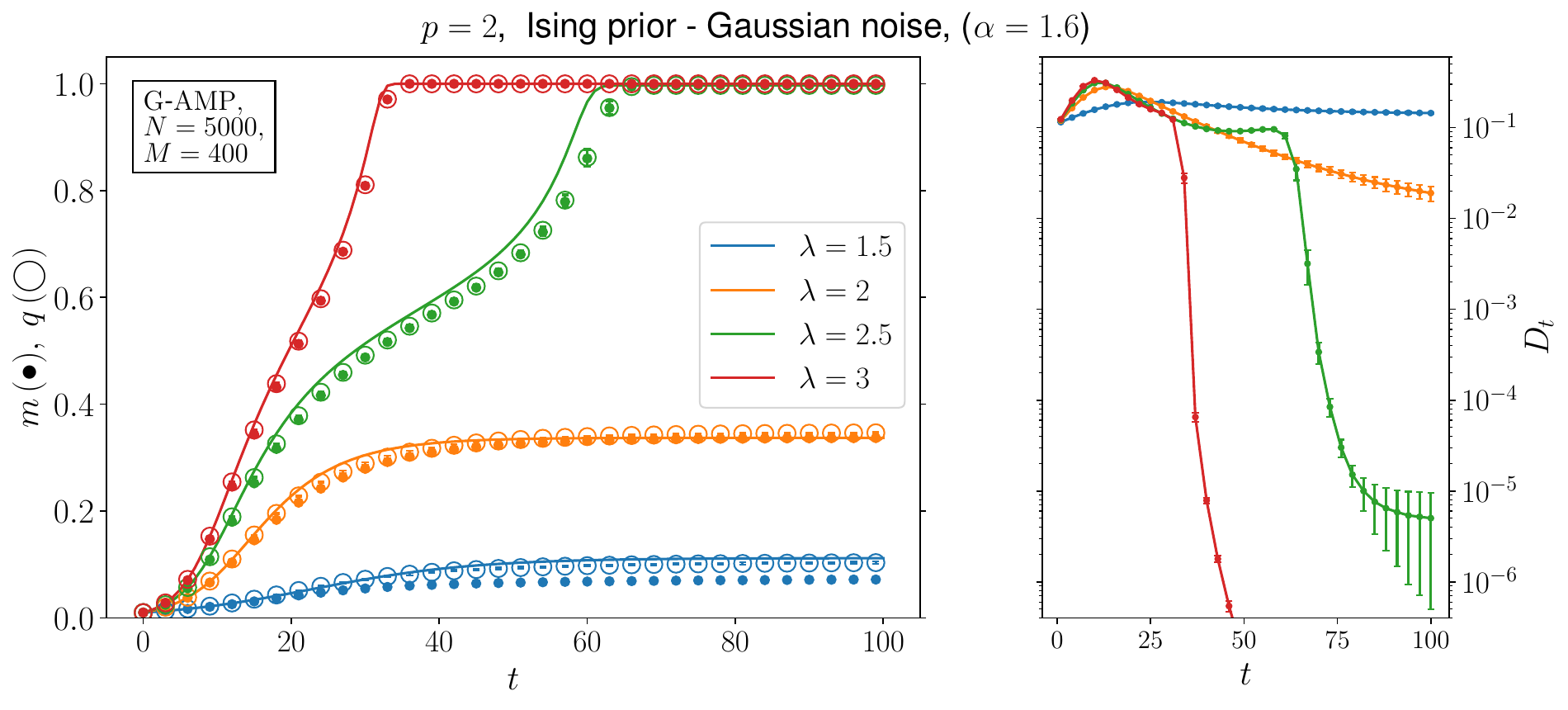}
	\includegraphics[width=0.495\textwidth]{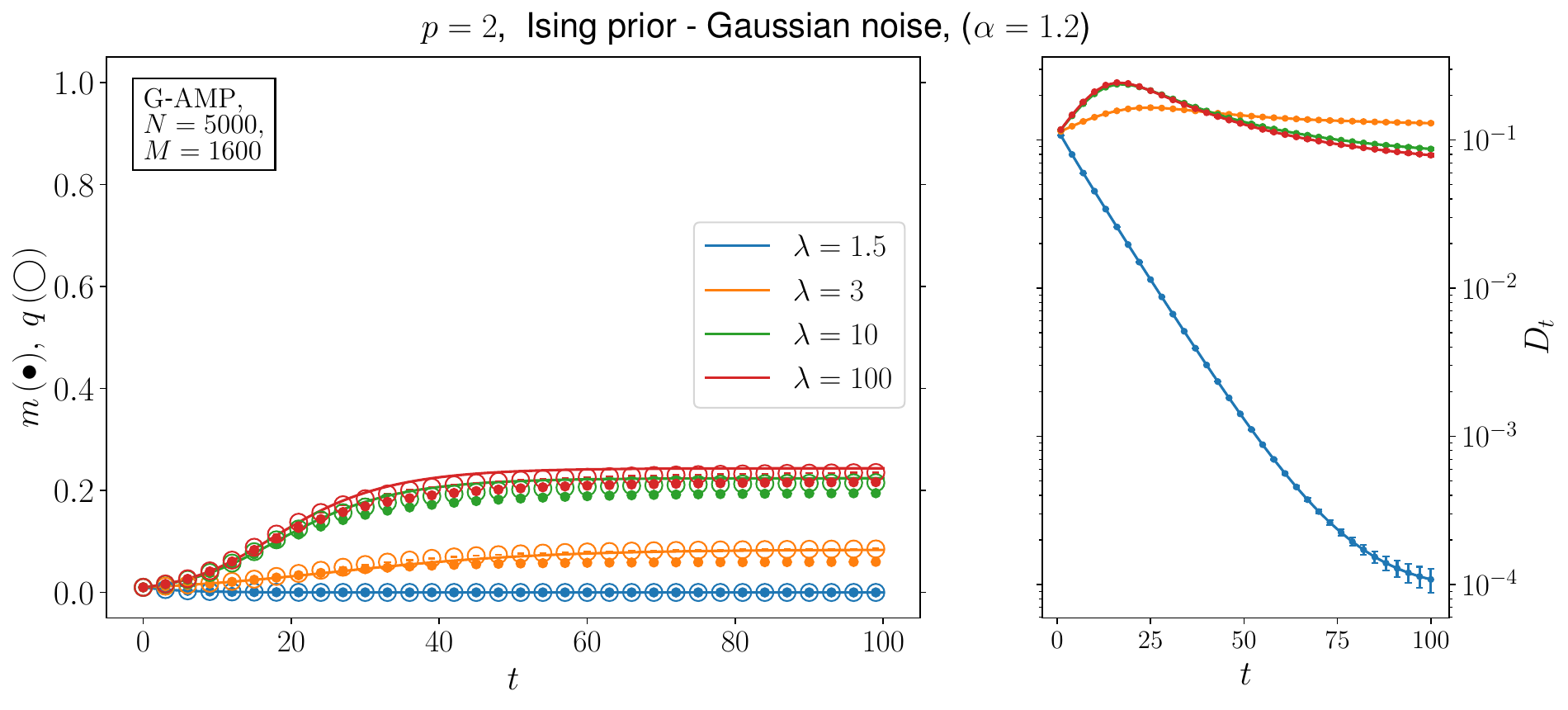}
	\includegraphics[width=0.495\textwidth]{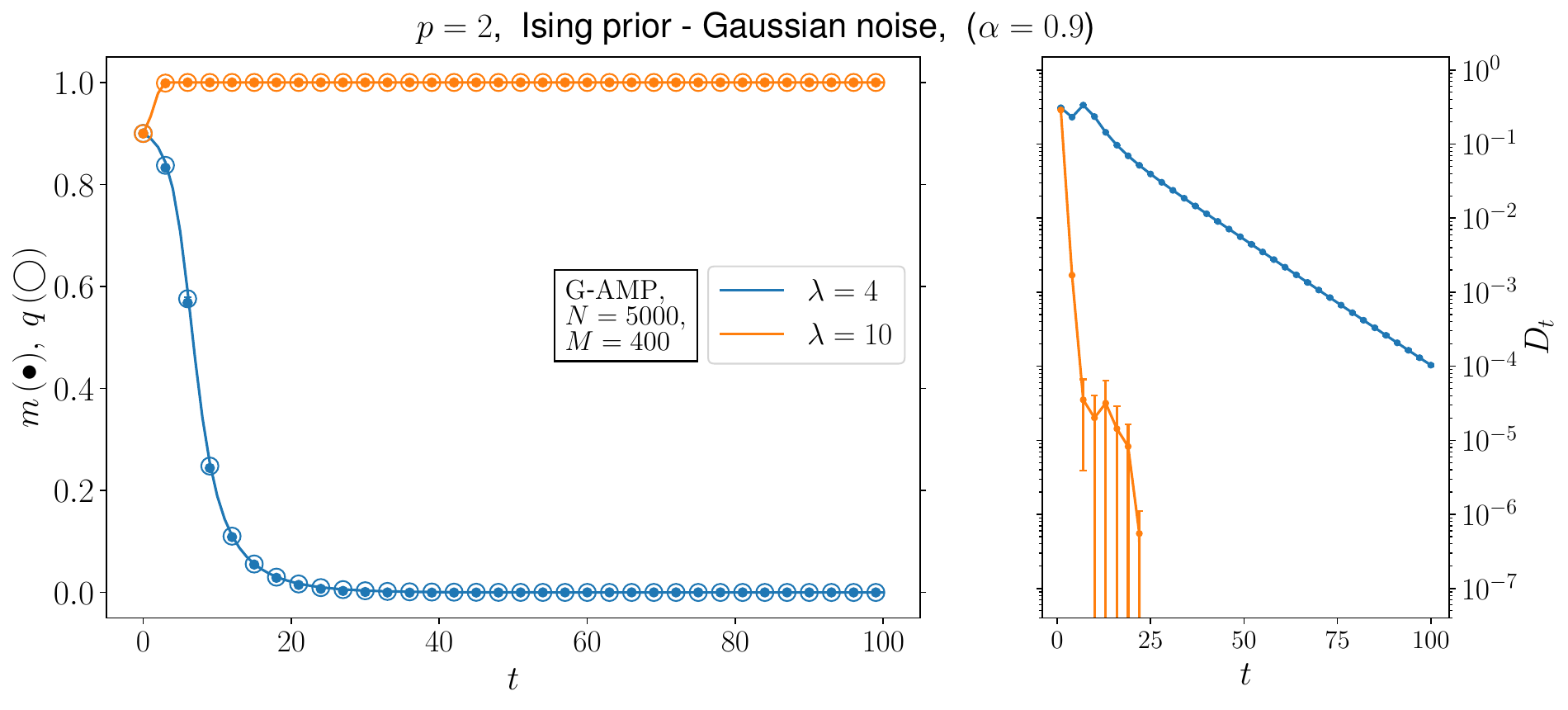}
	\caption{Evolution of the order parameters in the G-AMP algorithm for $\alpha=1.6$, $\alpha=1.2$  and $\alpha=0.9$ on random instances of the problem. The horizontal axis $t$ is for the time step.
          Data is averaged over 5 instances. Solid lines represent the State Evolution predictions. On the smaller panels, the convergence parameter is plotted.
       }
	\label{fig:G_AMP_Ising_Gaussian_p=2}
\end{figure}
%%%%%%%%%%%%%%%%%%%%%%
%For the $p=2$ case we used the random spreading factor $F_{\bs \mu}$ in order to
%break the rotational symmetry and assist convergence of the algorithms.
For the cases $\alpha=1.6$ and $\alpha=1.2$, uninformative initialization is used (see sec.~\ref{sec-initialization})
so that the initial configuration is chosen to be very far from the teacher's configuration.
The results are compared with the numerical solution of the SE equation. On the left plot, the case $\alpha=1.6$ is considered: for large enough values of $\lambda$ the system falls in Region III in which inference is easy. On the right panel, the case $\alpha=1.2$ is considered. This case falls in Region II, that is a hard phase for which the perfect reconstruction is not possible, when starting from an uninformative initialization. In the third panel, the case $\alpha=0.9<1$ is considered with the informative initial condition (the initialization parameter in sec.~\ref{sec-initialization} is set to be $a=0.9$): this is for examining the predicted possible-to-impossible threshold $\alpha_s=0$ in the limit $\lambda\to\infty$, and actually both the SE and GAMP results converge to the $m\approx 1$ solution when $\lambda$ is large ($\lambda=10$ is examined). For comparison, the case $\lambda=4$ is also examined, exhibiting the convergence to the paramagnetic solution.

The results of G-AMP compare well with the solution of the SE equation.
However, for small values of the overlap, it can be seen that the algorithm slows down. Discrepancies in this case may be attributed to finite $M$ corrections. See also Fig.~\ref{fig:finitesize_Gaussian_Gaussian_p=2} in the appendix for a similar case. The convergence parameter $D_t$ is defined in \eq{eq-def-Dt} in appendix ~\ref{sec-comparisons-finite-size}. Finally, in Fig.~\ref{fig:r_BP_Ising_Gaussian_p=2} in  appendix
~\ref{sec-comparisons-finite-size},
we also show results of
r-BP algorithm compared with  G-AMP. For smaller $N$ r-BP is more accurate,
but increasing $N$ G-AMP drastically improves.

%%%%%%%%%%%%%%%%%%%%%%
%%%%%%%%%%%%%%%%%%%%%%%%%%%%%%%%%%%%%%%%%%%%%%%%%%%%%
\subsubsection{$p=3$ case}
\label{subsec-ising-gaussian-p=3}
In this case, a discontinuous transition between the paramagnetic $m=0$ and
the magnetized phase $m >0$ is universally found. We expect this is the case for $p > 2$.
In Fig.~\ref{fig:ising_gauss_p3} we show the phase diagram obtained by analyzing the equation of states.
%%%%%%%%%%%%%%%%%%%%%%
\begin{figure}[h]
\centering
\includegraphics[width=0.450\textwidth]{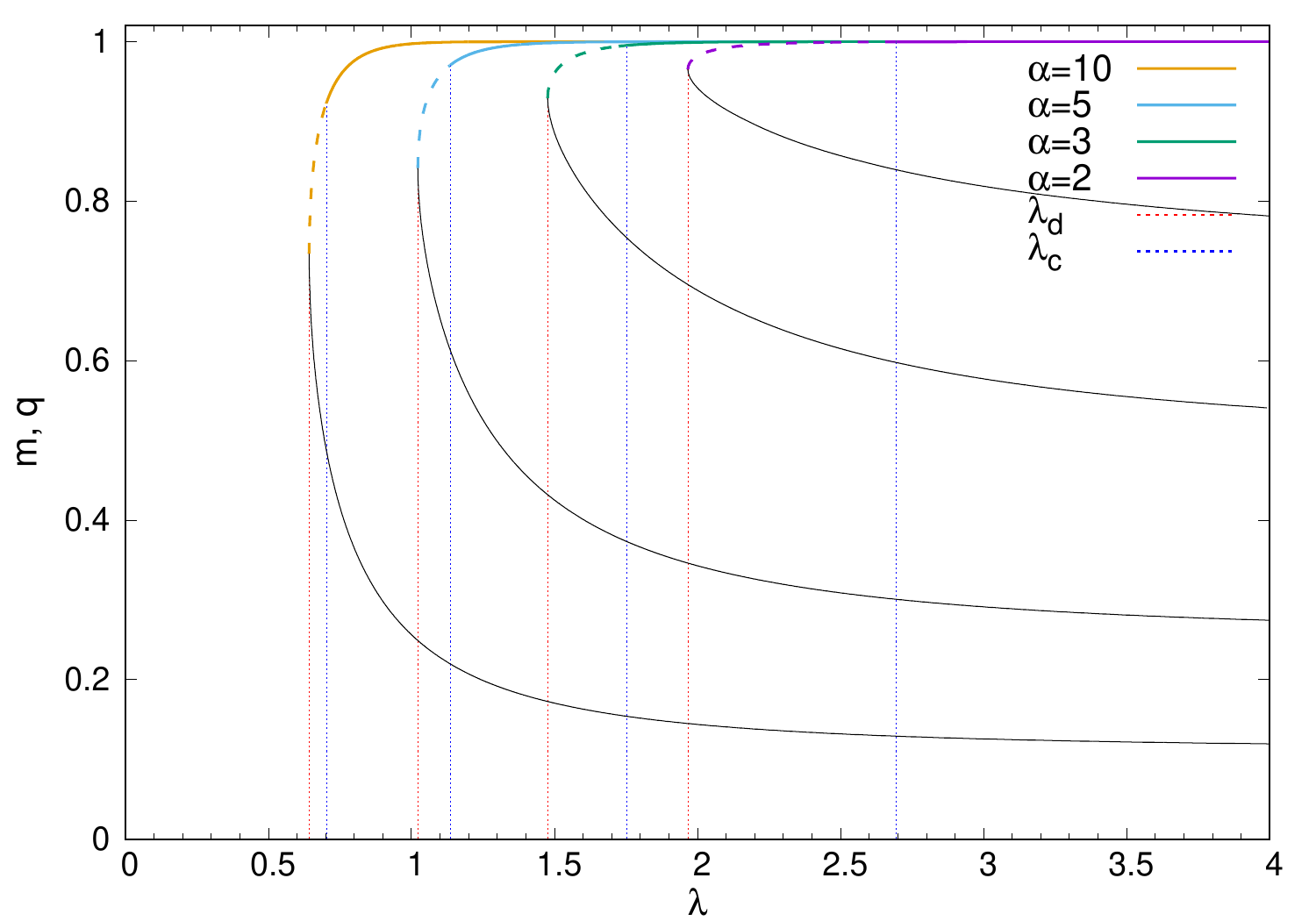}
\includegraphics[width=0.410\textwidth]{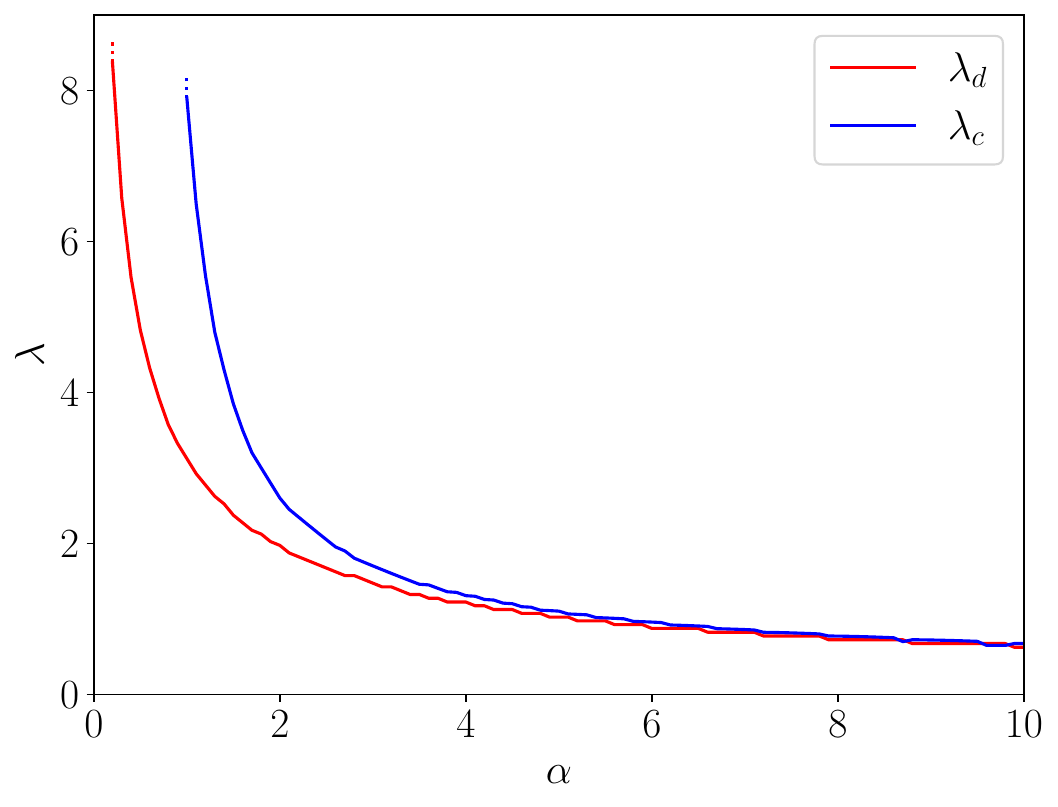}
\caption{Left: order parameter $m=q$ in the Bayes optimal case for Ising prior, Gaussian noise and $p=3$. \blue{The dashed portion of the lines indicate the values of $\lambda$ for which the magnetized phase is metastable with respect to the paramagnetic one.  Black continuous lines indicate branches of unstable solutions.} Right: phase diagram in the $\alpha-\lambda$ plane. The paramagnetic state $m=0$ is always locally stable.
}
\label{fig:ising_gauss_p3}
\end{figure}
%%%%%%%%%%%%%%%%%%%%%%
We define $\lambda_d(\alpha)$ the spinodal point of the non-trivial solution, that coexists with the trivial paramagnetic one since the transition is of first-order type.
The true thermodynamic transition point $\lambda_c(\alpha)$ is defined as the point where the free energy of the highly magnetized state becomes equal to the one of the paramagnetic state.

The magnetization becomes very large in the magnetized phase meaning that the quality of the inference is good. However, unlike some sparse systems with $c=O(1)$ \cite{zdeborova2016statistical},
the paramagnetic state $m=0$ is always stable \black{(see sec.~\ref{appendix-stability-parmagnet-ising} )} in the dense limit $c \gg 1$ for $p>2$\cite{yoshino2018}. This is problematic from the algorithmic point of view, as it constitutes a hard phase for inference, meaning that there is no easy-to-hard threshold at finite $\alpha$. Meanwhile, the possible-to-impossible threshold is $\alpha_s=0$ since the $m=1$ solution always exists at any finite $\alpha$ \blue{in the noise-less limit $\lambda\to\infty$}. Hence, the computational gap is quite serious in this case. As a way to circumvent this, we propose in sec.~\ref{sec-p=2+3} to mix $p=2$ interactions to destabilize the paramagnetic state.

%{\bf (Todo): discuss style adjustments to the  phase diagram $\alpha-\lambda$ plane.}
We report in Fig.~\ref{fig:G_AMP_Ising_Gaussian_p=3} the behavior of the G-AMP using
the informative (large overlap with the signal) initial condition  (see sec.~\ref{sec-initialization}).
%%%%%%%%%%%%%%%%%%%%%%
\begin{figure}[h]
	\centering
	\includegraphics[width=0.7\textwidth]{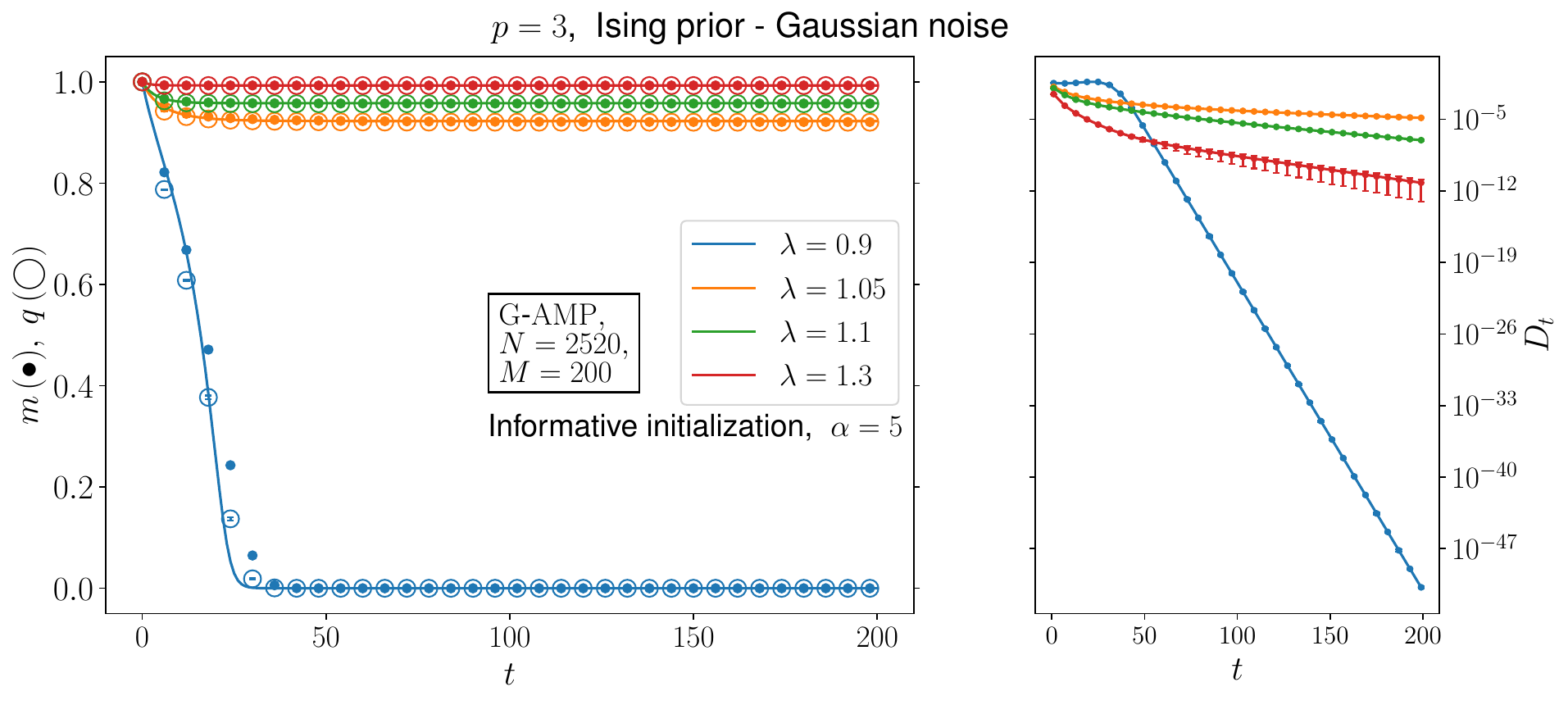}
        o	\caption{Evolution of the order parameters in the G-AMP algorithm for $\alpha=5$ on random instances of the problem with $N=2520$, $M=200$. Data is averaged over 10 instances. Solid lines represent the SE predictions. On the right panel, the convergence parameter is plotted.
        }
	\label{fig:G_AMP_Ising_Gaussian_p=3}
\end{figure}
%%%%%%%%%%%%%%%%%%%%%%
The results compare well with the solution of the SE equations.
In Fig.~\ref{fig:r_BP_Ising_Gaussian_p=3} in the appendix we show results of
r-BP algorithm which compare well with G-AMP.

%%%%%%%%%%%%%%%%%%%%%%%%%%%%%%%%%%%%%%%%%%%%%%%%%%%%%
%%%%%%%%%%%%%%%%%%%%%%%%%%%%%%%%%%%%%%%%%%%%%%%%%%%%%
\subsection{Gaussian prior and additive Gaussian noise}
\label{sec-gaussian-gaussian-noise}
Let us turn to the case of Gaussian prior. In sec. \ref{sec-Gaussian-Gaussian-replica} we obtained results
for the case of Gaussian prior and additive noise by the replica approach. We focus on the Bayes optimal case, and the corresponding equation of state and the free energy expression are given by \eq{eq:m_Bayesopt_gauss} and \eq{eq:free_ene_gaussian_Bayesopt}, respectively. We also checked in sec.~\ref{appendix-stability-parmagnet-gauss} the stability of the paramagnetic solution by computing the second derivative of the free energy. It turns out that for $p>2$ the paramagnetic state is stable for any $\alpha,\lambda>0$.

For the message passing approach, we examine the r-BP algorithm (Algorithm~\ref{alg:p-ary_r-BP})
and G-AMP algorithm (Algorithm~\ref{alg:p-ary_G-AMP}) with the input function \eq{eq:HRTD-input_Spherical}  and the output function \eq{eq:HRTD-output_AWGN_Gaussian}. The SE equations are \eq{eq:HRTD-SE_EqState_m_Gaussian}-\eq{eq:HRTD-SE_EqState_Q_Gaussian} combined with \eq{eq-SE-q-additive-noise}-\eq{eq-SE-Xi-additive-noise}, showing a consistency with the replica formula. 

%%%%%%%%%%%%%%%%%%%%%%%%%%%%%%%%%%%%%%%%%%%%%%%%%%%%%
\subsubsection{$p=2$ case}
The value $m=0$ is always a solution to~\eqref{eq:m_Bayesopt_gauss}. Two other solutions arise by solving the remaining second order equation in $m$, to find
\begin{equation}
    m_\pm=\frac{\alpha}{2}\pm\gamma(\alpha,\lambda), \quad\quad \gamma(\alpha,\lambda)\equiv\frac{1}{2}\sqrt{(\alpha-2)^2+\frac{4}{\lambda^2}}.
\end{equation}
Using $\gamma(\alpha,\lambda)\geq\frac{1}{2}\sqrt{(\alpha-2)^2}=\frac{\lvert\alpha-2\lvert}{2}$, one can easily show that $m_+\geq 1$ $\forall \alpha$, meaning that it represents an unphysical solution. The correct solution is thus $m=\frac{\alpha}{2}-\gamma(\alpha,\lambda)$. It satisfies, as a function of $\lambda$, $m\leq 1$ for $\alpha\geq 2$ and $m\leq\alpha-1$ for $\alpha<2$, implying that the perfect reconstruction could be possible only when $\alpha \geq 2$ and only the paramagnetic solution exists for $\alpha<1$. In the limit $\lambda\to\infty$, one has $m=0$ for $\alpha<1$, $m=\alpha-1$ for $1\leq\alpha<2$ and $m=1$ for $\alpha\geq2$. The critical value of $\lambda$ above which the non-trivial solution $m>0$ appears can also be obtained by assuming the continuity of the solution: the critical point $\lambda_c(\alpha)$ is thus derived from the equation $m=\frac{\alpha}{2}-\gamma(\alpha,\lambda_c)=0$, leading to
\begin{equation}
    \lambda_c(\alpha)=\frac{1}{\sqrt{\alpha-1}}.
\end{equation}
This also coincides with the stability condition for the paramagnetic solution obtained in sec.~\ref{appendix-stability-parmagnet-gauss}. The result of this analysis is displayed in Fig.~\ref{fig:gauss_gauss_p2}.
%%%%%%%%%%%%%%%%%%%%%%
\begin{figure}[h]
	\centering
	\includegraphics[angle=0,width=0.45\textwidth]{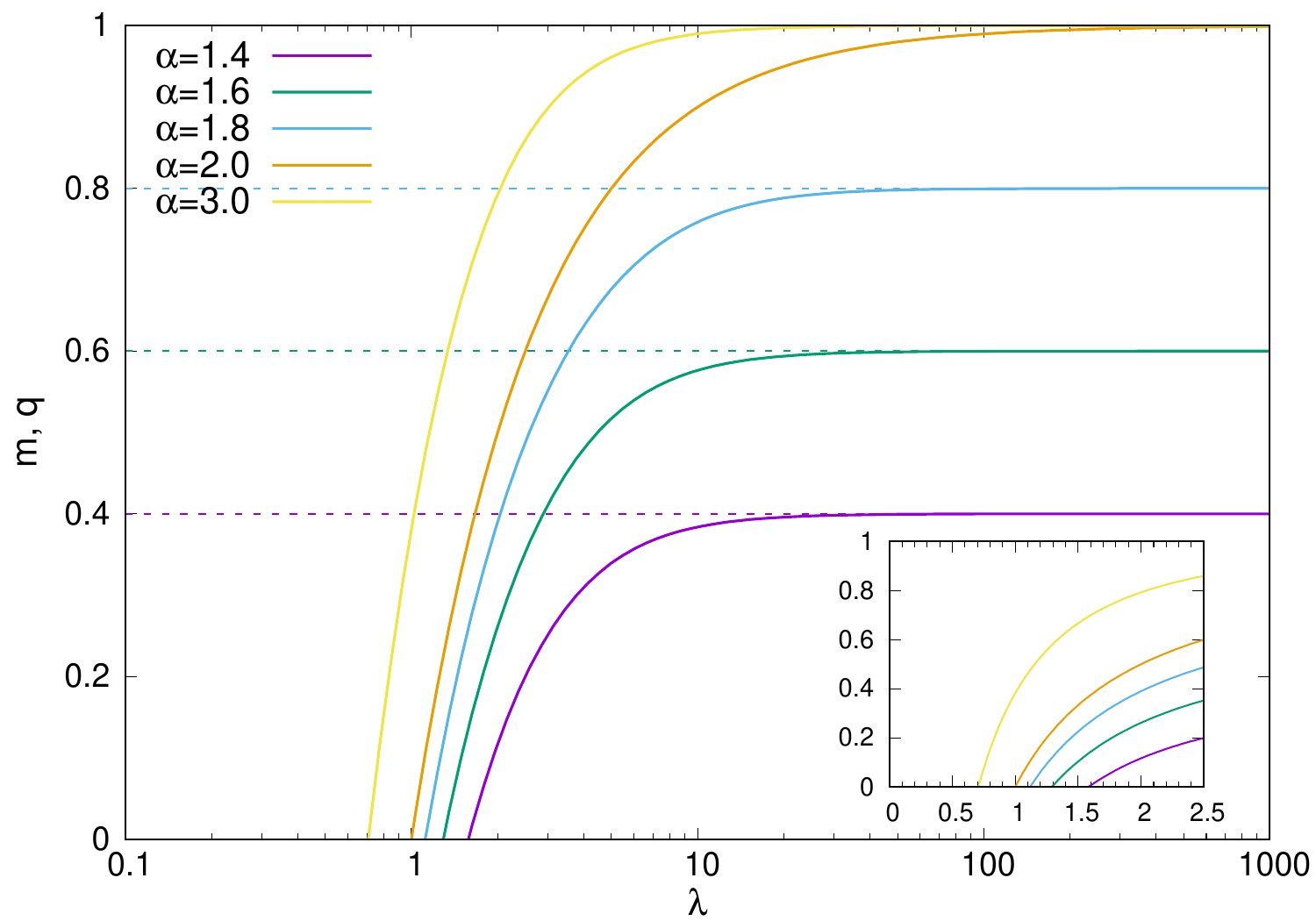}
	\includegraphics[width=0.45\textwidth]{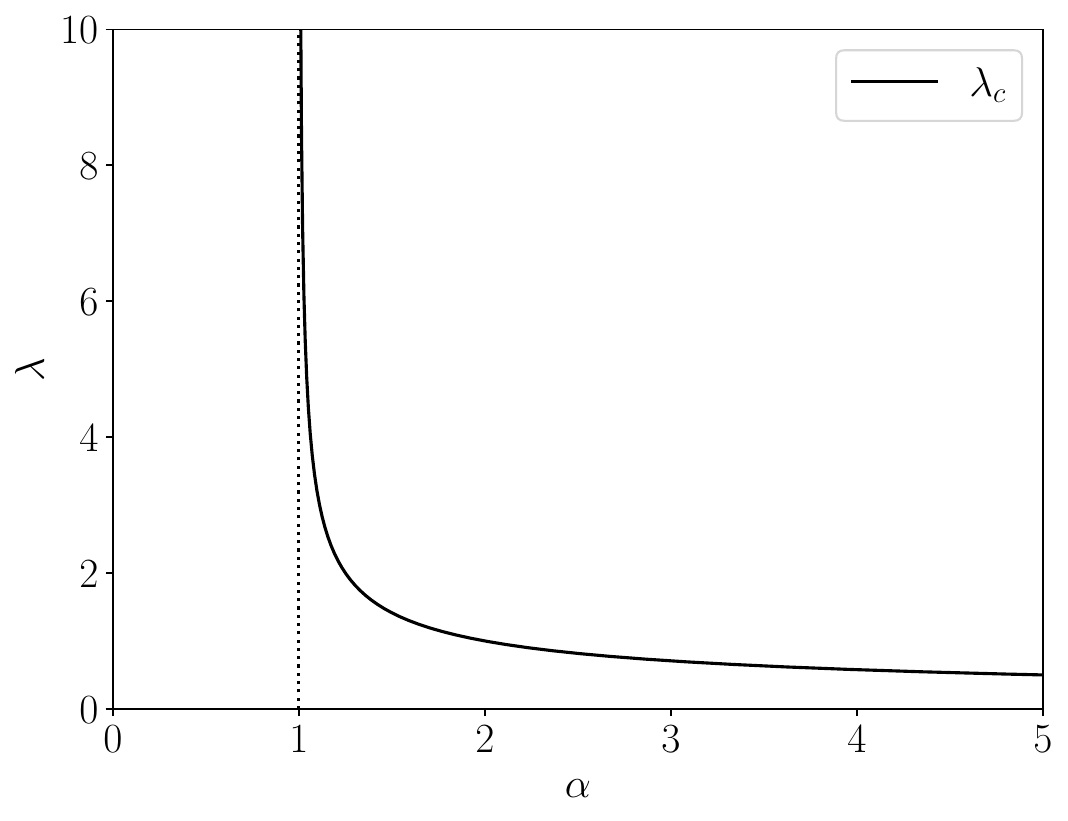}
	\caption{Left: Order parameter $m=q$ in the Bayes optimal case for the Gaussian prior, Gaussian noise and $p=2$. Inset: zoom on the region around $\lambda_c(\alpha)$. Right: phase diagram in the $\alpha-\lambda$ plane.}
	\label{fig:gauss_gauss_p2}
\end{figure}
%%%%%%%%%%%%%%%%%%%%%%

From the algorithmic point of view, one observes a phase where inference is impossible for $\lambda < \lambda_c(\alpha)$, and a phase where inference is easy for $\lambda > \lambda_c(\alpha)$. This yields the possible-to-impossible threshold $\alpha_s$ in the limit $\lambda \to \infty$ as $\alpha_s=1$, which is in contrast to the Ising prior case where the perfect reconstruction solution exists for any finite $\alpha$. Meanwhile, the easy-to-hard threshold for perfect reconstruction in the limit $\lambda \to \infty$ is given by $\alpha_P=2$ from the above discussion. 

In Fig.~\ref{fig:G_AMP_Gaussian_Gaussian_p=2} we show time evolution of the
overlaps $m^{t}$ and $q^{t}$ (points) computed using G-AMP,
with informative and uninformative initialization, at $N=2500,M=400,\alpha=1.6$. 
%%%%%%%%%%%%%%%%%%%%%%
\begin{figure}[h]
\centering
\includegraphics[width=0.45\textwidth]{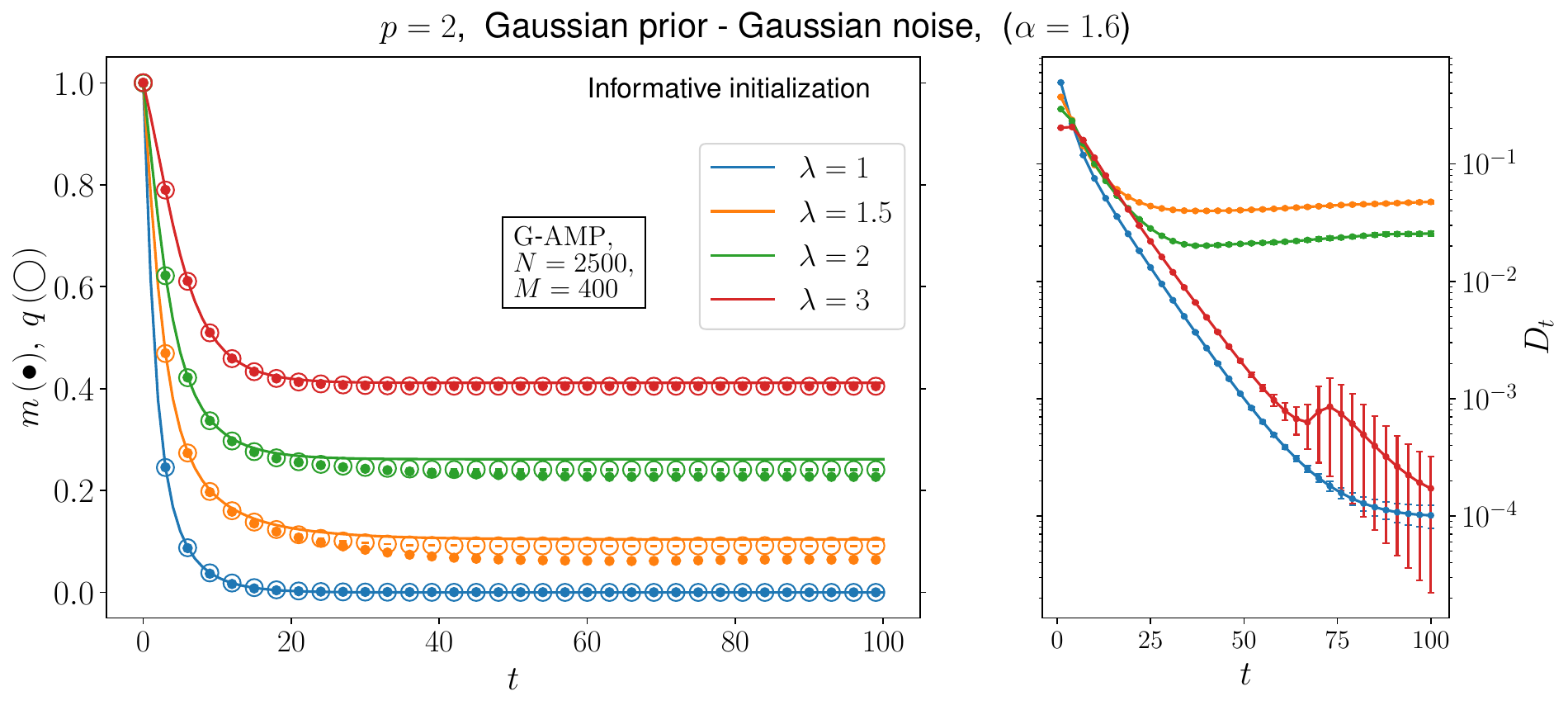}
\includegraphics[width=0.45\textwidth]{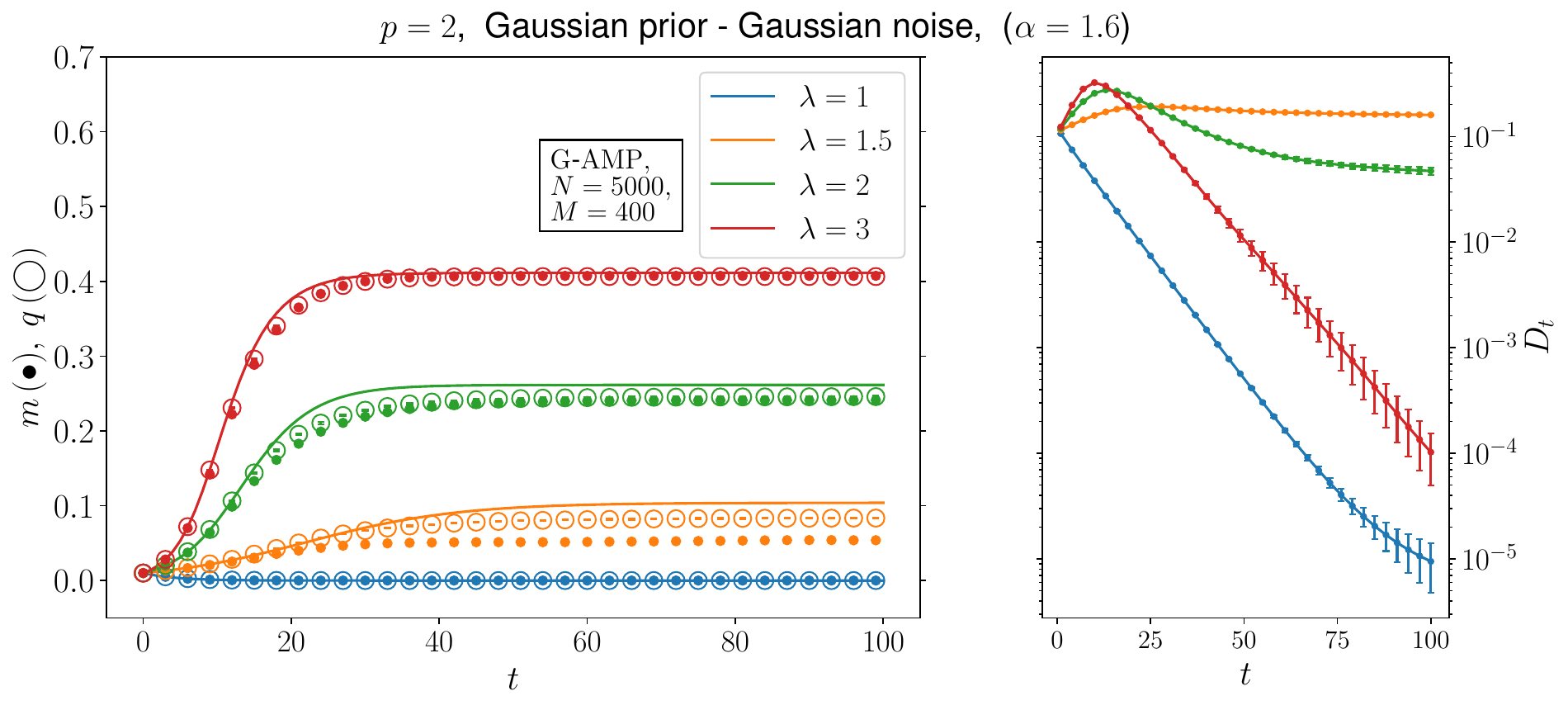}
\caption{Evolution of the order parameters in the G-AMP algorithm for $\alpha=1.6$ on random instances of the problem. Two different initializations (informative, uninformative) are considered. Solid lines represent the State Evolution prediction.}
\label{fig:G_AMP_Gaussian_Gaussian_p=2}
\end{figure}
%%%%%%%%%%%%%%%%%%%%%%
The results compare well with the solution of the SE equation (solid lines). However for small $\lambda$, it can be seen that the convergence parameter $D_t$ does not vanish. This is a general feature that we observe throughout the numerical simulation. We presume this is related to finite size corrections becoming particularly relevant close to the transition point. See also the discussion in sec.~\ref{sec-comparisons-finite-size}.

In Fig.~\ref{fig:r_BP_Gaussian_Gaussian_p=2} in the appendix we show results of
r-BP algorithm for the informative initialization which can be compared well with  G-AMP. 

%%%%%%%%%%%%%%%%%%%%%%%%%%%%%%%%%%%%%%%%%%%%%%%%%%%%%
\subsubsection{$p=3$ case}
Again, the value $m=0$ is always a solution to~\eqref{eq:m_Bayesopt_gauss}. A non-trivial $m\neq0$ solution has to satisfy $1 + \lambda^2(1-m^3)+\alpha\lambda^2m^2=\alpha\lambda^2m$, from which
\begin{equation}
    \lambda(m,\alpha) = \frac{1}{\sqrt{m^3-\alpha m^2+\alpha m-1}}.
    \label{eq:p=3_implicit_m}
\end{equation}
The plot of $m$ as a function of $\lambda$ can be obtained by inverting~\eqref{eq:p=3_implicit_m}, as shown in Fig.~\ref{fig:Gaussian_gauss_p3}.
%%%%%%%%%%%%%%%%%%%%%%
\begin{figure}[h]
	\centering
	\includegraphics[width=0.45\textwidth]{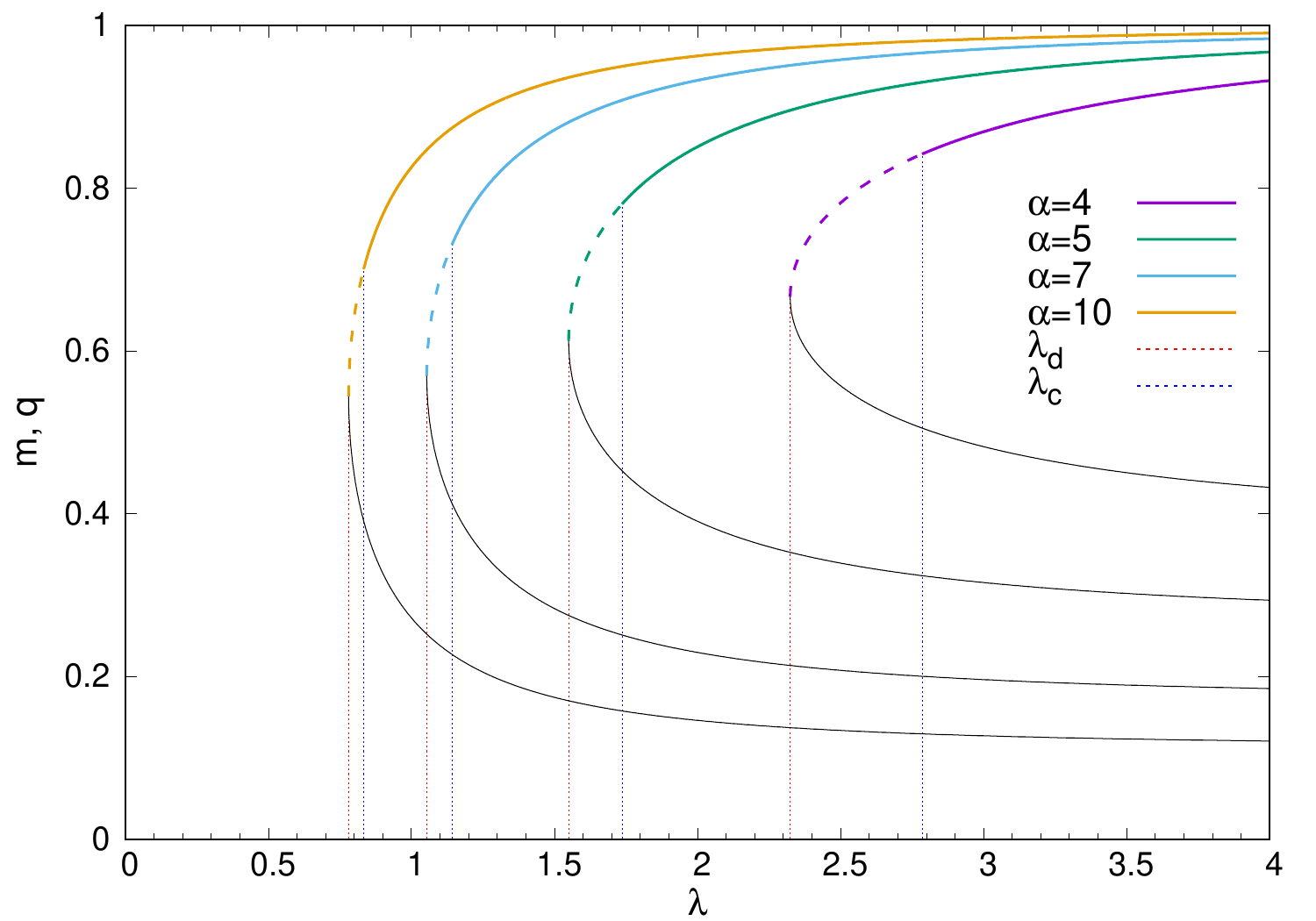}
	\includegraphics[width=0.45\textwidth]{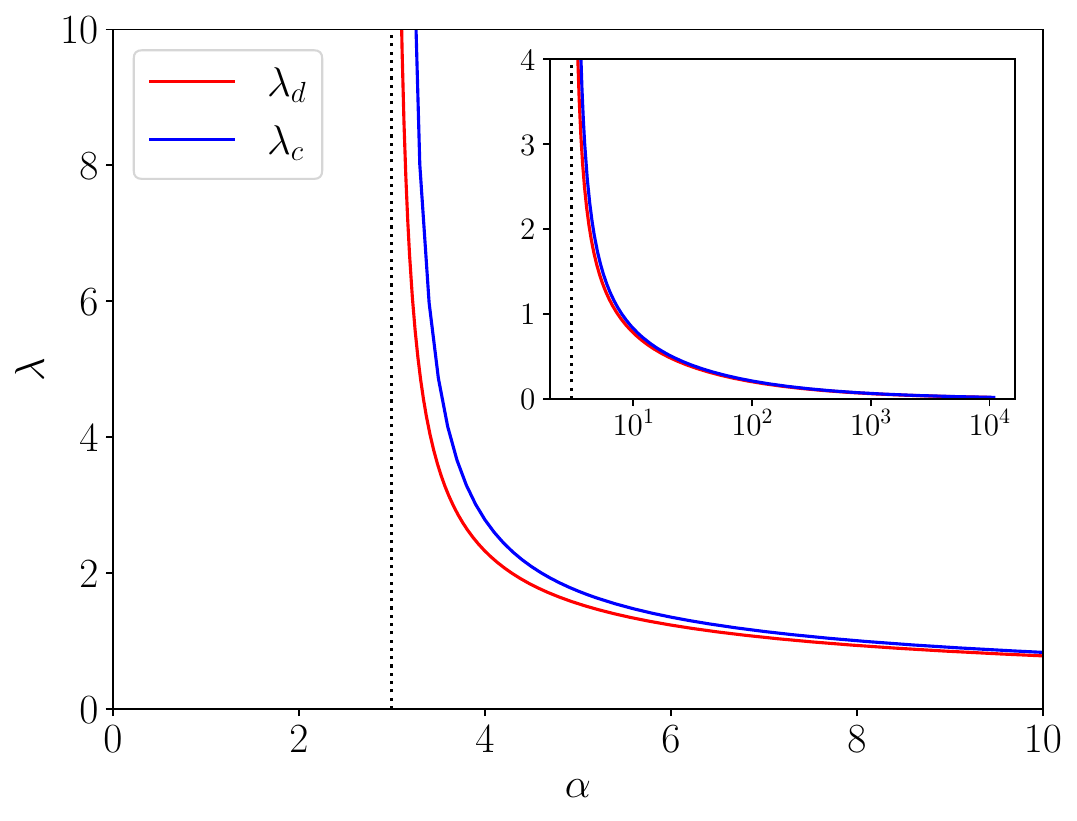}
	\caption{Left: Order parameter $m=q$ in the Bayes optimal case for the Gaussian prior, Gaussian noise and $p=3$. \blue{The dashed portion of the lines indicates a metastable magnetized state associated to a first-order transition. Black continuous lines indicate branches of unstable solutions.}
		Right: phase diagram in the $\alpha-\lambda$ plane.}
	\label{fig:Gaussian_gauss_p3}
\end{figure}
%%%%%%%%%%%%%%%%%%%%%%
Notice that the radicand at denominator of~\eqref{eq:p=3_implicit_m} can be written as $(m-1)(m^2-(\alpha-1)m+1)$. This implies that when $\Delta=(\alpha-1)^2-4=(\alpha+1)(\alpha-3)<0$, that is for $\alpha<3$, the radicand is negative, and the non-trivial solution cannot exist, implying that the possible-to-impossible threshold is $\alpha_s=3$ in this case. 

For $\alpha>\alpha_s$ we can obtain the spinodal point $\lambda_d(\alpha)$ of the non-trivial solution by imposing $\frac{\partial}{\partial m}\lambda(m,\alpha)=0$, that results in the second order equation $3m^2-2\alpha m + \alpha=0$. Of the two branches of solutions, we keep the one compatible with the physical constraint $0\leq m \leq 1$, which gives for the value of the order parameter at the spinodal point
\begin{equation}
    m_d(\alpha)=\frac{\alpha-\sqrt{\alpha(\alpha-3)}}{3}. 
\end{equation}
The dynamical transition line is then obtained as $\lambda_d(\alpha)=\lambda(m_d(\alpha),\alpha)$. In the large $\alpha$ limit $m_d(\alpha)\to\frac{1}{2}$, and one can observe that $\lambda_d(\alpha)$ is actually well approximated by $\lambda_d(\alpha)\sim\sqrt{\frac{8}{2\alpha-7}}$, being the relative error of the order of $0.5\%$ already for $\alpha=10$.

The point $\lambda_c(\alpha)$ signals a first order thermodynamic transition where the ferromagnetic state becomes dominant with respect to the paramagnetic one. We have the difference in free energy as
\begin{equation}
    f_{\rm RS}(m=0)-f_{\rm RS}(m_*)=\frac{1}{2}\ln(1-m_*) + \frac{m_*}{2} - \frac{\alpha}{p}\ln\left(1-\frac{\lambda^2}{1+\lambda^2} \,m_*^p\right),
\end{equation}
where $m_*$ is given by the non-trivial solution of Eq.~\eqref{eq:p=3_implicit_m}. The critical point $\lambda_c(\alpha)$ is computed from the zero point of this free energy difference. The two lines $\lambda_c(\alpha),\lambda_d(\alpha)$ are plotted in the right panel of Fig.~\ref{fig:Gaussian_gauss_p3}. Since the paramagnetic solution is stable for the whole $\alpha>\alpha_s$ region, the easy-to-hard threshold does not exist at finite $\alpha$ as in sec. \ref{subsec-ising-gaussian-p=3}.

In Fig.~\ref{fig:G_AMP_Gaussian_Gaussian_p=3} we show the time evolution of the
overlaps $m^{t}$ and $q^{t}$ computed using G-AMP at $N=5040,M=400,\alpha=5$ with several values of $\lambda$. Here the informative initialization is used. 
%%%%%%%%%%%%%%%%%%%%%%
\begin{figure}[h]
	\centering
	\includegraphics[width=0.7\textwidth]{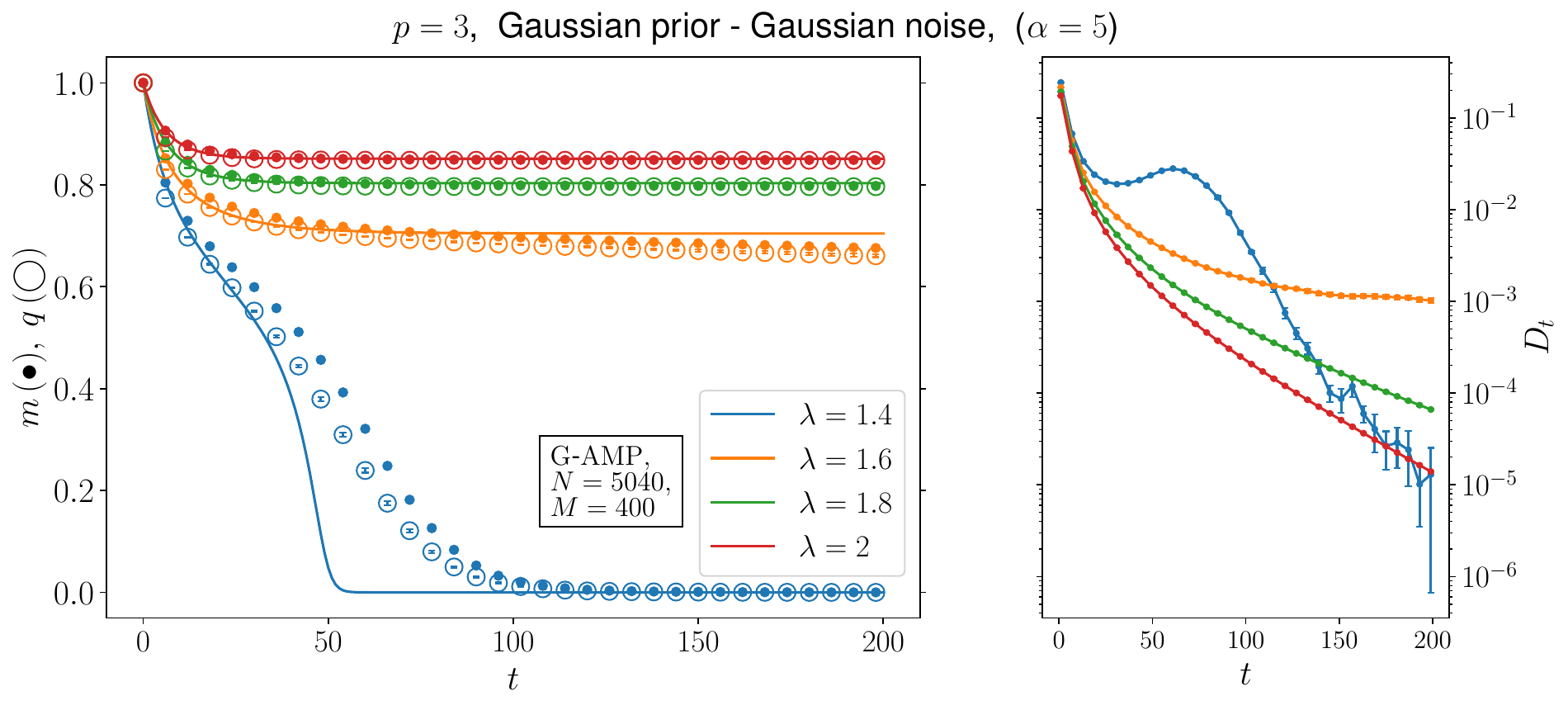}
	\caption{Evolution of the order parameters in the G-AMP algorithm for $\alpha=5$ on random instances of the problem with $N=5040$, $M=400$. Data is averaged over 10 instances. Solid lines represent the State Evolution predictions. On the right panel, the convergence parameter is plotted.
	}
	\label{fig:G_AMP_Gaussian_Gaussian_p=3}
\end{figure}
%%%%%%%%%%%%%%%%%%%%%%
The results of G-AMP compare well with the solution of the SE equation (\eq{eq:HRTD-SE_EqState_m_Gaussian}-\eq{eq:HRTD-SE_EqState_Q_Gaussian} combined with with \eq{eq-SE-q-additive-noise}-\eq{eq-SE-Xi-additive-noise}). 

In Fig.~\ref{fig:r_BP_Gaussian_Gaussian_p=3} in the appendix we show results of
r-BP algorithm for the informative initialization which can be compared well with  G-AMP. 
%%%%%%%%%%%%%%%%%%%%%%%%%%%%%%%%%%%%%%%%%%%%%%%%%%%%%
\subsubsection{mixed case: $p=2+3$}
\label{sec-p=2+3}
As we already pointed out in the previous sections (see also appendix~\ref{appendix-stability-parmagnet}), the $m=0$ solution is always stable for $p>2$ in the dense limit $c \gg 1$,  which is problematic in the context of inference. As a way to circumvent this, we propose a mixed model in which
two species of interactions ($p_1$ and $p_2$) are present. The total number of interactions, or measurements, is $N_\blacksquare=N_1+N_2$. The connectivity of each species is defined as $c_i=\alpha_i M$, then $N_i=\alpha_i M N / p_i$. The idea behind this is to contaminate a $p_2=3$ system with some fraction of $p_1=2$ measurements in order to have a first order transition in the magnetization and an instability of the paramagnetic solution at finite $\lambda$ at the same time.
\blue{While this trick should be useful also in sparse $c=O(1)$ systems,
it is even more useful in our dense systems because $m=0$ phase in $p>2$ systems
cannot become spontaneously unstable contrary to sparse systems.}

The only modification to the pure $p$ case is that the summation over $\blacksquare$ in the interaction part of the free energy Eq.~\eqref{eq-F1-additive} can be split according to the two species $p_1$ and $p_2$. This produces the following equation of state
\begin{equation}
	\frac{m}{1-m}=\frac{\alpha_1\lambda^2m^{p_1-1}}{1+\lambda^2(1-m^{p_1})}+\frac{\alpha_2\lambda^2m^{p_2-1}}{1+\lambda^2(1-m^{p_2})}\equiv D_{p_1,p_2}(\alpha_1,\alpha_2,\lambda,m),
\end{equation}
under the homogeneous, Bayes optimal and replica symmetric ansatz.

\blue{In the mixed case, there are loop corrections involving both $p=2$ and $p=3$ interactions which appear first at $O(\lambda^{3})$ (triangle) in the cumulant expansion (see sec.~\ref{appendix-cumulant-expansion}). The scaling arguments are essentially the
same and one can safely neglect such corrections in the dense limit. }

We specialize in the following to the case $p_1=2$, $p_2=3$ and we fix $\alpha_1=2$. This choice is due to the fact that this allows us to have an instability of the paramagnetic state for finite $\lambda_c=(\sqrt{\alpha_2-1})^{-1}=1$. The phase diagram in the $\alpha_2-\lambda$ plane is depicted in Fig.~\ref{fig:mixed_p=2+3}. 
%%%%%%%%%%%%%%%%%%%%%%
\begin{figure}[h]
	\centering
	\includegraphics[width=0.45\textwidth]{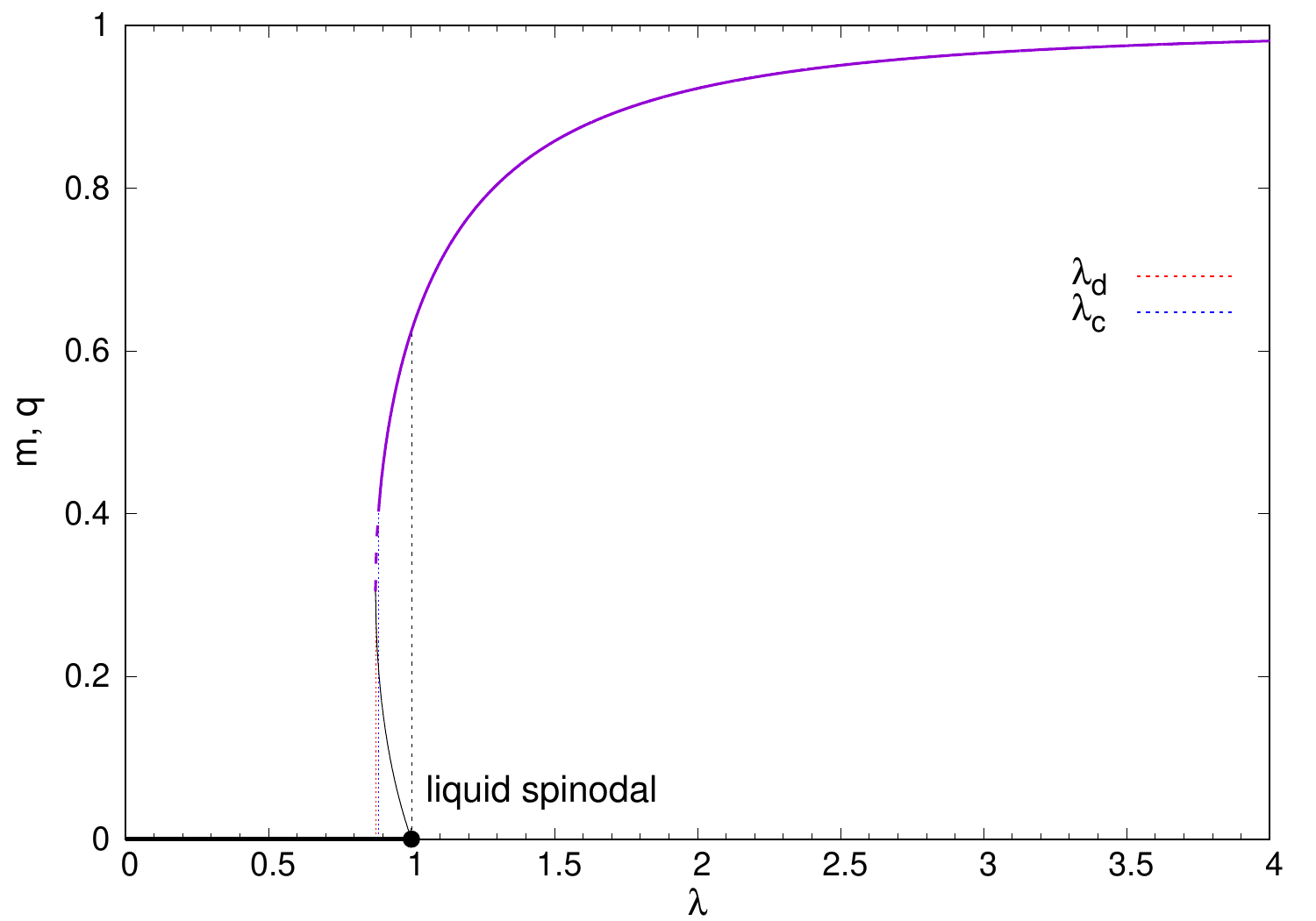}
	\centering
      	\includegraphics[width=0.41\textwidth]{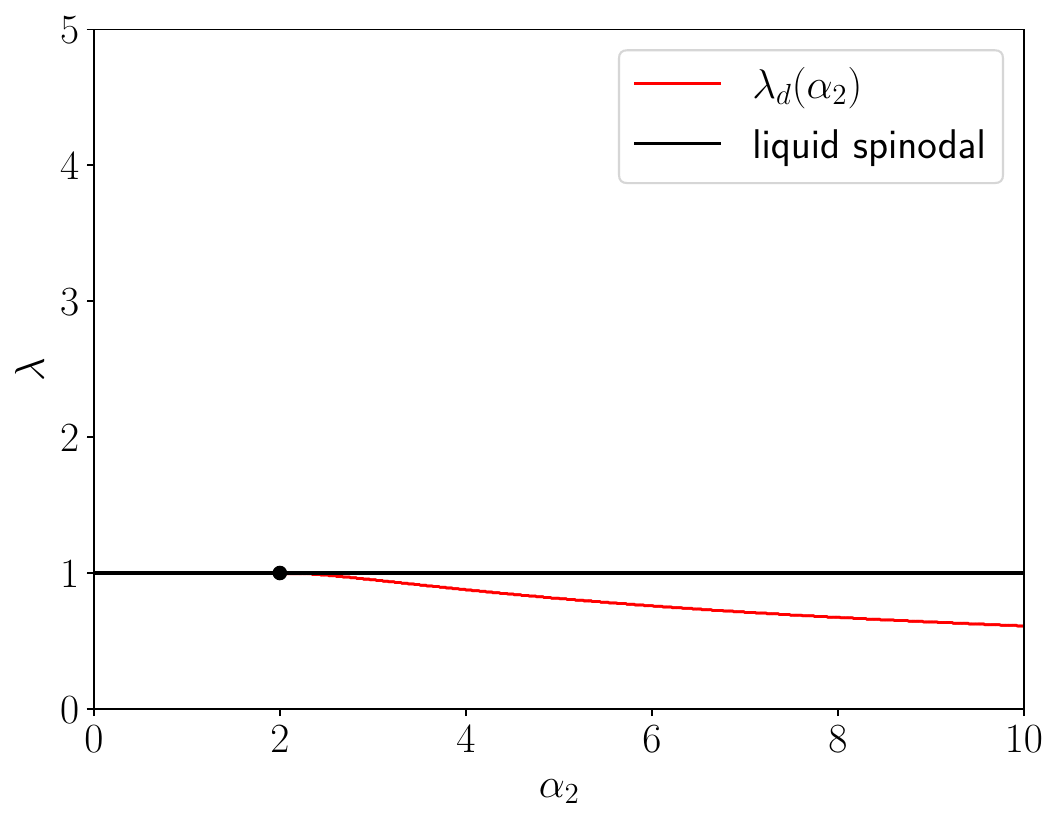}
	\caption{Left: order parameter $m=q$ for the Gaussian prior, Gaussian noise and mixed $p=2+3$. The parameter $\alpha$ of each species is $\alpha_1=2$ and $\alpha_2=4$. \blue{The thick black line for $m=0$ represents the paramagnetic phase which is locally stable up to the spinodal point $\lambda=1$.} Right: phase diagram in the $\alpha_2-\lambda$ plane, for fixed $\alpha_1=2$. \blue{Below the black horizontal line $\lambda=1$ the paramagnetic state is locally stable; above the red line $\lambda_d(\alpha_2)$ the magnetized state is present.}}
	\label{fig:mixed_p=2+3}
\end{figure}
%%%%%%%%%%%%%%%%%%%%%%
When $\alpha_2\leq2$ (which coincides to the condition $\alpha_2\leq\alpha_1$) the transition in $\lambda$ is second order, and it is first order otherwise. Interestingly, the discontinuity of the order parameter emerges continuously when moving away from $\alpha_2=2$. This means that for values of $\alpha_2$ bigger than but very close to $2$ the first order transition resembles very much a continuous one, see also Fig.~\ref{fig:mixed_p=2+3_type}.
%%%%%%%%%%%%%%%%%%%%%%
\begin{figure}[h]
	\centering
	\includegraphics[width=0.45\textwidth]{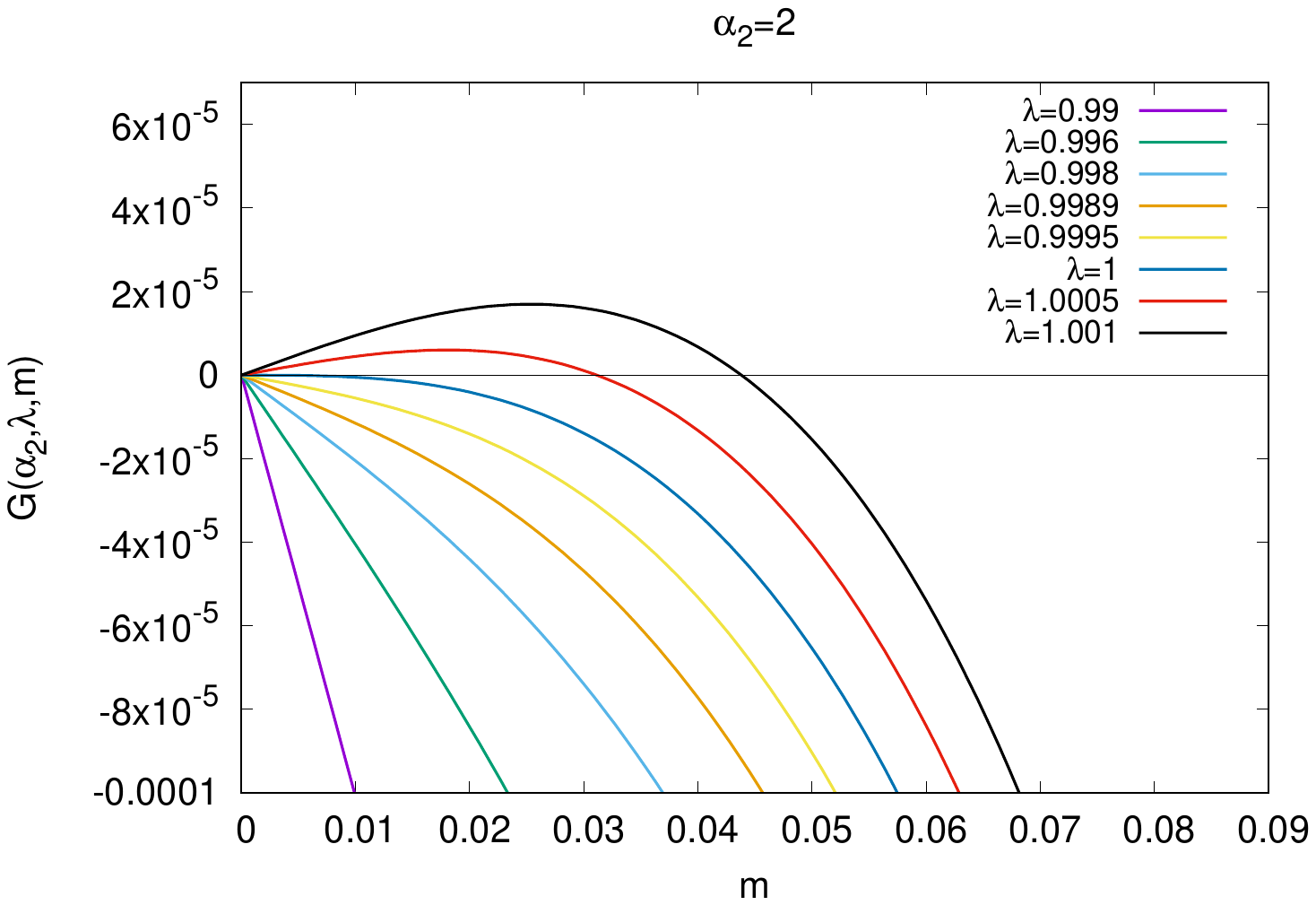}
	\centering
	\includegraphics[width=0.45\textwidth]{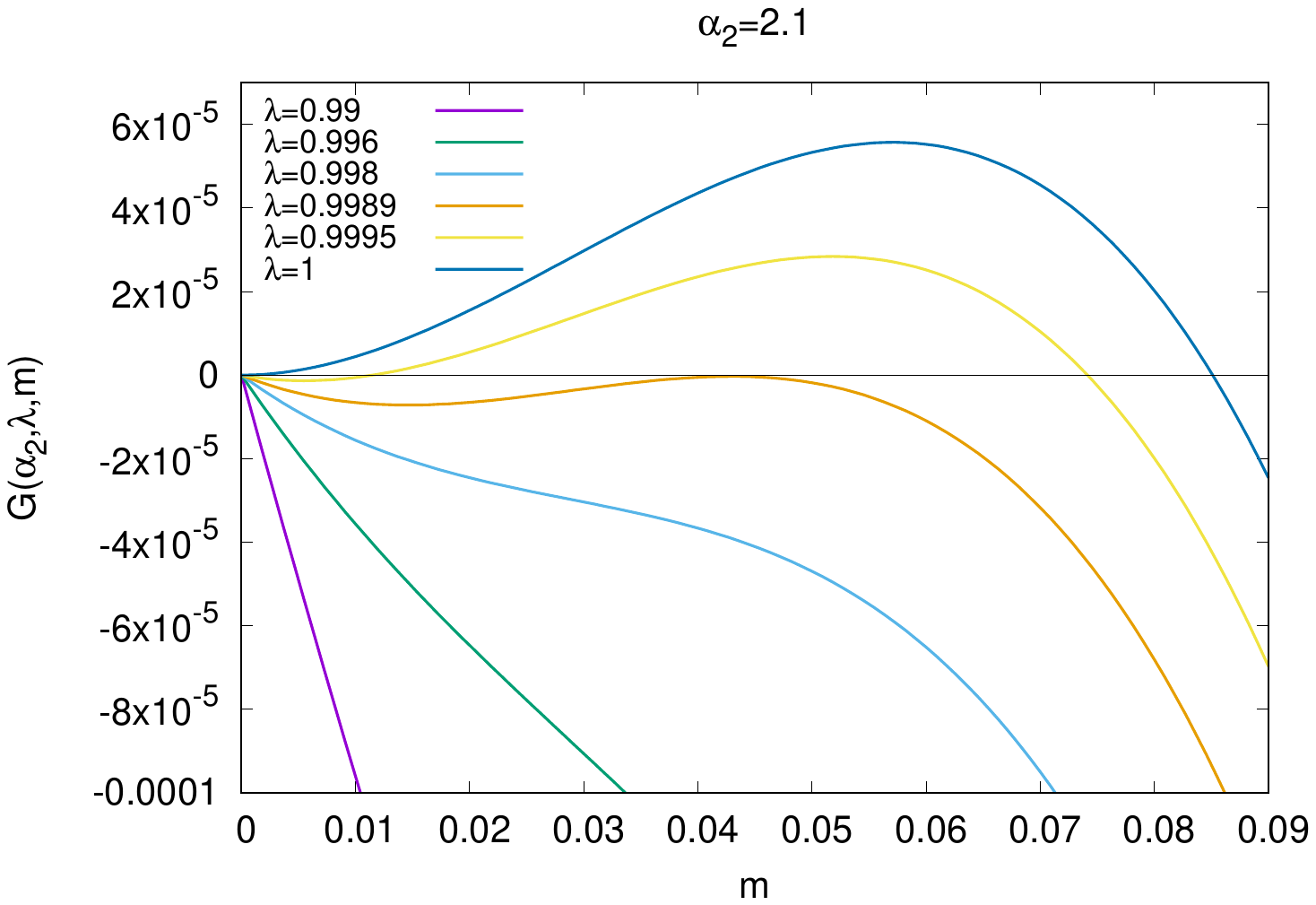}
	\caption{For the case $p_1=2$, $p_2=3$ and $\alpha_1=2$, we plot $G(\alpha_2,\lambda,m)\equiv (1-m) D(2,\alpha_2,\lambda,m)-m$. Left: for $\alpha_2 \leq 2$ the transition in $\lambda$ is second order. The derivative of $G(m)$ at $m=0$ becomes 0 exactly at $\lambda=1$. Right: for $\alpha_2>2$ there is a discontinuous appearance of a non-trivial solution already for $\lambda<1$. For $\alpha_2$ close to 2, however, the range of $\lambda$ and $m$ values interested is very small and the transition can resemble a continuous one.}
	\label{fig:mixed_p=2+3_type}
\end{figure}
%%%%%%%%%%%%%%%%%%%%%%

In Fig.~\ref{fig:G_AMP_Gaussian_Gaussian_p=2+3_A} we report the result of
G-AMP compared with the solution of the SE equation. Deviations close to the transition can be possibly ascribed to finite $N,M$ corrections.
%%%%%%%%%%%%%%%%%%%%%%
\begin{figure}[h]
\centering
\includegraphics[width=0.49\textwidth]{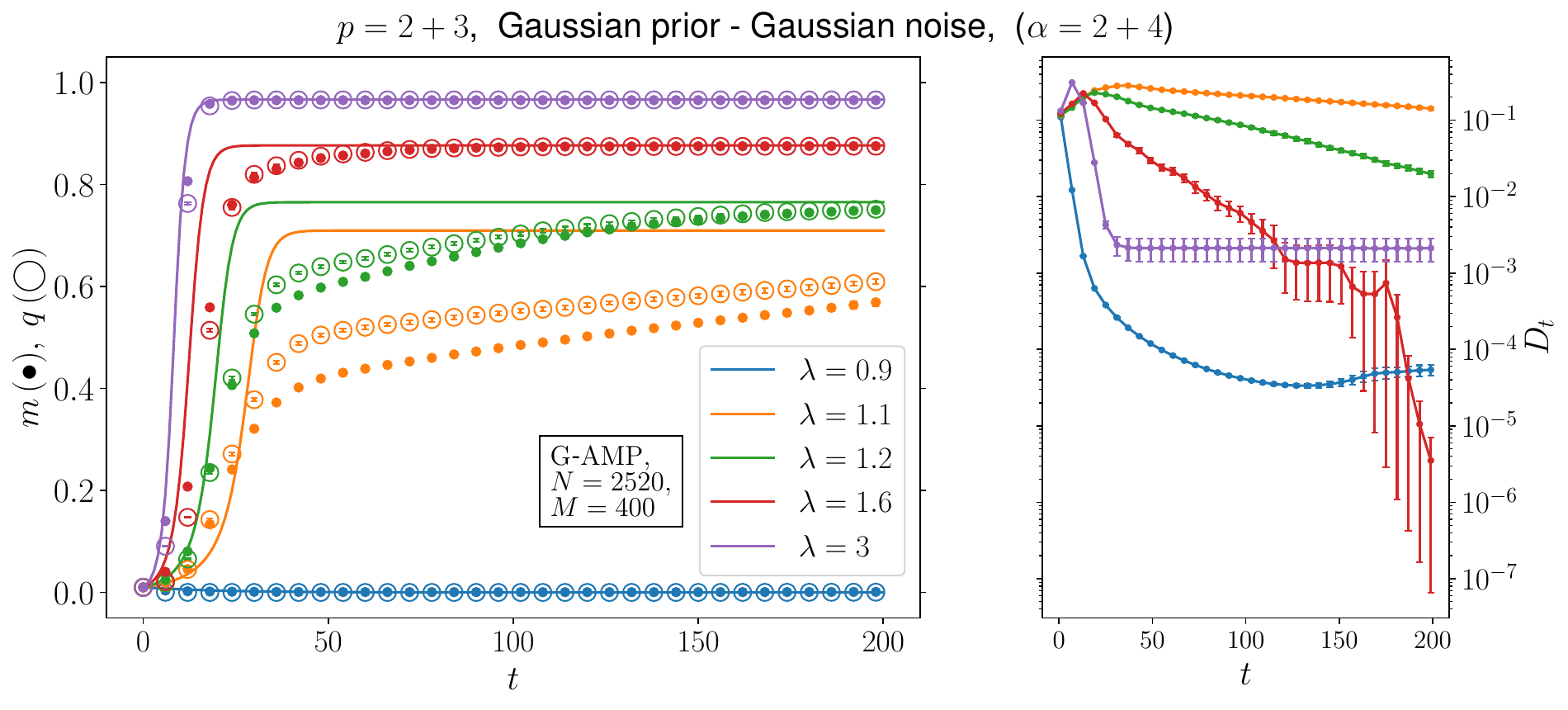}
\includegraphics[width=0.49\textwidth]{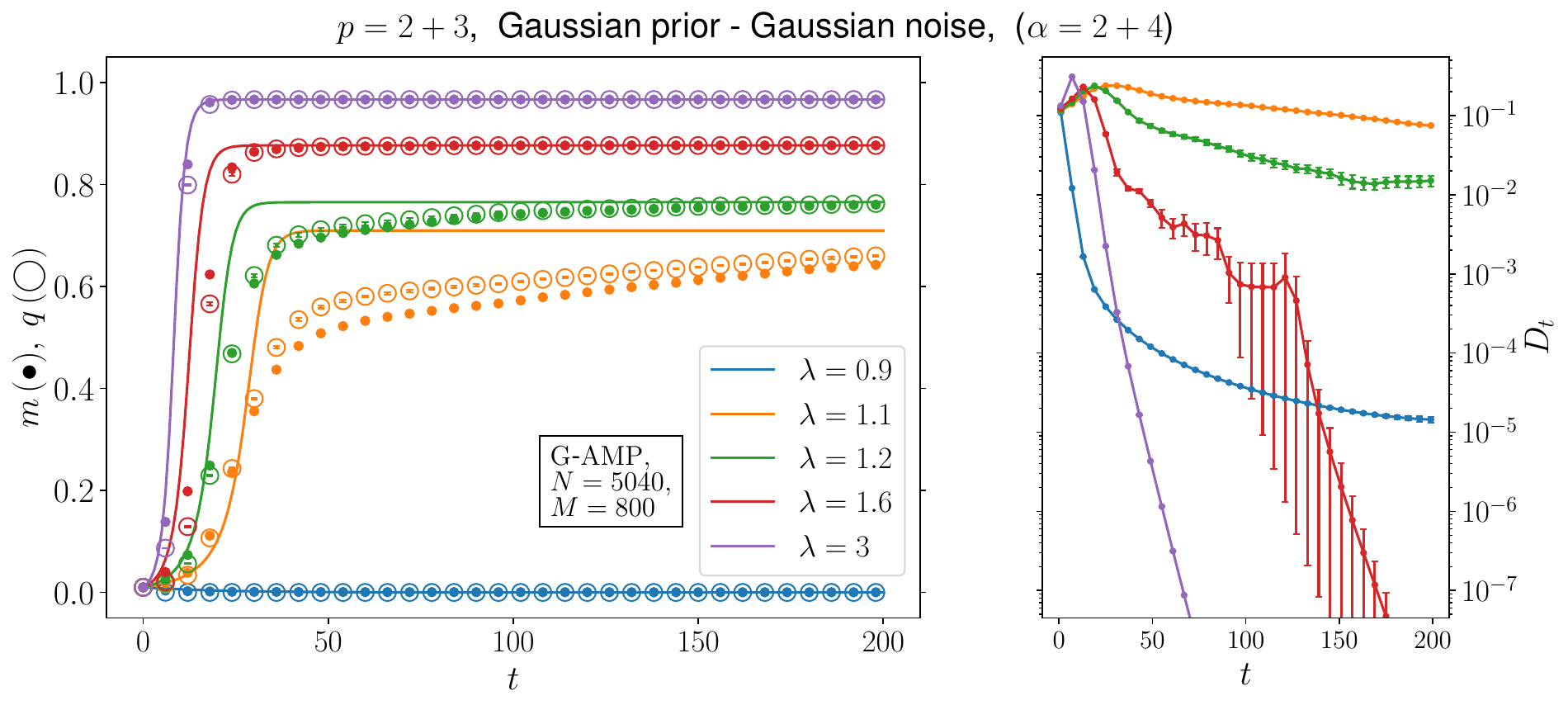}
\caption{Evolution of the order parameters in the G-AMP algorithm for the mixed $p=2+3$ case with $\alpha_1=2$, $\alpha_2=4$. Even starting from an uninformative initialization, partial recovery of the planted solution is possible thanks to the instability of the paramagnetic phase. Solid lines represent the State Evolution predictions. On the right panel, the convergence parameter is plotted.}
\label{fig:G_AMP_Gaussian_Gaussian_p=2+3_A}
\end{figure}
%%%%%%%%%%%%%%%%%%%%%%

\blue{Contrary to our dense systems, the paramagnetic state in sparse systems with $c=O(1)$ can become spontaneously unstable (for example, see Fig.~6 of \cite{zdeborova2016statistical}).
Still the trick can be useful also for sparse systems to facilitate inferences.}
%\clearpage 
%%%%%%%%%%%%%%%%%%%%%%%%%%%%%%%%%%%%%%%%%%%%%%%%%%%%%
%%%%%%%%%%%%%%%%%%%%%%%%%%%%%%%%%%%%%%%%%%%%%%%%%%%%%
\subsection{Gaussian prior and sign output}
\label{sec-gaussian-sign-output}

Here the case of Gaussian prior and sign output is treated. The equation of states and the free energy expression in the replica approach are given by \eq{eq-SP-gaussian-sign-RS-BayesOptimal} and \eq{eq-free-ene-signoutput-BayesOptimal}, respectively, where $H(x)$ is defined in \eq{eq-H}. This formula shows the consistency with the SE equation (\eq{eq:HRTD-SE_EqState_m_Gaussian}-\eq{eq:HRTD-SE_EqState_Q_Gaussian} combined with \eq{eq-SE-chi-q-m-sign-output}) in the message passing approach.

For the message passing approach, we examine the r-BP algorithm (Algorithm~\ref{alg:p-ary_r-BP}) and G-AMP algorithm (Algorithm~\ref{alg:p-ary_G-AMP}) with the input function
\eq{eq:HRTD-input_Spherical}  and output function \eq{eq:HRTD-outputFunction_sign}-\eq{eq:def-H-function}. Again, we limit ourselves to the Bayes optimal case. The only control parameter is in this case $\alpha$. The transition is second order for $p=2$ and first order for $p=3$. See Fig.~\ref{fig:sign_output}.
%%%%%%%%%%%%%%%%%%%%%%
\begin{figure}[h]
\centering
\includegraphics[width=0.45\textwidth]{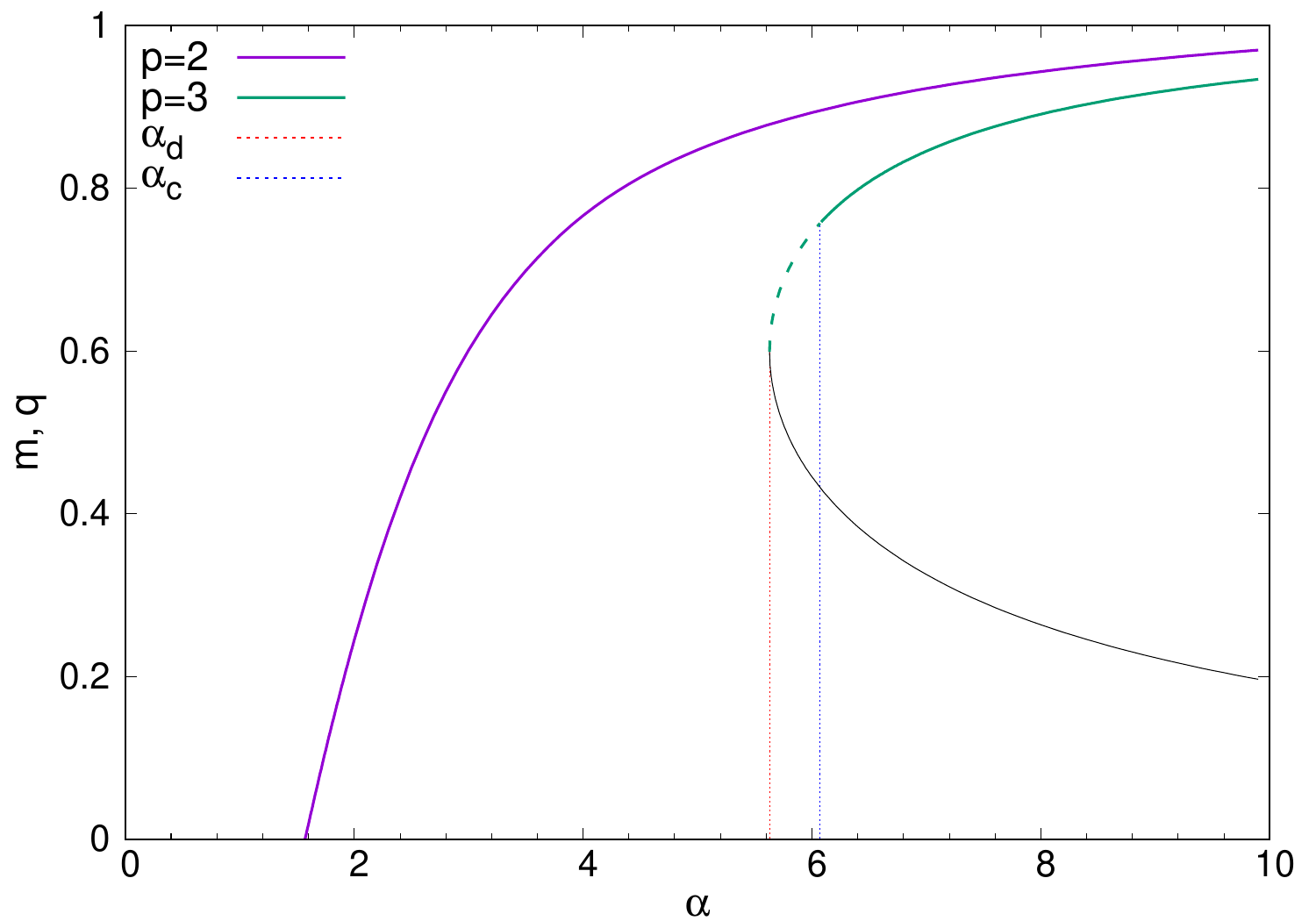}
\caption{Order parameter $m=q$ for the Gaussian prior and sign output. The transition in $\alpha$ is second order for $p=2$ and first order for $p=3$. \blue{The dashed portion of the $p=3$ line indicates a metastable magnetized state associated to the first-order transition. The black continuous line indicate instead a branch of unstable solutions.}}
\label{fig:sign_output}
\end{figure}
%%%%%%%%%%%%%%%%%%%%%%
In this case, in the context of perfect reconstruction, there are no meaningful threshold at finite $\alpha$ due to the mismatch between the prior and the noise model. 

We report in Fig.~\ref{fig:G_AMP_Gaussian_Sign_p=2} and Fig.~\ref{fig:G_AMP_Gaussian_Sign_p=3} the behavior of the G-AMP using the informative/uninformative initial conditions for $p=2$ and $p=3$ case respectively. The results again compare well with the solution of the SE equations.
%%%%%%%%%%%%%%%%%%%%%%
\begin{figure}[h]
\centering
\includegraphics[width=0.495\textwidth]{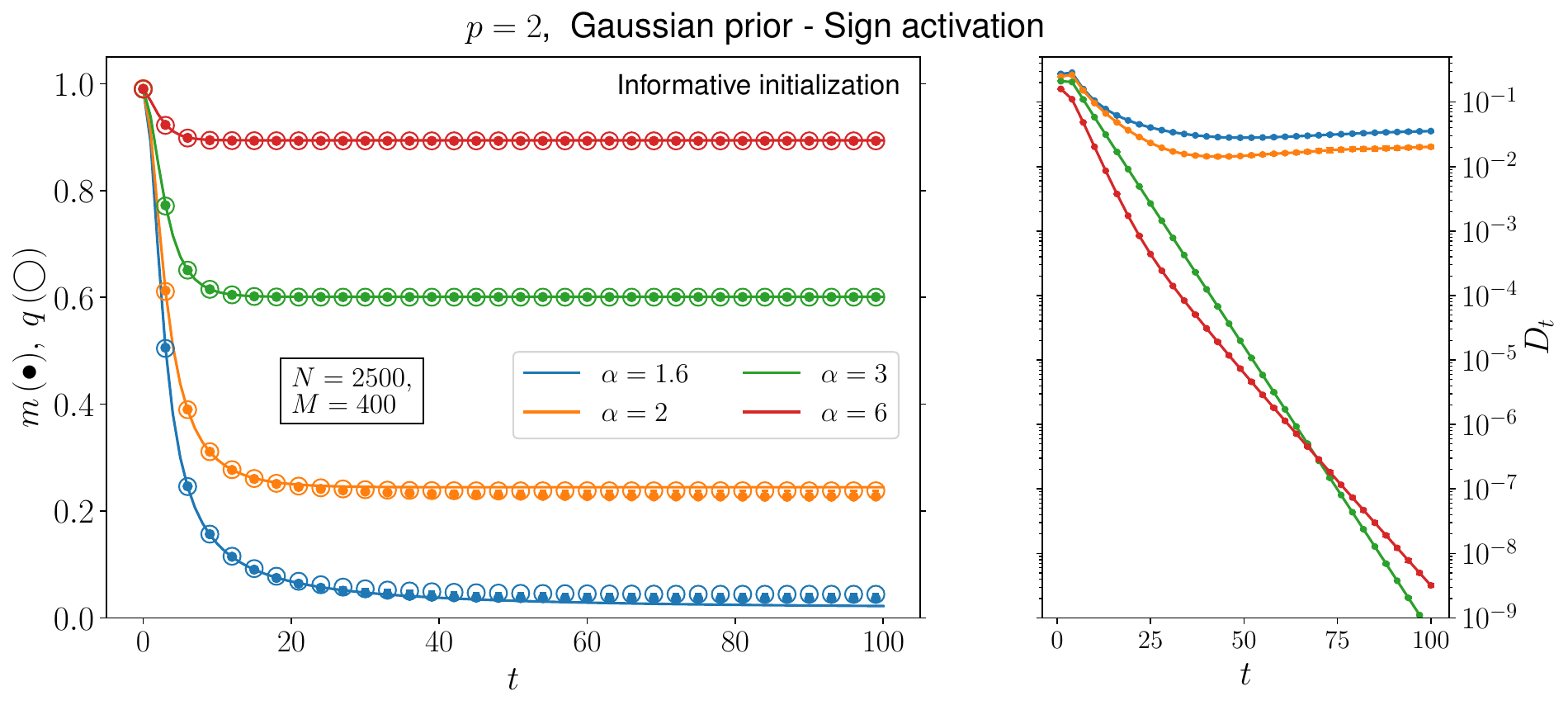}
\includegraphics[width=0.495\textwidth]{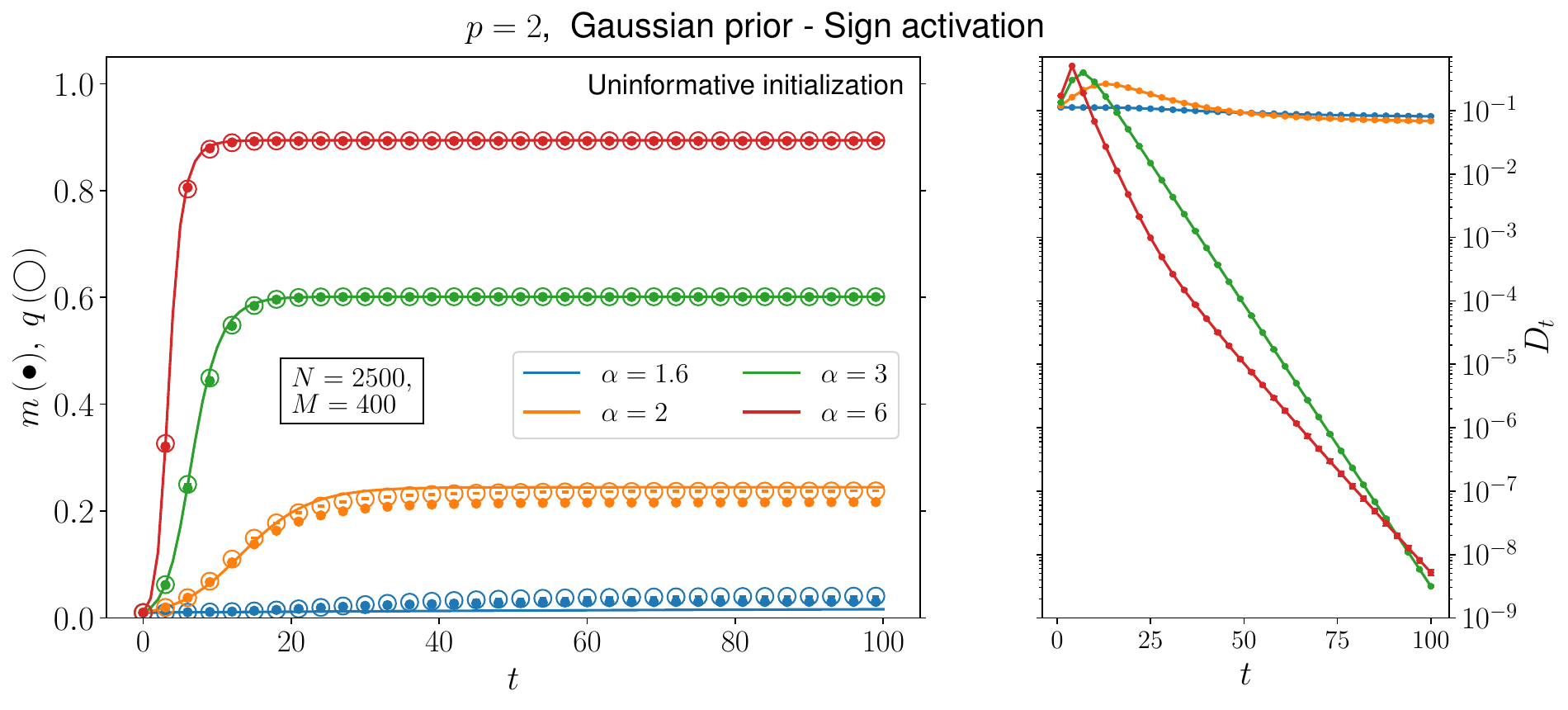}
\caption{Evolution of the order parameters in the G-AMP algorithm averaged over 10 random instances and for two different initializations. Solid lines represent the State Evolution prediction}
\label{fig:G_AMP_Gaussian_Sign_p=2}
\end{figure}
%%%%%%%%%%%%%%%%%%%%%%
%%%%%%%%%%%%%%%%%%%%%%
\begin{figure}[h]
\centering
\includegraphics[width=0.6\textwidth]{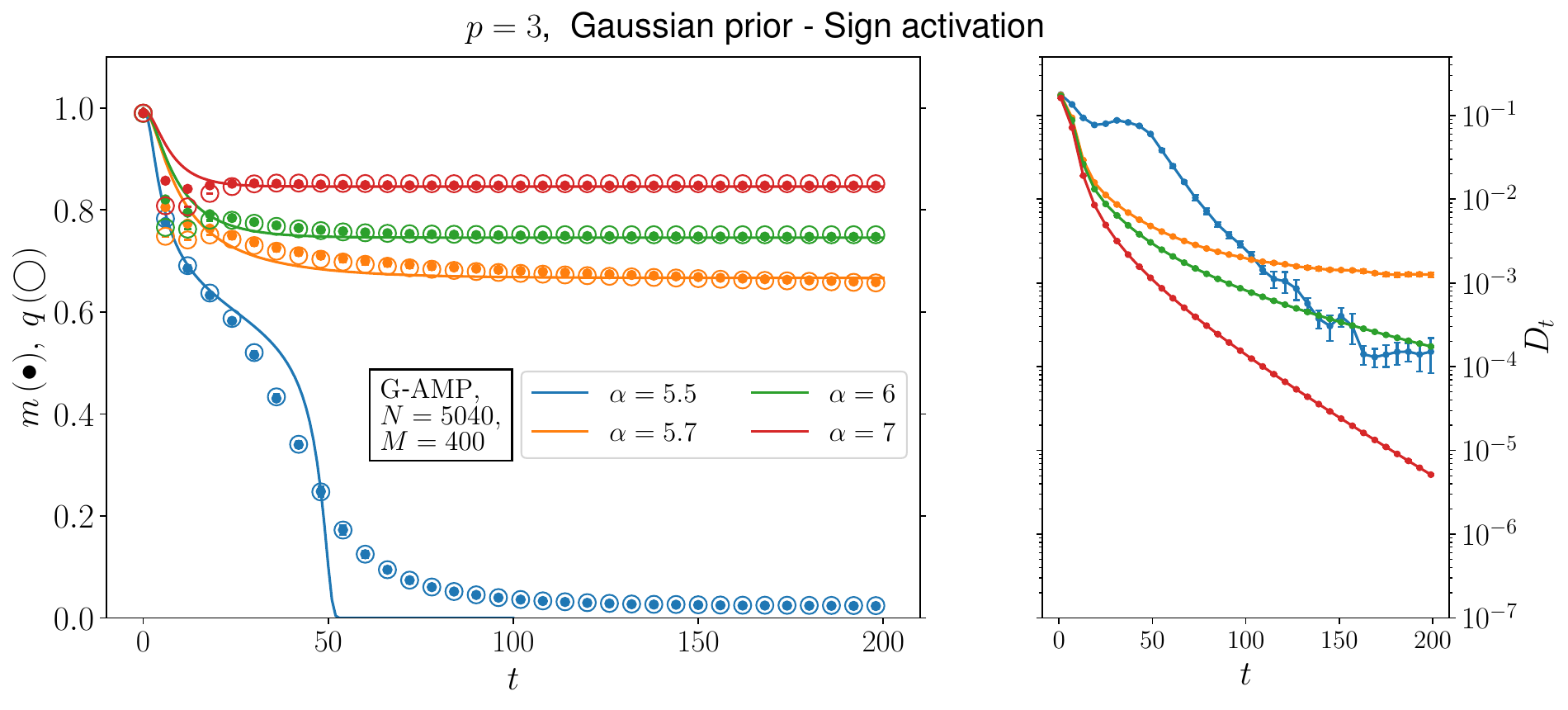}
\caption{Evolution of the order parameters in the G-AMP algorithm on random instances of the problem with $N=5040$, $M=400$. Data is averaged over 5 instances. Solid lines represent the State Evolution predictions. On the right panel, the convergence parameter is plotted.}
%\label{fig:}
\label{fig:G_AMP_Gaussian_Sign_p=3}
\end{figure}
%%%%%%%%%%%%%%%%%%%%%%

In Fig.~\ref{fig:r_BP_Gaussian_Sign_p=2} in the appendix we show results of
the r-BP algorithm for the $p=2$ case. The comparison with SE equations is good up to $\alpha=3$, though for $\alpha=6$ we observe strong deviations; we also notice that the observable $Q_t=\langle v_{i\mu}^{t}\rangle$ is in this case consistently different from unity. This is presumably due to strong finite $M$ corrections, since some improvement is present when increasing $M$ to $M=300$ as shown in Fig.~\ref{fig:r_BP_Gaussian_Sign_p=2_corrections}.
In Fig.~\ref{fig:G_AMP_Gaussian_Sign_p=2}, we have shown that for $M=400$ the G-AMP algorithm does not suffer from this instability. See also Fig~\ref{fig:random_init} discussed in appendix concerning uninformative initialization and truly random initialization.

%%%%%%%%%%%%%%%%%%%%%%%%%%%%%%%%%%%%%%%%%%%%%%%%%%%%%
%%%%%%%%%%%%%%%%%%%%%%%%%%%%%%%%%%%%%%%%%%%%%%%%%%%%%
\subsection{Dependence of the convergence of algorithms on the choice of $F$}

As discussed in previous sections we found the macroscopic quantities like the order-parameters, equation of states and the free energy are the same for the deterministic and disordered factors $F_{\bs \mu}$ from the theoretical perspective in the dense limit.
However behavior of the message passing algorithms applied on systems with finite $N,c,$ and $M$
can be different between two cases.
For the $p=2$ case (both with Ising/Gaussian priors) we found the message passing algorithms (r-BP and G-AMP) 
do not converge properly with the deterministic model $F_{\bs,\mu}=1$ \eq{eq-F-uniform} and we had to resort to 
the model with random $F_{\bs,\mu}$ \eq{eq-F-random} in numerical analysis. On the other hand we did not encounter such problems in the $p=3$ cases.

One might expect that the residual global symmetries of the problems may affect convergence of the algorithms.
Among the global symmetries listed in sec~\ref{sec-symmetries}, the rotational symmetry exists only in the $p=2$ model
with Gaussian prior and $F_{\bs, \mu}=1$. Such a symmetry does not exist with the Ising prior.
Thus we cannot attribute the reason for the failure of convergence in the $p=2$ model,
which happen with both the Ising/Gaussian priors, entirely to the rotational symmetry.
%In $p=2$ cases the global reflection symmetry remain
%also with the disordered $F_{\bs,\mu}$ but the convergence of the algorithms suggests 
%that the reflection symmetry can be broken dynamically by the message passing algorithms.
The absence of the convergence problem
in the $p=3$ case with the deterministic model $F_{\bs,\mu}=1$ suggests that the permutation symmetry can be broken dynamically
by the message passing algorithms. Let us note also that we performed some Monte Carlo simulations (not reported in the present paper)
on the $p=2$ Ising/Gaussian models with $F_{\bs,\mu}=1$ and found that statistical inferences are indeed possible modulo the rotational/permutation symmetries
(these symmetries are not broken in the Monte Carlo simulations of finite size systems in equilibrium).

We speculate that the reason for the difference between the convergence of the message passing algorithms
between the $p=2$ and $p=3$ cases may be attributed to some differences in the
abundance of short loops. In general, message passing algorithms assume that the system is locally tree-like.
Presence of short loops can hamper the convergence of the algorithm.
In the $p=2$ cases, the shortest loops are the triangular loops (See Fig.~\ref{fig:diagrams} a)) while
triangular loops are absent and the shortest loops are rectangular in the $p=3$ model.
In the idealized dense limit $N \gg c \gg 1$ the presence of such short loops become
negligible but may be the values of $N$ and $c$ used in our analysis reported below are not separated enough.

\clearpage
%%%%%%%%%%%%%%%%%%%%%%%%%%%%%%%%%%%%%%%%%%%%%%%%%%%%%
%%%%%%%%%%%%%%%%%%%%%%%%%%%%%%%%%%%%%%%%%%%%%%%%%%%%%
%%%%%%%%%%%%%%%%%%%%%%%%%%%%%%%%%%%%%%%%%%%%%%%%%%%%%
\section{Conclusions}
\label{sec-conclusion}

To summarize, we proposed a scheme of tensor factorization 
based on sparse measurements of tensor components represented by a graphical model which is dense but not globally coupled. We expect that the setup will be useful in cases where substantial amount of data is missing as in recommendation systems heavily used in social network services.
We studied the Bayesian inference of the tensor factorization by statistical mechanics approaches: the replica method and message passing algorithms.

We considered the dense limit $N \gg M \gg 1$ \cite{yoshino2018,yoshino2023spatially} which amounts to consider dense but not fully connected graph. This limit is useful from theoretical point of view in two aspects: 1) it allows exact theoretical analysis
since loop corrections vanish and 2) the r-BP (relaxed belief propagation)
and G-AMP (generalized approximate message passing) become valid in this limit.
\blue{It will be interesting to test our theoretical results
independently by Monte Carlo (MC) sampling.
MC sampling will provide useful insights into the magnitudes of 
the finite $N$ and $M$ corrections.}

From practical point of views, our result on the dense limit will be useful to consider situations where effective rank of tensors is not quite low. As mentioned in sec. \ref{sec-model}, this can happen for instance in the case of facial images \cite{zhao2015bayesian} and in recommendation systems \cite{koren2009matrix}.
%Note that our graphical model itself can be applied for general $M$ including $M=O(1)$ cases.
Conventional algorithm like alternating least-square, gradient based methods can also
be applied on our problem. It will be very interesting to test our setup for real-world data using various algorithms including the message passing algorithms developed in the present paper.
\blue{
However one would have to take into account the large node-to-node fluctuation of the
  connectivity in real-world data. Our present analysis based on a regular random graph may be
  too naive in this respect.}

There are numerous possibilities to extend our present work. In the present paper we considered
just one specie of vectors $x_{i\mu}$. The theory can be straightforwardly extended to the
case of multi-species $x_{i\mu},y_{i\mu},z_{i\mu},\ldots$. This will be useful for instance in the context of recommendation systems
and dictionary learning $Y=DX$
%(matrix factorization $p=2$)
where one usually consider two species, 'dictionary' $D$ and
  'sparse representation' $X$.

%Let us list below some perspective on further theoretical developments.
%\item  Finite $N$ and finite $M$ corrections to the theory: In principle, this is possible considering higher order terms
%  in the cumulant expansion which vanish in the dense limit
%  and also finite-$M$ corrections to the saddle point integrations.
%    
%\item Beyond the Bayes optimal case : There are various possibilities.
%  For example, it is straight forward to consider  MAP case.

%\item In the case of matrix factorization $p=2$, our problem amount to consider
%  matrix of size $N \times M$ with $N \gg M \gg 1$.
%  In the context of dictionary learning $Y=DX$
%  this means strongly undercomplete dictionary while usually overcomplete dictionary
%  $N < M$ is considered assuming 'sparse' (number of non-zero components is small) representation $X$ .
%  If one knows in advance which components of $X$ are non-zero, a sparse graphical model could be constructed but the selection of non-zero components is part of the inference procedure in dictionary learning.
%  This is more challenging. One may consider ensemble of graphs?
%\end{itemize}

%%%%%%%%%%%%%%%%%%%%%%%%%%%%%%%%%%%%%%%%%%%%%%%%%%%%%
%%%%%%%%%%%%%%%%%%%%%%%%%%%%%%%%%%%%%%%%%%%%%%%%%%%%%
%%%%%%%%%%%%%%%%%%%%%%%%%%%%%%%%%%%%%%%%%%%%%%%%%%%%%
\section*{Acknowledgments}

We thank Yoshiyuki Kabashima, Sakata Ayaka, Takashi Takahashi, Koki Okajima
and Lenka Zdeborov{\'a}, for useful discussions.
This work was supported by KANENHI (No. 19H01812), (No. 21K18146), (No. 22H05117),
(No. 18K11463), (No. 17H00764), (No. 22K12179)
%, and No. 22H05117 (TO)
 from MEXT, Japan.

%%%%%%%%%%%%%%%%%%%%%%%%%%%%%%%%%%%%%%%%%%%%%%%%%%%%%
%%%%%%%%%%%%%%%%%%%%%%%%%%%%%%%%%%%%%%%%%%%%%%%%%%%%%
%%%%%%%%%%%%%%%%%%%%%%%%%%%%%%%%%%%%%%%%%%%%%%%%%%%%%
\appendix
 \clearpage
%\section{Appendix}

\section{Some useful formulas}
In the following we collect some useful formulas.
We will use the short-hand notation \eq{eq-def-Dz}, which reads as,
\beq
    \int\mathcal{D}z\ldots=\int_{-\infty}^{\infty}\frac{dz}{\sqrt{2\pi}}
    e^{-\frac{z^{2}}{2}}\ldots.
\eeq

\begin{lemma}
    For any analytic function $f(h)$ and any constant $C$,
    \begin{equation}
        f(h+C) = e^{C\frac{\partial}{\partial h}}f(h).
    \end{equation}
    \label{lemma:uno}
\end{lemma}

\begin{lemma}
    For any analytic function $f(h)$ and any constant $C\geq0$,
    \begin{equation}
        e^{\frac{C}{2}\frac{\partial^2}{\partial h^2}}f(h)=\int\mathcal{D}z f(h\pm\sqrt{C}z).
    \end{equation}
    \label{lemma:due}

\end{lemma}

\begin{lemma}
    For any $m$-times differentiable functions $F(h_1,h_2,\dots,h_n)$ and any positive integer $m$
    \begin{equation}
        \left(\sum_{a=1}^n\frac{\partial}{\partial h_a}\right)^mF(h_1,h_2,\dots,h_n) \Biggl\lvert_{\{h_a=h\}}=\frac{\partial^m}{\partial h^m} F(h,h,\dots,h).
    \end{equation}
    \label{lemma:tre}
\end{lemma}

\begin{corollary}
    For any analytic functions $\{f_a(h)\}_{a=1,\dots,n}$ and any constant $C\geq0$,
    \begin{equation}
        e^{\frac{C}{2}\left(\sum_{a=1}^n\frac{\partial}{\partial h_a}\right)^2}\prod_{a=1}^nf_a(h_a)\Biggl\lvert_{\{h_a=0\}}=\int\mathcal{D}z\prod_{a=1}^nf_a(\pm\sqrt{C}z).
    \end{equation}
    \label{coroll:uno}
\end{corollary}

\begin{lemma}
For any positive integer $n$,
    \begin{equation}
        \int{\rm d}wW(w)\mathcal{D}z_0\frac{\int\mathcal{D}z_1W^{(n)}(\Xi)}{\int\mathcal{D}z_1W(\Xi)}=0,
    \end{equation}
where
    \begin{equation}
    \Xi\equiv w - Az_0 - Az_1.
    \end{equation}
    \label{lemma:quattro}
\end{lemma}

        \begin{lemma}\label{thm:derivations:RS_solutions:spherical:int_zero}
          If $\beta = 1$ and $m = q$, for any positive integer $n$,
          \begin{equation*}
            \int \dd{w_*} W_*(w_*)
              \int \gdd{z_0}
              \frac{\int \gdd{z_1} W_*^{(n)}(\Xi)}{\int \gdd{z_1} W_*(\Xi)}
            = 0.
          \end{equation*}
        \end{lemma}
        \begin{proof}
          Change variables from $z_1$ to $\zeta_1 = w_* - A z_0 - A z_1$.
          \begin{equation*}
            \lhs = \int \dd{w_*} W_*(w_*)
              \int \gdd{z_0}
              \frac{
                \int \dd{\zeta_1} e^{
                    - \frac{1}{2 A^2} \zeta_1^2 + \frac{1}{A^2} (w_* - A z_0) \zeta_1
                  }
                  W_*^{(n)}(\zeta_1)
              }{
                \int \dd{\zeta_1} e^{
                    - \frac{1}{2 A^2} \zeta_1^2 + \frac{1}{A^2} (w_* - A z_0) \zeta_1
                  }
                  W_*(\zeta_1)
              }.
          \end{equation*}
          Change variables again from $z_0$ to $\zeta_0 = \frac{w_*}{A} - z_0$.
          \begin{equation*}
            \begin{aligned}
              \lhs
              &= \int \frac{\dd{\zeta_0}}{\sqrt{2 \pi}} e^{- \frac{1}{2} \zeta_0^2}
                \qty(
                  \int \dd{w_*} e^{- \frac{1}{2 A^2} w_*^2 + \frac{1}{A} \zeta_0 w_*}
                    W_*(w_*)
                ) \\
                &\qquad \times \frac{
                  \int \dd{\zeta_1} e^{
                      - \frac{1}{2 A^2} \zeta_1^2 + \frac{1}{A} \zeta_0 \zeta_1
                    }
                    W_*^{(n)}(\zeta_1)
                }{
                  \int \dd{\zeta_1} e^{
                      - \frac{1}{2 A^2} \zeta_1^2 + \frac{1}{A} \zeta_0 \zeta_1
                    }
                    W_*(\zeta_1)
                }.
            \end{aligned}
          \end{equation*}
          The denominator can be reduced and it becomes
          \begin{equation*}
            \begin{aligned}
              \lhs
              &= \int \frac{\dd{\zeta_0}}{\sqrt{2 \pi}} e^{- \frac{1}{2} \zeta_0^2}
                \int \dd{\zeta_1} e^{
                    - \frac{1}{2 A^2} \zeta_1^2 + \frac{1}{A} \zeta_0 \zeta_1
                  }
                  W_*^{(n)}(\zeta_1) \\
              &= \int \dd{\zeta_1} W_*^{(n)}(\zeta_1)
                \int \frac{\dd{\zeta_0}}{\sqrt{2 \pi}}
                  e^{- \frac{1}{2} \qty(\zeta_0 - \frac{1}{A} \zeta_1)^2} \\
              &= \int \dd{\zeta_1} W_*^{(n)}(\zeta_1) \\
              &= \qty[W_*^{(n - 1)}(\zeta_1)]_{\zeta_1 = - \infty}^\infty.
            \end{aligned}
          \end{equation*}
          We expect that $W_*(w_*)$ is a well-behaved probability distribution, and its derivatives converge to zero at infinity.
          Therefore
          \begin{equation*}
            \lhs = 0.
          \end{equation*}
        \end{proof}

\begin{lemma}
\label{lemma:determinant-RS}

Given a matrix of size $n\times n$
in the form$ M_{ab}=M_{1}\delta_{ab}+M_{2}$, its determinant is obtained
as,
\beq
{\rm det}M=M_{1}^{n}{\rm det}(\delta_{ab}+M_{2}/M_{1})=M_{1}^{n}(1+nM_{2}/M_{1})
\eeq
    
\end{lemma}

\begin{lemma}
    \label{lemma:determinant}

With $A$ and $C$ being certain square (sub)matrices, we have the following formula for the determinant of a square matrix.
\beq
    {\rm det} \left (
    \begin{array}{cc}
      A & B
      \\
      B^{t} & C
    \end{array}
    \right )
    ={\rm det} A {\rm det}(C-BA^{-1}B^t)
    \eeq

\end{lemma}

%\section{\blue{Saddle point integration in dense limit $N \gg M \gg 1$}}
\section{Saddle point integration in dense limit $N \gg M \gg 1$}
        \label{appendix-sp-indegration}

In the replica theory discussed in sec~\ref{sec-replica} we consider
integrals  of the form (see \eq{eq-z-1+n-and-F} and \eq{eq-hat-F})
\beqn
I=\int \prod_{i=1}^{N} dx_{i} e^{-M f(x_{1},x_{2},\ldots,x_{N})}.
\eeqn
in the limit where both $N$ and $M$ are large. 
In particular we are interested with "free-energy density" (see \eq{eq-def-f})
in the dense limit $N \gg M \gg 1$,
\beq
\lim_{M \to \infty}\lim_{N \to \infty} \frac{1}{MN}\log I
\eeq
Although we do not have fully rigorous justification of the method,
we present below a sketch of the method.

We may evaluate the integral via saddle point method assuming $M \gg 1$ as,
\beqn
I &=& e^{-Mf(x_{1}^{*},x_{2}^{*},\ldots,x_{N}^{*})}
\int \prod_{i=1}^{N} dx_{i} e^{-\frac{M}{2}\sum_{i,j}H_{ij}
(x_{i}-x^{*}_{i})(x_{j}-x^{*}_{j}) 
+\ldots}
\eeqn
where the saddle point $(x^{*}_{1},x^{*}_{2},\ldots,x^{*}_{N})$ is determined by

\beq
\frac{\partial}{\partial x_{i}}\left. f (x_{1},x_{2},\ldots,x_{N}) \right |_{(x_{1},x_{2},\ldots,x_{N})=(x^{*}_{1},x^{*}_{2},\ldots,x^{*}_{N})}
=0
%\qquad i=1,2,\ldots,N
\eeq
and $H$ is the Hessian matrix defined as
\beq
H_{ij}=\left. \frac{\partial^{2}f}{\partial x_{i}\partial x_{j}}\right |_{(x_{1},x_{2},\ldots,x_{N})=(x^{*}_{1},x^{*}_{2},\ldots,x^{*}_{N})}
\eeq
Suppose that the eigen modes of the Hessian matrix verify
\beq
\sum_{j=1}^{N}H_{ij}u_{j\mu}=\lambda_{\mu}u_{i\mu} 
\qquad i=1,2,\ldots,N \qquad \mu=1,2,\ldots,N
\eeq
and that the eigen values $\lambda_{\mu}$ ($\mu=1,2,\ldots,N$)
associated with eigen vectors $\bm{u}_\mu=(u_{1\mu},u_{2\mu},\ldots,u_{N\mu})$
are all non-negative.
Then we find
\beqn
&& \frac{1}{MN}\log I - \frac{1}{N}(-f)(x_{1}^{*},x_{2}^{*},\ldots,x_{N}^{*}) =
\frac{1}{MN}\log \sqrt{\frac{(2\pi)^{N}}{M {\rm det}H}}+\ldots \nonumber \\
&&= \frac{1}{2M}\log (2\pi)-\frac{1}{2N}\frac{\log M}{M}-\frac{1}{2M}\frac{1}{N} \sum_{\mu=1}^{N}\log \lambda_{\mu}+ \ldots
\eeqn
For large but finite $M$ the terms on the r.h.s which represent corrections to the saddle point are finite. We consider dense limit $N \gg M \gg 1$ so we first 
consider $N \to \infty$. 
We find the 2nd term vanishes but other terms remain finite.
Next by considering $M \to \infty$ the 1st, 3rd
as well as higher order terms vanish.
Thus we find
\beq
\lim_{M \to \infty}\lim_{N \to \infty}\frac{1}{MN}\log  I=-
\lim_{N \to \infty} \frac{1}{N}f(x_{1}^{*},x_{2}^{*},\ldots,x_{N}^{*})
\eeq

In the replica theory (sec~\ref{sec-replica}),
we obtain the free-energy functional $F[Q]$
(see \eq{eq-F-total-summary} and \eq{eq-f-ex})
with $Q_{i}$ allowed to vary in 'space' $i=1,2,\ldots,N$.
Although the spatially uniform solution \eq{eq-homogeneous-ansatz} is enough
for the specific problems we considered in the present paper (see sec.~\ref{sec-result}),
the generic free-energy functional can become useful
in situations where spatial inhomogeneity is expected.
For example it may be useful to consider strong fluctuations of
connectivity mentioned in the conclusion.
Indeed in \cite{yoshino2020complex,yoshino2023spatially}
(see also \cite{tomita2025replicated} for an application to the  structural glass problem),
one dimensional variation of the order parameter
were analyzed based on generic free-energy expression constructed in similar manners
as in the present paper.

        \section{Analysis of cumulant expansion}
        \label{appendix-cumulant-expansion}

In the following we explain
the cumulant expansion discussed in 
sec.~\ref{sec-interaction-part-of-free-energy}
which is used to extract the interaction part of the replicated free-energy.

        \subsection{Order $\lambda$}

At the lowest order the contribution can be represented by the diagrams shown in Fig.\ref{fig:diagrams} a). We readily find using \eq{eq-averaging-with-epsilon}.
\beq
\langle\pi^{a}_{\bs}\rangle_{\epsilon,0}
=\sum_{\mu=1}^{M}\frac{\lambda}{\sqrt{M}} F_{\bs,\mu}\langle\prod_{j\in bs_{1}}x_{j\mu}^{a}\rangle_{\epsilon,0}=0
\eeq
because of the reflection symmetry
$P_{\rm pri.}(x)=P_{\rm pri.}(-x)$ which holds in both the Ising and Gaussian priors that we consider in the present paper
(see sec~\ref{sec-symmetries}).

\subsection{Order $\lambda^{2}$}

At the next order in the expansion \eq{eq-cumulant-expansion-interaction} we find
\beqn
&& \langle
\pi^{a}_{\bs_{1}}
\pi^{b}_{\bs_{2}}
\rangle_{\epsilon,0}
-
\langle
\pi^{a}_{\bs_{1}}
\rangle_{\epsilon,0}
\langle
\pi^{b}_{\bs_{2}}
\rangle_{\epsilon,0} \nonumber \\
&=&\frac{\lambda^{2}}{M}\sum_{\mu_{1}=1}^{M}\sum_{\mu_{2}=1}^{M}F_{\bs_{1},\mu_{1}}F_{\bs_{2},\mu_{2}}
\underbrace{\left[
\langle
\prod_{j_{1} \in \bs_{1}}x_{j_{1}\mu_{1}}^{a}
\prod_{j_{2} \in \bs_{1}}x_{j_{2}\mu_{2}}^{b}
\rangle_{\epsilon,0}
-
\langle
\prod_{j_{1} \in \bs_{1}}x_{j_{1}\mu_{1}}^{a}
\rangle_{\epsilon,0}
\langle
\prod_{j_{2} \in \bs_{1}}x_{j_{2}\mu_{2}}^{b}
\rangle_{\epsilon,0}
\right]}_{\delta_{\bs_{1},\bs_{2}}\delta_{\mu_{1},\mu_{2}}\prod_{i \in \bs_{1}}Q_{i}^{ab}} \nonumber \\
&=& \delta_{\bs_{1},\bs_{2}}\frac{\lambda^{2}}{M}\sum_{\mu=1}^{M}
F_{\bs_{1},\mu}^{2}
\prod_{i \in \partial\bs_{1}}Q_{i}^{ab}
=\delta_{\bs_{1},\bs_{2}}\lambda^{2}\prod_{i \in \partial\bs_{1}}Q_{i}^{ab}
\eeqn
%Note that the term computed above involves connected correlation function. Disconnected diagrams do not contribute to $\hat{G}$.

To derive the above equation we used
\beq
\langle x^{a}_{i\mu} x^{b}_{j\nu} \rangle_{\epsilon,0}=\delta_{ij}\delta_{\mu\nu}Q_{i}^{ab}
\eeq
which follows from  \eq{eq-def-G0}, \eq{eq-averaging-with-epsilon} and \eq{eq-epsilon0}. We also assumed
\beq
\frac{1}{M}\sum_{\mu=1}^{M}F^{2}_{\bs,\mu}=1
\eeq
for the factor $F_{\bs,\mu}$ introduced in sec~\ref{sec-techinical-assumptions}.
This trivially holds for
the 'deterministic model' $F_{\bs,\mu}=1$  but also holds
for the 'random model'
where $F_{\bs,\mu}$ is an iid random variable with unit variance.

This is a type (A) contribution
 discussed in sec.~\ref{sec-interaction-part-of-free-energy}
which can be represented by doubling the diagrams shown in
Fig.\ref{fig:diagrams} a) : two diagrams associated with function nodes $\bs_{1}$ and $\bs_{2}$ which are identical but carrying two replicas $a$ and $b$.
Equivalently this also corresponds to diagrams Fig.\ref{fig:diagrams} b) with two function nodes carrying replicas $a$ and $b$ under the condition 
that the variable nodes 'j' and 'k' are identical (otherwise
the contribution becomes zero by the same symmetry reason as for the vanishing of $O(\lambda)$ terms).

So far we found $O(\lambda)$ term in the expansion
\eq{eq-cumulant-expansion-interaction} is zero while $O(\lambda^{2})$ is non-zero.
Up to this order, it is natural to choose $A=\lambda^{2}$ in the Plefka expansion 
\eq{eq-G-plefka-expansion} and \eq{eq-F-and-epsilon-plefka-expansion}.
Using the above results in
\eq{eq-cumulant-expansion-interaction-0}, \eq{eq-cumulant-expansion-interaction}
and \eq{eq-F-G-relations}) we readily find,
\beq
-\beta \hat{F}_{1}[\hat{Q}][\partial/\partial h_{\bs}^{a}]=
\frac{1}{2} \sum_{\bs} \sum_{a,b=0}^{n}
\prod_{j \in \partial\bs}Q_{j}^{ab} \frac{\partial^{2}}{\partial h_{\bs}^{a}\partial h_{\bs}^{b}}
%\prod_{a=0}^{n}
%\left.
%P_{\rm out}^{\beta}(y_{\bs}\lvert h_{\bs}^{a})
%\right|_{\{h_{\bs}^{a}=0\}}
\label{eq-F1-appendix}
\eeq

\subsection{Order $\lambda^{3}$}
\label{section-order-lambda-3}

At $O(\lambda^{3})$ we find terms which can be represented as the 'triangular' diagrams shown in
Fig.~\ref{fig:diagrams} c) with the three function nodes '1', '2' and '3' carrying three
replicas $a$, $b$ and $c$. Notice that the function nodes adjacent to each other
(joined by the variable nodes in between) exhibit a 'triangle' which is the shortest loop.
Each function node is associated with a $\pi$ and carries a factor $1/\sqrt{M}$.
There is one common running index $\mu$ for the vector component which produces a factor $M$.
For example the triangular diagram for $p=2$ and $p=4$ contains three and six variable nodes
so that they yield
\beq
\left(\frac{1}{\sqrt{M}}\right)^{3}MQ_{i}^{ab}Q_{j}^{bc}Q_{k}^{ca}
\qquad
\left(\frac{1}{\sqrt{M}}\right)^{3}MQ_{i}^{ab}Q_{j}^{bc}Q_{k}^{ca}Q_{l}^{ab}Q_{m}^{bc}Q_{n}^{ca} 
\eeq
with $M=c/\alpha$ with $\alpha=O(1)$.

An important feature is that all function nodes carry a replica while all variable nodes carry two replicas so that a variable node participating the triangle
carries a replica overlap like $Q^{ab}$. For this to happen
with 3 function nodes,  $p$ has to be an even number. Note that for odd $p$ models, including $p=3$, similar single loop diagrams in which each variable node carries a replica overlap like $Q^{ab}$
appear first at $O(\lambda^{4})$ (rectangle).
Generalization of the following discussion to such cases is straightforward.

%In Fig.~\ref{fig:diagrams} c) we show such diagrams for $p=2$ and $p=4$ but not for $p=3$.

Now let us consider whether the triangular diagrams (assuming $p$ is even)
give relevant contributions to the free-energy.
To this end let us consider the set of such diagrams containing a function node, say '1'.
It is adjacent to $p$ variable nodes. Take one of them and call it as 'i'.
The variable node 'i' has $c-1$ adjacent function nodes other than '1'. Choose one of them and call it as '2'.
The function node '2'  has $p-1$ variable nodes adjacent to it other than 'i'. Choose one of them and call it as 'j'.
The variable node 'j' has $c-1$ adjacent function nodes other than '2'. Choose one of them and call it as '3'.
The function node '3'  has $p-1$ variable nodes adjacent to it other than 'j'. Choose one of them and call it as 'k'.
Now what is the probability that 'k' is adjacent to the function node '1' to close a loop?
Let us denote it as $p_{\rm connect}$.
To sum up the total number of such diagrams will scale
as $p(p-1)^{2} (c-1)^{2}p_{\rm connect}$
so that their contribution to the free-energy per function node scale as
\beq
M\left(\frac{1}{\sqrt{M}}\right)^{3}p(p-1)^{2}(c-1)^{2} p_{\rm connect}
\sim c^{3/2}p_{\rm connect}
\eeq
for $c \gg 1$ assuming $\alpha$ and $p$ are finite.
In the globally coupled case ($c \propto N^{p-1}$) surely we have  $p_{\rm connect}=1$
so that the contribution {\it cannot} be neglected.
This is in agreement with the recent observation in the matrix factorization problem \cite{maillard2022perturbative} which corresponds to $p=2$ case with global coupling.
On the other hand in our dense limit $N \gg c \gg 1$, we expect $p_{\rm connect}=O(c/N)$
so that it can be neglected by taking $N \to \infty$ limit first.

Note that in the presence of the random spreading code $F_{\bs}^{\mu}$
the contribution vanishes whatever the connectivity
by taking average over the realization of $F_{\bs}^{\mu}$
(see sec~\ref{sec-techinical-assumptions}).

\subsection{Order $\lambda^{4}$}

At order $O(\lambda^{4})$ we find terms which contribute to $\hat{G}_{2}$.
They consist of $4$ function nodes carrying replicas $a$,$b$,$c$ and $d$.

\begin{itemize}
\item The simplest possibility is to use again diagrams shown in Fig.~\ref{fig:diagrams} a)
with $\bs$ carrying four replicas. In each of the $p$ variable nodes we attach a four point connected correlation function like
$\langle x^{a}_{i\mu}x^{a}_{i\mu}x^{c}_{i\mu}x^{d}_{i\mu}\rangle^{c}_{\epsilon,0}$.
Each function node carries $1/\sqrt{M}$ factor and there is one running global index $\mu$ 
so that the contribution of such a term to the free-energy per function node
scales as $M(1/\sqrt{M})^{4} \sim M^{-1}$ which vanish in the limit $M \to \infty$.

\item (One particle reducible diagram) Another possibility is to use the
diagrams shown in Fig.~\ref{fig:diagrams} b) where $\bs_{1}$ and $\bs_{2}$ carry
replica pairs $(a,b)$ and $(c,d)$ respectively.  At the variable node
$i$ we attach a four point connected correlation function $\langle x^{a}_{i\mu}x^{a}_{i\mu}x^{c}_{i\mu}x^{d}_{i\mu}\rangle^{c}_{\epsilon,0}$
while on other variable nodes we attach
$Q^{ab}$ or $Q^{cd}$. This case contributes to $\hat{G}_{2}$ but disappears in $\hat{F}_{2}$ by the Legendre
transformation \eq{eq-F-G}.
Indeed one can show that they become canceled by $\beta\frac{(\hat{G}_{0}^{'}[\epsilon_{0}])^{2}}{\hat{G}_{0}^{''}[\epsilon_{0}]}$ which appear in \eq{eq-F-G-relations}.
Note that the diagram can be disconnected into two parts by cutting the diagram at $i$.
Such diagrams are examples of  the so called one particle reducible diagrams 
and disappear in  $\hat{F}$ which consists only of one particle irreducible (1PI) diagrams
\cite{hansen1990theory,zinn2021quantum}.

%\item (Loop diagram) Yet another possibility is the single closed loop diagrams similar to 
%the triangles we considered  at $O(\lambda^{3})$ but extended to the case of four function nodes.
%in which  each variable node carry two replicas so that a replica overlap like $Q^{ab}$ is attached to such a variable node
%Their contribution vanishes in the dense limit $N \gg c \gg 1$.
%and/or in the presence of the random spreading factor $F_{\bs}^{\mu}$
\end{itemize}

\subsection{Higher orders}

Let us consider a generic closed one-loop diagram of length $l(\geq 3)$ which goes through $l$ function nodes
and $l$ variable nodes and carries $k$ replicas on each function node. 
Repeating the argument in sec~\ref{section-order-lambda-3}, we find that
the contribution to the free-energy per function node can be estimated as the following: 
\begin{itemize}
\item Running $\mu$ index gives a factor $M$
\item A function node yields a factor $(\lambda/\sqrt{M})^{k}$
\item The number of closed loop starting from a function node is  $(c-1)^{l-1}p(p-1)^{l-1} p_{\rm connect}$
\item A variable node caries a $2k$ point connected correlation function
of $x_{i \mu}^{a}$s.
\item A function node caries also a factor $F^{k}$.
\end{itemize}
Combing these we find a contribution of order
\beq
M\left(\frac{\lambda}{\sqrt{M}}\right)^{lk}(c-1)^{l-1}p(p-1)^{l-1} p_{\rm connect}
\propto c^{l(1-\frac{k}{2})}p_{\rm connect}
\eeq

For $k=1$ we find it scales as $c^{l/2}p_{\rm connect}$ so that
such a contribution is relevant in globally coupled system for which $p_{\rm connect}=1$.
However, in the dense limit $p_{\rm connect}\sim O(c/N)$ so that it vanishes
by taking $N \to \infty$ limit first.
The above analysis implies
$N$ must grow faster than any power law of $c$ in order to
get rid of contribution from loop diagrams.
%In the presence of the random spreading factor $F_{\bs}^{\mu}$ it vanishes after averaging over $F_{\bs}^{\mu}$ (even for the globally coupled case).

%For $k=2$ the contribution is marginal, it scales as $c^{0}p_{\rm connect}$ so that it is relevant
%in globally coupled system for which $p_{\rm connect}=1$.
%Even in the presence random spreading factor $F_{\bs}^{\mu}$ it does not vanishes.

For $k>2$ we find the contribution vanishes in $c\to \infty$
limit either in the dense limit or in globally coupled systems. 

The contributions of diagrams with larger number of closed loops
(2-loop, 3-loops,...) will vanish
faster in $c \to \infty$ limit. 

To summarize the relevant
loop diagrams which can contribute to the free-energy
in the $c \to \infty$ limit are those with $k=1$ in the case of global
coupling. Using the dense coupling $N \gg c \gg 1$ we can avoid this.
Such contributions of loops (with $k=1$) disappear also in the presence of the random spreading code $F_{\bs}^{\mu}$ after averaging over $F_{\bs}^{\mu}$.
Indeed this is the situation in the standard fully connected $p$-spin spin-glass models
($p=2$ case is the Sherrington-Kirkpatrick model 
\cite{kirkpatrick1978infinite}, more general $p$-spin model introduced first in  \cite{derrida1980random}).
%as well as the case where the average of $F_{\bs}^{\mu}$ is non-zero but small enough $O(1/\sqrt{c})$
%as it happens in the fully connected $p$-spin ferromagnetic models ($p=2$ is the Curie-Weiss model).

\section{Stability of the paramagnetic solution}
\label{appendix-stability-parmagnet}
The instability of the $m=0$ solution can be detected by studying the change of sign of the second derivative of the RS free-energy (we consider here the Bayes Optimal case $\beta=1$ and $q=m$) at $m=0$.
\subsection{Ising prior and additive Gaussian noise}
\label{appendix-stability-parmagnet-ising}
Let us begin with the entropic part of the free-energy. Our goal is to compute $\frac{d^2}{d m^2}\frac{\partial}{\partial n}\frac{-\beta F_0(m,m)}{MN}\bigg\lvert_{n=0}$. We start from the definition of the entropic part of the free-energy as given in the first line of Eq.~\eqref{eq-f0-ising-RS}, equipped with Eq.~\eqref{eq-eps-phi-Ising-RS}. One finds
\begin{equation}
	\frac{\partial}{\partial n}\frac{-\beta F_0(q,m)}{MN}\bigg\lvert_{n=0,q=m}=-\frac{1}{2}(1+m) A(m,\alpha,\lambda) + e^{\frac{A(m,\alpha,\lambda)}{2}\frac{\partial^2}{\partial h^2}} \log 2\cosh(h+A(m,\alpha,\lambda))\bigg\lvert_{h=0},
\end{equation}
where $A(m,\alpha,\lambda)$ is defined, for the additive Gaussian noise case, in Eq.~\eqref{eq:m_Ising_Bayesopt}. Taking two derivatives with respect to $m$ we obtain
\begin{equation}
	-\frac{\partial A}{\partial m} +  B \left(\frac{\partial A}{\partial m}\right)^2 + \left(C-\frac{1}{2}(1+m)\right)\frac{\partial^2 A}{\partial m^2} ,
\end{equation}
where 

\begin{align}
	C &= e^{\frac{A}{2} \frac{\partial^2}{\partial h^2} } \frac{1}{2} \frac{\partial}{\partial h} \tanh(h+A)\Big\lvert_{h=0}  + e^{\frac{A}{2}\frac{\partial^2}{\partial h^2}} \tanh(h+A)\Big\lvert_{h=0} , \nonumber \\
	B &= e^{\frac{A}{2} \frac{\partial^2}{\partial h^2} } \left[\frac{1}{4} \frac{\partial^3}{\partial h^3} + \frac{\partial^2}{\partial h^2} + \frac{\partial}{\partial h} \right]\tanh(h+A)\Bigg\lvert_{h=0} ,
\end{align}
which evaluated at $m=0$ gives $C_{m=0}=B_{m=0}=\frac{1}{2}$.

Let us turn now to the interaction part of the free-energy and compute $\frac{d^2}{d m^2}\frac{\partial}{\partial n}\frac{-\beta F_{\rm ex}(m,m)}{MN}\bigg\lvert_{n=0}$. From Eq.~\eqref{eq-f1-additive-noise-RS-derivatives}, equipped with Eq.~\eqref{eq-theta0-theta1-gaussian-noise-replica} in the case of Gaussian noise, we obtain
\begin{align}
	\frac{d}{d m}\frac{\partial}{\partial n}\frac{-\beta F_{\rm ex}(m,m)}{MN}\bigg\lvert_{n=0} &= \partial_{n} \frac{\partial }{\partial q}  \frac{ -\beta F_{\rm ex}(q,m)}{MN} \Big\lvert_{n=0,q=m}+\partial_{n} \frac{\partial }{\partial m}  \frac{ -\beta F_{\rm ex}(q,m)}{MN} \Big\lvert_{n=0,q=m} = \nonumber \\
	&=\frac{A}{2}.
\end{align}

Putting the pieces together
\begin{align}
	\frac{d^2}{d m^2}(-f(m)) &= \left(B\frac{\partial A}{\partial m} - \frac{1}{2}\right)\frac{\partial A}{\partial m} + \left(C-\frac{1}{2}(1+m)\right)\frac{\partial^2 A}{\partial m^2}, \nonumber \\
	\frac{d^2}{d m^2}(-f(m))\Big\lvert_{m=0} &= \frac{1}{2}\left(\frac{\partial A}{\partial m}\Big\lvert_{m=0} - 1\right)\frac{\partial A}{\partial m} \Big\lvert_{m=0}.
\end{align}
For $p=2$, we have that $\frac{\partial}{\partial m}A\lvert_{m=0}=\frac{\alpha \lambda^2}{1+\lambda^2}>0$. The condition of stability then implies $\frac{\partial A}{\partial m}\Big\lvert_{m=0} < 1 $ or
\begin{equation}
\begin{cases}
	0<\alpha < 1, \quad&\forall \lambda >0, \\
	\alpha>1, \quad&\lambda < \frac{1}{\sqrt{\alpha-1}}.
\end{cases}
\end{equation}
For $p>2$, we have instead $\frac{\partial}{\partial m}A\lvert_{m=0} = 0$ and the second derivative is not enough to assess the stability of the paramagnetic phase. We can go on and compute the third derivative
\begin{equation}
	\frac{d^3}{d m^3}(-f(m)) = \left(\frac{\partial B}{\partial m}\frac{\partial A}{\partial m}+B\frac{\partial^2 A}{\partial m^2}\right)\frac{\partial A}{\partial m} +  \left(2B\frac{\partial A}{\partial m} - 1\right)\frac{\partial^2 A}{\partial m^2} +  \left(C-\frac{1}{2}(1+m)\right)\frac{\partial^3 A}{\partial m^3},
\end{equation}
where we used the fact that $\frac{\partial C}{\partial m} = B\frac{\partial A}{\partial m}$. For $p=3$, the only term that survives at $m=0$ is the one proportional to $\frac{\partial^2 A}{\partial m^2}$,  i.e. $\frac{d^3}{d m^3}(-f(m))\lvert_{m=0}=-\frac{\partial^2 A}{\partial m^2}\lvert_{m=0}=\frac{2\alpha\lambda^2}{1+\lambda^2}<0$ $\forall \alpha, \lambda >0$. For $p=4$, the third derivative is zero since $C_{m=0}=\frac{1}{2}$, and we expect a non-trivial contribution appearing in $\frac{d^4}{d m^4}(-f(m))\lvert_{m=0}$ associated to the $\frac{\partial^3 A}{\partial m^3}$ term.

\subsection{Gaussian prior and additive Gaussian noise}
\label{appendix-stability-parmagnet-gauss}
In the case of Gaussian prior, from the expression of the entropic term Eq.~\ref{eq-f0-gaussian-RS-derivatives} one obtains
\begin{align}
	\frac{d}{d m}\frac{\partial}{\partial n}\frac{-\beta F_0(m,m)}{MN}\bigg\lvert_{n=0} &= \partial_{n} \frac{\partial }{\partial q}  \frac{ -\beta F_0(q,m)}{MN} \Big\lvert_{n=0,q=m}+\partial_{n} \frac{\partial }{\partial m}  \frac{ -\beta F_0(q,m)}{MN} \Big\lvert_{n=0,q=m} = \nonumber \\
	&=-\frac{1}{2}\frac{m}{1-m}.
\end{align}

Putting the pieces together (the interaction term is the same as in the Ising case)
\begin{align}
	\frac{d^2}{d m^2}(-f(m)) &= -\frac{1}{2}\frac{1-2m}{(1-m)^2} +\frac{1}{2}\frac{\partial A}{\partial m}, \nonumber \\
	\frac{d^2}{d m^2}(-f(m))\Big\lvert_{m=0} &= \frac{1}{2}\left(\frac{\partial A}{\partial m}\Big\lvert_{m=0} - 1\right).
\end{align}s
For $p=2$, this gives the same stability condition as in the Ising case. For $p>2$, one obtains $\frac{d^2}{d m^2}(-f(m))\Big\lvert_{m=0}=-\frac{1}{2}<0$ $\forall \alpha, \lambda >0$.

\section{Details on numerical analysis of the message passings}
\label{appendix-mp-detail}

\subsection{Initialization}
\label{sec-initialization}

In both r-BP and G-AMP experiments we consider two main kinds of initializations: informative initialization and uninformative initialization.

\begin{itemize}

\item In the {\bf informative initialization}, we set the values of the messages $m^{t=0}_{i\mu\to\blacksquare}$, $v^{t=0}_{i\mu\to\blacksquare}$ (or equivalently $m^{t=0}_{i\mu}$, $v^{t=0}_{i\mu}$ in the G-AMP case) to
\begin{equation}
	\begin{cases}
		m^{t=0}_{i\mu\to\blacksquare} = x^*_{i\mu} \\
		v^{t=0}_{i\mu\to\blacksquare}=(x^*_{i\mu})^2 \quad\quad {\rm (Gaussian\,prior)}.
	\end{cases}
\end{equation}
The starting point of SE equations is then $m_0\equiv\langle x^*_{i\mu}m^{t=0}_{i\mu}\rangle=q_0\equiv\langle(m^{t=0}_{i\mu})^2\rangle=1$ and $Q_0=\langle v^{t=0}_{i\mu}=1\rangle$.  In the case of sign output, this definition is slightly modified
\begin{equation}
	\begin{cases}
		m^{t=0}_{i\mu\to\blacksquare} = x^*_{i\mu}(a+\sqrt{a-a^2}\mathcal{N}(0,1)),  \quad\quad a=0.99\\
		v^{t=0}_{i\mu\to\blacksquare}=a^{-1} (m^{t=0}_{i\mu\to\blacksquare})^2 .
	\end{cases}
\end{equation}
In this way, $m_0=q_0=a$ and $Q_0=1$. Furthermore, $v^{t=0}_{i\mu\to\blacksquare}>(m^{t=0}_{i\mu\to\blacksquare})^2$ $\forall i,\mu$, which makes the quantity under square root $V^{t=0}_{\blacksquare\to\mu}$ strictly positive.

\item In the {\bf uninformative initialization}, messages are set to
\begin{equation}
	\begin{cases}
		m^{t=0}_{i\mu\to\blacksquare} = ax^*_{i\mu} +\sqrt{a-a^2}\mathcal{N}(0,1), \quad\quad &a=0.01\\
		v^{t=0}_{i\mu\to\blacksquare}=1 &{\rm (Gaussian\,prior)}.
	\end{cases}
\end{equation}
Also in this case we have $m_0=q_0=a$ and $Q_0=1$.

\end{itemize}

\black{The main advantage of the 'uninformative' initialization with respect to a truly random one is twofold: it satisfies the Bayes-optimal condition $m=q$ from $t=0$, and it helps selecting the 'correct' state with respect to a remaining global symmetry of the system when $p$ is even. This symmetry is realized under the change of sign of a given $\mu$ component of all of the $N$ vectors, and we may refer to it as the symmetry under inversion on the $\mu$-plane. This symmetry remains even after the introduction of the random spreading factor for $p=2$, but being global and discrete it happens to be spontaneously broken by the message passing algorithms. We show it in Fig.~\ref{fig:random_init}, by comparing the result of uninformative and {\bf truly random initialization}
	\begin{equation}
		\begin{cases}
			m^{t=0}_{i\mu\to\blacksquare} = \sqrt{a}\mathcal{N}(0,1), \quad\quad &a=0.01\\
			v^{t=0}_{i\mu\to\blacksquare}=1 &{\rm (Gaussian\,prior)}.
		\end{cases}
	\end{equation}
With this choice, $m_0=0$, $Q_0=1$ but $q_0=a$, so the algorithm is initialized slightly off the Bayes Optimal ansatz. In general, without any prior knowledge of the planted solution, we are able to infer it only up to a global sign on each $\mu$ component. In order to get non-zero magnetization, one has to modify its definition according to
\begin{math}
	m= \frac{1}{M}\sum_{\mu=1}^M\left\lvert \frac{1}{N}\sum_{i=1}^N x^*_{i\mu}m_{i\mu} \right\rvert.
\end{math}
}

\subsection{Spreading factor}

\black{For $p=2$, we used the 'disordered model' with random spreading factor $F_{\bs \mu}$
(see sec ~\ref{sec-techinical-assumptions})
in order to eliminate the
rotational symmetry and help the convergence of the algorithms. The spreading factor was also  beneficial to the mixed case $p=2+3$.
Otherwise, for $p=3$, we found the algorithms converge well with the 'disorder-free model' $F_{\bs \mu}=1$.
}

\subsection{Update of messages and check of convergence}
\label{subsubsec-mp-convergence}

We introduce a parameter $\epsilon$ to control the magnitude of the updates
\begin{align}
	m^{t+1} &= m^{t} + \epsilon (f^{t+1}_{\rm input} - m^{t}),  \\
	v^{t+1} &= v^{t} + \epsilon (f^{t+1}_{\rm input,II} - v^{t}).   \nonumber
\end{align}
Usually values of $\epsilon$ between 0.5 and 1 work well; in some cases, such as the mixed $p=2+3$ Gaussian model, an accurate tuning of the parameter is crucial in order to get a good convergence of the algorithms.

For each set of parameters, the data is averaged over 10 random instances.
Convergence is tracked via the parameter
\begin{equation}
	D_t \equiv \langle \lvert f^{t+1}_{\rm input}-m^t\lvert\rangle,
        \label{eq-def-Dt}
\end{equation}
where $\langle\dots\rangle$ is an average over all the messages.

\section{Comparison between r-BP and G-AMP, finite size corrections}
\label{sec-comparisons-finite-size}
\black{In this section we show more extensive numerical results, providing comparisons between the r-BP and G-AMP algorithms. In general, both algorithms give comparable performances. There are cases, such as the $p=2$ Ising prior or the $p=3$ Gaussian prior equipped with Gaussian noise, where at relatively small $N,M$ sizes the r-BP algorithm performs slightly better. And cases such as the $p=2$ sign output, where instead the better scalability of G-AMP allows to easily reach bigger $M$ sizes and get rid of some strong finite $M$ corrections.
}
	
\black{We also provide comparisons for different sizes of $N,M$.  We found in general that when the transition is continuous, as in the $p=2$ cases, stronger deviations are observed for values of $\lambda$ ($\alpha$ in the sign output case) closer to the transition. Also the convergence rate of the algorithm deteriorates. The closer to the transition, the more relevant the finite $M$ corrections seem to become, see for example Fig.~\ref{fig:finitesize_Gaussian_Gaussian_p=2}. Finite $M$ corrections can arise both as loops corrections and at the level of the saddle point evaluation of the replicated free-energy. Further away from the transition, also finite $N$ effects become more evident.
}

\black{When the transition is discontinuous, as in the $p=3$ cases, finite $N$ effects can destabilize the system close to the transition, see in particular Figs.~\ref{fig:r_BP_Gaussian_Gaussian_p=3} and~\ref{fig:finitesize_Gaussian_Sign_p=3}.
}

\subsection{Ising prior - Gaussian noise}

\subsubsection{$p=2$}

\begin{figure}[H]
	\centering
	\includegraphics[width=0.495\textwidth]{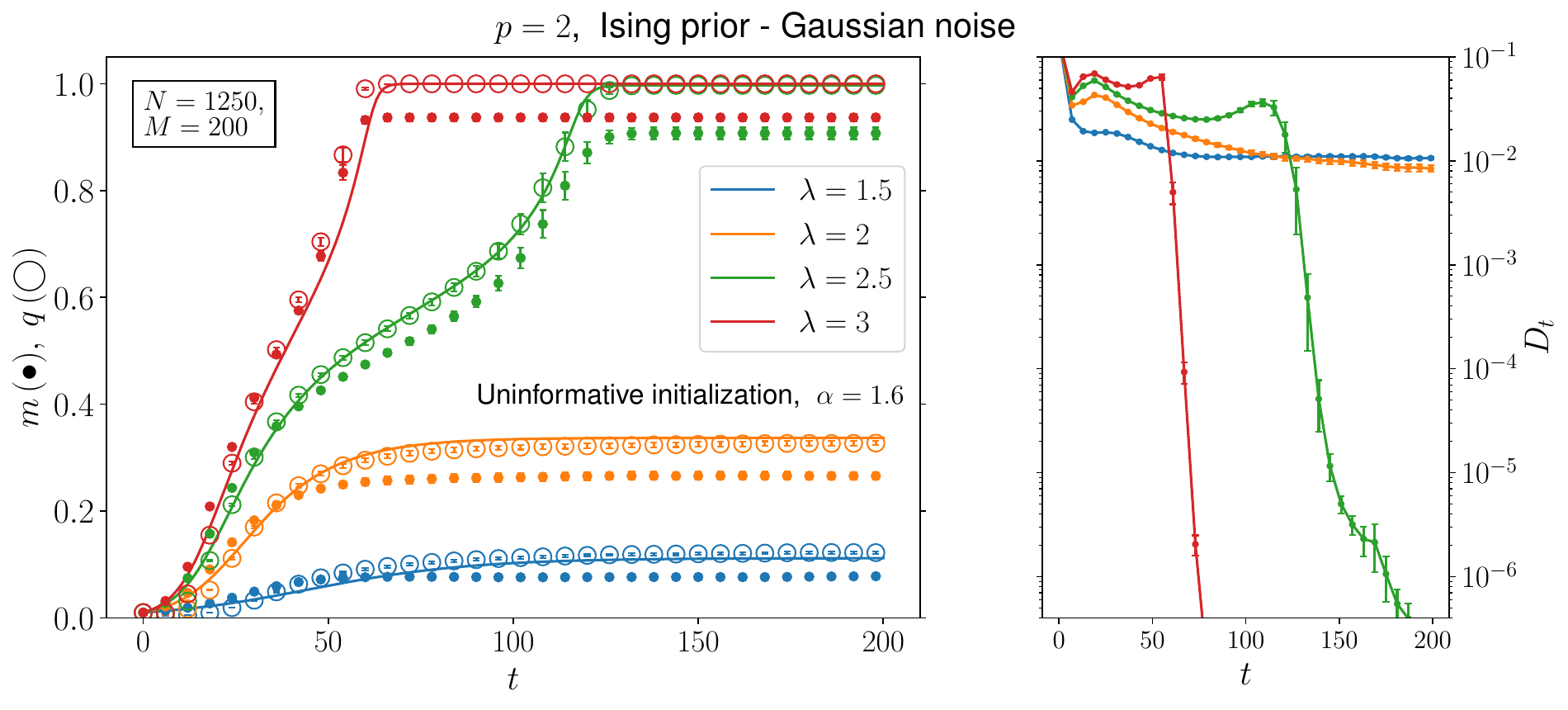}
	\includegraphics[width=0.495\textwidth]{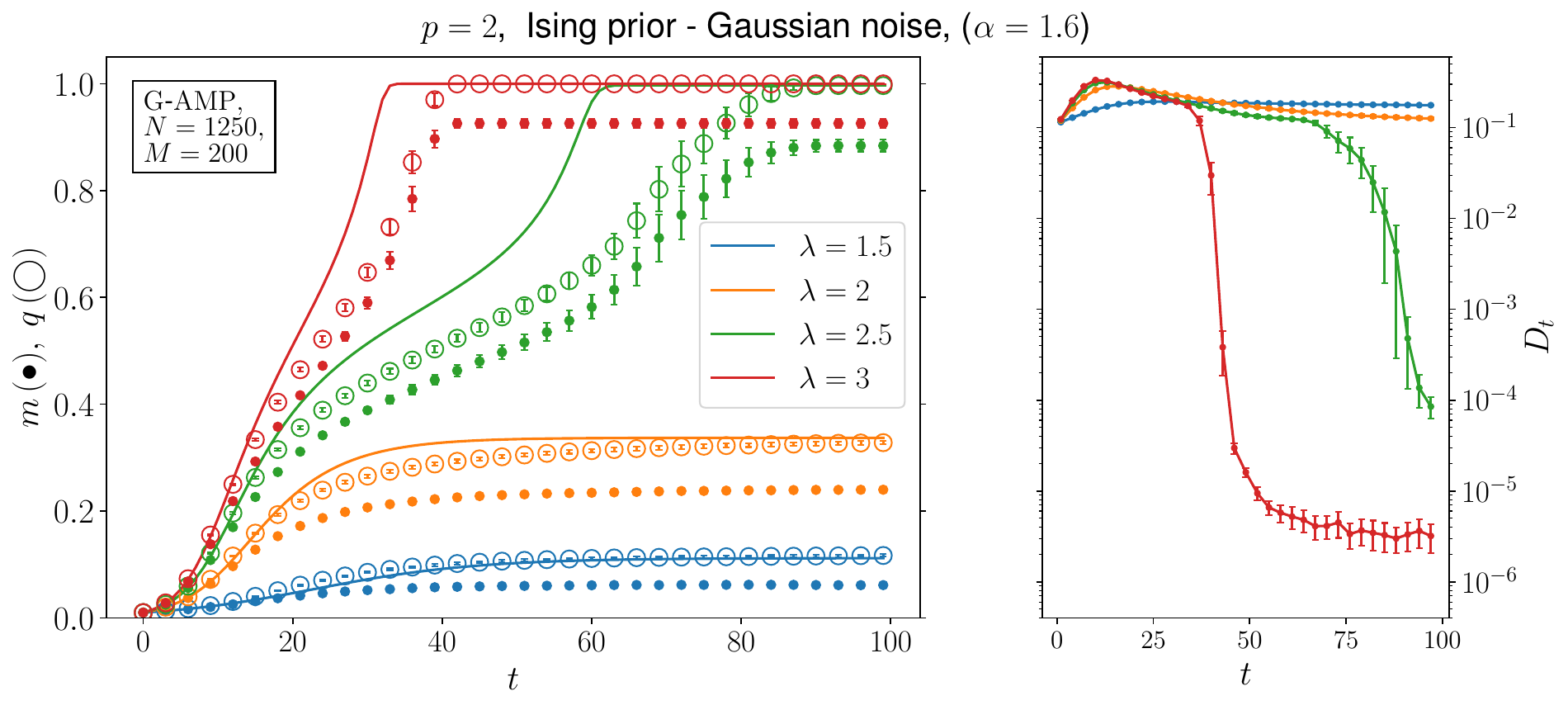}
	\includegraphics[width=0.495\textwidth]{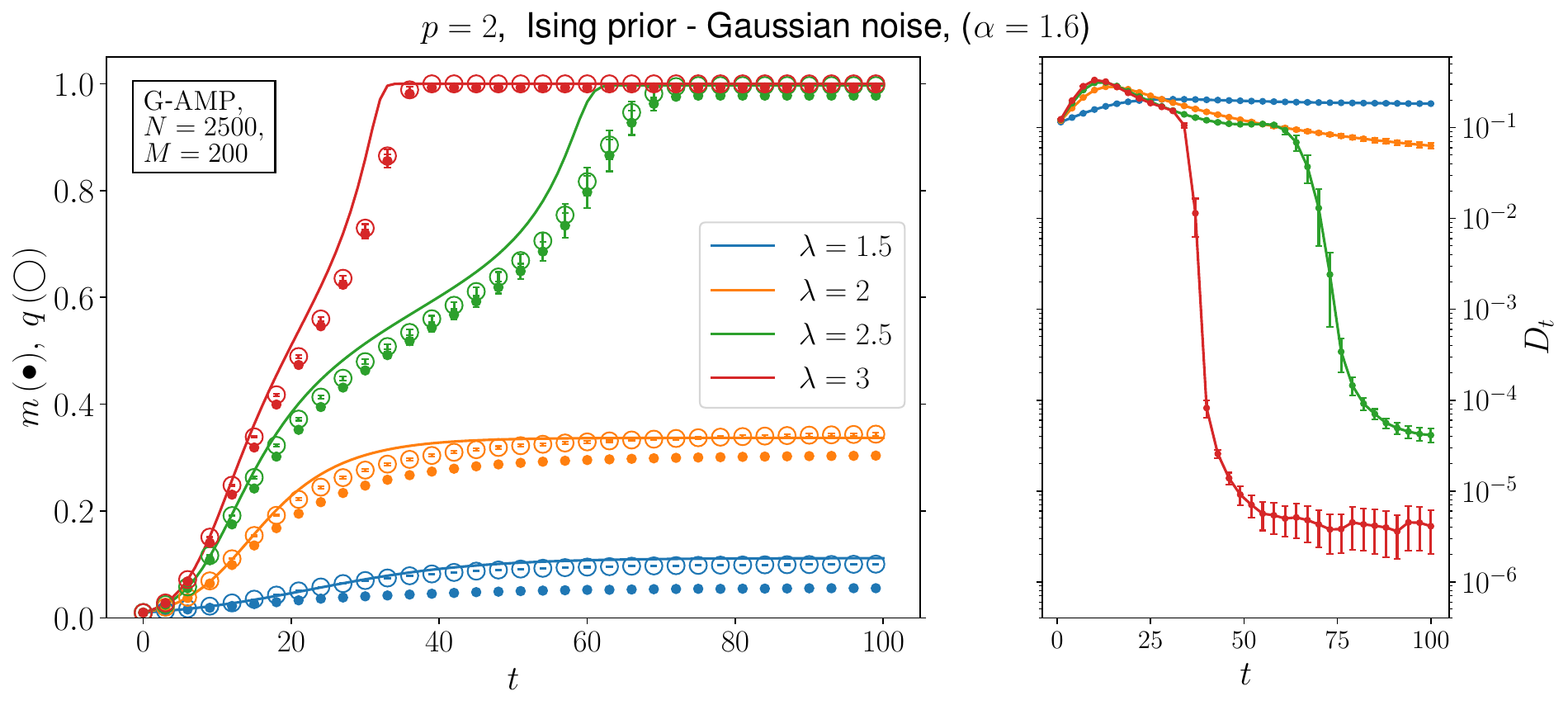}
	\includegraphics[width=0.495\textwidth]{fig_AMP_IsingGaussian_p2_alpha1.6_N5000_M400_unInform.pdf}
	\caption{Comparison between r-BP algorithm (first panel) and G-AMP. For smaller sizes r-BP is more accurate, but increasing the size G-AMP drastically improves (last panel is the same as in Fig.~\ref{fig:G_AMP_Ising_Gaussian_p=2}). Solid lines represent the predictions from State Evolution. 
		The values of $\epsilon$ used are $\epsilon=0.5$ with r-BP and  $\epsilon=1$ with G-AMP. 
	}
	\label{fig:r_BP_Ising_Gaussian_p=2}
\end{figure}

\subsubsection{$p=3$}

\begin{figure}[H]
	\centering
	\includegraphics[width=0.495\textwidth]{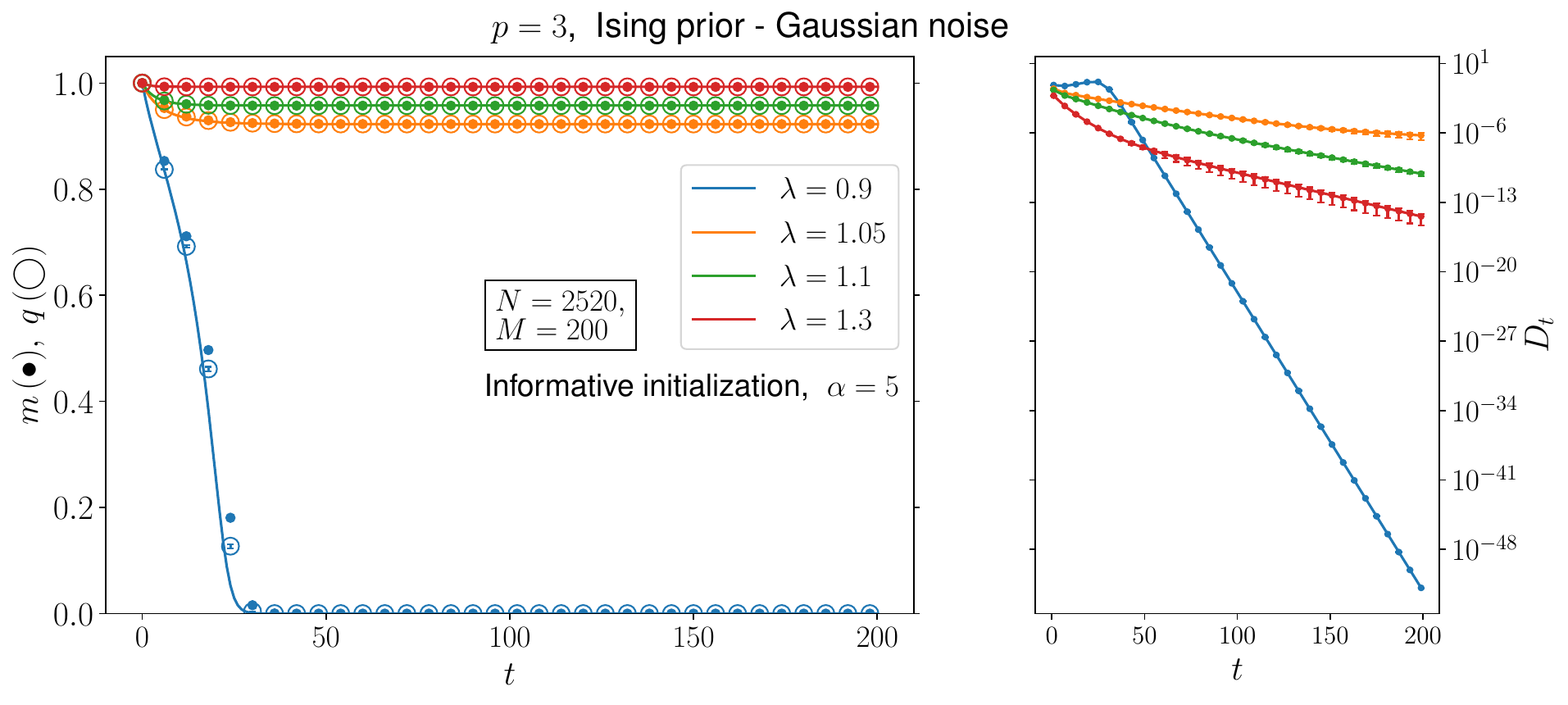}
	\includegraphics[width=0.495\textwidth]{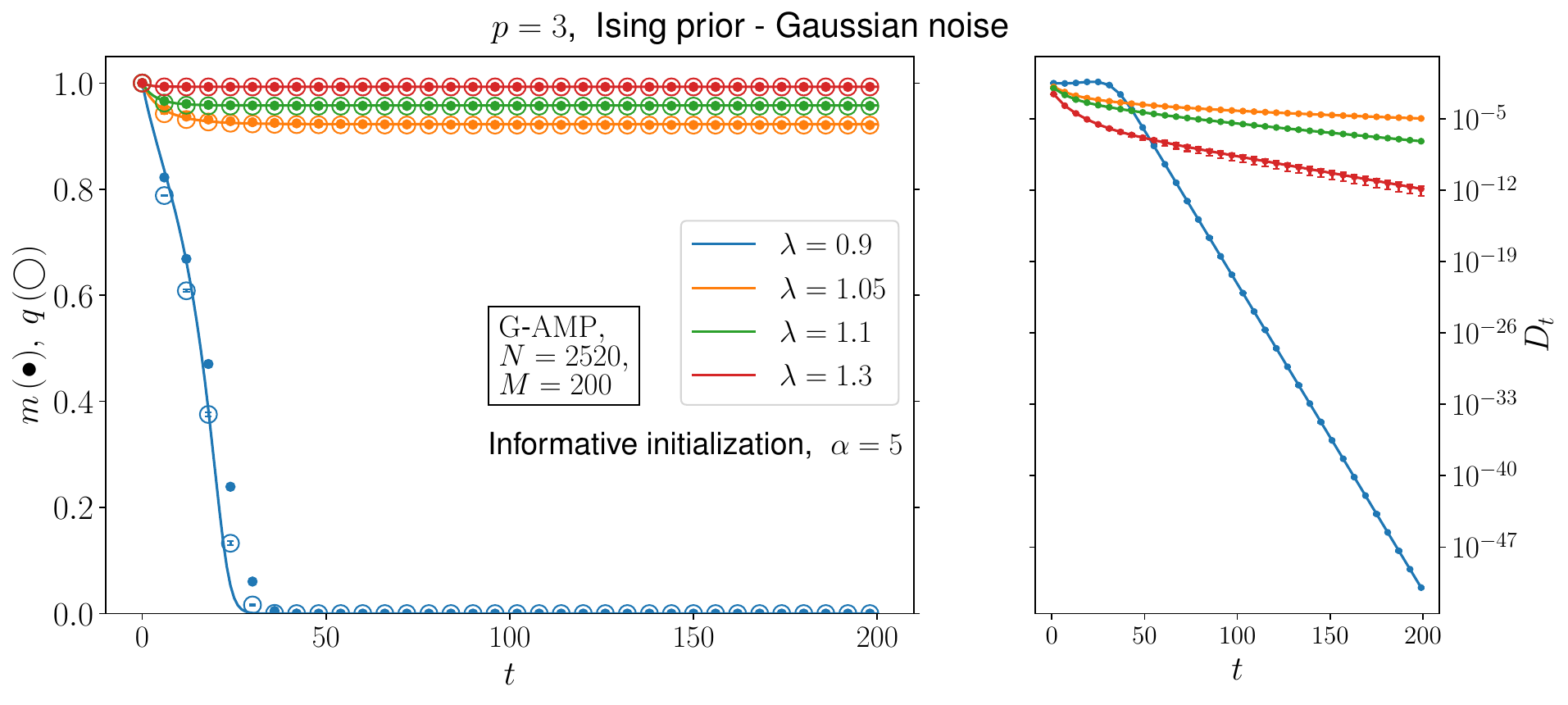}
	\caption{Comparison between r-BP algorithm (left panel) and G-AMP (right panel, from Fig.~\ref{fig:G_AMP_Ising_Gaussian_p=3}). In both cases $\epsilon=0.5$.}
	\label{fig:r_BP_Ising_Gaussian_p=3}
\end{figure}

\subsection{Gaussian prior - Gaussian noise}

\subsubsection{$p=2$}

\begin{figure}[H]
	\centering
	\includegraphics[width=0.495\textwidth]{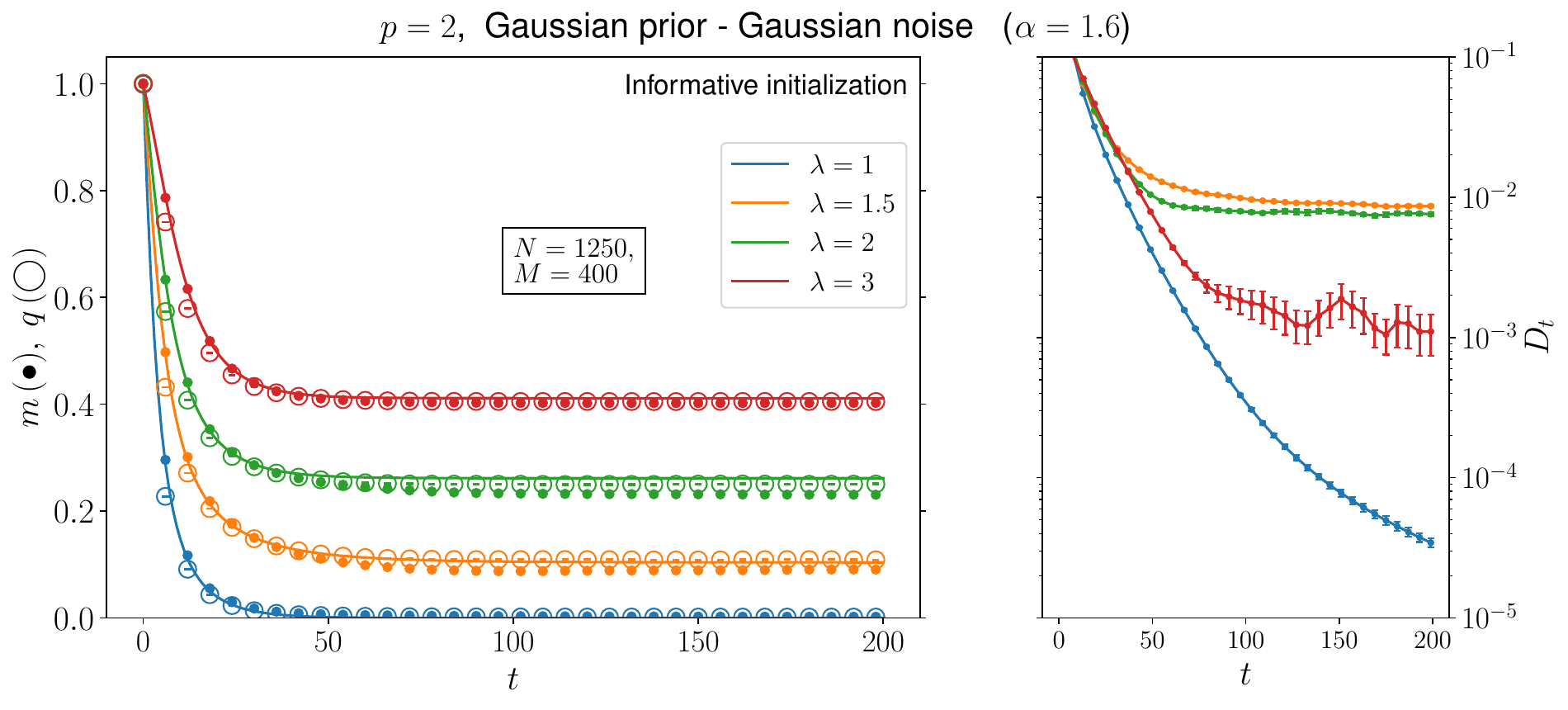}
	\includegraphics[width=0.495\textwidth]{fig_AMP_GaussGauss_p2_alpha1.6_N2500_M400_Inform.pdf}
	\caption{Comparison between r-BP algorithm (left panel) and G-AMP (right panel, the same as in Fig.~\ref{fig:G_AMP_Gaussian_Gaussian_p=2}). The values of $\epsilon$ in this case are respectively $\epsilon=0.5,1$
}
	\label{fig:r_BP_Gaussian_Gaussian_p=2}
\end{figure}

\begin{figure}[H]
	\centering
	\includegraphics[width=0.495\textwidth]{fig_AMP_GaussGauss_p2_alpha1.6_N5000_M400_unInform.pdf}
	\includegraphics[width=0.495\textwidth]{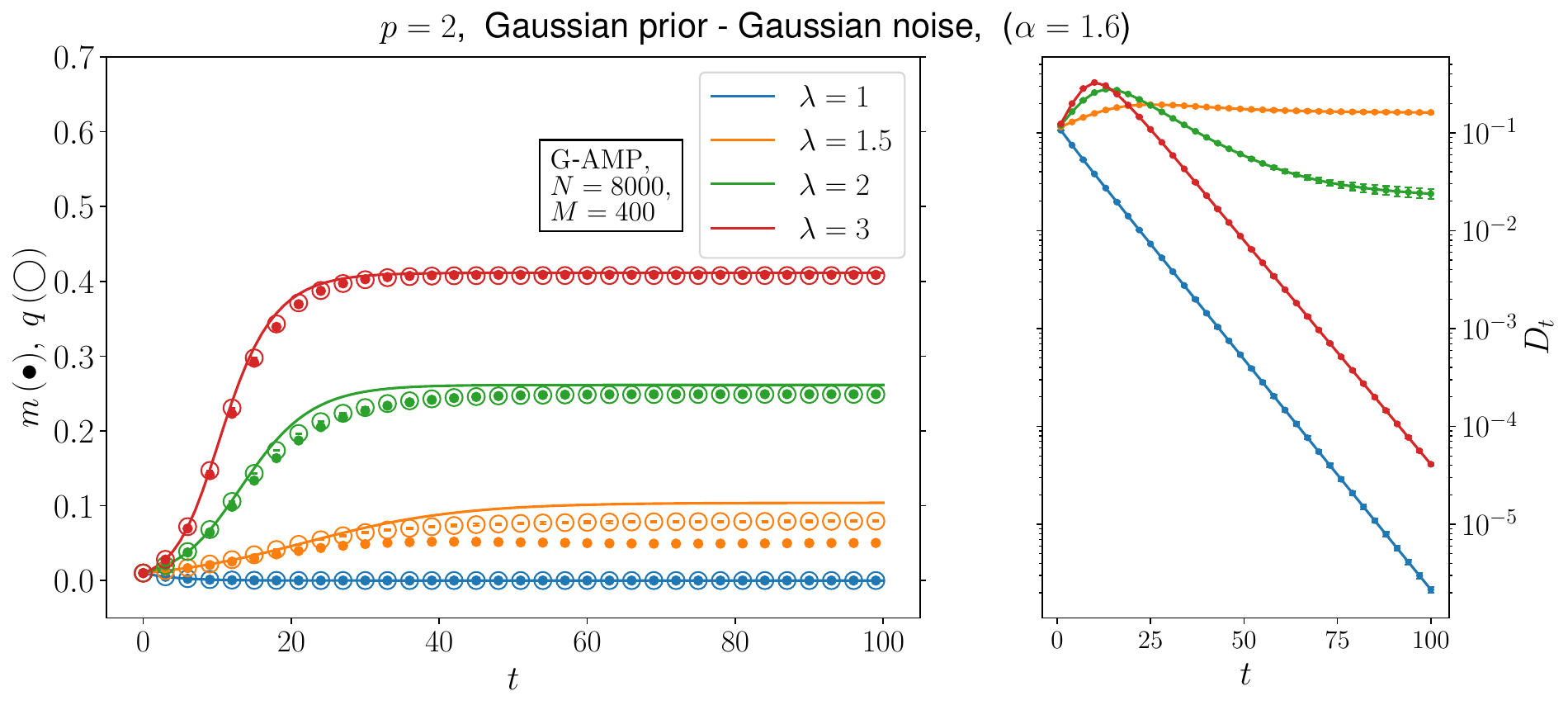}
	\includegraphics[width=0.495\textwidth]{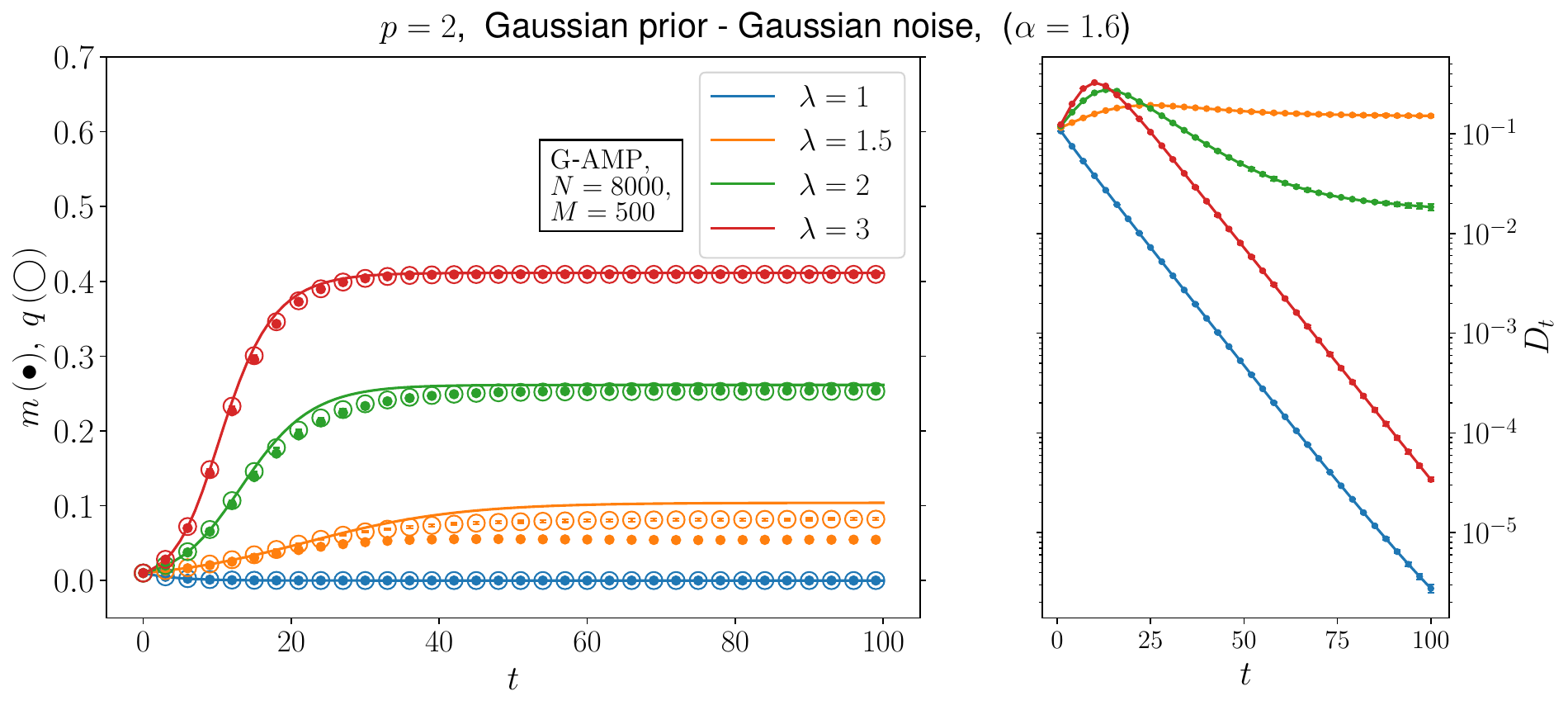}
	\includegraphics[width=0.495\textwidth]{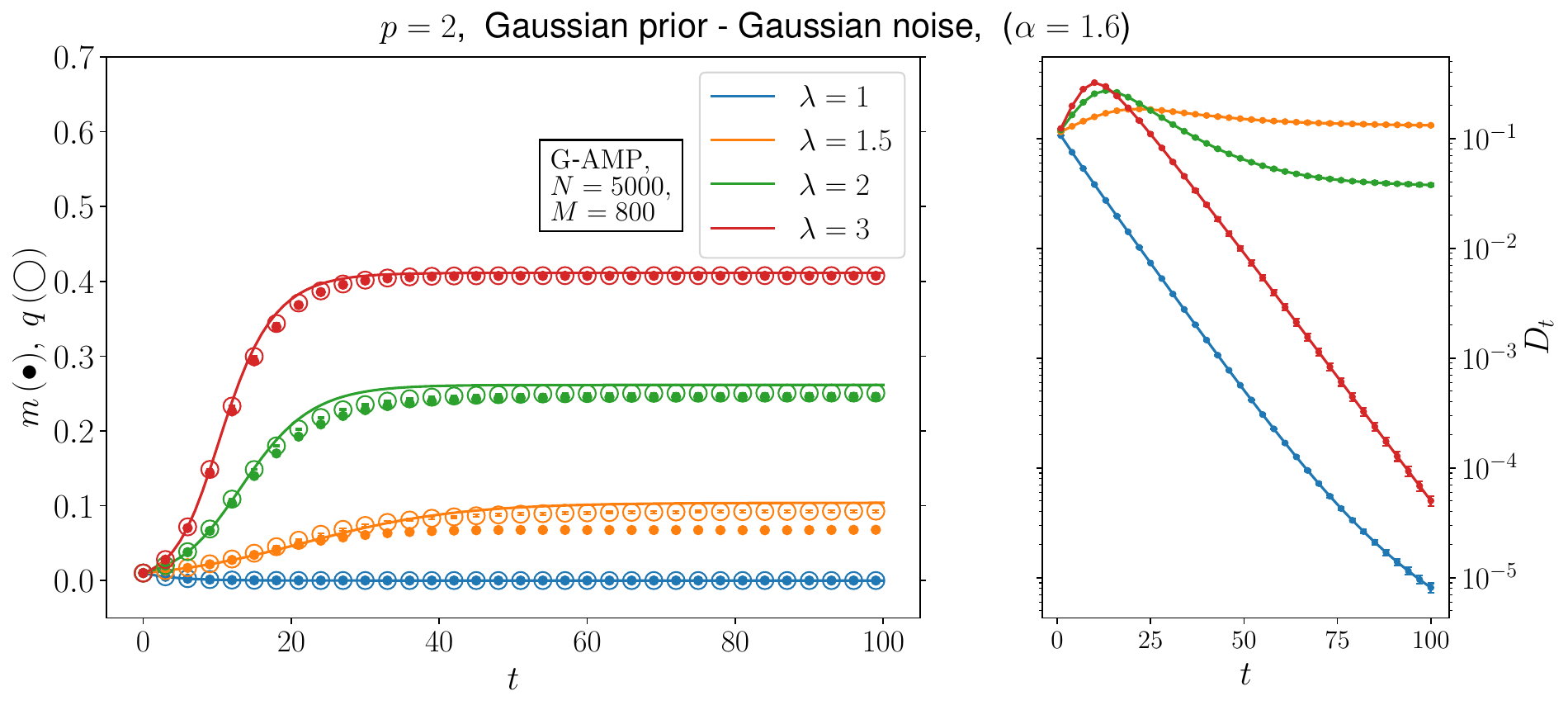}
	\caption{\black{Finite $N,M$ corrections. Closer to the continuous transition, for $\lambda=1.5$,  convergence is the worst and deviations are present. Some slight improvement is observed when increasing $M$, in the last panel. Further away from the transition, as for $\lambda=2$, also some finite $N$ effect becomes evident.
		}
	}
	\label{fig:finitesize_Gaussian_Gaussian_p=2}
\end{figure}

\subsubsection{$p=3$}
\begin{figure}[H]
	\centering
	\includegraphics[width=0.495\textwidth]{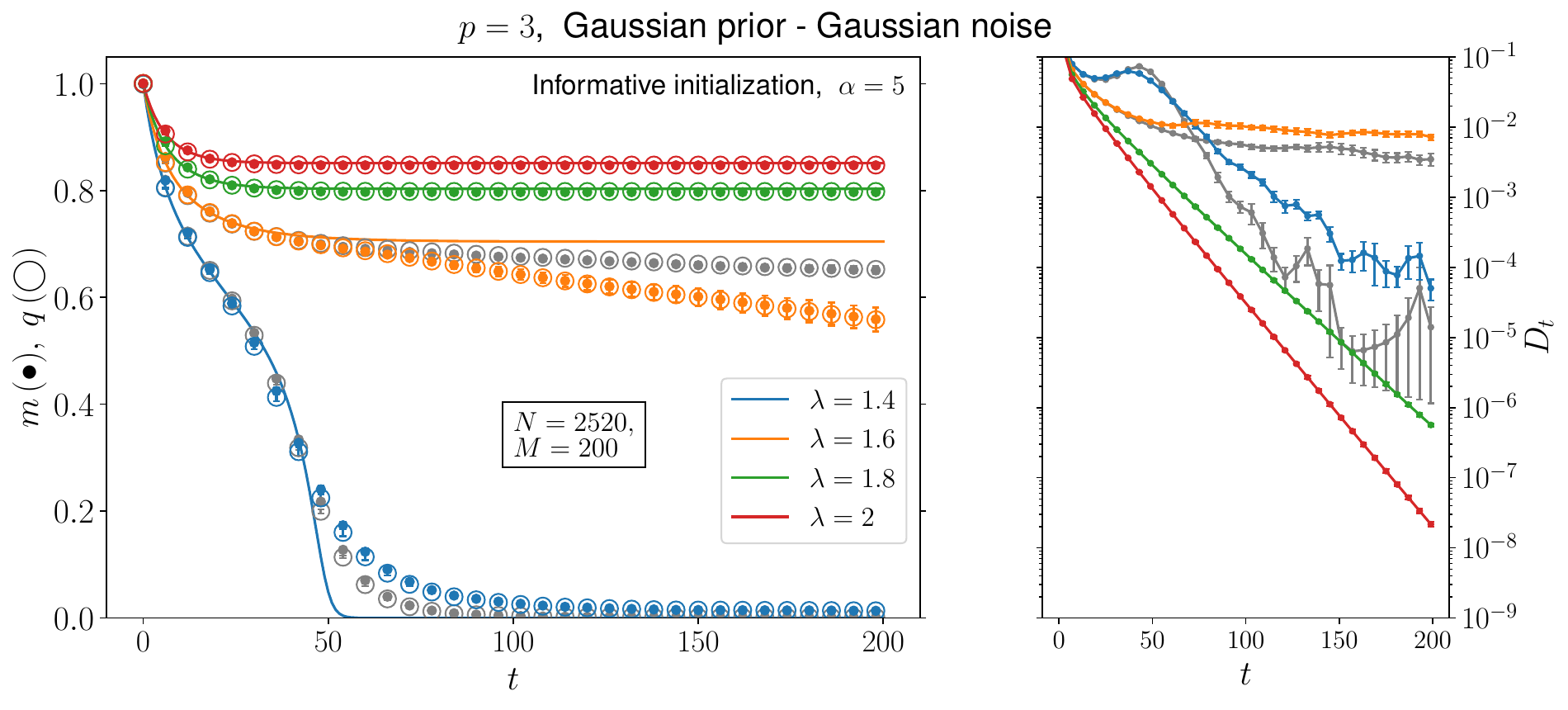}
	\includegraphics[width=0.495\textwidth]{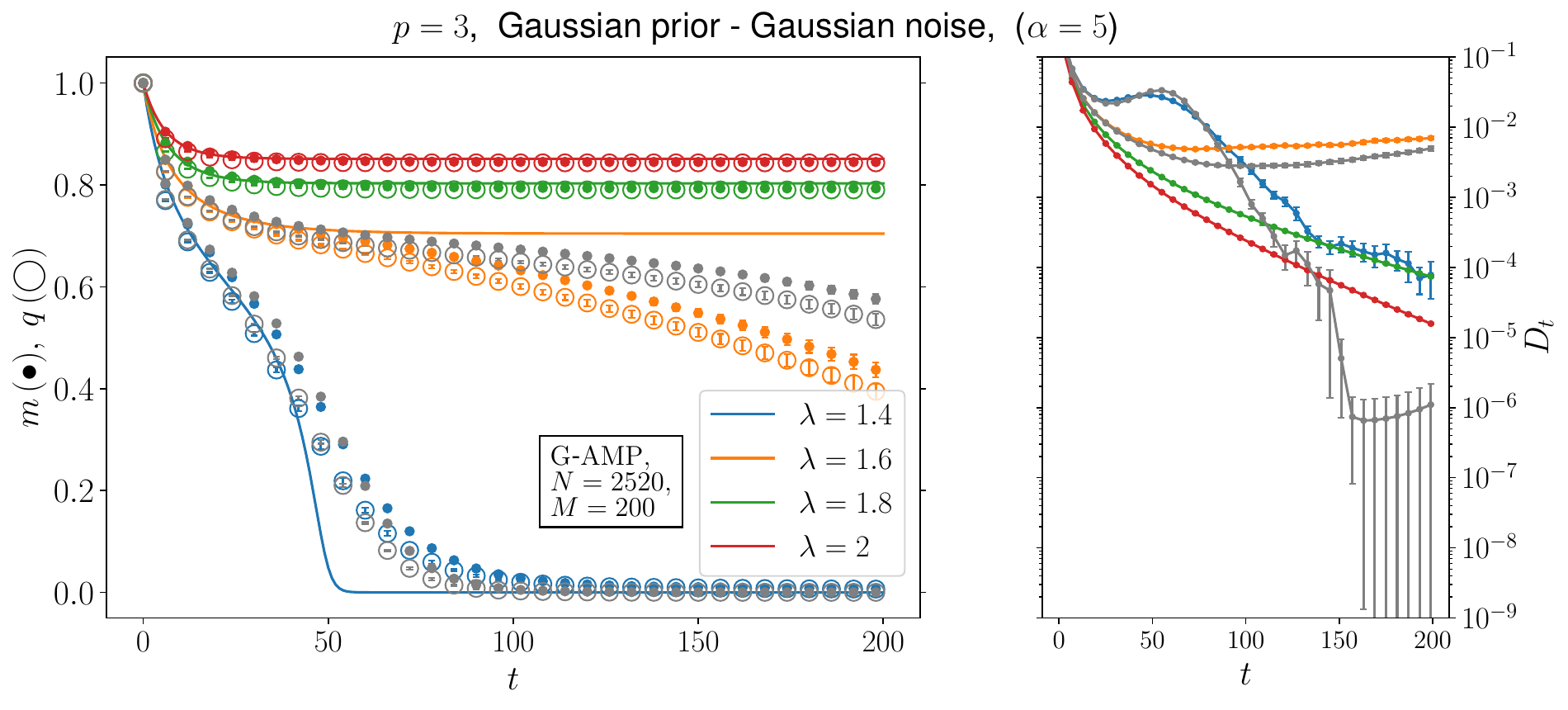}
	\caption{Comparison between r-BP algorithm (left panel) and G-AMP (right panel). Grey data points are obtained considering a bigger system size, $N=5040$, while keeping fixed $M=200$, for $\lambda=1.4$ and $1.6$. In both cases $\epsilon=0.5$. Compare also to Fig.~\ref{fig:G_AMP_Gaussian_Gaussian_p=3} in the main text for bigger values of $N,M$ in the G-AMP case.
	}
	\label{fig:r_BP_Gaussian_Gaussian_p=3}
\end{figure}

\subsection{Gaussian prior - SignOutput}

\begin{figure}[H]
	\centering
	\includegraphics[width=0.5\textwidth]{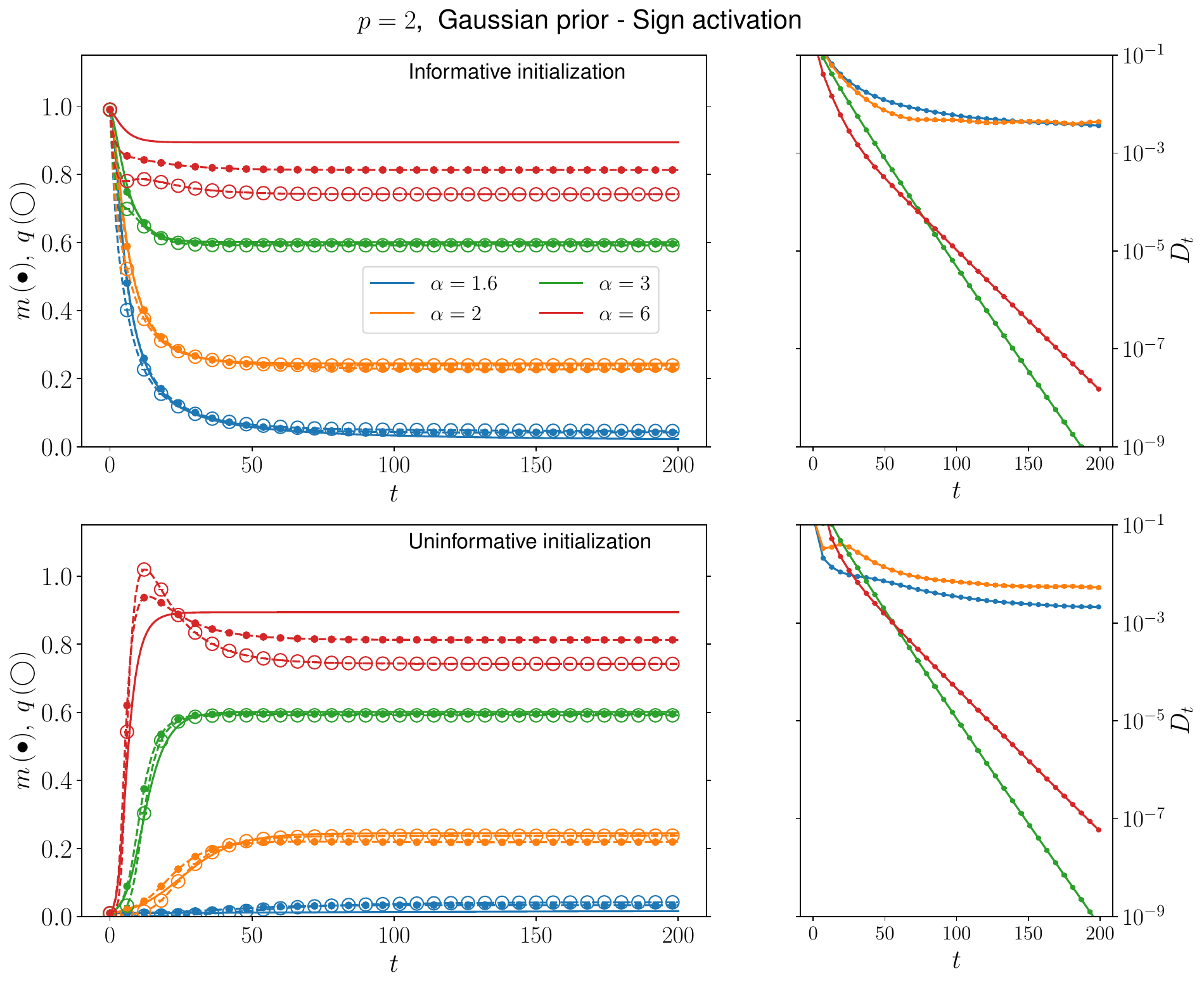}
	\caption{Results from the r-BP algorithm, $N=2500$ and $M=200$. For the largest value of $\alpha$ there are strong finite $M$ corrections. Solid lines represent the prediction from SE equations.}
	\label{fig:r_BP_Gaussian_Sign_p=2}
\end{figure}

\begin{figure}[h]
	\centering
	\includegraphics[width=0.645\textwidth]{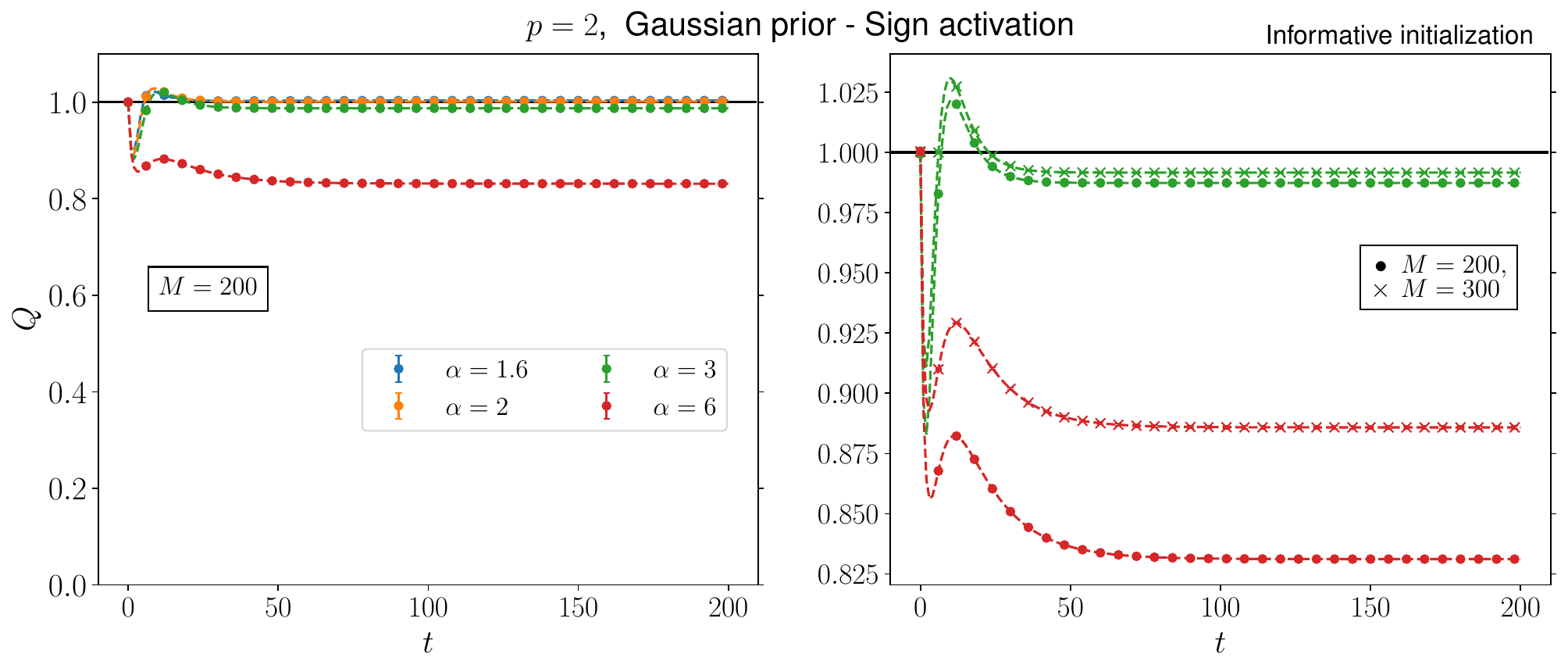}
	\includegraphics[width=0.645\textwidth]{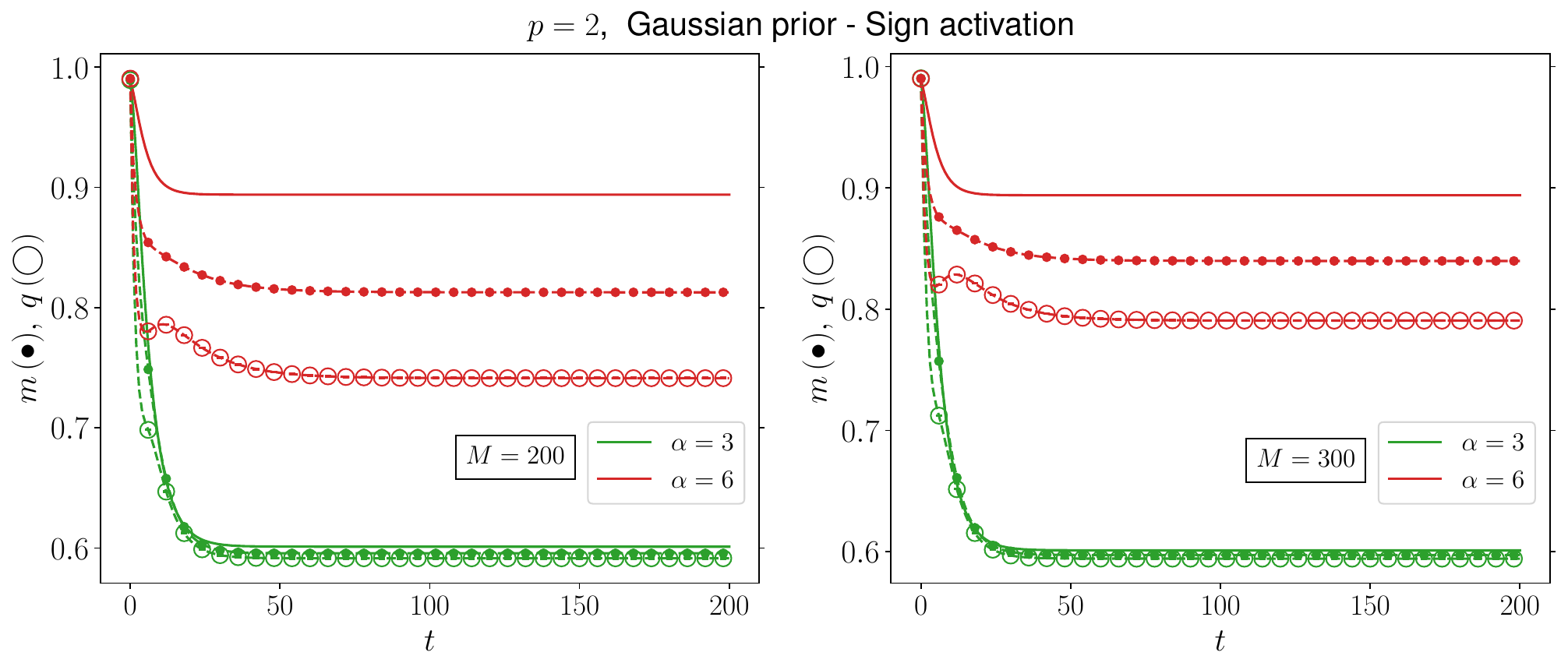}
	\caption{Results from the r-BP algorithm. Top panels: the finite $M$ corrections for $\alpha=6$ in Fig.~\ref{fig:r_BP_Gaussian_Sign_p=2} are associated to a value of $Q_t$ consistently different from the Bayes Optimal expectation $Q_t=1$. In the right panel, $Q_t$ for $\alpha=3$ and $6$ is shown to improve when increasing $M$. Bottom panels: also the observables $q$ and $m$ show an improvement for $\alpha=6$ when increasing $M$.}
	\label{fig:r_BP_Gaussian_Sign_p=2_corrections}
\end{figure}

\begin{figure}[H]
	\centering
	\includegraphics[width=0.495\textwidth]{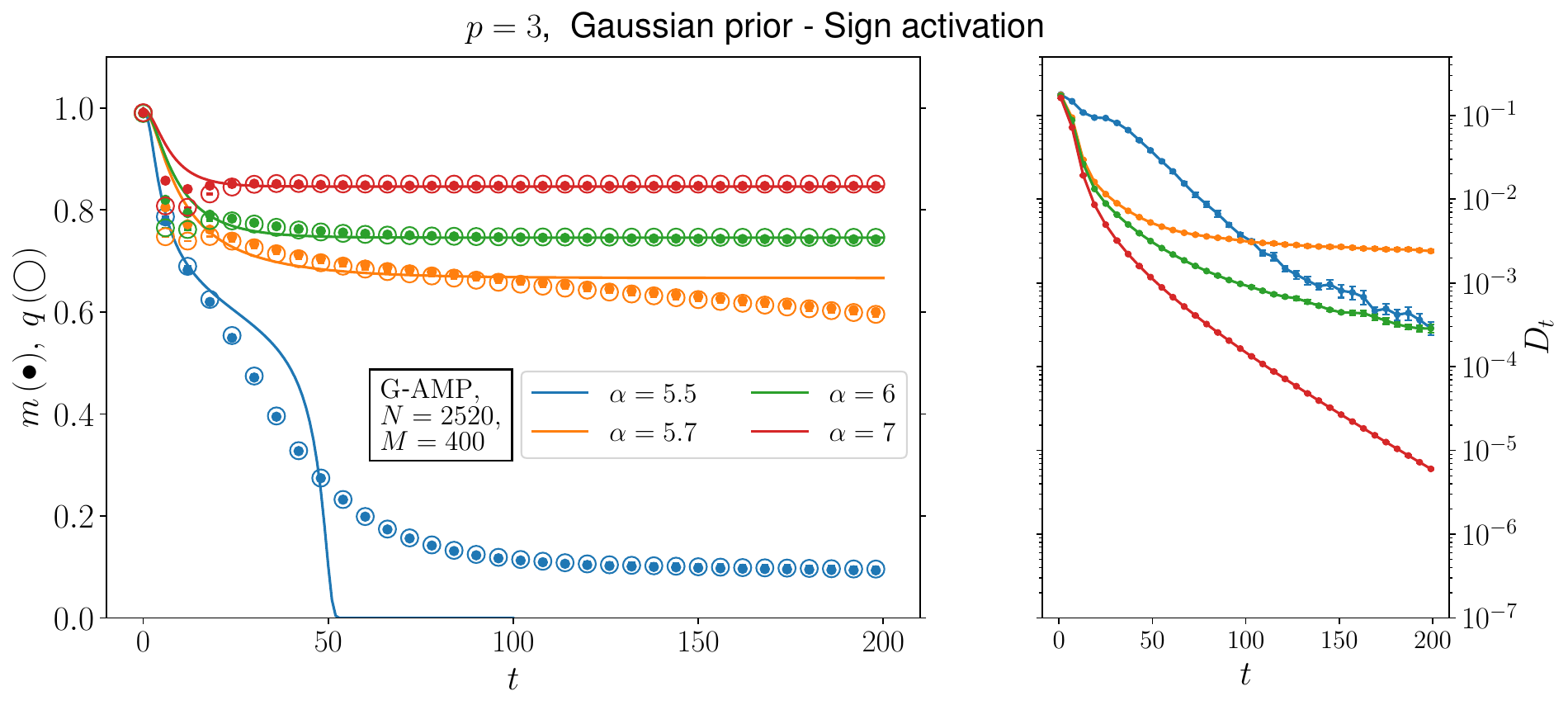}
	\includegraphics[width=0.495\textwidth]{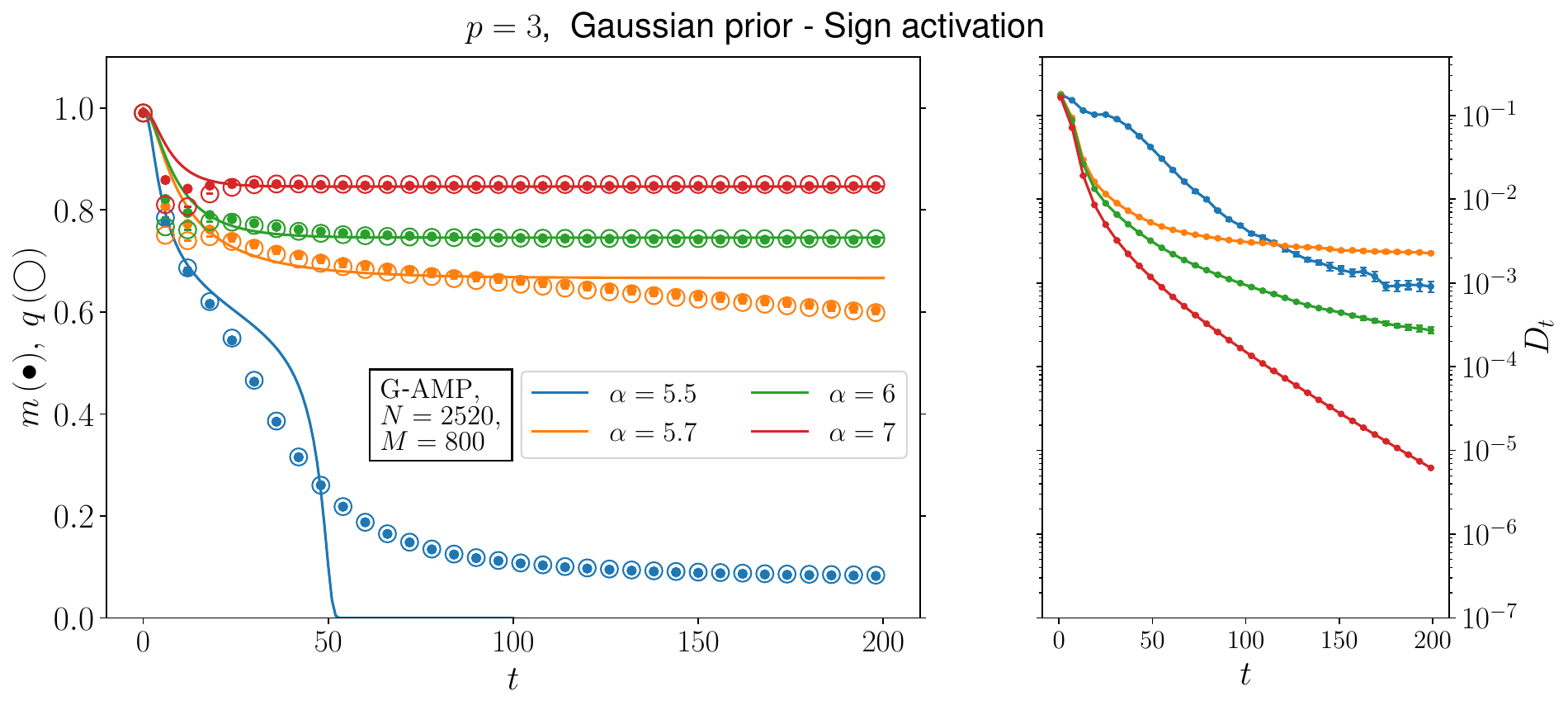}
	\includegraphics[width=0.495\textwidth]{fig_AMP_GaussSign_p3_N5040_M400_inf.pdf}
	\includegraphics[width=0.495\textwidth]{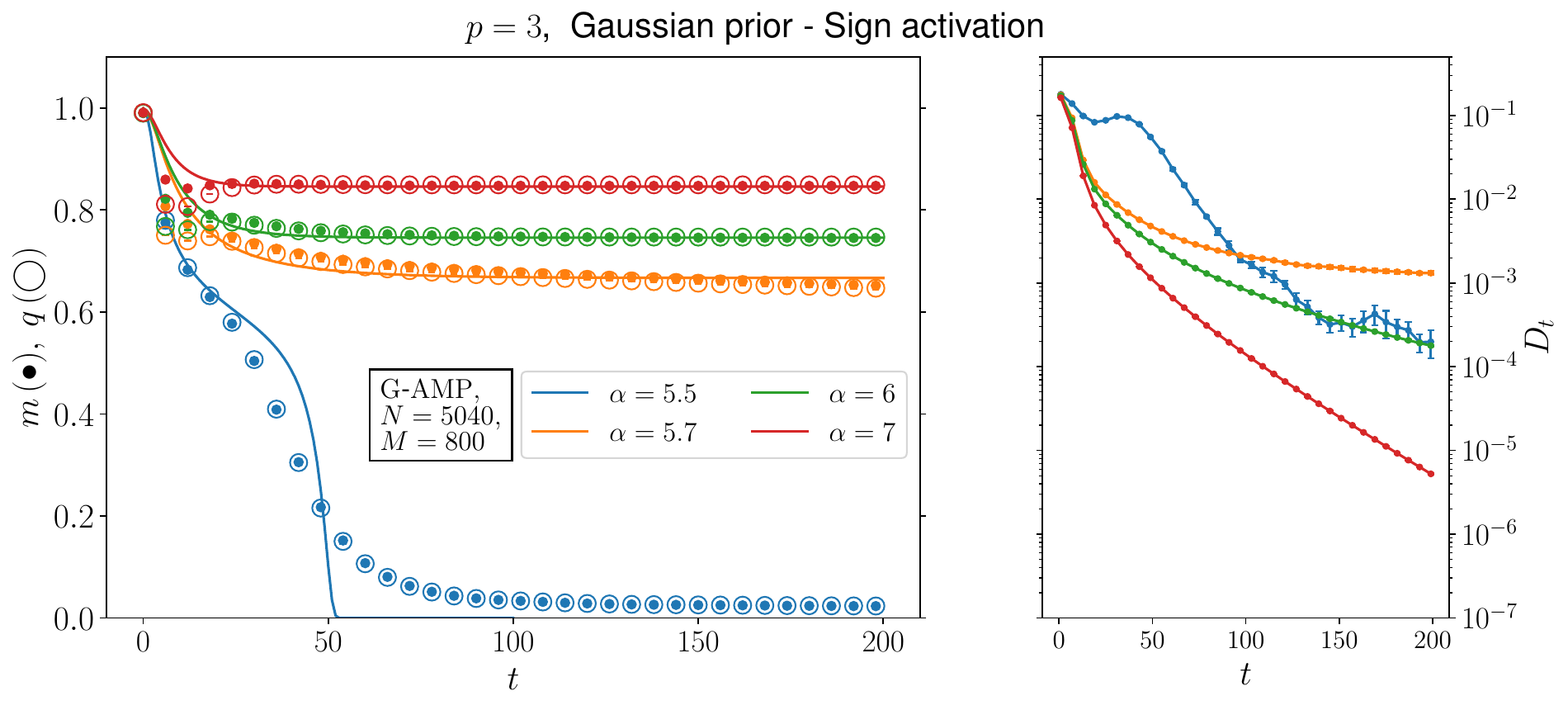}
	\caption{\black{Finite $N,M$ corrections. Closer to the discontinuous transition, for $\alpha=5.5$ and $\alpha=5.7$,  convergence is the worst and strong deviations due to finite size $N$ are present. 
		}
	}
	\label{fig:finitesize_Gaussian_Sign_p=3}
\end{figure}

\begin{figure}[H]
	\centering
	\includegraphics[width=0.495\textwidth]{fig_AMP_GaussSign_p2_M400_uni.pdf}
	\includegraphics[width=0.495\textwidth]{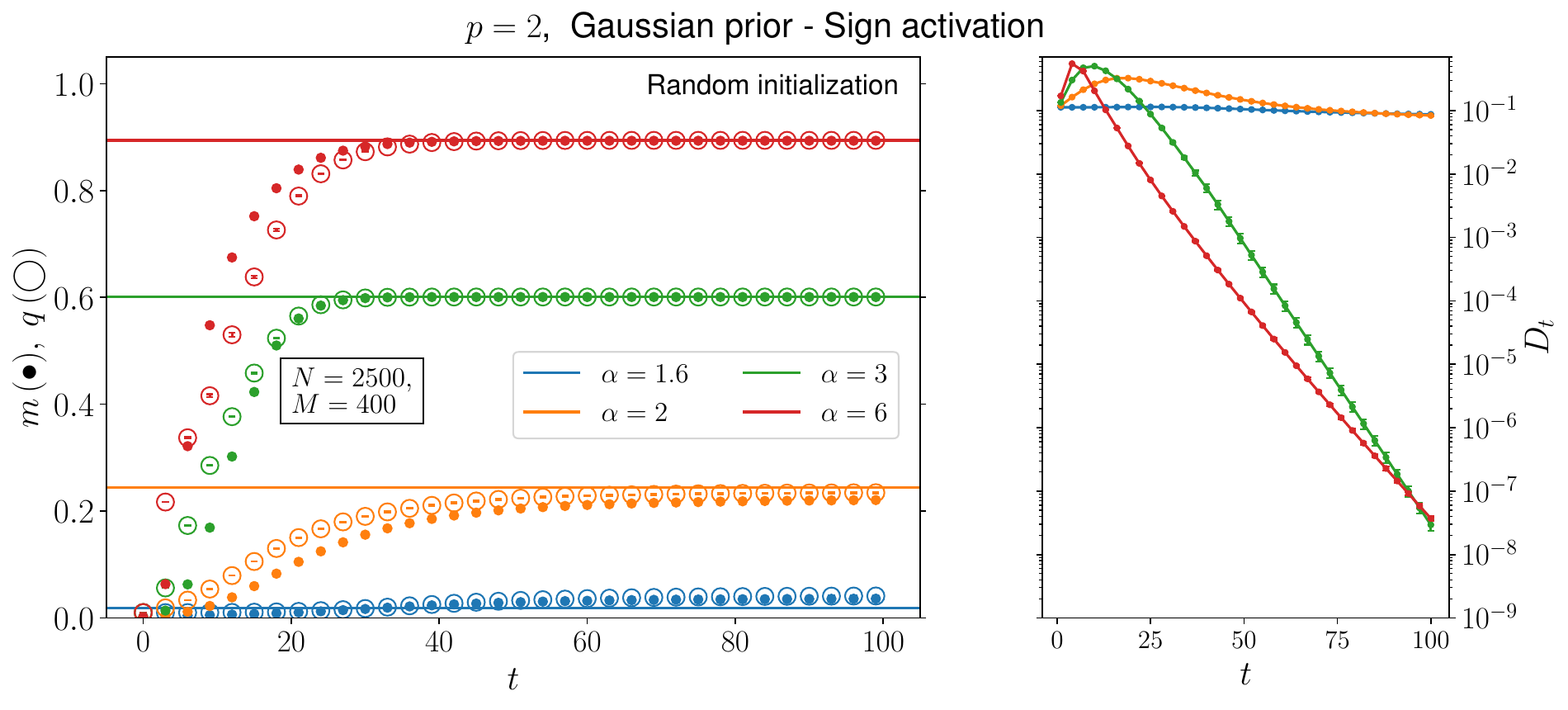}
	\caption{Comparison between the `uninformative' (left) and 'truly random' (right) initializations. Right panel: the long time limit of the SE equations (coincident with the replica result) is shown as solid lines. The definition of the magnetization is corrected (see sec.~\ref{sec-initialization}) in order to take into account the impossibility of recovering a global sign on each $\mu$-plane.}
	\label{fig:random_init}
\end{figure}

%%%%%%%%%%%%%%%%%%%%%%%%%%%%%%%%%%%%%%%%%%%%%%%%%%%%%
%%%%%%%%%%%%%%%%%%%%%%%%%%%%%%%%%%%%%%%%%%%%%%%%%%%%%
%%%%%%%%%%%%%%%%%%%%%%%%%%%%%%%%%%%%%%%%%%%%%%%%%%%%%
%\bibliography{ref_yoshino,obuchi}
\bibliography{p_ary.v4}

\end{document}